\documentclass[11pt]{article}
\usepackage[T1]{fontenc}
\usepackage[utf8]{inputenc}
\usepackage[margin=1in]{geometry}
\usepackage{microtype}
\usepackage[numbers]{natbib}
\usepackage{hyperref}
\usepackage{url}
\usepackage{algorithm}
\usepackage[noend]{algpseudocode}
\usepackage{graphicx}
\usepackage{amsmath,amssymb,amsthm}
\usepackage{wrapfig}
\usepackage{tcolorbox}
\usepackage{paralist}
\usepackage{booktabs}
\usepackage{comment}
\usepackage{listings}
\usepackage[capitalize,noabbrev]{cleveref}
\usepackage{xspace}

\newcommand{\eg}{\emph{e.g.}\xspace}
\newcommand{\ie}{\emph{i.e.}\xspace}
\newcommand{\keywords}[1]{\par\noindent\textbf{Keywords:} #1}
\newcommand{\Description}[2][]{}

\theoremstyle{definition}
\newtheorem{definition}{Definition}[section]
\theoremstyle{plain}
\newtheorem{theorem}[definition]{Theorem}

\title{Rethinking the Harmonic Loss via Non-Euclidean Distance Layers}
\author{
  Maxwell Miller-Golub\\
  \small American University, Washington, DC, USA\\
  \small \texttt{mm9628a@american.edu}
  \and
  Collin Coil\\
  \small University of Vermont, Burlington, USA
  \and
  Kamil Faber\\
  \small AGH University, Krakow, Poland
  \and
  Marcin Pietron\\
  \small AGH University, Krakow, Poland
  \and
  Panpan Zheng\\
  \small Xinjiang University, \"{U}r\"umqi, China
  \and
  Pasquale Minervini\\
  \small University of Edinburgh, Edinburgh, United Kingdom
  \and
  Roberto Corizzo\\
  \small American University, Washington, DC, USA\\
  \small \texttt{rcorizzo@american.edu}
}
\date{}

\begin{document}
\maketitle

\begin{abstract}
Cross-entropy loss has long been the standard choice for training deep neural networks, yet it suffers from interpretability limitations, unbounded weight growth, and inefficiencies that can contribute to costly training dynamics.
The \emph{harmonic loss} is a distance-based alternative grounded in Euclidean geometry that improves interpretability and mitigates phenomena such as \emph{grokking}, or delayed generalization on the test set.
However, the study of harmonic loss remains narrow: only the Euclidean distance is explored, and no systematic evaluation of computational efficiency or sustainability was conducted.
We extend harmonic loss by systematically investigating a broad spectrum of distance metrics as replacements for the Euclidean distance.
We comprehensively evaluate \emph{distance-tailored harmonic losses} on both vision backbones and large language models.
Our analysis is framed around a three-way evaluation of \emph{model performance}, \emph{interpretability}, and \emph{sustainability}.
On vision tasks, cosine distances provide the most favorable trade-off, consistently improving accuracy while lowering carbon emissions, whereas Bray-Curtis and Mahalanobis further enhance interpretability at varying efficiency costs.
On language modeling tasks, cosine-based harmonic losses improve gradient and learning stability, strengthen representation structure, and reduce emissions relative to cross-entropy and Euclidean harmonic loss.
Our code is available at: \url{https://anonymous.4open.science/r/rethinking-harmonic-loss-5BAB/}.
\end{abstract}


\section{Introduction} \label{sect:intro}
Cross-entropy is the \emph{de facto} loss function for classification tasks. %  due to its effectiveness in combination with the softmax output layer.
However, it has shortcomings in terms of model interpretability and training dynamics.
Cross-entropy training provides no inherent meaning to the learned weight vectors (they serve as abstract parameters rather than intuitive prototypes) and can drive those weights to grow without bound in pursuit of confident predictions~\cite{baek2025harmonic}.
This unbounded weight growth can lead to phenomena such as \emph{grokking}: a delayed generalization in which the model closes the train–test performance gap only after extensive overtraining~\cite{DBLP:journals/corr/abs-2201-02177}.
Moreover, in high-stakes applications where transparency is critical (\eg, healthcare or finance), the opaque nature of cross-entropy–trained models poses challenges for trust and error diagnosis.
These issues motivate the exploration of alternative loss functions that may yield more \emph{interpretable}, \emph{efficient}, and \emph{robust} model behavior.
Recently, the \emph{harmonic loss} was proposed as an alternative training objective to address some of these concerns~\cite{baek2025harmonic}.
The harmonic loss replaces the conventional inner-product logits and softmax normalization with a distance-based formulation: model predictions are derived from the distances between the sample’s representation and \emph{class prototype} vectors (learned weight vectors for each class).
Intuitively, this means that a model is trained to move each sample toward its correct class center in the feature space rather than simply increasing a classification score.
This approach endows the learning process with two key properties: \emph{i) scale invariance} -- distance comparisons do not depend on vector norm, and \emph{ii) finite convergence point} -- training aims for a distance of zero to the correct prototype. %, rather than driving logits to $\pm\infty$ as in cross-entropy.
As a result, each class weight converges to an anchor point that can be interpreted as the \emph{center} of that class’s feature distribution.
Empirically, \citet{baek2025harmonic} demonstrated that harmonic loss can close the train–test gap faster and yield more interpretable representations than cross-entropy.
For example, the learned weight vectors in a harmonic-loss model directly reflect class prototypes, making them semantically meaningful.
Models trained with harmonic loss have been shown to require less data to generalize and to mitigate grokking, while achieving competitive or better accuracy on vision and language benchmarks.
These findings suggest that \emph{distance-based loss functions} are a promising direction for improving both performance and transparency in deep learning.
However, research on harmonic loss has been limited in scope so far.
\citet{baek2025harmonic} focused exclusively on Euclidean distance as the metric for their loss function and did not examine the broader impacts on computational efficiency or energy consumption.
On the other hand, distance-based metrics have been explored in other contexts and problems.
Notably, \citet{coil2025distance} investigated a wide range of distance measures for a problem of change point detection in concept-drift scenarios for anomaly detection.
Their study found that the choice of distance metric can drastically affect both the accuracy and efficiency of detecting distribution shifts.
For instance, replacing a costly metric (\eg, Wasserstein) with simpler alternatives yielded comparable detection performance at substantially lower computational cost.
This evidence that ``metric matters'' in learning algorithms raises a natural question: \emph{might other distance measures offer advantages over Euclidean in a harmonic loss setting?}
To date, no work has evaluated harmonic loss using distance metrics beyond Euclidean, nor has it benchmarked their impacts across different domains.
In this paper, we present the first comprehensive study of \emph{custom distance-based loss functions} in deep learning classification, extending the harmonic loss framework to a variety of distance measures across multiple problem domains.
We experiment with a rich set of non-Euclidean distance metrics, including Manhattan, Euclidean, Chebyshev, Minkowski, and cosine distance, as well as specialized metrics such as Hamming, Canberra, Bray-Curtis, and Mahalanobis.
These metrics are integrated as drop-in replacements for Euclidean distance in the harmonic loss formulation.
We evaluate the resulting non-Euclidean harmonic losses on two heterogeneous task families: \emph{image classification} (MLP, ResNet, PVT) and \emph{language modeling} with transformer-based LLMs (GPT-2, BERT, and others).
This diversity enables us to assess whether certain distance-based losses consistently outperform cross-entropy and Euclidean harmonic loss on metrics of \emph{effectiveness}, \emph{efficiency}, and \emph{explainability}.
Specifically, we pursue the following research questions:

\noindent \textbf{RQ1 (Model Performance):} Do non-Euclidean harmonic loss functions offer higher accuracy or faster convergence compared to cross-entropy and Euclidean harmonic loss?

\noindent \textbf{RQ2 (Interpretability):} Do models trained with non-Euclidean harmonic losses exhibit more interpretable representations than those trained with cross-entropy?

\noindent \textbf{RQ3 (Efficiency \& Sustainability):} If a custom non-Euclidean harmonic loss outperforms cross-entropy on a given downstream task, does it do so without incurring a higher computational cost? We track training time, resource utilization, and energy consumption to assess the \emph{Green AI} perspective~\cite{schwartz2019greenai}.

By addressing these questions, our aim is to explore a \emph{three-way trade-off} between \emph{accuracy, interpretability, and sustainability} in the training process of deep learning models.
Previous work has typically optimized one or two of these aspects in isolation: for instance, improving accuracy at the cost of enormous compute, known as ``Red AI''~\citep{schwartz2019greenai}, or simplifying models for interpretability while losing accuracy.
In contrast, we seek solutions that improve predictive performance while also reducing energy consumption and yielding more transparent models.
\textbf{Contributions.}
This paper introduces \emph{distance-tailored harmonic losses} and provides an extensive empirical and analytical evaluation of their merits.
To our knowledge, this is the first work to:
\begin{inparaenum}[i)]
\item extend the harmonic loss beyond Euclidean distance, and benchmark a wide spectrum of metrics on both vision and NLP tasks;
\item assess the carbon footprint and resource usage of different loss functions in a controlled setting; and
\item investigate interpretability outcomes of distance-based losses.
\end{inparaenum}
We also offer preliminary theoretical insights into how different distance metrics influence the geometry of the learned model (\eg, relating $L_1$ losses to median-based class centers vs.\ $L_2$ to mean-based centers), which could inform the selection of an optimal loss for a given objective.

\section{Harmonic loss}
% Harmonic loss : Definition 

Harmonic loss replaces the conventional inner-product logits and softmax normalization with a distance-based formulation: model predictions are derived from the distances between the sample’s representation and \emph{class prototype} vectors (the learned weight vectors for each class).
Intuitively, this means a model is trained to bring each sample closer to its correct class center in the feature space, rather than simply increasing a classification score.

From \citet{baek2025harmonic}, given the training set $D = \{(x_i, y_i)\}_{i=1}^n$ with $y_i \in \{1, ..., K\}$ and class prototypes $W = \{\mathbf{w}_k\}_{k=1}^K$ with $\mathbf{w}_k \in \mathbb{R}^d$, the harmonic logit is the $L_2$ distance between the prototype $\mathbf{w}_k$ and the instance representation $\mathbf{h} \in \mathbb{R}^{d}$, \ie, $d_k = \|\mathbf{h} - \mathbf{w}_k\|_2$.
Then, the harmonic probabilities are given by:
\begin{align} \label{eq:harmonic}
p_W(y_k \mid x)\;=\;\frac{d_k^{-n}}{\sum_{j=1}^Kd_j^{-n}},
\end{align}
\noindent where the harmonic exponent $n$ is a hyperparameter that controls the heavy-tailedness of the probability distribution.
The Harmonic loss is then given by:
\begin{align}
%
%\mathcal{L}(\{\mathbf{w}_i\}) = -\sum_{i=1}^n \log p_i(x).
\mathcal{L}(W) = -\sum_{(x, y) 
\in D} \log p_{W}(y \mid x).
\end{align}
This approach endows the learning process with two key properties:
\begin{inparaenum}[i)]
%
%\item \emph{scale invariance:} distance comparisons do not depend on the overall norm of $\mathbf{h}$ or $\mathbf{w}_k$, in contrast to inner-product logits; and
%
\item \emph{ratio invariance:} for 1-homogeneous distances (\eg, Euclidean, Manhattan), the class probabilities $p_{W}(y \mid x)$ are invariant to uniform rescaling of both representations and prototypes, since factors of $c^{-n}$ cancel in the ratio (\cref{thm:finite}); this contrasts with cross-entropy, where scaling logits changes the softmax distribution;~\footnote{\emph{Individual} scale invariance---invariance to scaling $\mathbf{h}$ alone---holds only for angular distances such as cosine; for Euclidean distance, $\|\alpha \mathbf{h} - \mathbf{w}\|_2 \neq \|\mathbf{h} - \mathbf{w}\|_2$ in general.} and
\item \emph{finite convergence point:} optimization seeks a distance of zero to the correct prototype. %, rather than driving logits to $\pm \infty$ as in cross-entropy.
\end{inparaenum}

As a result, each class weight converges to an anchor point that can be interpreted as the \emph{center} of that class’s feature distribution.
Empirically, \citet{baek2025harmonic} demonstrated that harmonic loss can close the train--test gap faster and yield more interpretable representations than cross-entropy.
For example, the learned weight vectors in a harmonic-loss model directly reflect class prototypes, making them semantically meaningful. Models trained with harmonic loss were also shown to require less data to generalize and to mitigate grokking, all while achieving competitive or better accuracy on both vision and language benchmarks.
These findings suggest that \emph{distance-based loss functions} are a promising direction for improving performance and transparency in deep learning.
\section{Non-Euclidean Harmonic Losses} \label{eq:method}
%
%We extend harmonic loss beyond Euclidean distance by considering a broad family of distances. 
Our framework introduces \emph{non-Euclidean harmonic losses} as a generalization of the harmonic loss, and as a replacement for conventional cross-entropy training. 
%
% The workflow begins with input data $\mathbf{x}$ drawn from one of four benchmark datasets: MNIST, CIFAR-10, CIFAR-100 (vision tasks), or OpenWebText (language modeling).
%
%An encoder $f_\theta(\cdot)$ maps the input to a feature representation $\mathbf{h} \in \mathbb{R}^d$, using one of four neural backbones: a multi-layer perceptron (MLP), a convolutional neural network (CNN), a residual network (ResNet-50), or a pyramid vision transformer (PVT).
%
%For language modeling experiments, we employ transformer-based architectures such as GPT-2, BERT, or MiniLM.
%
The idea is that, in \cref{eq:harmonic}, the Euclidean distance $d_k = \|\mathbf{h} - \mathbf{w}_k\|_2$ is replaced by a non-Euclidean distance.

\subsection{Class Prototypes, Distances, and non-Euclidean Harmonic Loss functions}
Each class $k \in \{1,\dots,K\}$ is associated with a \emph{prototype vector} $\mathbf{w} \in \mathbb{R}^d$.
Given an instance representation $\mathbf{h}$, we compute its distance to all prototypes via a chosen metric $d(\cdot,\cdot)$.
Prototypes are learned parameters, just like the weight matrix in linear classification.
%
%In cross-entropy, the final linear layer maintains a weight vector per class. In harmonic loss, these weight vectors are simply reinterpreted as prototypes. 
%
Thus, prototype learning is no more computationally expensive than learning a final linear layer.

We extend the Euclidean formulation of harmonic loss~\citep{baek2025harmonic} with the following distances:
\textbf{Euclidean.}
Baseline Euclidean distance between feature and prototype:
\[
d_{\text{Euclidean}}(\mathbf{h},\mathbf{w}) = \|\mathbf{h} - \mathbf{w}\|_2.
\]
%
%%%%%%%%%%%%%%%%%%%%%%%%%%%%%%%%%%%
% Euclidean (L2) Distance
%%%%%%%%%%%%%%%%%%%%%%%%%%%%%%%%%%%
% Claim: It serves as the baseline distance measure between feature vectors and prototypes. Supporting Evidence: Euclidean distance is widely used as a default similarity metric in machine learning (e.g. clustering, nearest neighbors, prototype classifiers) due to its simplicity and convenience[1]. For instance, k-means and prototype-based classifiers typically rely on Euclidean distance as the standard measure of closeness[1]. This prevalence justifies treating Euclidean distance as a baseline in most feature-based tasks.

%
% [1] Euclidean Distance - an overview | ScienceDirect Topics
%https://www.sciencedirect.com/topics/computer-science/euclidean-distance

\textbf{Manhattan ($L_1$).}
%
%Emphasizes absolute differences, making it potentially more robust to outliers~\citep{keeling2016robust,ye2012l1pca,giloni2003breakdown}. {\color{red} check}
%
The $L_1$ norm emphasizes absolute differences, providing robustness to outliers since large deviations contribute linearly rather than quadratically to the total distance~\citep{keeling2016robust,ye2012l1pca,giloni2003breakdown}.
\[
d_{\text{Manhattan}}(\mathbf{h},\mathbf{w}) = \|\mathbf{h} - \mathbf{w}\|_1.
\]
%
%It can stabilize training and reduce unnecessary computations, thereby lowering energy costs.
%
Furthermore, computing $L_1$ distance requires only additions and subtractions without squaring or square-root operations, making it computationally efficient.
%
%{\color{red} source? also, can remove}
%
Recent work on AdderNets demonstrates that replacing multiplication-based convolutions with $L_1$-based operations can reduce energy consumption while maintaining competitive accuracy~\citep{chen2020addernet}.
\textbf{Chebyshev ($L_\infty$).}
\[
d_{\text{Chebyshev}}(\mathbf{h},\mathbf{w}) = \|\mathbf{h} - \mathbf{w}\|_{\infty}.
\]
Captures the maximum coordinate deviation, offering a highly interpretable measure of the most discriminative feature dimension.
Its simplicity makes it computationally efficient.
\textbf{Minkowski ($L_p$).}
%
%Generalizes both $L_1$ and $L_2$, with a tunable $p$ that enables a trade-off between robustness and sensitivity {\color{red} source?}:
%
Generalizes both $L_1$ and $L_2$, with a tunable exponent $p$ that controls the trade-off between robustness and sensitivity~\citep{aggarwal2001surprising,hu2016distance}:
\[
d_{\text{Minkowski}}(\mathbf{h},\mathbf{w};p) = \|\mathbf{h} - \mathbf{w}\|_{p}.
\]
%
%This flexibility allows tailoring the loss to dataset complexity, improving accuracy while balancing sustainability.
%
Lower values of $p$ (closer to 1) down-weight large individual coordinate differences, yielding behavior more robust to outliers, while higher values (toward 2 or above) increase sensitivity to larger deviations.
This flexibility allows the metric to be tailored to dataset characteristics, treating $p$ as a hyperparameter to optimize accuracy.

\textbf{Cosine.}
%
%Ignores magnitude and instead measures angular similarity, making it particularly effective in high-dimensional embeddings (\eg, CNNs, Transformers)~\citep{reimers2019sentencebert,deng2019arcface,wang2018cosface,sun2016learning,karpukhin2020dense}.
%
Ignores vector magnitudes and instead measures angular similarity, making it particularly effective for high-dimensional embeddings~\citep{reimers2019sentencebert,deng2019arcface,wang2018cosface,sun2016learning,karpukhin2020dense}.
\[
d_{\text{cosine}}(\mathbf{h},\mathbf{w}) = 1 - \tfrac{\mathbf{h}^\top \mathbf{w}}{\|\mathbf{h}\|_2 \, \|\mathbf{w}\|_2}.
\]
%
%This often improves generalization with minimal computational overhead.
%
Because the output is bounded, cosine similarity decreases activation variance compared to unbounded dot products, which can improve generalization~\citep{luo2017cosine}.
%
%{\color{red} source?}
%
The computational overhead is minimal, only requiring an additional normalization step beyond the standard dot product.

\textbf{Hamming.}
%
%Counts mismatches directly, providing highly interpretable signals. {\color{red} why?}
%
Counts mismatches directly, providing a highly interpretable signal: each unit of distance corresponds to exactly one differing feature dimension.
\[
d_{\text{Hamming}}(\mathbf{h},\mathbf{w}) = \tfrac{1}{d} \sum_{i=1}^d \mathbf{1}_{\{h_i \ne w_i\}}.
\]
Although inherently discrete, Hamming distance can be integrated into neural network training via continuous relaxations such as the Gumbel-Softmax trick~\citep{jang2017gumbel,maddison2017concrete}, which provide differentiable approximations to binary sampling.
%
%With soft or Gumbel relaxations, it becomes suitable for continuous embeddings and can reduce emissions when binary approximations are leveraged.
%
%{\color{red} source?}
%
Binary embeddings offer substantial efficiency gains -- comparing binary codes via Hamming distance can be significantly faster and more memory-efficient than comparing full-precision vectors.
\textbf{Canberra.}
%
%Normalizes differences by feature magnitudes, enhancing sensitivity to small but meaningful variations. {\color{red} sure?}
%
Normalizes each coordinate difference by the sum of the magnitudes, yielding heightened sensitivity when both values are near zero~\citep{lance1967general}.
\[
d_{\text{Canberra}}(\mathbf{h},\mathbf{w}) = \sum_{i=1}^d \frac{|h_i - w_i|}{|h_i| + |w_i| + \varepsilon}.
\]
%
%This can improve performance on fine-grained tasks while stabilizing optimization. {\color{red} source?}
%
Since each term is bounded in $[0,1]$, the metric emphasizes proportional rather than absolute differences.
This property can benefit fine-grained recognition tasks in which small variations in low-magnitude features carry important discriminative information.
\textbf{Bray-Curtis.}
%
%Captures proportional differences across feature vectors, making it efficient and interpretable for compositional data~\citep{fuschi2025microbiome,chao2010decomposition,song2020systematic}.
%
A normalized variant of the $L_1$ distance that measures proportional differences across all features~\citep{fuschi2025microbiome,chao2010decomposition,song2020systematic}.
\[
d_{\text{Bray-Curtis}}(\mathbf{h},\mathbf{w}) = \frac{\sum_{i=1}^d |h_i - w_i|}{\sum_{i=1}^d (|h_i| + |w_i|) + \varepsilon}.
\]
%
%It often balances accuracy with sustainability better than covariance-based measures.
%
Widely used in ecology for comparing species abundance profiles, Bray--Curtis produces values in $[0,1]$ that are easy to interpret: 0 indicates identical composition, 1 indicates no overlap.
%
%{\color{red} source?}
%
Because it does not require estimating a covariance matrix, it offers a computationally lighter alternative to Mahalanobis distance while still accounting for relative differences.
\textbf{Mahalanobis.}
%
%Incorporates feature correlations, offering superior accuracy in complex datasets and deep CNNs~\citep{pang2018maxmahalanobis,lee2018simple,gomez2021mahalanobis,omara2021novel}. {\color{red} check; also really don't like "superior"}
%
Incorporates feature correlations by weighting distances according to the inverse covariance matrix $\Sigma^{-1}$, effectively measuring distances in a whitened feature space~\citep{pang2018maxmahalanobis,lee2018simple,gomez2021mahalanobis,omara2021novel}.
\[
d_{\text{Mahalanobis}}(\mathbf{h},\mathbf{w};\Sigma) = \sqrt{(\mathbf{h}-\mathbf{w})^\top \Sigma^{-1}(\mathbf{h}-\mathbf{w})}.
\]
%
%Although covariance estimation may increase computational cost, its interpretability and classification power justify the trade-off in high-capacity models.
%
This can improve discrimination when features have different scales or are correlated; for example, \citet{lee2018simple} achieve state-of-the-art out-of-distribution detection in deep networks using Mahalanobis distance on CNN features.
%
%{\color{red} source?}
%
The computational trade-off is that estimating and inverting the covariance matrix incurs overhead, particularly in high dimensions.
However, the resulting distance is scale-invariant and unitless, facilitating the interpretability of class boundaries as ellipsoids in feature space.
In our framework, we generalize the harmonic loss by replacing the Euclidean distance in \cref{eq:harmonic} with one of the distances defined above. 
%
%Harmonic loss is applied only at the final classification layer, replacing the standard softmax + cross-entropy objective.
%
This substitution is applied only at the final classification layer, replacing the standard softmax cross-entropy head.
%
%{\color{red} improve wording}
%
%All intermediate layers remain unchanged, and no normalization is applied inside the backbone.
%
The feature extraction backbone remains unchanged; we do not introduce additional normalization layers or architectural modifications beyond the classification head.
%
%{\color{red} ?}
%

%
%Overall, compared to cross-entropy, these distance-based harmonic losses reduce reliance on probabilistic normalization {\color{red} ???} and can lower the number of required operations.
%
Compared to standard cross-entropy, distance-based harmonic losses offer several potential benefits: the scale-invariant formulation avoids unbounded logit growth, the geometric interpretation provides transparency (class decisions correspond to proximity to prototype vectors), and certain distance choices (\eg, $L_1$, cosine) can reduce computational cost.
%
%This translates into potential accuracy gains, reduced carbon emissions, and improved interpretability, depending on the chosen distance and backbone.
%

%
A formal treatment of our distance–based probabilistic layer is provided in \cref{sec:theory}.
There, we generalize the harmonic-loss analysis to broad distance families and prove: i) \emph{scale invariance} and the existence of \emph{finite} minimizers under 1-homogeneous distances (\cref{thm:finite}), and ii) a \emph{margin-style PAC–Bayes generalization bound} whose finiteness follows from the finite–norm solution (\cref{thm:margin}).
These results clarify when geometry choices are well-posed and why the resulting classifiers admit standard generalization guarantees.  %; full assumptions and proofs are deferred to the appendix.

\section{Experiments}

\subsection{Training and Evaluation}
% The workflow begins with input data $\mathbf{x}$ drawn from one of four benchmark datasets: MNIST, CIFAR-10, CIFAR-100 (vision tasks), or OpenWebText (language modeling).
% %
% An encoder $f_\theta(\cdot)$ maps the input to a feature representation $\mathbf{h} \in \mathbb{R}^d$, using one of four neural backbones: a multi-layer perceptron (MLP), a convolutional neural network (CNN), a residual network (ResNet-50), or a pyramid vision transformer (PVT).
% %
% For language modeling experiments, we employ transformer-based architectures such as GPT-2, BERT, or Qwen.
% %
% Training proceeds by minimizing $\mathcal{L}_{\text{harm}}$ using stochastic gradient descent (SGD). 
% All models are trained under matched conditions against a cross-entropy baseline, ensuring fair comparison. 
% A detailed analysis of hyperparameter configurations is provided in Appendix \ref{appendix:hyperparams}.

% Evaluation follows a \emph{multi-criteria protocol} along three axes:
% %\begin{enumerate}
% %
% i) \emph{Accuracy:} classification accuracy or perplexity depending on the task.
% %
% ii) \emph{Interpretability:} quality of prototype vectors as class centers, assessed via feature clustering and visualization.
% %
% iii) \emph{Sustainability:} energy consumption and CO$_2$ emissions measured via CodeCarbon.
% %
% % \end{enumerate}
% %
% This workflow allows us to directly compare how different distance measures, when embedded into harmonic loss, affect effectiveness, sustainability, and interpretability across both vision and language domains.

\textbf{Datasets.}
We evaluate on five \emph{vision} benchmarks (MNIST, CIFAR-10, CIFAR-100, MarathiSignLanguage, TinyImageNet) and one \emph{language} corpus (OpenWebText).
%For vision, inputs are preprocessed as follows: MNIST is normalized to mean/std $(0.5)$ (MLP/CNN) or to $(0.1307,0.3081)$ (ResNet-50), with grayscale replicated to 3 channels for ViT/PVT when needed; CIFAR-10 uses mean/std $(0.4914,0.4822,0.4465)$ / $(0.2023,0.1994,0.2010)$ with random horizontal flips, random crops ($32{\times}32$, padding 4), and small rotations for ResNet-50; CIFAR-100 uses mean/std $(0.5071,0.4867,0.4408)$ / $(0.2675,0.2565,0.2761)$ with stronger augmentation (flips, crops, color jitter). ViT inputs are resized to $224{\times}224$; PVT uses $32{\times}32$ with 3 channels.
%For language modeling, we tokenize OpenWebText into memory-mapped \texttt{train.bin}/\texttt{val.bin}, and read contiguous $L$-length blocks for batching.

\textbf{Vision.}
We consider a simple MLP with two hidden layers (512, 256, ReLU), a simple CNN (two $3{\times}3$ conv blocks with $[32,64]$ channels and $2{\times}2$ max-pooling, then a 128-dim FC), ResNet-50 (standard $[3,4,6,3]$ bottleneck stages; for small inputs we remove the initial max-pool and use a $3{\times}3$ stride-1 stem), and PVTv2-B0 (four hierarchical stages with overlapping patch embeddings; output pooled to a 256-dim vector).

\textbf{Language.}
We study three Transformer families: GPT-style (decoder-only causal LM), BERT (encoder-only masked LM with 15\% masking), and Qwen2-style decoders. %(RoPE, RMSNorm, grouped KV heads). 
%Vocabulary size is read from metadata when available; Qwen2 runs can use the native tokenizer.

% \textbf{Distance-based heads (harmonic loss).}
% For a final-layer feature $\mathbf{h}\in\mathbb{R}^d$ and class (or token) prototype $\mathbf{w}_c\in\mathbb{R}^d$, our \emph{harmonic} head replaces the linear classifier with
% \[
% z_c \;=\; -\,D(\mathbf{h},\mathbf{w}_c;\Theta), \qquad
% p_c \;=\; \frac{\exp(z_c)}{\sum_{j}\exp(z_j)},
% \]
% and trains with standard cross-entropy on $p_c$.\footnote{For cosine/similarity variants we may use $\log$-similarities for numerical stability; details match the robust implementations in \S\ref{sec:methods}.}
% We instantiate $D(\cdot)$ with the distances formalized in \S\ref{sec:distances}: Euclidean, Manhattan, Minkowski ($p$), Cosine, Chebyshev (optionally smooth), Canberra (standard/robust/weighted), Bray--Curtis (standard/abs/normalized), Mahalanobis (standard/diagonal/Cholesky), and Hamming (soft/gumbel/hard).
% The \emph{baseline} uses the conventional linear layer (inner-product logits).

%\subsection{Training Protocol}
\textbf{Optimization.}
Unless noted, models are trained from scratch with Adam/AdamW-style optimizers (weight decay, $(\beta_1,\beta_2)$ as configured), cosine learning-rate decay with linear warmup, mixed precision (FP16/BF16 when available), and gradient accumulation. We apply gradient clipping, dataset-specific schedulers, and early stopping with \emph{dataset-specific patience} and a minimum improvement threshold ($\Delta_{\min}$). For fairness, all harmonic layers and the baseline share the same backbone, batch size, scheduler, and data order. Additional details about optimization are reported in Appendix \ref{appendix:vision-arch}.

% \textbf{Batching.}
% Vision uses standard mini-batches with the preprocessing described above. GPT/Qwen use causal LM batches $(X,Y)$ built from contiguous blocks ($Y$ is $X$ shifted by one). BERT uses masked-LM batches with 15\% corruption: 80\% \texttt{[MASK]}, 10\% random token, 10\% unchanged; loss is computed only on masked positions.

% \textbf{Compute and distribution.}
% We train on single- and multi-GPU setups using \texttt{DistributedDataParallel}. For small images (MNIST/CIFAR), the ResNet-50 stem is adapted (no initial max-pool) to preserve resolution; for distance-head ResNet we $L_2$-normalize features before cosine-style heads if needed.

%\subsection{Evaluation Protocol}
\textbf{Model Performance.}
For vision tasks, we report average Accuracy and F1. %\textbf{top-1 accuracy} (per epoch curves and maxima) and \textbf{F1}. 
% For language we report \textbf{token-level accuracy} (causal LM next-token or BERT MLM-at-masked). We also summarize convergence via \emph{area-under-accuracy-curve} (AUCC) across epochs to capture \emph{both} final accuracy and training speed.
For language task, we report the following metrics: % Gradient Stability, %(higher values indicate smoother training with lower variance in gradient norms), 
% Model Health, % ($-\Delta\,$\texttt{model\_collapse\_score}): higher values means stronger resistance to representation collapse, 
% Clipping Quality, %(higher scores indicate healthier gradient flow without extreme values requiring intervention), 
% and Learning Quality, defined as: %(captures both how well the model learned and how much it improved). More formally:

\textbf{Perplexity (Train / Val).}
%
%To quantify language modeling quality, we report standard token--level \emph{perplexity} on both the training and validation splits.
%
Given a sequence of targets $\{y_t\}_{t=1}^T$ and model probabilities $p_\theta(y_t \mid \text{context})$, where $\theta$ denotes the model parameters, the average negative log--likelihood is:
\[
\mathcal{L}_{\text{NLL}}
= - \frac{1}{T} \sum_{t=1}^{T} \log p_\theta\big(y_t \mid \text{context}\big),
\]
\noindent and the corresponding perplexity is
\(
\mathrm{PPL} 
= \exp\!\big(\mathcal{L}_{\text{NLL}}\big).
\)
Lower perplexity indicates better next-token prediction.
\footnote{
For visualization in the radar plots, we invert perplexity (and all metrics where lower values indicate better performance) and then normalize to the range $[0,10]$ relative to the harmonic Euclidean baseline.
The effect is that the greater the coverage on the plot, the better the relative performance compared to Euclidean.
Importantly, the absolute numeric values on the radial axis do not have a direct “good/bad” interpretation in perplexity space; they are meaningful only as normalized, experiment–specific comparisons against the Euclidean harmonic.
}

{
\textbf{Gradient Stability (GS).}
To quantify the smoothness of optimization, we measure the variance of the $L_2$-norm of the gradient across consecutive training steps. Let $\mathcal{L}_t$ denote the training loss at step~$t$; then:
\[
\mathrm{GS} 
= 1 - \frac{\operatorname{Var}\big(\lVert \nabla_\theta \mathcal{L}_t \rVert_2 \big)}
           {\operatorname{Var}\big(\lVert \nabla_\theta \mathcal{L}_t \rVert_2 \big)_{\text{CE}}},
\]
\noindent where $\operatorname{Var}(\cdot)$ is computed over a fixed evaluation window (\eg, 500 steps) and the denominator corresponds to the variability under
cross–entropy (CE).
%
% {\color{red} is there a source/ref for GS?}
% 
This metric is motivated by gradient variance analyses~\citep{DBLP:journals/corr/abs-2007-04532,DBLP:conf/icml/LeiY20},
%{\color{red} PLEASE CHECK.}
which demonstrate that gradient variance dynamics correlate with training stability and convergence behavior.
Thus, $\mathrm{GS}\!=\!0$ indicates equal smoothness as CE, $\mathrm{GS}>0$ indicates reduced gradient variance (smoother training), and $\mathrm{GS}<0$ reflects more unstable gradient dynamics (higher variance than CE).
%
%Higher values indicate smoother training with lower variance in gradient norms.
%
This metric is anchored in standard variance-of-gradient analyses used in optimizing large-scale LLMs.
\textbf{Effective Rank.} This metric is inspired by analyses of dimensional collapse~\cite{DBLP:conf/iclr/JingVLT22,DBLP:conf/iclr/BardesPL22}, %
%{\color{red} PLEASE CHECK}
that track the covariance spectrum of learned representations to detect embedding collapse and maintain representation quality in neural networks.
%
% Higher values means stronger resistance to representation collapse.
%
%This relates directly to standard metrics used in collapse detection and embedding drift.
Let $\Sigma_h = \operatorname{Cov}(\mathbf{h}_t)$ with eigenvalues $\{\lambda_k\}_{k=1}^d$. Define the normalized eigenvalues $p_k = \lambda_k / \sum_j \lambda_j$. The effective rank is:
\[
\operatorname{ER}(\Sigma_h) = \exp\!\left(-\sum_{k=1}^d p_k \ln p_k\right).
\]
This is the exponential of the Shannon entropy of the eigenvalue distribution \cite{DBLP:conf/icml/GarridoBNL23} \cite{roy2007effective}. 
Metric values are in the range $[1, d]$, where a value of 1 means total collapse (one dimension dominates), and value of $d$ means perfectly uniform spread. Higher values indicate that the learned weight matrices retain higher intrinsic dimensionality, suggesting the distance metric discourages low-rank collapse in the parameter space and is consistent with more distributed use of the embedding capacity. 
}
% Clipping Quality
% Model Health
% Gradient Stability
% Learning Quality Combines final training 
% improvements $\Delta$ with respect to the Euclidean harmonic-loss baseline (positive is better unless noted). \textbf{Model performance:} \emph{ 

\textbf{Interpretability.}
We probe whether learned prototypes/weights act as class centers and whether features become more structured by computing PCA explained variance on the penultimate features: i) \textit{PC2~EV} (variance explained by the top two PCs), and ii) \textit{PCA@90\%} (dimensions required to reach 90\% variance). Lower PCA@90\% and higher PC2~EV indicate more concentrated, low-dimensional structure. For language, %we compute these on the final hidden states (causal: last token; MLM: masked-token positions).
% \textbf{Sustainability:} 
%\emph{Emissions}—$\Delta$\,CO$_2$ (grams), where lower/negative is greener. \textbf{Interpretability:} 
we report \emph{PCA5}: $\Delta$ variance explained by the top 5 principal components of final hidden states (causal LM: last token; MLM: masked positions); higher values implies more concentrated, low-dimensional structure.

\textbf{Sustainability.}
We perform training with \textit{CodeCarbon} to log \emph{duration}, \emph{energy}, and \emph{CO$_2$ emissions}. Emissions are reported per run and \emph{differentially} vs.\ the cross-entropy baseline (grams CO$_2$; negative means greener-than-baseline). We aggregate by (dataset, backbone, distance) and also report cumulative figures across seeds. For language, we also report  Speed ($-\Delta\,$\texttt{time\_to\_90\_percent}): higher values denotes fewer steps to reach $90\%$ of final performance.

%\textbf{Comparison Protocol and Controls}
To isolate the effect of the \emph{loss geometry}, we \emph{only} swap the classifier head (linear vs.\ distance-based) while keeping: backbone weights initialization scheme, data preprocessing/augmentation, optimizer and LR schedule, batch size, number of epochs, early-stopping rule, and randomness controls (seeds). For ResNet-50/PVT we use identical augmentation; for LLMs we use the same context length $L$, optimizer, and schedule across heads. We run multiple seeds and report means. %where space allows we include per-seed curves in the appendix.
% 
%\subsection{Reporting}
%We present: (i) epoch-wise accuracy curves with legends annotated by per-method \emph{max accuracy}; (ii) radar plots that jointly summarize \emph{Effectiveness} (Accuracy, F1), \emph{Interpretability} (PC2~EV, PCA@90\%), and \emph{Sustainability} (Duration/Epoch, Emissions); (iii) bar plots of emissions deltas vs.\ baseline; and (iv) tables of aggregate metrics with percentage changes relative to baseline (in parentheses). 
Exact architectures and preprocessing pipelines are detailed in Appendix~\ref{appendix:vision-arch}.
Full hyperparameter grids (including head-specific parameters $\Theta$, e.g., $p$ for Minkowski or covariance settings for Mahalanobis) are provided in Appendix~\ref{appendix:hyperparams}.
This unified protocol lets us \emph{systematically} test how replacing the Euclidean harmonic head with alternative distances impacts: i) final model performance, ii) representation structure and prototype semantics, and iii) measured energy and carbon footprint. %across both vision and language settings under matched training conditions.

\subsection{Vision: Radar Plots}

{
Figure~\ref{fig:acc:spider:vision}
summarizes the behavior of distance-based harmonic losses across all
vision settings, including a high–resolution sign language
dataset (Marathi Sign) and TinyImageNet in addition to CIFAR-100. 
Additional results on MNIST and CIFAR10 are provided in Appendix \ref{sec:spider-extra}.
Together, these radar plots expose how the choice of distance in the
harmonic loss shapes performance, representation geometry, and
sustainability.

\textbf{RQ1: Model Performance (F1, Accuracy).}
Across datasets and backbones, \emph{cosine–based harmonic losses}
remain the most reliable all–round performers.  
On CIFAR-100, cosine (stable/unstable) typically attains the highest or
near–highest accuracy and F1 on CNN and ResNet50, and is consistently
among the top curves on PVT.  
On the more realistic, higher–resolution Marathi Sign and
TinyImageNet, the same pattern largely persists: cosine (stable) and
Bray--Curtis (normalized) frequently improve or match Euclidean and
cross–entropy on CNN, ResNet50, and PVT, while also appearing in the
top group on MLP.  TinyImageNet is the most challenging setting:
cross–entropy remains a strong baseline, but cosine heads still achieve
competitive accuracy on ResNet50 and PVT, demonstrating that the
benefits of distance–tailored heads extend beyond small benchmarks.
Other non–Euclidean distances (Bray-Curtis variants, Manhattan,
Minkowski) can occasionally match or exceed cosine in specific
architecture–dataset combinations. %, but their gains are more
%configuration–dependent.

\textbf{RQ2: Interpretability (PC2 EV, PCA 90\%).}
Non–Euclidean distances reshape the final embedding geometry in a
systematic, dataset–agnostic way.  
Across Marathi Sign, TinyImageNet, and CIFAR-100, Bray--Curtis
(standard/normalized) and Chebyshev (standard) repeatedly yield the largest PC2 explained variance and the lowest dimensionality required to reach $90\%$ EV, indicating compact, prototype–aligned feature spaces with sharper class clusters than those produced by Euclidean harmonic loss or cross–entropy.
Cosine harmonic loss generally provides substantial EV gains over Euclidean, while retaining top accuracy, offering a favorable accuracy–interpretability balance on both convolutional backbones and PVT.
Mahalanobis variants often achieve extreme variance concentration (very high EV) and pronounced cluster separation, but this representation clarity sometimes co–occurs with less stable optimization on the hardest datasets.
Overall, the same geometric trends observed on earlier small benchmarks persist when moving to higher resolutions and deeper models: non–Euclidean harmonic losses, especially Bray-Curtis and Chebyshev, produce more structured, low–dimensional embeddings than Euclidean or cross–entropy heads.

\textbf{RQ3: Sustainability (Duration/Epoch/GFLOPs, Emissions).}
Distance choice also affects efficiency, but in a controlled way.
Across all datasets, cosine harmonic loss is typically neutral to favorable in emissions relative to Euclidean and cross–entropy: normalized Duration/Epoch/GFLOPs and gCO$_2$eq remain comparable, and in several ResNet50 and PVT runs, cosine achieves slightly lower emissions due to a faster approach to high accuracy.
Bray-Curtis losses incur modest overhead while delivering strong interpretability gains, whereas Mahalanobis distances are the most costly, reflecting their covariance–related computation and sometimes slower convergence on complex data.
Even on high–resolution Marathi Sign and TinyImageNet, the harmonic head accounts for only a small fraction of total FLOPs; thus, differences in Duration/Epoch are smaller than differences in accuracy or EV, yet cumulative emissions still meaningfully separate distances. % over many
%epochs.

Across all vision workloads, three regularities emerge: i) cosine harmonic loss is the best all–around choice, offering consistently strong accuracy/F1, clear geometric structure relative to Euclidean, and neutral–to–lower emissions from MLPs up to ResNet50/PVT on Marathi Sign and TinyImageNet; ii) Bray--Curtis and Chebyshev are the most interpretability–forward options, reliably increasing variance concentration and reducing PCA~90\% dimensionality, with accuracy effects that are positive but more configuration–dependent; iii) Mahalanobis emphasizes representation clarity at a higher sustainability cost. %, making it attractive when prototype sharpness and cluster separation are prioritized over raw efficiency.
Taken together, the radar plots show that the geometry of the harmonic loss, especially non–Euclidean choices, has a consistent, architecturally robust effect on performance, structure, and sustainability across both small and large vision benchmarks.
}

\subsection{Language: Radar Plots}
%Results in \ref{fig:acc:spider:language}.
% %%% PLOT HERE %%%
% %%% 3 HORIZONTAL 1 ROW %%%
%\paragraph{Language Radar Plots.}
Figure~\ref{fig:acc:spider:language} summarizes the effect of distance-tailored harmonic losses on \emph{BERT}, \emph{GPT}, and \emph{Qwen}-style decoders across the three perspectives. %: Model Performance (Model Health, Gradient Stability, Learning Stability, Clipping Quality), Interpretability (PCA Structure), and Sustainability (Emissions). 
Scores are normalized so that larger areas indicate more desirable behavior.

\textbf{RQ1: Model Performance (Perplexity, Health, Stability).}
Across architectures, cosine–based harmonic losses remain the most reliable choices on
performance–oriented axes.  
For BERT, cosine heads %(with moderate temperatures) 
achieve low train and validation perplexity
while improving Gradient Stability and preserving high effective rank relative to both
cross–entropy and Euclidean harmonic loss. %Very sharp cosine temperatures can over–constrain the logits,
%leading to slightly worse perplexity and less stable training, but the ``simple'' and medium–temperature
%variants define the best overall envelope.  
For GPT, cosine and Minkowski ($p{=}2$) again provide steady training dynamics with competitive perplexity, whereas the cross–entropy baseline exhibits higher variability and weaker stability.
On Qwen, none of the approaches provides a satisfactory trade-off. Euclidean harmonic loss offers the strongest effective rank and gradient stability,
while Minkowski provides the best perplexity. % a close alternative; the cross–entropy head is consistently dominated on at least
%one of these axes.  
In summary, in two out of three cases, replacing the linear classifier with distance–based harmonic heads reduces gradient volatility and collapse symptoms while maintaining or improving perplexity. % across LLM architectures.

\textbf{RQ2: Interpretability (PCA Structure).}
Non–Euclidean distances consistently concentrate token representations into more structured latent spaces.
In BERT and GPT, cosine and Minkowski enlarge the PCA Structure wedge (higher variance explained by a small
number of components), indicating more organized, prototype–aligned embeddings than those produced by
cross–entropy or Euclidean harmonic loss. Qwen shows a similar pattern: distance–based heads achieve
clearer low–dimensional structure even when Euclidean is slightly stronger on stability. As in the vision
experiments, geometries that emphasize angles (cosine) or $L_p$ structure (Minkowski) tend to
yield hidden states that are easier to summarize with a few principal components.

\textbf{RQ3: Sustainability (Emissions).}
Results confirm that distance–based harmonic heads introduce little computational overhead
and can be greener than cross–entropy in practice. In all three models, the cross–entropy baseline
occupies the largest emissions wedge, while cosine and Minkowski are neutral–to–favorable, often matching or
improving on Euclidean harmonic loss. Extremely sharp cosine temperatures may reduce emissions slightly but
at the cost of stability and perplexity; moderate settings avoid this trade–off. Because the classifier head
is lightweight compared to the Transformer backbone, these sustainability gains primarily arise from smoother
optimization and faster convergence rather than per–step FLOPs.

In summary, cosine–based harmonic losses are the most robust all–around choice for LLMs, jointly improving
perplexity, stability, and representation structure with neutral or reduced emissions. Minkowski ($p{=}2$)
provides a strong alternative when cosine hyperparameters are poorly tuned, while Euclidean remains a solid reference but is rarely dominant over non–Euclidean geometries. % are available.
Additional results showcasing optimization dynamics for all models, including a larger GPT2 (2B) model, are reported in Appendix \ref{sec:convergence}.

\begin{figure*}[h!]
    \begin{centering}
    % MLP
    \includegraphics[page=14,width=0.28\textwidth,trim=45 80 80 15,clip]{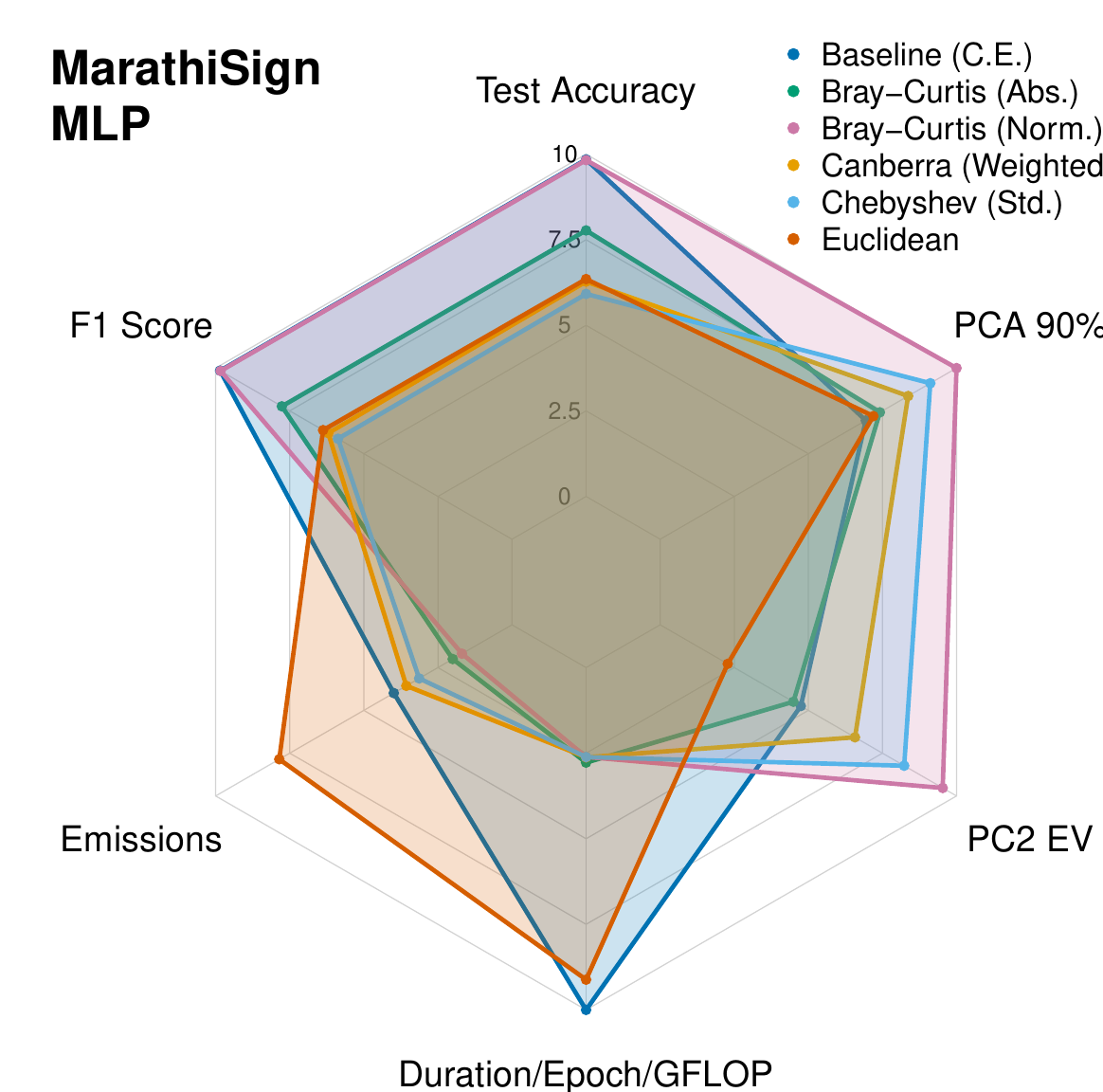}
    \includegraphics[page=18,width=0.28\textwidth,trim=45 80 80 15,clip]{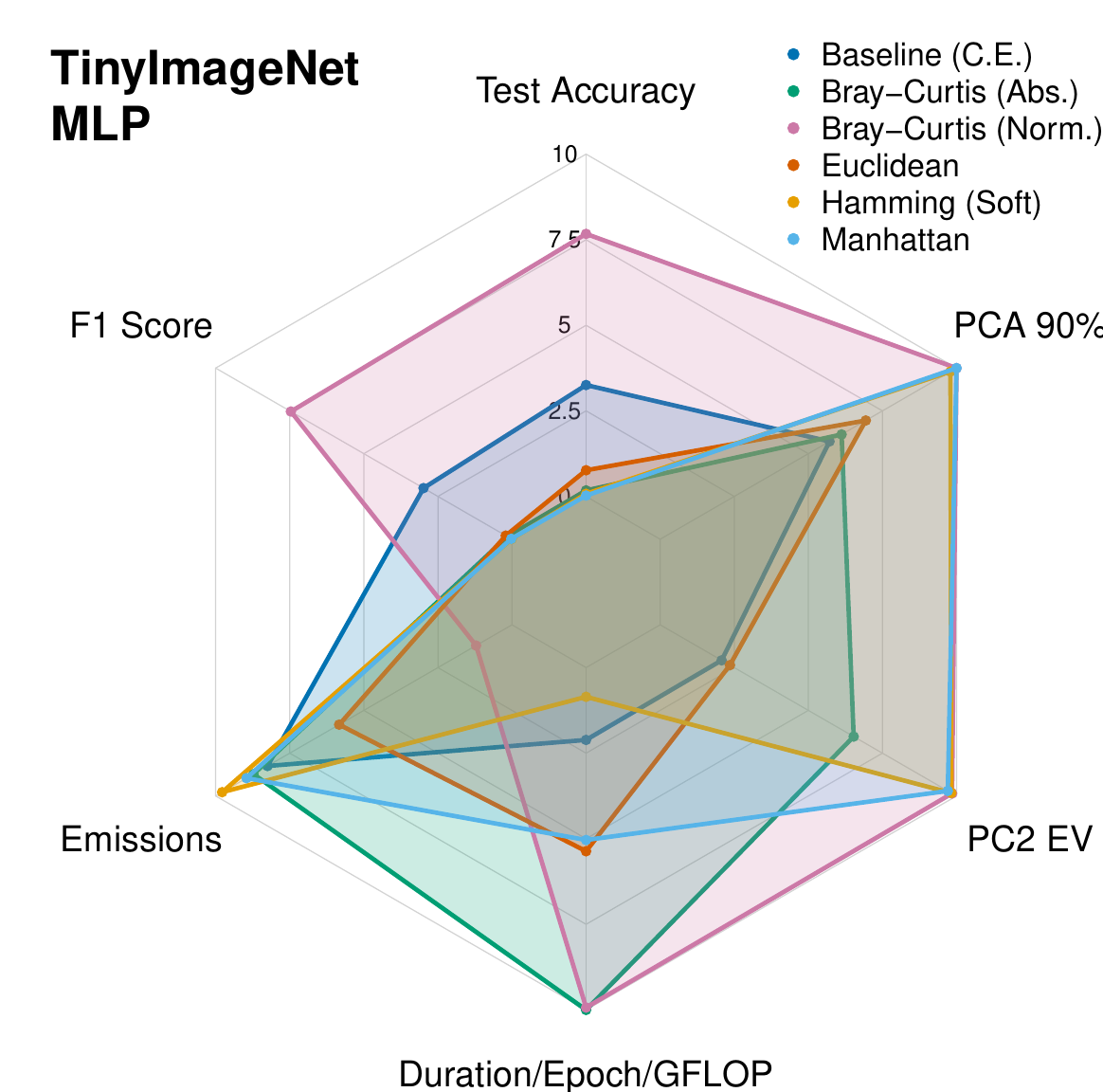}
    \includegraphics[page=6,width=0.28\textwidth,trim=45 80 80 15,clip]{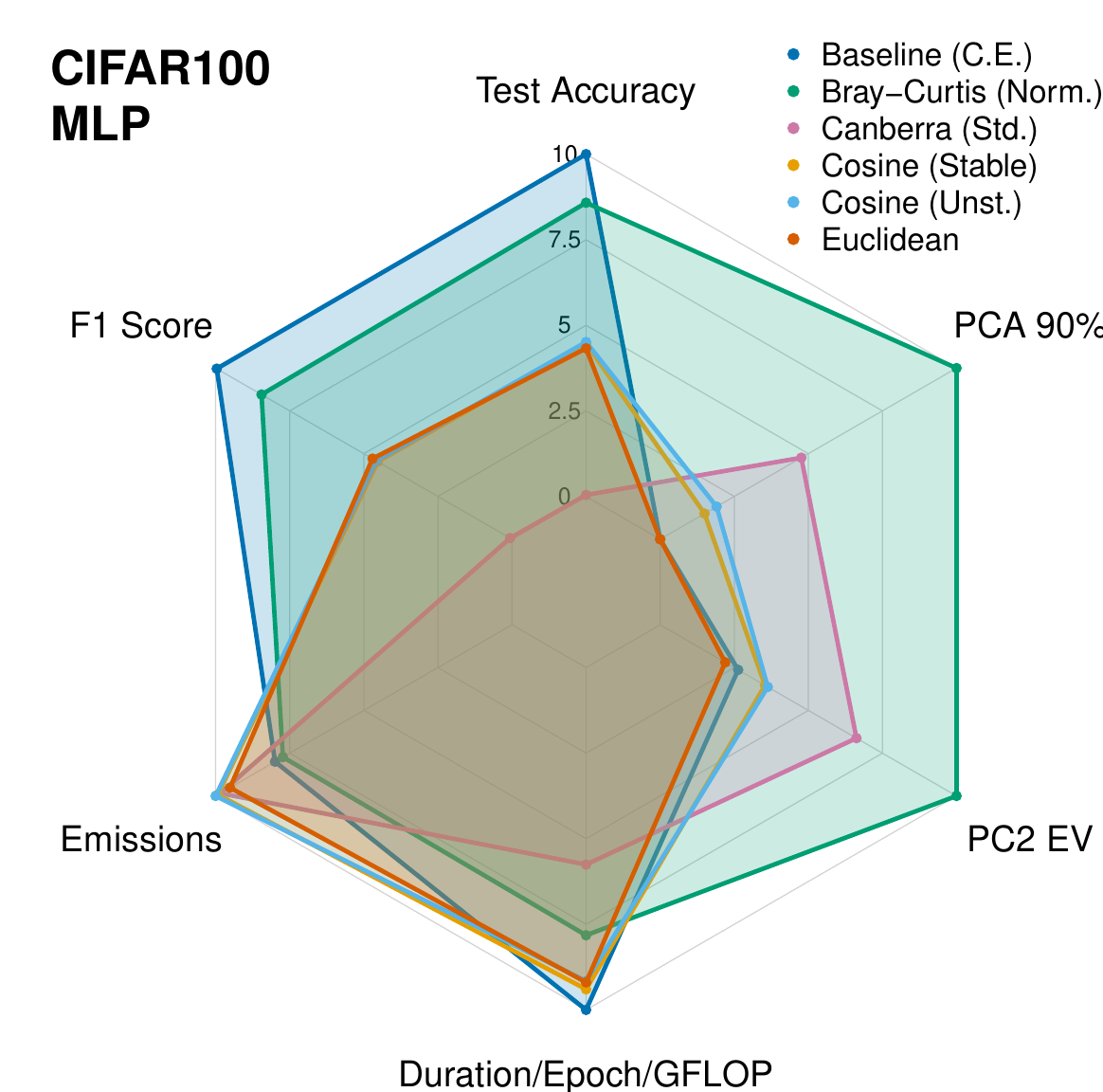}
    % CNN 
    \includegraphics[page=13,width=0.28\textwidth,trim=45 80 80 15,clip]{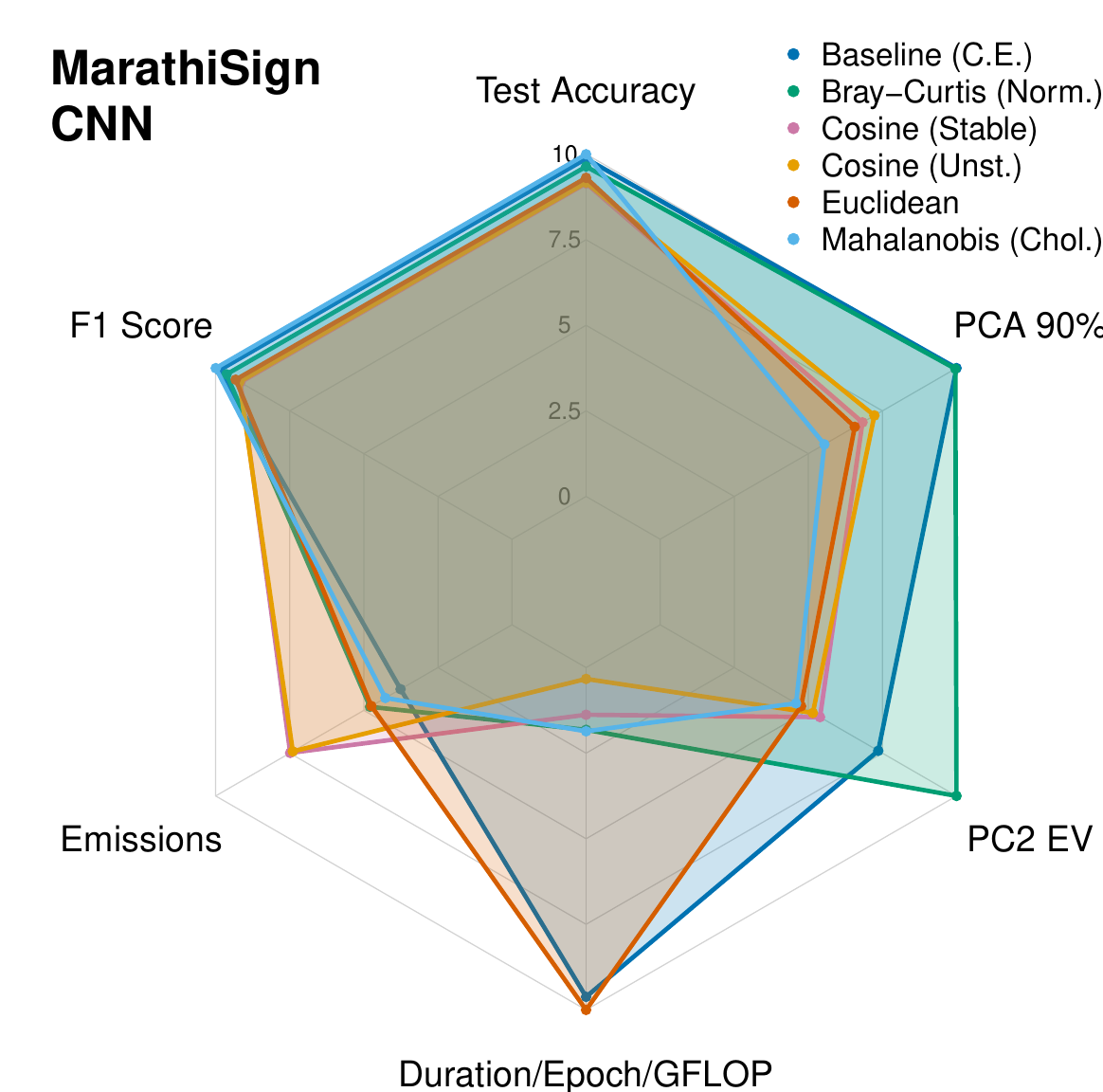}
    \includegraphics[page=17,width=0.28\textwidth,trim=45 80 80 15,clip]{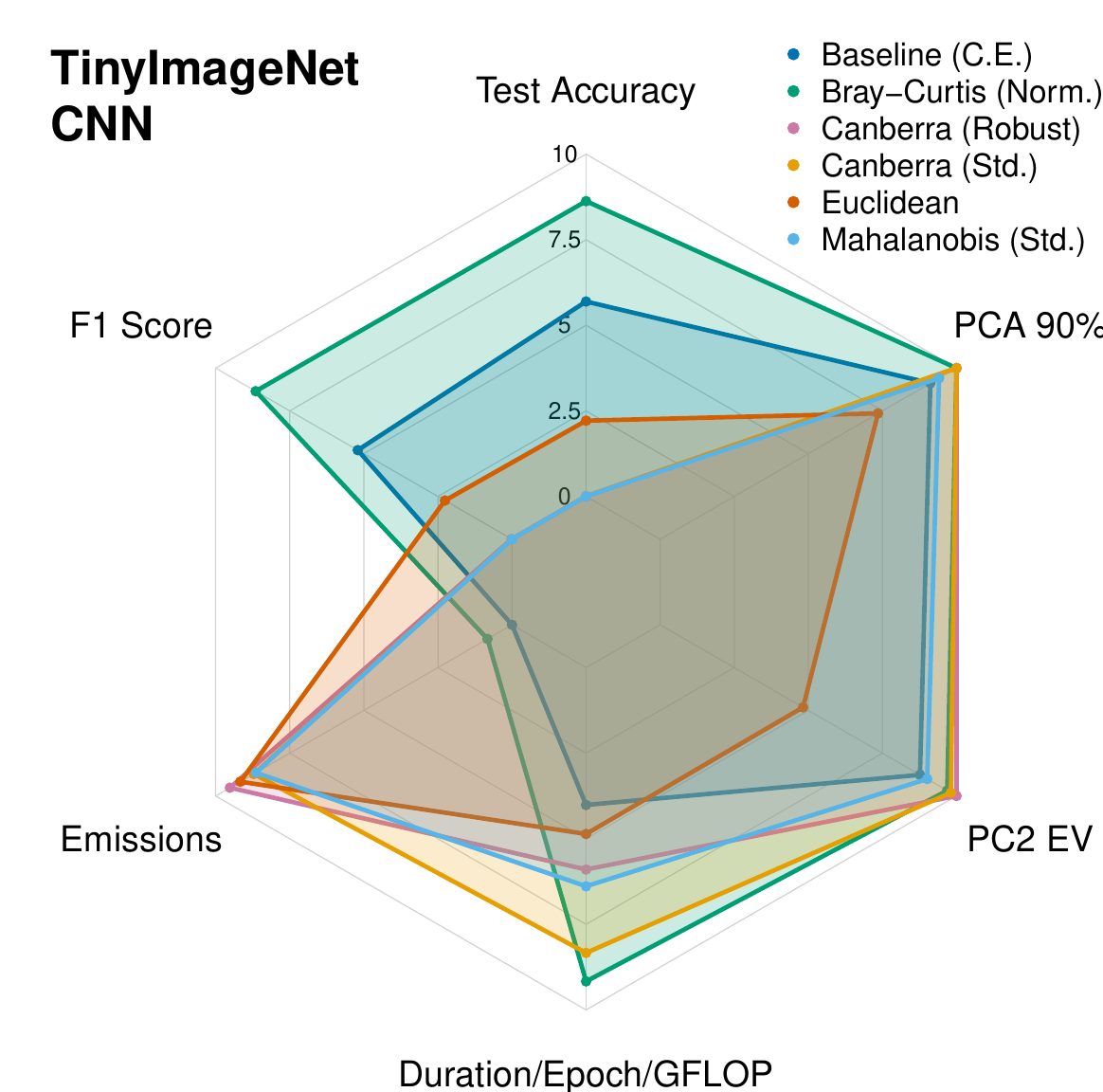}
    \includegraphics[page=5,width=0.28\textwidth,trim=45 80 80 15,clip]{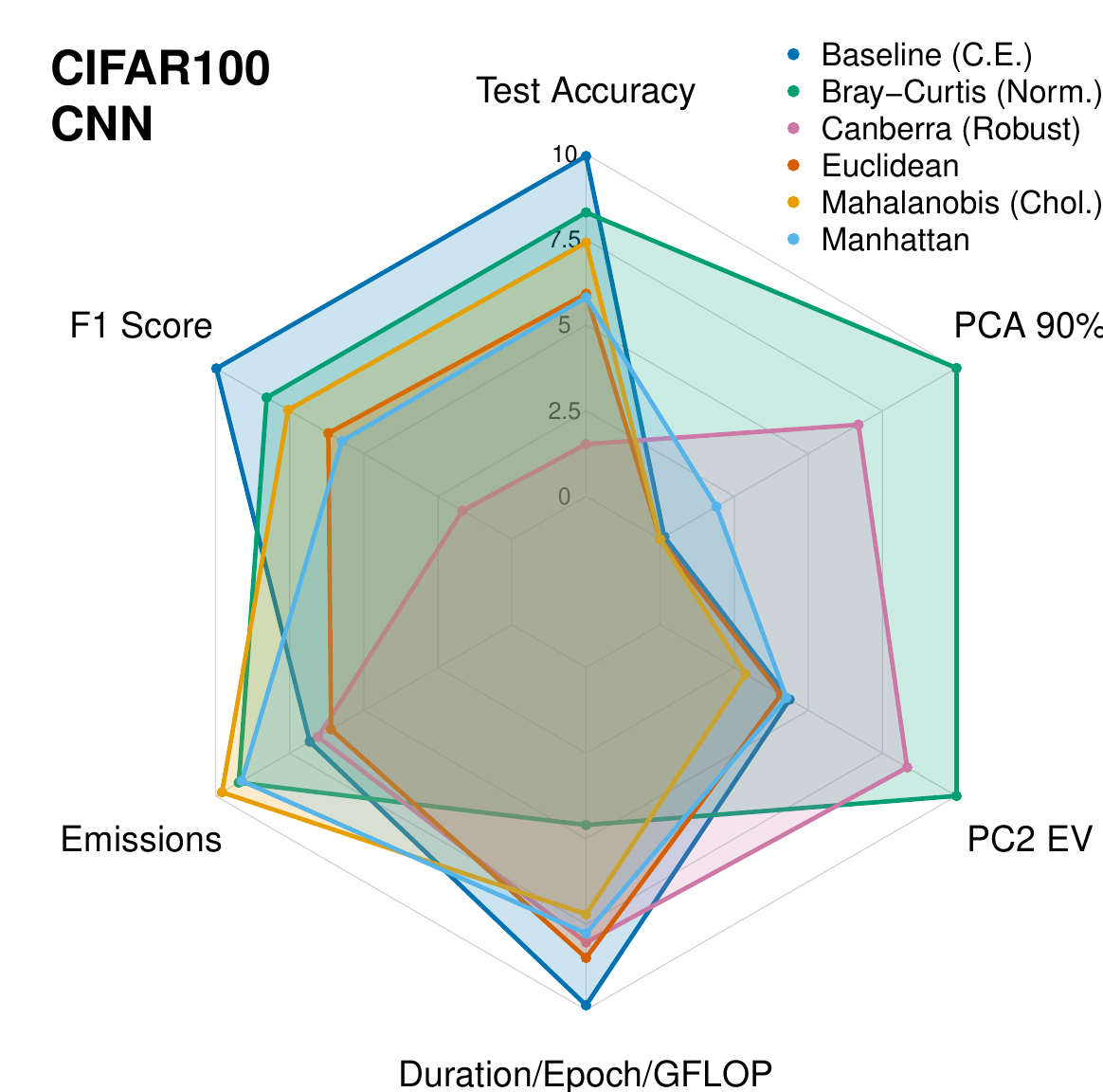}\\
    % RESNET
    \includegraphics[page=16,width=0.28\textwidth,trim=45 80 80 15,clip]{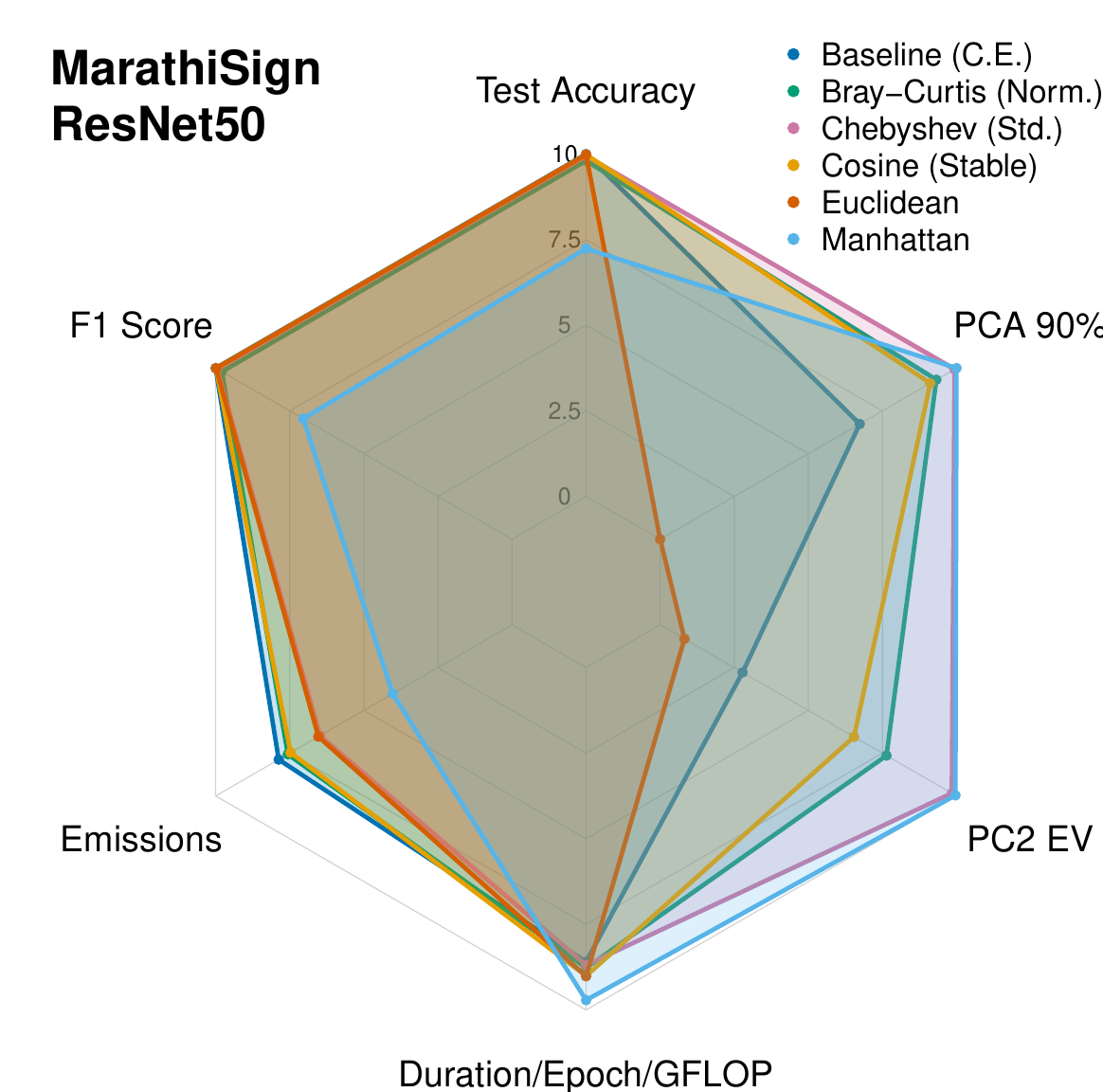} 
    \includegraphics[page=20,width=0.28\textwidth,trim=45 80 80 15,clip]{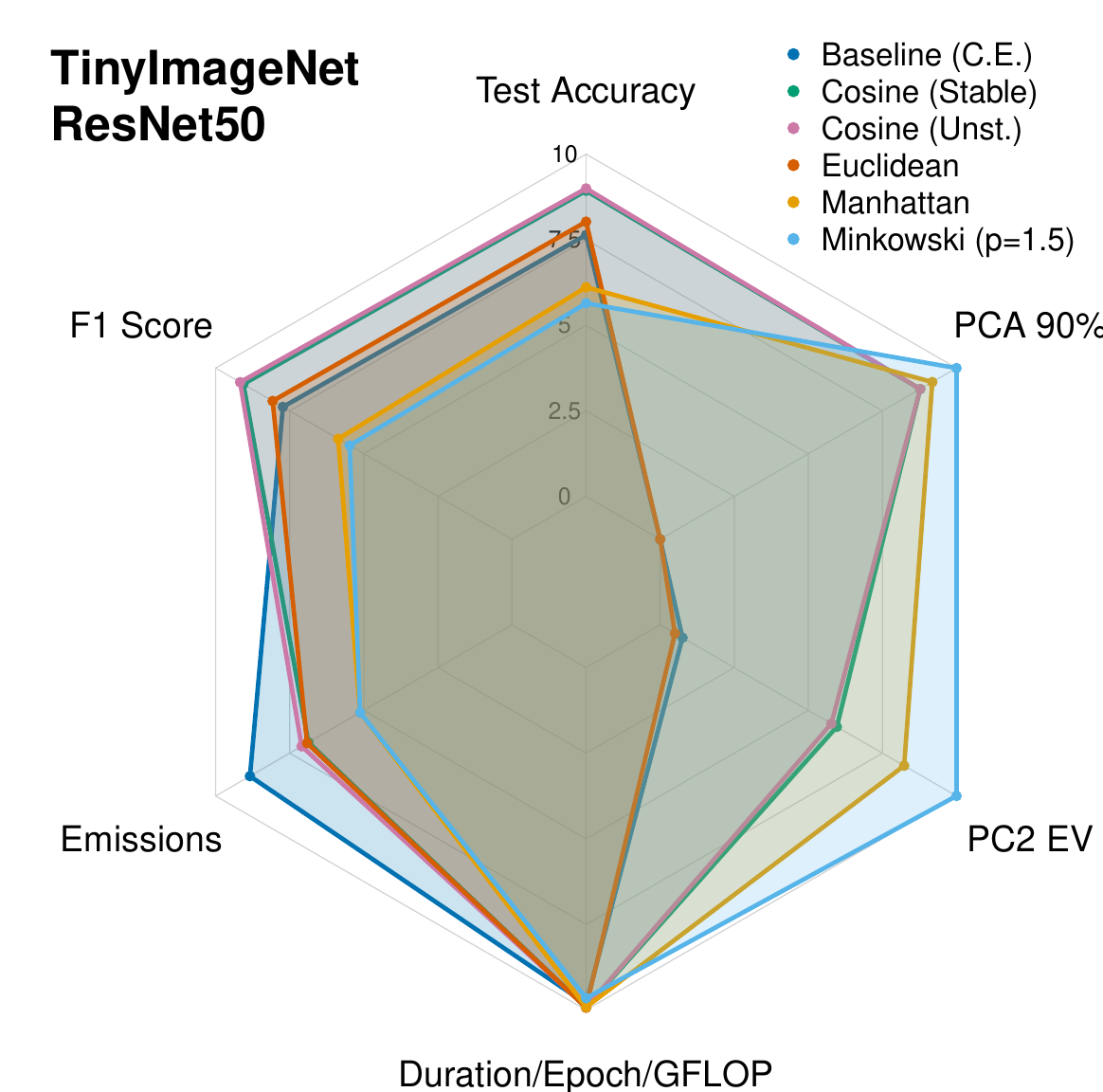}
    \includegraphics[page=8,width=0.28\textwidth,trim=45 80 80 15,clip]{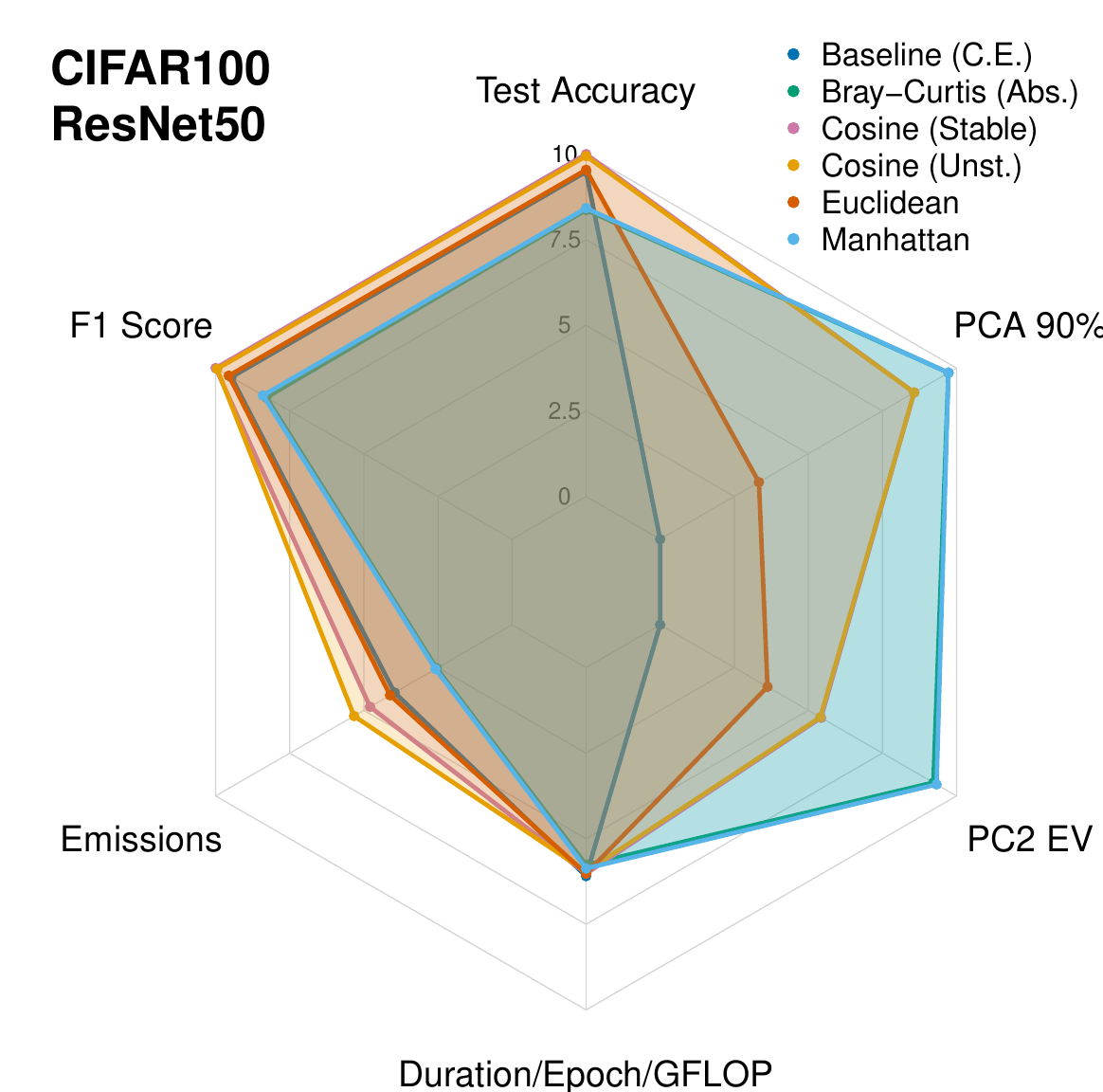}\\
    % PVT
	    \includegraphics[page=15,width=0.28\textwidth,trim=45 80 80 15,clip]{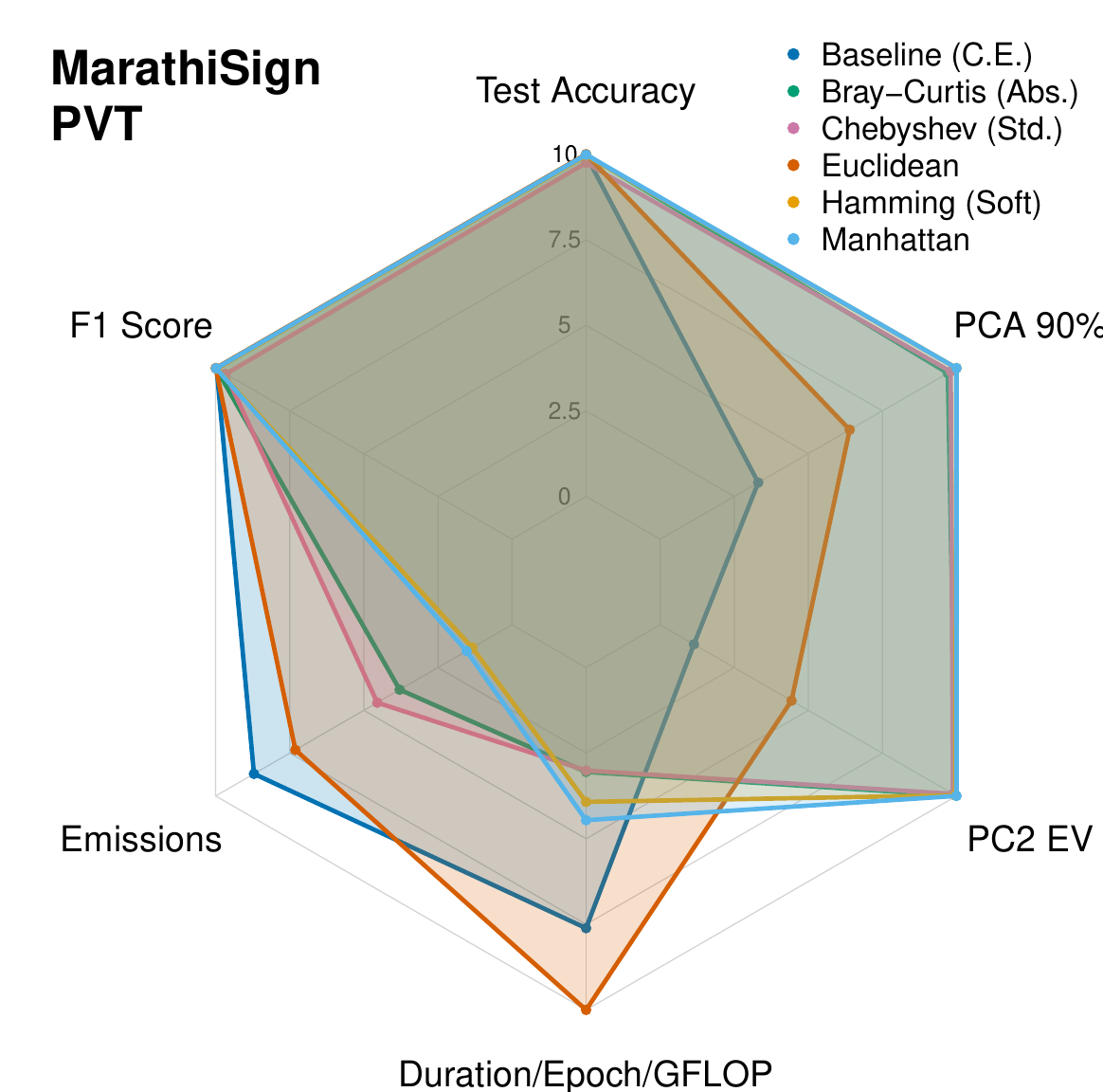} 
	    \includegraphics[page=19,width=0.28\textwidth,trim=45 80 80 15,clip]{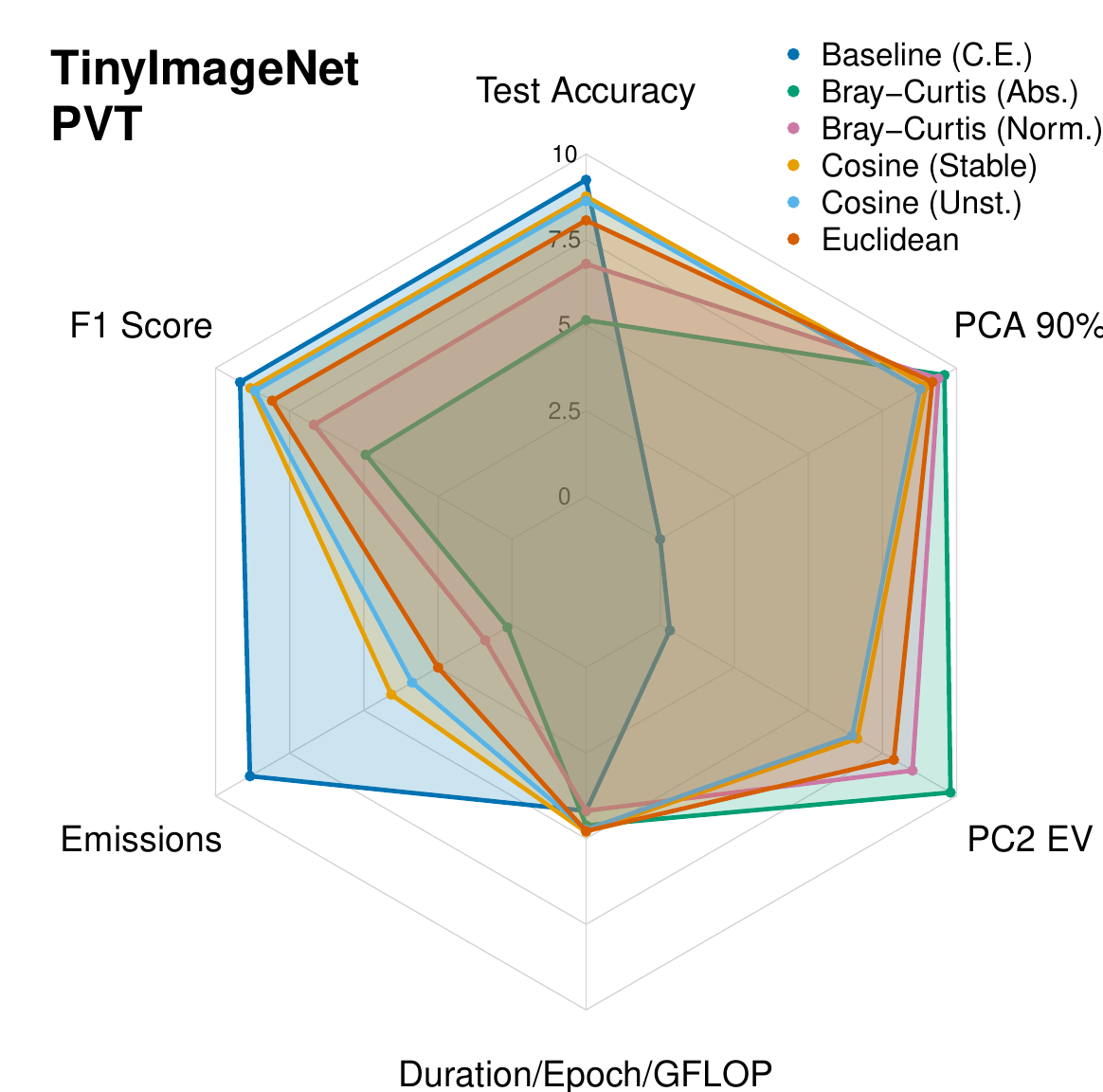}
	    \includegraphics[page=7,width=0.28\textwidth,trim=45 80 80 15,clip]{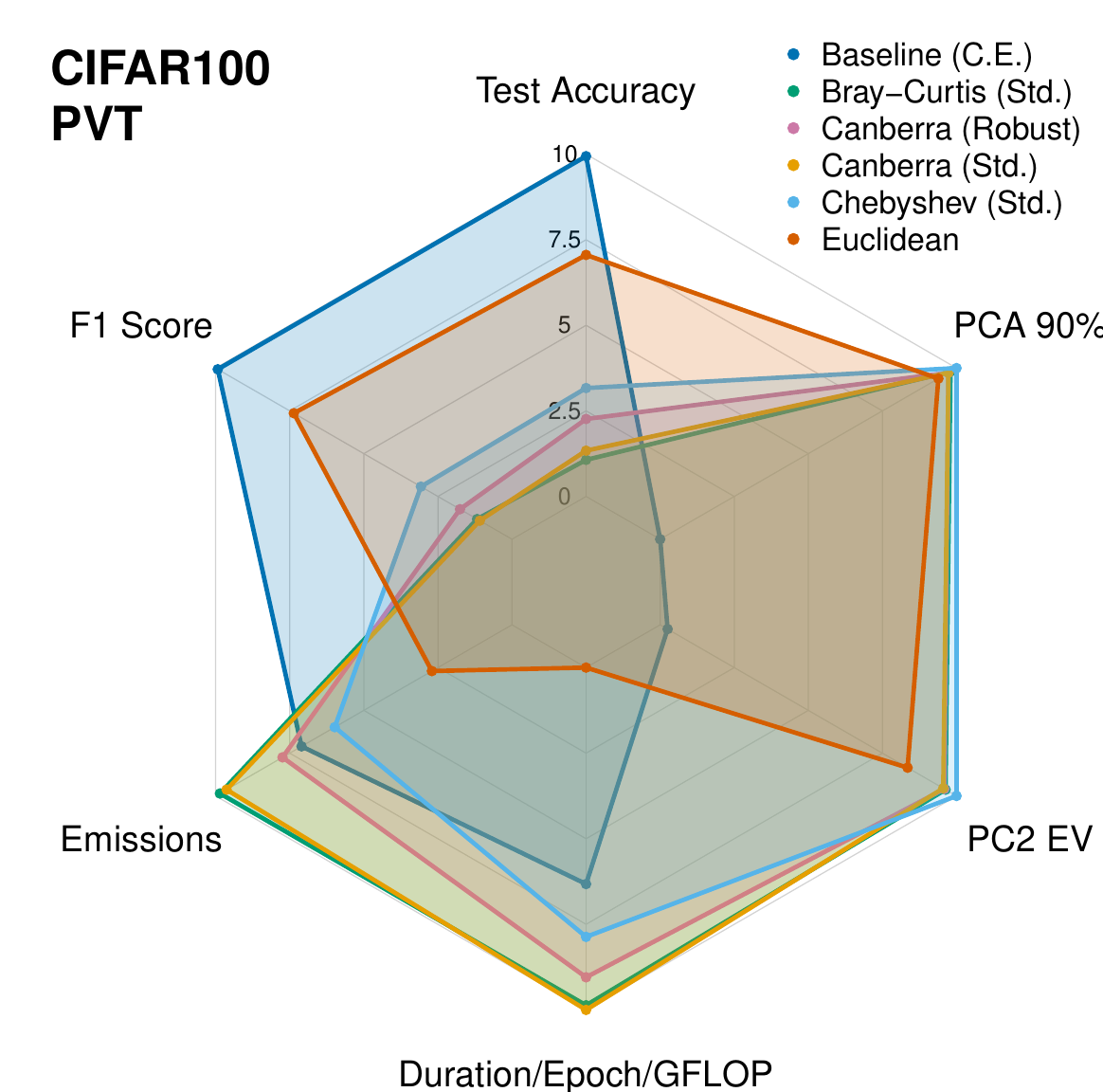}\\     %\includegraphics[page=17,width=0.3\textwidth,trim=45 80 80 15,clip]{figures/Image_Classification/Spider/euclidean_baseline_plus_top4_tinyimagenet_no_new_losses.pdf}
	        \caption{Vision: Radar plots: 1) \textit{Model Performance} (F1, Accuracy); 2) \textit{Interpretability} (PC2 EV, PCA 90\%), and 3) \textit{Sustainability} (Duration/Epoch/GFLOPs, Emissions). Plots feature Baseline (Cross-Entropy), Euclidean harmonic, and the four top-performing non-Euclidean harmonic losses.}
	        \Description[Vision radar plots comparing harmonic losses]{A grid of radar (spider) charts comparing distance-based harmonic losses on vision tasks. Each radar chart has axes grouped into three perspectives: model performance (F1 and accuracy), interpretability (explained variance of the first two principal components and the number of PCs needed to reach 90\% explained variance), and sustainability (training duration per epoch, GFLOPs where reported, and estimated CO2 emissions). Within each radar, multiple colored polygons/lines represent the cross-entropy baseline, a Euclidean harmonic loss, and several top-performing non-Euclidean harmonic losses, enabling visual comparison across metrics.}
	        \label{fig:acc:spider:vision}
	    \end{centering}
\end{figure*}

% NEW RADAR TINYIMAGENET

% %\begin{wrapfigure}{h!}{0.32\textwidth}
% \begin{figure}[h!]
% \begin{centering}
% \includegraphics[page=18,width=0.35\textwidth,trim=45 80 80 15,clip]{figures/Image_Classification/Spider/euclidean_baseline_plus_top4_tiny_imagenet_no_new_losses_11-26.pdf}
% \includegraphics[page=17,width=0.35\textwidth,trim=45 80 80 15,clip]{figures/Image_Classification/Spider/euclidean_baseline_plus_top4_tiny_imagenet_no_new_losses_11-26.pdf}
%     \caption{Vision: Radar plot - TinyImageNet : 1) \textit{Model Performance} (F1, Accuracy); 2) \textit{Interpretability} (PC2 EV, PCA 90\%), and 3) \textit{Sustainability} (Duration/Epoch/GFLOPs, Emissions).} %The plot features Baseline (Cross-Entropy), Euclidean harmonic, and the four top-performing non-Euclidean harmonic losses.}
%     \label{fig:acc:spider:tiny:resnet}
% \end{centering}
% \end{figure}
% %\end{wrapfigure}

\begin{figure*}
\includegraphics[page=1,width=0.3\textwidth,trim=65 75 40 15,clip]{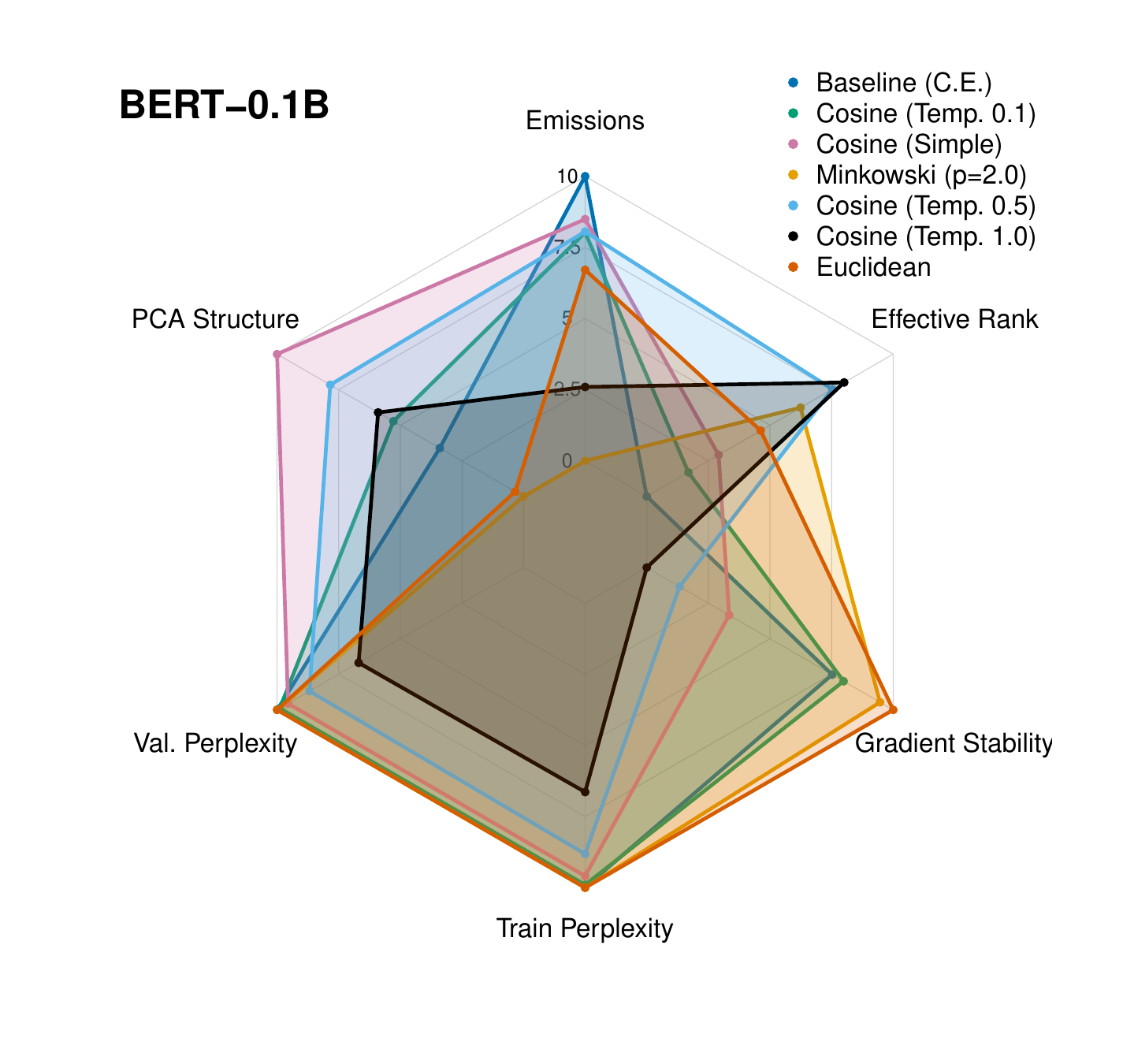}
\includegraphics[page=2,width=0.3\textwidth,trim=65 75 40 15,clip]{figures/Color-Blind_Accessible_Images/color_blind_llm_spiders.pdf}
\includegraphics[page=4,width=0.3\textwidth,trim=65 75 40 15,clip]{figures/Color-Blind_Accessible_Images/color_blind_llm_spiders.pdf}
	        \caption{Language: Radar plots: 1) \textit{Model Performance} (Perplexity, Effective Rank, Gradient Stability); 2) \textit{Interpretability} (PCA5 EV), and 3) \textit{Sustainability} (Emissions). Plots feature Baseline (CE), Euclidean harmonic, and the top-performing non-Euclidean harmonic losses.}
	        \Description[Language-model radar plots comparing harmonic losses]{A set of radar (spider) charts comparing distance-based harmonic losses for language models. Each chart summarizes model performance (perplexity, effective rank, and gradient stability), interpretability (explained variance captured by the top five principal components), and sustainability (estimated CO2 emissions). Curves/polygons within each radar correspond to the cross-entropy baseline, a Euclidean harmonic loss, and the strongest non-Euclidean harmonic loss variants, allowing side-by-side comparison of trade-offs across metrics.}
	        \label{fig:acc:spider:language}
\end{figure*}

%\clearpage\newpage

% In your preamble:
% \usepackage{booktabs}
% \usepackage{siunitx} % optional

\section{Related Work}
\textbf{Loss functions for classification.} The majority of classification models are trained with cross-entropy loss due to its empirical effectiveness and probabilistic interpretation.
However, it only cares about separating classes, not about how the representations are separated, often yielding features that are separable but not necessarily interpretable.
Over the years, alternative loss functions have been proposed to address these limitations.
Metric learning losses, such as contrastive and triplet loss, train models to preserve distances between examples, but require sampling strategies that add training complexity.
\citet{boudiaf2020mutualinfo} propose a unifying mutual information framework connecting cross-entropy to standard pairwise losses, showing that cross-entropy implicitly bounds pairwise distance objectives.
These insights motivate a deeper theoretical understanding of distance-based training.
Regularization-based approaches such as \emph{center loss} \citep{wen2016center} explicitly encourage compact intra-class clusters and large inter-class separation. These works foreshadow the idea that directly leveraging distances to class prototypes can improve representation quality. Angular margin losses, such as AMC-Loss in \citet{choi2020amcloss}, introduce geometric constraints on angular separations to enhance interpretability via hyperspherical metrics.
Orthogonal Projection Loss (OPL) introduced by \citet{ranasinghe2021opl} encourages inter-class orthogonality and intra-class cohesion without sampling overhead. Several studies have assessed how loss functions affect neural network performance.
\citet{miller2021cac} introduce \emph{Class Anchor Clustering} (CAC) loss that encourages tight class clusters centered on anchored prototypes, enhancing distance-based open-set classification performance. This approach aligns with the prototype-centered philosophy underlying harmonic loss. \citet{cho2019performance} analyzed how eight loss functions impact neural network accuracy and convergence speed, finding that additive-margin softmax loss resulted in the fastest convergence and highest performance on multiple datasets. \citet{janocha2017loss} assessed 12 loss functions for classification, finding that choice of loss function impacted learning speed and testing accuracy. \citet{gonzalez2020improved} used genetic programming to develop Baikal loss, which not only led to networks achieving higher accuracy than networks trained with cross-entropy loss, but also faster training and higher performance in low-data settings. These studies demonstrate a large focus on the impact of loss function on neural networks performance. Our work builds on the discussion of the importance of loss function choice by drilling deeper on harmonic loss, examining how distance metric choice impacts the effectiveness of neural networks. Our focus is not on comparing harmonic loss with other loss functions, which was done by \citet{baek2025harmonic}, but rather to shed light on the performance of a generalized harmonic loss. 

\textbf{Efficiency and Green AI.} Green AI is an emerging initiative that calls for efficiency and energy usage to be treated as first-class evaluation criteria \citep{schwartz2019greenai}. Many works on green AI focus on model compression \citep{paula2025comparative, rafat2023mitigating}, comparing multiple models \citep{verma2024performance} or fine-tuning strategies \citep{wang2023energy}, or hyperparameter optimization for carbon emission reduction \cite{wang2025carbon}. While prior works on new loss functions rarely report sustainability metrics, we incorporate carbon footprint analysis into our evaluation due to claims that models trained with harmonic loss are more data efficient and have less grokking \citep{baek2025harmonic}. 

\textbf{Interpretability in neural networks.}
Neural networks are complex and not inherently interpretable, but substantial effort has been devoted to improving interpretability \citep{zhang2021survey}.
The push for \emph{interpretable by design} models argues that transparency should be built into model architectures and losses rather than added post-hoc~\citep{rudin2019stop}.
Harmonic loss aligns with this vision by structurally linking model weights to class prototypes.
The study by \citep{bereska2024mechanistic} discusses how internal model components reveal human-understandable circuits and features in LLMs.
Techniques such as activation patching, sparse autoencoders, transcoders, and crosscoders enable structural interpretations of model behavior.
Parallel to our interpretability focus, \citet{wen2025interpgn} introduced a framework combining interpretable models with deep networks for time-series tasks, preserving understandable reasoning where possible; though not loss-centric, it reflects the growing emphasis on transparency in deep learning research.
Some work has focused on using loss functions specifically to improve model interpretability.
\citet{liu2022sparsity} combine sparse coding constraints with cross-entropy to produce concise, interpretable word-level attributions.
\citet{dong2017improving} introduced \emph{interpretative loss} to improve the interpretability of learned features during video captioning tasks. Within classification tasks.
\citet{zhang2018interpretable} designed a loss function to improve CNN filter interpretability.
Methods such as the one proposed by \citet{hagos2023xbl} augment standard losses with distance-based penalties that align model attributions with user-provided annotations, strengthening interpretability.
\textbf{Distance metrics in learning algorithms.}
Beyond supervised classification, the choice of distance measure is known to be crucial.
\citet{coil2025distance} compared twelve distance metrics in anomaly detection for concept drift; their results highlighted that performance depends heavily on the chosen metric and that efficient alternatives can sometimes match the performance of more costly distances.
A variety of other works have shown the importance of distance metric choice.
\citet{amaya2024distance} used a kernel for SVMs that supported a variety of kernels, finding that the choice of distance metric impacted performance.
\citet{kalra2022effect} and \citet{hu2016distance} both found that distance metric choice impacted performance of \textit{k}-nearest neighbors algorithms on a variety of datasets.
These result highlights the importance of systematically exploring metrics in different contexts.
To our knowledge, our paper is the first to bring this perspective into loss functions.

{
\section{Conclusion}
\label{sec:conclusion}
This work examined \emph{distance–based harmonic losses} as drop–in replacements for cross–entropy across image classification
(MNIST, CIFAR-10, CIFAR-100, Marathi Sign Language, TinyImageNet) with four vision backbones (MLP, CNN, ResNet50, PVT) and
LLM pretraining (GPT, BERT, Qwen, GPT-2B), leveraging a broad family of distances (cosine, Euclidean, Bray–Curtis,
Mahalanobis, Minkowski, Chebyshev, Canberra, \emph{etc.}) and comparing them against strong modern baselines
(Focal Loss, Label Smoothing, Center Loss, Confidence Penalty, ArcFace).

What we learned:
i) \textbf{Geometry matters for optimization.}
Across vision and language tasks, Cosine consistently delivers smoother training dynamics, higher or competitive final performance, and reduced grokking–like behavior in toy modulo–addition experiments.
Euclidean remains a solid reference; Bray–Curtis is often competitive but architecture–sensitive; Mahalanobis exhibits
the largest variance---sometimes yielding very sharp, well–separated clusters, but with less stable plateaus in more difficult scenarios (larger datasets and model backbones).
Loss–convergence curves for both vision and LLMs show that all investigated distances (including cosine and Mahalanobis) exhibit smooth optimization without problematic instabilities.
ii) \textbf{Sustainability depends jointly on distance and architecture.}
On vision tasks, several non–Euclidean harmonic losses are carbon–negative per step relative to cross–entropy for CNN/ResNet50 (largest gains occur on deeper CNNs), mixed on MLP, and closer to neutral on PVT and TinyImageNet, where
backbone FLOPs dominate.
For LLM pretraining, the classifier head is lightweight, so differences arise primarily via convergence: the cross–entropy baseline typically incurs the largest cumulative emissions, while cosine and Minkowski heads are neutral–to–favorable.
Our FLOPs–normalized analysis and extended emissions study show that the best non–Euclidean harmonic losses lie on or near the sustainability–accuracy Pareto frontier.
iii) \textbf{Interpretability can be quantified.}
PCA–based probes (variance concentration and PCA@90\%) and geometric visualizations of prototype neighborhoods provide reproducible evidence that distance–tailored heads yield more structured representations.
Bray–Curtis and Chebyshev consistently increase variance concentration and reduce intrinsic dimensionality, while Mahalanobis emphasizes representation clarity at a higher computational cost.
These trends hold for image features and for token representations in LLMs (last–token and masked–token states) and
are supported by statistical tests (Wilcoxon) and confidence intervals across seeds.
\textbf{Language.}
Cosine–based harmonic losses markedly improve gradient/learning stability, perplexity, and representation structure for GPT, BERT, Qwen, and GPT-2B, while keeping emissions on par with or below
cross–entropy and Euclidean harmonic loss. 
Mahalanobis remains less attractive for large–scale pretraining due to covariance overheads and sensitivity to ill–conditioned statistics.
\textbf{Vision.}
For accuracy–focused workloads across MNIST, CIFAR, Marathi Sign, and TinyImageNet, cosine (stable) is the preferred all–round choice; Bray–Curtis is a strong secondary option; Mahalanobis should be used when its inductive bias (sharp, anisotropic clusters) is explicitly desired.
For sustainability on CNN/ResNet50, several non–Euclidean distances reduce per–step CO$_2$; on PVT and LLMs, the lightest geometries (cosine/Euclidean) should be favored, or cross–entropy retained unless a distance–based head reduces steps to target enough to offset higher per–step cost.

In summary, our \emph{framework} including a plug–and–play harmonic distance layer, a catalogue of non-Euclidean distances, and a three–axis evaluation protocol (\emph{performance, interpretability, sustainability}) %with concrete metrics, visualizations, and statistical tests.
%
%This framework 
can be effectively exploited in future work: practitioners can choose distances according to their priorities, and researchers can extend our study to new geometries, learning settings, and domain–specific constraints.
In this sense, distance–based harmonic losses provide a principled, empirically validated toolbox for rethinking the geometry of classification layers in both vision and language models.

\newpage
%
\begin{comment}
%
%\textbf{Outlook.}
%

\textbf{Reproducibility Statement.}
We took several steps to facilitate exact and statistical reproducibility. The main paper specifies the learning objectives, training protocol, model families, and evaluation metrics used in all studies. The \emph{Appendix} contains: i) complete hyperparameter and backbone-specific settings; %, with early-stopping and optimizer details; {\color{red}
ii) dataset descriptions and end-to-end preprocessing pipelines (including splits and any filtering); iii) detailed experimental studies and analyses; %(multiple random seeds, alternative batch sizes, and learning-rate sweeps); and iv) the methodology for energy/emissions measurement.}; 
iv) technical details with code snippets to integrate our non-Euclidean harmonic losses in conventional deep learning pipelines. 
Our code repository provides: ready-to-run scripts for data acquisition and preprocessing; configuration files for every experiment; training/evaluation entry points; instructions for reproducing results.  %logging utilities that reproduce all tables and figures; and experiment manifests to rerun all results across seeds. 
%The repository also includes instructions for reproducing results with a single command. % per figure/table, and we release outputs (CSV logs and checkpoints) to support verification without retraining. 
Together, these materials are intended to enable independent researchers to audit, rerun, and extend our findings with minimal effort.
\end{comment}
}

\bibliographystyle{ACM-Reference-Format}
\bibliography{refs}

\appendix
\section*{Appendix}

% =========================
% Theory for Distance-based DistLayers / Harmonic-style losses
% =========================
\section{Theoretical Properties of Distance-based Probabilistic Layers}
\label{sec:theory}

\textbf{Setup.}
Let $\{(x_i,y_i)\}_{i=1}^n$ be the training set with $y_i\in\{1, ..., K\}$.
Each class has a prototype $w_k\!\in\!\mathbb{R}^d$ and a nonnegative distance $d(x,w)\!\ge\!0$.
Given a decreasing link $\kappa:\mathbb{R}_+\!\to\!\mathbb{R}_+$ we define
\[
p_k(x_i)\;=\;\frac{\kappa\!\big(d(x_i,w_k)\big)}{\sum_{j=1}^K\kappa\!\big(d(x,w_j)\big)}\!,
\qquad
\mathcal{L}(\{w_k\})\;=\;-\sum_{i=1}^n\log p_{y_i}(x_i).
\]
% We will compare two links:
\emph{harmonic} $\kappa(r)=r^{-\omega}$ with $\omega>0$; while
% (ii) \emph{Gibbs/softmax-distance} $\kappa(r)=\exp(-r/\tau)$ with temperature $\tau>0$.
distances include Euclidean/Mahalanobis, %$\|A(x-w)\|_2$, 
$L_p$, Bregman divergences, and cosine/angle on the sphere.

\subsection{Scale invariance and finite minimizers}
We begin by generalizing the finite-minimizer result of the harmonic loss (cf. Thm.~1, Sec.~G in \citet{baek2025harmonic}). %:contentReference[oaicite:4]{index=4}

\begin{definition}[Metric separability and homogeneity]
A dataset is \emph{metric-separable} if for each $i$ there exists $\{w_k\}$ s.t.
$d(x_i,w_{y_i})<\min_{j\neq y_i}d(x_i,w_j)$.
A distance $d$ is \emph{1-homogeneous} if $d(c x,c w)=|c|\,d(x,w)$ for all $c>0$.
\end{definition}

\begin{theorem}[Finite minimizer and scale invariance for harmonic link]
\label{thm:finite}
Assume $d$ is $1$-homogeneous and the training set is metric-separable.
For $\kappa(r)=r^{-\omega}$, the empirical loss $\mathcal{L}$ is invariant to the joint rescaling
$(x,w)\mapsto(c x, c w)$ and attains a global minimum at \emph{finite} $\{w_k\}$.
In particular, %once all $x_i$ are strictly nearest to $w_{y_i}$, 
increasing $\|w_k\|$ further does not reduce $\mathcal{L}$.
\end{theorem}

\begin{proof}[Proof]
Following the proof of Sec.~G Thm.~1 in \citet{baek2025harmonic}, the probabilities remain unchanged under uniform scaling for any $1$-homogeneous distance $d$. For the probabilities, if we replace $x_i$ by $cx_i$ and $w_j$ with $cw_j$, then $d(c x_i,c w_j)=c\,d(x_i,w_j)$, so the scaling factors cancel when using a harmonic link $\kappa$. Therefore, once the correct classification is achieved, no further reduction in loss is obtained by increasing $\|w_k\|$, and the loss achieves a global minimum at a finite $\{w_k\}$.%:contentReference[oaicite:5]{index=5}
\end{proof}

\textbf{Compactness Argument for Finite Minimizers:}
Since the loss $\mathcal{L}(W)$ is scale-invariant, we can restrict the optimization to the unit hypersphere $\mathcal{S} = \{W : \|W\|_\star = 1\}$. Given that the training set is metric-separable, the loss function $\mathcal{L}$ is continuous and bounded on the compact set $\mathcal{S}$. By the \textit{Extreme Value Theorem}, $\mathcal{L}$ must attain a global minimum $W^* \in \mathcal{S}$. Any scaled version $cW^*$ ($c > 0$) is also a global minimizer with a finite norm, thus ensuring that the "weight explosion" seen in Softmax training is theoretically impossible here.

% \textbf{Contrast (Gibbs link).}
% For $\kappa(r)=\exp(-r/\tau)$, the rescaling $(x,w)\mapsto c(x,w)$ is equivalent to changing $\tau\mapsto\tau/c$,
% so the objective is \emph{not} scale-invariant; weight norms can “run” unless temperature is controlled.

\subsection{Margin-style generalization (PAC-Bayes view)}
Sec.~G gives a PAC-Bayes margin bound that is finite because the harmonic solution has finite norm (Thm.~2) in \citet{baek2025harmonic}.% . :contentReference[oaicite:6]{index=6}

\begin{definition}[Distance margin]
Given prototypes $W=\{w_k\}$, define
\(
\gamma(W)\;=\;\min_i\big[\min_{j\neq y_i}d(x_i,w_j)-d(x_i,w_{y_i})\big].
\)
That is, $\gamma(W)$ measures the gap between the nearest \emph{incorrect} prototype and the \emph{correct} prototype; $\gamma(W) > 0$ implies all training points are strictly closer to their true class prototype.
\end{definition}

\begin{theorem}[Generalization with metric margin]
\label{thm:margin}
Assume all $x_i$ lie in a ball of radius $R$ (in the native norm of $d$ or its inducing space). 
Let $\|W\|_\star$ denote a capacity measure compatible with $d$. %(e.g.\ operator norm of $A$ for Mahalanobis).
With probability at least $1-\delta$, the generalization error of the classifier satisfies
\[
\Pr_{(x,y)}\big[h_W(x)\neq y\big]
\;=\;
\mathcal{O}\!\left(\frac{R\,\|W\|_\star}{\gamma(W)\sqrt{n}}+\sqrt{\frac{\log(1/\delta)}{n}}\right),
\]
where $h_W(x)=\arg\max_k p_k(x)$ denotes the predicted class and $n$ is the number of training samples.
For the harmonic link, $\|W\|_\star$ is finite by Thm.~\ref{thm:finite}, yielding a finite bound (cf.\ Sec.~G Thm.~2) in \citet{baek2025harmonic}. % :contentReference[oaicite:7]{index=7}
\end{theorem}

\begin{proof}[Proof]
Mirroring the proof for Sec.~G Thm.~2 in \citet{baek2025harmonic}, applying the standard PAC-Bayes margin bounds, one obtains that with at least probability $1-\delta$, 
\[
\Pr_{(x,y)}\big[h_W(x)\neq y\big]
\;=\;
\mathcal{O}\!\left(\frac{R\,\|W\|_\star}{\gamma(W)\sqrt{n}}+\sqrt{\frac{\log(1/\delta)}{n}}\right).
\]

Since $\|W\|_\star$ is finite by \cref{thm:finite}, the bound is finite. 
\end{proof}

\textbf{Intuition:}
To see why the harmonic link achieves a finite minimizer, we examine the behavior of the probability $p_k$ under a global scaling factor $c > 0$. For a $1$-homogeneous distance $d$ (where $d(cx, cw) = c \cdot d(x, w)$) and a harmonic link $\kappa(r) = r^{-\omega}$:

\begin{equation}
p_k(cx_i; cW) = \frac{\kappa\big(d(cx_i, cw_k)\big)}{\sum_{j=1}^K \kappa\big(d(cx_i, cw_j)\big)} = \frac{(c \cdot d(x_i, w_k))^{-\omega}}{\sum_{j=1}^K (c \cdot d(x_i, w_j))^{-\omega}}
\end{equation}

Factoring out the constant $c^{-\omega}$:

\begin{equation}
p_k(cx_i; cW) = \frac{c^{-\omega} \cdot d(x_i, w_k)^{-\omega}}{c^{-\omega} \cdot \sum_{j=1}^K d(x_i, w_j)^{-\omega}} = p_k(x_i; W)
\end{equation}

Since $p_k$ is invariant to the magnitude of the parameters $W$, the loss $\mathcal{L}(W)$ is constant along any radial ray in the parameter space once the prototypes are correctly oriented. Unlike the Gibbs link (Softmax), which requires $\|W\| \to \infty$ to drive the loss to zero, the harmonic gradient $\nabla_W \mathcal{L}$ vanishes at a finite manifold.

\textbf{Theoretical Impact on Generalization:}
In standard PAC-Bayes frameworks, the complexity term is proportional to the weight norm $\|W\|_\star$. In Softmax-based models, if the weights "run" to infinity to minimize the cross-entropy loss, the generalization bound:
\begin{equation}
\epsilon_{gen} \leq \mathcal{O}\left(\frac{R \cdot \|W\|_\star}{\gamma \sqrt{n}}\right)
\end{equation}
becomes vacuous as $\|W\|_\star \to \infty$. Our Theorem \ref{thm:finite} guarantees that $\|W\|_\star$ remains bounded at the optimum, ensuring that the generalization bound remains tight and informative. This provides a formal guarantee that the harmonic geometry acts as an inherent structural regularizer.

\color{black}

\section{Integration into Deep Learning Pipelines}

The \texttt{DistLayer} abstraction highlights that distance-based harmonic loss functions are highly modular and can be seamlessly integrated into existing deep learning pipelines. 
The \texttt{forward} method requires only three operations: i) computing pairwise distances between sample embeddings and class prototype weights, ii) clamping values for numerical stability, and iii) applying a softmin via \texttt{log\_softmax} to obtain normalized class probabilities. 
This makes the substitution of Euclidean distance with alternative metrics essentially a one-line change in the distance registry, with no modifications required in the broader training loop.

Several design choices make the implementation robust. 
First, all distance functions are implemented in a vectorized form, ensuring GPU efficiency and avoiding explicit loops. 
Second, numerical safeguards (e.g., \(\varepsilon\)-offsets, clamping before roots and divisions, regularization of covariance matrices) prevent instability across diverse datasets and architectures. 
Third, the registry-based design allows new distance functions to be added without disrupting the existing workflow, reinforcing the flexibility of harmonic loss as a general framework.

From a methodological perspective, this implementation highlights one of the key contributions of this work: the ease of replacing cross-entropy with distance-based harmonic loss. 
Unlike cross-entropy, which relies on unbounded logit growth, the harmonic formulation treats classification as a problem of minimizing distances to interpretable class prototypes. 
The plug-and-play nature of the \texttt{DistLayer} demonstrates that alternative geometries (e.g., cosine, Mahalanobis, Bray--Curtis) can be explored at negligible engineering cost, paving the way for systematic evaluation of accuracy, sustainability, and interpretability across diverse tasks.

\section{Model Architectures}
\subsection{Vision}
\label{appendix:vision-arch} 
We detail the architectures of the vision models used in our experiments – including a simple MLP, a small CNN, ResNet-50, and PVTv2-B0 – specifying their layers and neuron counts for reproducibility. All models were implemented in PyTorch, and for distance-based variants, %(denoted \texttt{\_DIST} in code)
the final fully-connected layer is replaced by a specialized distance layer as noted below.

\textbf{MLP}: \textbf{Input Layer:} Accepts the flattened image input (e.g., $28\times28=784$ features for MNIST, $32\times32\times3=3072$ for CIFAR).  
\textbf{Hidden Layer 1:} Fully-connected layer with 512 neurons, followed by ReLU.  
\textbf{Hidden Layer 2:} Fully-connected layer with 256 neurons, followed by ReLU.  
\textbf{Output Layer:} Linear mapping from 256 units to the number of classes (10 for MNIST/CIFAR-10, 100 for CIFAR-100). In \texttt{\_DIST} variants, this layer is replaced with a distance-based classification head (e.g. Euclidean, cosine) that computes distances between the embedding and class prototypes, outputting \emph{negative distances} as logits.

\textbf{CNN}: \textbf{Conv Layer 1:} 2D convolution, 32 filters, kernel size $3\times3$, padding 1, followed by ReLU, then $2\times2$ max pooling.  
\textbf{Conv Layer 2:} 2D convolution, 64 filters, kernel size $3\times3$, padding 1, followed by ReLU, then $2\times2$ max pooling.  
\textbf{Fully Connected Layer:} Flattened output fed into a 128-unit linear layer with ReLU.  
\textbf{Output Layer:} Linear layer mapping the 128-D representation to the number of classes. In \texttt{\_DIST} variants, this is replaced by a distance metric layer.

\textbf{ResNet-50}: \textbf{Stem:} Standard $7\times7$ convolution with 64 filters and stride 2, batch norm, ReLU, then $3\times3$ max pooling. For CIFAR/MNIST, we use a $3\times3$ conv with stride 1 and remove max pooling.  
\textbf{Stage 1:} 3 bottleneck blocks, output 256 channels.  
\textbf{Stage 2:} 4 bottleneck blocks, output 512 channels.  
\textbf{Stage 3:} 6 bottleneck blocks, output 1024 channels.  
\textbf{Stage 4:} 3 bottleneck blocks, output 2048 channels.  
\textbf{Global Pooling and Output:} Global average pooling yields a 2048-D vector. In the baseline, a linear FC layer maps to logits. In \texttt{\_DIST} variants, the FC is replaced by a distance layer (e.g. cosine similarity) that outputs similarity-based logits.

\textbf{{Pyramid Vision Transformer (PVTv2-B0)}}: \textbf{Stage 1:} Overlapping patch embedding with a $7\times7$ conv (stride 4), output 32 channels, followed by 2 Transformer encoder layers (1 attention head).  
\textbf{Stage 2:} $3\times3$ conv (stride 2), output 64 channels, followed by 2 encoder layers (2 heads).  
\textbf{Stage 3:} $3\times3$ conv (stride 2), output 160 channels, followed by 2 encoder layers (5 heads).  
\textbf{Stage 4:} $3\times3$ conv (stride 2), output 256 channels, followed by 2 encoder layers (8 heads).  
\textbf{Global Pooling and Output:} Global average pooling yields a 256-D vector. A linear classifier maps to the number of classes in the baseline, while in \texttt{\_DIST} variants this is replaced with a distance layer producing log-similarity or negative distance scores.

\textbf{Preprocessing Pipelines}: \textbf{MNIST:} For MLP/CNN, grayscale input normalized to mean 0.5, std 0.5. For ResNet, normalization uses dataset statistics (mean 0.1307, std 0.3081). For PVT, grayscale converted to 3 channels, resized to 32 (PVT), normalized to mean/std 0.5.  
\textbf{CIFAR-10:} Normalization with mean (0.4914, 0.4822, 0.4465) and std (0.2023, 0.1994, 0.2010). ResNet uses data augmentation (random flips, crops, small rotations).
\textbf{CIFAR-100:} Normalization with mean (0.5071, 0.4867, 0.4408) and std (0.2675, 0.2565, 0.2761). Stronger augmentation (random flips, crops, rotations, color jitter). PVT models use $32\times32$ resized inputs with normalization.

\subsection{LLMs}
\label{appendix:llm-arch}
This section documents the \textbf{LLM} configurations used in our experiments for reproducibility. We report data preprocessing, architectural details for \textbf{GPT}, \textbf{BERT}, and \textbf{Qwen2}-style models, how \textbf{distance-based heads} are integrated in place of the standard linear classifier, and the training/evaluation/emissions-logging pipeline. All models are implemented in PyTorch and trained with mixed precision when available.

\textbf{{Data and Preprocessing}}
\textbf{{Corpus and Storage.}}
We pre-process a text corpus into contiguous token ID arrays and store them as memory-mapped files:
\begin{itemize}
\item \texttt{train.bin} and \texttt{val.bin}: \texttt{np.memmap} arrays of type \texttt{uint16} containing token IDs.
\item \texttt{meta.pkl}: contains metadata including \texttt{vocab\_size} (used to configure model embeddings).
\end{itemize}
Let $V$ denote the discovered vocabulary size from \texttt{meta.pkl} (fallback $V{=}50304$ if not found).

\textbf{{Batching.}}
For a given \texttt{block\_size} $L$, batches are sampled by picking random starting indices and slicing $L$ tokens:
\begin{center}
\texttt{X = data[i : i+L], \quad Y = data[i+1 : i+1+L]} \quad (causal LM)
\end{center}
All batching is performed on-device with pinned memory. We denote \texttt{batch\_size} by $B$.

\textbf{Masking for MLM (BERT).}
For BERT runs, we construct masked language modeling (MLM) batches with the standard 15\% corruption:
\begin{itemize}
  \item Select $\approx 15\%$ token positions per sequence to form mask indices $\mathcal{M}$.
  \item For each $i\in\mathcal{M}$: with 80\% probability replace $x_i$ with \texttt{[MASK]} (id $\le 103$ or capped by $V{-}1$), with 10\% replace by a random token in $[0,V)$, with 10\% keep $x_i$ unchanged.
  \item Labels use the original token at masked positions and $-100$ (ignore index) elsewhere.
\end{itemize}
This yields \texttt{input\_ids}, \texttt{attention\_mask} (all-ones here), and \texttt{labels} containing ground-truth only at masked positions.

\textbf{Architectures}
Across models below, the principal hyperparameters are: layers  ($n_{\ell}$), heads ($n_{h}$), embedding dim ($d$), context length ($L{=}\texttt{block\_size}$), vocab size ($V$).
Unless otherwise specified, positional encodings follow each model's default (e.g., learned or rotary).

\subsubsection*{GPT2 (Causal LM)}
\textbf{Backbone.} A standard decoder-only Transformer with $n_{\ell}$ blocks. Each block has:
\begin{itemize}
  \item Multi-Head Causal Self-Attention with $n_{h}$ heads, hidden size $d$, and causal mask.
  \item Position-wise MLP of width typically $\approx 4d$ with nonlinearity (e.g., GELU).
  \item Pre/post LayerNorm and residual connections as in GPT-style decoders.
\end{itemize}
\textbf{Token Embeddings.} Learnable token and (implicit) position embeddings of sizes $V\times d$ and $L\times d$ (or rotary embeddings if enabled).
\textbf{Projection Head (baseline).} A linear layer $W_{\text{lm}}\in\mathbb{R}^{d\times V}$ producing logits over the vocabulary at each position.
\textbf{Distance Head (\_DIST).} The linear projection is replaced by a \emph{distance-based layer} that treats the vocabulary columns as \emph{prototypes} $\{w_{v}\in\mathbb{R}^{d}\}_{v=1}^{V}$. Given a hidden state $h_t\in\mathbb{R}^{d}$, the head returns per-token logits $z_{t,v}=-D(h_t,w_v;\Theta)$ (or $\log S(h_t,w_v)$ for similarity-type layers), where $D(\cdot,\cdot;\Theta)$ is one of the distances defined in the main text (Euclidean, cosine, Manhattan, Minkowski, Canberra, Bray--Curtis, Chebyshev, Mahalanobis, Hamming). This integrates seamlessly with the causal LM objective (next-token prediction via softmax over $V$).

\subsubsection*{BERT (Masked LM)}
\textbf{Backbone.} An encoder-only Transformer with $n_{\ell}$ layers, each with:
\begin{itemize}
  \item Multi-Head Self-Attention (bidirectional) with $n_{h}$ heads.
  \item Position-wise MLP, LayerNorm, residual connections.
\end{itemize}
\textbf{Embeddings.} Token embeddings $V\times d$, segment/type embeddings (size 2), and positional embeddings of length $L$.
\textbf{Head (baseline).} The standard MLM classifier projects $d\!\to\!V$ (optionally via an intermediate nonlinearity tied to the embedding matrix).
\textbf{Distance Head (\_DIST).} We replace the MLM classifier with the same prototype-based distance layer used for GPT, but applied \emph{only at masked positions}. For each masked token representation $h_i$, logits are $z_{i,v}=-D(h_i,w_v;\Theta)$ (or log-similarity), and cross-entropy is computed against the ground-truth token at $i$.

\subsubsection*{Qwen2-style Decoder (Causal LM)}
\textbf{Backbone.} A decoder-only Transformer similar to GPT, with model-specific details:
\begin{itemize}
  \item Rotary Position Embeddings (RoPE) with $\theta$ (e.g., $\theta=10^6$).
  \item RMSNorm with $\epsilon$ (e.g., $10^{-6}$) in place of LayerNorm.
  \item Grouped key/value heads: \texttt{num\_key\_value\_heads} may be $< n_h$.
  \item Intermediate MLP width (\texttt{intermediate\_size}) configurable.
\end{itemize}
\textbf{Vocabulary.} By default, we use Qwen's native vocabulary \\(\texttt{vocab\_size}=151{,}936); alternatively, one can adapt to the dataset vocab.
\textbf{Head (baseline vs. \_DIST).} As with GPT, the final projection is either a linear layer to $V$ or a distance-based head over $V$ prototype vectors.

\subsection*{Distance-Based Output Layer}
For all three families (GPT, BERT/MLM, Qwen2), the baseline $d\!\to\!V$ classifier is replaced in \texttt{\_DIST} runs by a distance head:
\[
\scriptsize
z_v(h)\;=\;\begin{cases}
-\,\|h-w_v\|_2 & \text{(Euclidean)}\\
-\,\|h-w_v\|_1 & \text{(Manhattan)}\\
-\,\|h-w_v\|_p & \text{(Minkowski, $p$ specified)}\\
-\,\big(1-\tfrac{h^\top w_v}{\|h\|_2\|w_v\|_2}\big) & \text{(Cosine)}\\
-\,D_{\text{Canberra}}(h,w_v) \;\;\text{or}\;\; -\,D_{\text{Bray--Curtis}}(h,w_v) & \text{(variants as defined)}\\
-\,\|h-w_v\|_\infty & \text{(Chebyshev)}\\
-\,\sqrt{(h-w_v)^\top \Sigma^{-1}(h-w_v)} & \text{(Mahalanobis, variants)}\\
-\,D_{\text{Hamming}}(h,w_v) & \text{(soft/Gumbel/hard)}
\end{cases}
\]
where $w_v$ are learned prototype vectors (analogous to classifier weights). We adopt the numerically robust implementations given in the main text (e.g., small $\varepsilon$, clamping, optional normalization of $h$ and/or $w_v$ where appropriate). For cosine, we may output $\log$-similarities for stability. Loss is standard cross-entropy over the $V$ logits per position (causal) or per masked position (MLM).

\subsection*{Training Setup and Optimization}
\paragraph{Device and Precision.}
We use \texttt{bfloat16/float16/float32} (configurable) with automatic mixed precision:
\[
\texttt{torch.autocast(device\_type='cuda', dtype=ptdtype)}.
\]
Training can run in single-GPU or \textbf{DDP} (\texttt{torch.distributed}) multi-GPU mode. In DDP, \texttt{LOCAL\_RANK} selects the device, and gradients are synchronized across ranks.

\textbf{Initialization and Checkpointing.}
Models are initialized \textit{from scratch} using the specified architecture config (layers, heads, width, $L$, $V$). For GPT-only runs we optionally support \texttt{init\_from='gpt2*'}, and for BERT we support \texttt{init\_from='bert*'} (when provided), with appropriate overrides. Checkpoints store model/optimizer state, \texttt{iter\_num}, \texttt{best\_val\_loss}, and the configuration.

\textbf{Optimizer and LR Schedule.}
We use the model’s \\ \texttt{configure\_optimizers} helper to instantiate an Adam/AdamW-style optimizer with weight decay and $(\beta_1,\beta_2)$. Learning rate follows cosine decay with warmup:
\[
\scriptstyle
\text{lr}(t)=
\begin{cases}
\text{lr}_{\max}\cdot t/\text{warmup} & t<\text{warmup},\\
\text{lr}_{\min} + \tfrac{1}{2}\!\left(1+\cos\!\frac{\pi(t-\text{warmup})}{T-\text{warmup}}\right)\!\big(\text{lr}_{\max}-\text{lr}_{\min}\big) & t\le T,
\end{cases}
\]
where $T$ is \texttt{lr\_decay\_iters}. We apply gradient accumulation \\(\texttt{gradient\_accumulation\_steps}), optional gradient clipping \\(\texttt{grad\_clip}), and AMP scaling (\texttt{GradScaler}).

\textbf{Objectives.}
\begin{itemize}
  \item \textbf{GPT/Qwen2 (causal LM):} next-token cross-entropy over $V$ at each position.
  \item \textbf{BERT (MLM):} cross-entropy computed only at masked positions; non-masked labels set to $-100$ (ignored).
\end{itemize}
Accuracy reporting: we compute token-level accuracy for monitoring (on next-token for causal LM, on masked tokens for MLM).

\textbf{Evaluation and Early Signals.}
At fixed \texttt{eval\_interval}, we run \texttt{estimate\_loss()} over \texttt{eval\_iters} batches on train/val splits (model in \texttt{eval()}), then resume training. Best validation loss checkpoints are saved; optional compile (\texttt{torch.compile}) can be enabled.

\subsection*{Sustainability Tracking}
We integrate \textit{CodeCarbon} to measure energy and emissions. At each evaluation interval:
\begin{enumerate}
  \item Stop the tracker and record interval-level metrics: emissions (kg CO$_2$), duration, estimated CPU/GPU/RAM power and energy.
  \item Log cumulative emissions and training metrics (loss, lr) to W\&B (if enabled).
  \item Restart the tracker for the next interval to avoid long-running file locks and to attribute emissions to training phases cleanly.
\end{enumerate}
At the end of training, we stop the tracker one final time and persist all accumulated records to a CSV (\texttt{emissions\_*.csv}) alongside model checkpoints.

\textbf{Key Configurations (Reproducibility)}
The following knobs are saved in run configs/checkpoints and should be reported alongside results:
% \[
\begin{equation*}
\scriptsize
  \begin{aligned}
(n_{\ell},\,n_h,\,d,\,L,\,V),\quad \text{distance head type and parameters }(\Theta),\quad
\text{batch size }B,\;\text{precision},\; \\
\text{optimizer \& betas},\;
\text{lr schedule }(\text{warmup},\,T,\,\text{lr}_{\max},\,\text{lr}_{\min}),\;
\text{grad accumulation},\;\\ 
\text{grad clip},\;\text{DDP world size}.
\end{aligned}
\end{equation*}

When using \texttt{\_DIST} variants, we additionally report which distance (Euclidean, cosine, Manhattan, Minkowski($p$), Canberra, Bray--Curtis, Chebyshev, Mahalanobis, Hamming), any normalization/scaling flags, and regularization choices (e.g., Mahalanobis covariance learning/regularization).

In all models, the sole architectural change introduced by harmonic loss is confined to the \textbf{output head}: a drop-in replacement of the linear classifier with a \textbf{distance-based prototype head} over the vocabulary. This isolates the effect of the loss geometry while keeping the Transformer backbone (and training recipe) unchanged, enabling controlled comparisons across distances in terms of \emph{accuracy}, \emph{interpretability} (e.g., PCA-based analyses), and \emph{sustainability} (emissions and runtime).

\section{Hyperparameter Configurations}
\label{appendix:hyperparams}
The hyperparameter settings in Tables~\ref{tab:llm-core}--\ref{tab:dist-shared} were chosen to balance \emph{comparability}, \emph{training stability}, and \emph{sustainability}. Below we highlight several important considerations.

\subsection{Language Models (OpenWebText)}
Table~\ref{tab:llm-core} specifies the core training parameters for GPT, BERT, and Qwen on OpenWebText. The main goal was to maintain a fair comparison across models of varying scale by using effective batch sizes of similar order (76--128). This ensures that any differences observed in performance or emissions are attributable to the \emph{loss formulation}, not simply to batch scaling. The use of AdamW with default $\beta$ values (0.9, 0.999) follows current best practices for stability.  

Table~\ref{tab:llm-arch} details architecture-specific modifications. BERT includes type embeddings and a masked language modeling (MLM) setup, while GPT and Qwen use causal language modeling (CLM). Qwen, being substantially larger, incorporates more advanced design elements such as grouped query attention (GQA) and rotary position embeddings (RoPE). Table~\ref{tab:llm-summary} summarizes these differences: GPT and Qwen follow causal objectives, while BERT relies on bidirectional context, which may affect the degree to which distance-based losses interact with their representations.  

\begin{table*}[h]
\centering
\caption{Core configuration for GPT, BERT, Qwen, and GPT-2B on OpenWebText.}
\label{tab:llm-core}
\begin{tabular}{lcccc}
\toprule
\textbf{Configuration} & \textbf{GPT} & \textbf{BERT} & \textbf{Qwen} & \textbf{GPT-2B} \\
\midrule
$n_{\text{layer}}$ & 12 & 12 & 24 & 48 \\
$n_{\text{head}}$ & 12 & 12 & 14 & 20 \\
$n_{\text{embd}}$ & 768 & 768 & 896 & 1600 \\
Vocab size & 50304 & 50304 & 151936 & 50304 \\
Dropout & 0.1 & 0.1 & 0.0 & 0.1 \\
Bias & True & True & True & True \\
Batch size & 16 & 38 & 6 & 3 \\
Grad. accum. steps & 8 & 2 & 10 & 21 \\
Effective batch size & 128 & 76 & 60 & 63 \\
Learning rate & 2e-4 & 1e-4 & 1e-4 & 1e-4 \\
Warmup iters & 500 & 1000 & 1000 & 1000 \\
Weight decay & 0.01 & 0.01 & 0.01 & 0.01 \\
Grad clip & 1.0 & 1.0 & 1.0 & 1.0 \\
Min LR & 2e-6 & 1e-6 & 1e-6 & 1e-6 \\
Decay LR & True & True & True & True \\
LR decay iters & 10000 & 10000 & 10000 & 10000 \\
Max iters & 10000 & 10000 & 10000 & 10000 \\
Dataset & OpenWebText & OpenWebText & OpenWebText & OpenWebText \\
dtype & bfloat16 & bfloat16 & bfloat16 & bfloat16 \\
Optimizer & AdamW & AdamW & AdamW & AdamW \\
$\beta_1,\beta_2$ & 0.9, 0.999 & 0.9, 0.999 & 0.9, 0.999 & 0.9, 0.999 \\
Eval interval & 1000 & 1000 & 1000 & 500 \\
Eval iters & 100 & 100 & 100 & 25 \\
Log interval & 50 & 50 & 50 & 25 \\
Scale attn by inverse layer idx & False & False & False & False \\
\bottomrule
\end{tabular}
\end{table*}

\begin{table}[h]
\centering
\caption{Architecture-specific settings for GPT, BERT, Qwen, and GPT-2B.}
\label{tab:llm-arch}
\begin{tabular}{lcccc}
\toprule
\textbf{Configuration} & \textbf{GPT} & \textbf{BERT} & \textbf{Qwen} & \textbf{GPT-2B} \\
\midrule
Block size / Seq length & 1024 & 512 & 1024 & 512 \\
Type vocab size & -- & 2 & -- & -- \\
Pad token id & -- & 0 & -- & -- \\
MLM probability & -- & 0.15 & -- & -- \\
Intermediate size & -- & -- & 4864 & -- \\
\# key–value heads & -- & -- & 2 & -- \\
RMSNorm $\epsilon$ & -- & -- & 1e-6 & -- \\
RoPE $\theta$ & -- & -- & 1{,}000{,}000.0 & -- \\
\bottomrule
\end{tabular}
\end{table}

\begin{table}[h]
\setlength{\tabcolsep}{2pt}
\centering
\caption{Key differences summary (task and position encoding).}
\label{tab:llm-summary}
\begin{tabular}{lcccc}
\toprule
\textbf{Aspect} & \textbf{GPT} & \textbf{BERT} & \textbf{Qwen} & \textbf{GPT-2B} \\
\midrule
Model size (approx.) & $\sim$124M & $\sim$110M & $\sim$494M & $\sim$2B \\
Attention & Causal & Bidirectional & Causal (GQA) & Causal \\
Training task & CLM & MLM & CLM & CLM \\
Position encoding & Learned & Learned & RoPE & Learned \\
\bottomrule
\end{tabular}
\vspace{0.3em}
\footnotesize CLM = Causal Language Modeling;\quad MLM = Masked Language Modeling;\quad GQA = Grouped Query Attention.
\end{table}

\subsection{Vision Models}
Tables~\ref{tab:vision-core}--\ref{tab:vision-dataset} provide the vision settings across datasets. As shown in Table~\ref{tab:vision-core}, optimizer and learning-rate schedules are backbone-specific: Adam for MLPs and CNNs, AdamW for transformers (PVT), and SGD with momentum for ResNet50. This reflects both convention and empirical stability in preliminary experiments.  %Longer schedules are used for CIFAR-100 due to its greater difficulty, with patience for early stopping (Table~\ref{tab:vision-parameters}) scaled accordingly. 
Batch size selection (Table~\ref{tab:vision-bsz}) reflects hardware utilization on H100 GPUs. Notably, lightweight backbones (e.g., CNNs) leverage very large batches (up to 8192 for MNIST), while transformer-based PVT is limited to much smaller batches (128--256) to fit memory constraints. These design choices affect emissions profiles: large-batch training can reduce wall-clock time but at the cost of GPU memory overhead. %
Learning-rate schedulers %(Table~\ref{tab:vision-sched}) 
differ across models. For example, PVT employs cosine annealing, which smooths convergence and interacts well with distance-based loss formulations. ResNet50 relies on multi-step decay, ensuring stability across the long 200-epoch training horizon on CIFAR-100.

\textbf{Distance Layer Parameters.}  
Table~\ref{tab:dist-shared} summarizes the shared hyperparameters across all distance functions. The exponent $n$ is fixed to 1.0 and $\varepsilon=10^{-4}$ provides numerical stability. Importantly, distances are not scaled post hoc, ensuring that differences in results are directly attributable to the geometric properties of the chosen distance (Euclidean, Manhattan, Mahalanobis, etc.), rather than to auxiliary tuning.

\begin{table}[h]
\centering
\caption{Core training configuration by backbone and dataset.}
\begin{tabular}{lcccc}
\toprule
\textbf{Configuration} & \textbf{MLP} & \textbf{CNN} & \textbf{PVT} & \textbf{ResNet50} \\
\midrule
LR (MNIST) & 3e-4 & 3e-4 & 1e-3 & 0.1 \\
LR (CIFAR-10) & 3e-4 & 3e-4 & 1e-3 & 0.1 \\
LR (CIFAR-100) & 3e-4 & 3e-4 & 5e-4 & 0.1 \\
LR (MarathiSign) & 1.5e-4 & 1.5e-4 & 5e-4 & 0.05 \\
LR (TinyImageNet) & 3e-4 & 3e-4 & 5e-4 & 0.1 \\
Epochs (MNIST) & 40 & 40 & 80 & 100 \\
Epochs (CIFAR-10) & 40 & 40 & 80 & 100 \\
Epochs (CIFAR-100) & 150 & 150 & 150 & 200 \\
Epochs (MarathiSign) & 50 & 50 & 100 & 75 \\
Epochs (TinyImageNet) & 100 & 100 & 200 & 150 \\
Optimizer & Adam & Adam & AdamW & SGD \\
Weight decay & 0 & 0 & 0.01 & 1e-4 \\
Momentum & -- & -- & -- & 0.9 \\
\bottomrule
\end{tabular}
\label{tab:vision-core}
\end{table}

\begin{table}[h]
\setlength{\tabcolsep}{1pt}
\small
\centering
\caption{Batch size configuration on H100 GPU.}
\begin{tabular}{lcccccc}
\toprule
\textbf{Model} & \textbf{MNIST} & \textbf{CIFAR-10} & \textbf{CIFAR-100} & \textbf{MarathiSign} & \textbf{TinyImageNet} \\
\midrule
MLP & 2048 & 1024 & 1024 & 128 & 256 \\
CNN & 8192 & 4096 & 512 & 512 & 512 \\
PVT & 256 & 512 & 256 & 64 & 128 \\
ResNet50 & 512 & 512 & 256 & 128 & 256 \\
\bottomrule
\end{tabular}
\label{tab:vision-bsz}
\end{table}

\begin{table*}[h]
\small
\centering
\caption{Learning-rate schedulers by backbone and dataset.}
\begin{tabular}{lccccc}
\toprule
\textbf{Model} & \textbf{MNIST} & \textbf{CIFAR-10} & \textbf{CIFAR-100} & \textbf{MarathiSign} & \textbf{TinyImageNet} \\
\midrule
MLP & None & None & StepLR (50, 0.5) & ReduceLR* & StepLR (50, 0.5) \\
CNN & None & None & StepLR (50, 0.5) & ReduceLR* & StepLR (50, 0.5) \\
PVT & CosineAnn. (80) & CosineAnn. (80) & CosineAnn. (150) & CosineAnn. (100) & CosineAnn. (200) \\
ResNet50 & StepLR (30, 0.1) & StepLR (30, 0.1) & MultiStep** & StepLR (25, 0.1) & MultiStep*** \\
\bottomrule
\multicolumn{6}{l}{\footnotesize *ReduceLROnPlateau (mode=max, factor=0.5, patience=5, min\_lr=1e-6)} \\
\multicolumn{6}{l}{\footnotesize **MultiStepLR (milestones=[60,100,140], $\gamma$=0.2)} \\
\multicolumn{6}{l}{\footnotesize ***MultiStepLR (milestones=[80,120], $\gamma$=0.2)} \\
\end{tabular}
\label{tab:vision-schedulers}
\end{table*}

\begin{table*}[h]
\small
\centering
\caption{Dataset metadata and early-stopping settings (vision).}
\label{tab:vision-dataset}
\begin{tabular}{lccccc}
\toprule
\textbf{Parameter} & \textbf{MNIST} & \textbf{CIFAR-10} & \textbf{CIFAR-100} & \textbf{MarathiSign} & \textbf{TinyImageNet} \\
\midrule
Num classes & 10 & 10 & 100 & 43 & 200 \\
Early stopping patience & 15 & 15 & 25 & 10 & 15 \\
Min improvement (\%) & 0.01 & 0.01 & 0.01 & 0.01 & 0.01 \\
Native image size & $28\!\times\!28\!\times\!1$ & $32\!\times\!32\!\times\!3$ & $32\!\times\!32\!\times\!3$ & varies & $64\!\times\!64\!\times\!3$ \\
Processed size (MLP/CNN/PVT) & $28\!\times\!28\!\times\!1$ & $32\!\times\!32\!\times\!3$ & $32\!\times\!32\!\times\!3$ & $32\!\times\!32\!\times\!3$ & $224\!\times\!224\!\times\!3$ \\
Processed size (ResNet50) & $28\!\times\!28\!\times\!1$ & $32\!\times\!32\!\times\!3$ & $32\!\times\!32\!\times\!3$ & $224\!\times\!224\!\times\!3$ & $224\!\times\!224\!\times\!3$ \\
\bottomrule
\end{tabular}
\end{table*}

\begin{table}[h]
\centering
\caption{Distance-layer shared parameters (all backbones).}
\label{tab:dist-shared}
\begin{tabular}{lc}
\toprule
\textbf{Parameter} & \textbf{Value} \\
\midrule
$n$ & 1.0 \\
$\varepsilon$ & 1e-4 \\
Scale distances & False \\
\bottomrule
\end{tabular}
\end{table}

\subsection{Discussion}
\textbf{Language models.} GPT and BERT use comparable depth/width with learned positional encodings, while Qwen is larger, adopts RoPE, and GQA. Effective batch sizes (via gradient accumulation) normalize throughput across models for fair comparison on OpenWebText. 

\textbf{Vision models.} Optimizer and scheduler choices follow common practice: Adam/AdamW for MLP/CNN/PVT, SGD with momentum for ResNet50; deeper/longer CIFAR-100 runs employ stepped or cosine schedules. Early-stopping patience scales with dataset difficulty. 

\textbf{DistLayer defaults.} A unified setting ($n{=}1.0$, $\varepsilon{=}10^{-4}$, no scaling) ensures distance variants differ only in geometry, not in auxiliary hyperparameters. These settings match the configuration used in our main experiments and figures.

\newpage

{
\section{Statistical significance against Euclidean harmonic loss.}

\subsection{Wilcoxon Signed-rank Tests}
To quantify whether non--Euclidean harmonic losses differ systematically from
the Euclidean reference, we ran paired Wilcoxon signed--rank tests over all
dataset--backbone combinations ($N{=}16$ pairs per distance).  The resulting
median score improvements and $p$--values are reported in
Tables~\ref{tab:wilcoxon-accuracy}--\ref{tab:wilcoxon-emissions}.

On \emph{model performance}
(Table~\ref{tab:wilcoxon-accuracy}), the non--Euclidean distances
do not achieve a statistically significant \emph{positive} median improvement over the
Euclidean harmonic loss.  Several metrics (e.g., Mahalanobis~(Std.),
Bray--Curtis~(Std.), Canberra variants, Manhattan, Minkowski, Hamming) show
significant \emph{negative} medians ($p<0.05$), indicating that when a
difference is present it tends to favor the Euclidean reference in raw
accuracy.  This is consistent with our main results, where non--Euclidean
geometries target interpretability and sustainability rather than headline
accuracy gains.

For \emph{interpretability}
(Table~\ref{tab:wilcoxon-id}), we observe the opposite pattern.  Distances such
as Bray--Curtis (Norm.), Canberra (Robust/Std.), Chebyshev (Std.), Manhattan,
and both cosine variants exhibit statistically significant shifts in the number
of principal components needed to explain $90\%$ of the variance.  The median
differences are large in magnitude (e.g., $+12.8$ for Bray--Curtis~(Norm.),
$+9.8$ for Canberra~(Robust)), confirming that switching away from Euclidean
induces a consistent and substantial change in representation geometry across
datasets and backbones.

For \emph{sustainability} (Table~\ref{tab:wilcoxon-emissions}), four distances
reach $p<0.05$: Mahalanobis~(Std.) and Bray--Curtis~(Std.) with positive
medians, and Canberra~(Weighted) and Mahalanobis~(Chol.) with negative medians.
This suggests that the carbon footprint differences between Euclidean and most
non--Euclidean harmonic losses are modest and model--dependent: some geometries
slightly increase emissions, others slightly decrease them, but strong
systematic effects are rare once we fix backbone, data, and training budget.

Overall, these nonparametric tests support our main claims: non--Euclidean
harmonic losses do not uniformly dominate Euclidean in accuracy, but several of
them induce statistically significant changes in representation structure, with
only mild and mixed effects on emissions.

\begin{table*}[t]
\centering
\small
\caption{Wilcoxon signed--rank test comparing each non--Euclidean harmonic loss
against the Euclidean harmonic baseline on \emph{average final test accuracy}.
Median score improvement is the median paired difference (non--Euclidean minus Euclidean)
across $N$ dataset--backbone combinations. Positive values indicate that the
non--Euclidean distance attains higher accuracy; the last column marks tests
with $p<0.05$.}
\label{tab:wilcoxon-accuracy}
\begin{tabular}{lrrrrc}
\toprule
\textbf{Comparison} & \textbf{N Pairs} & \textbf{Median Impr.} & \textbf{$p$-value} & \textbf{$p_{\text{adj}}$} & \textbf{Sig.\ ($p<0.05$)} \\
\midrule
Bray--Curtis (Norm.)   & 16 &   0.5317  & 0.17060 & 1.0000 &  \\
Cosine (Stable)        & 16 &   0.1598  & 0.45339 & 1.0000 &  \\
Cosine (Unst.)         & 16 &   0.0965  & 0.55208 & 1.0000 &  \\
Mahalanobis (Chol.)    & 16 &   0.0400  & 0.58717 & 1.0000 &  \\
Bray--Curtis (Abs.)    & 16 &  -1.1967  & 0.10335 & 1.0000 &  \\
Mahalanobis (Diag.)    & 16 &  -1.4633  & 0.08323 & 1.0000 &  \\
Chebyshev (Std.)       & 16 &  -1.7232  & $<0.001$ & 0.0736 & Yes \\
Minkowski ($p{=}3.0$)  & 16 &  -1.8183  & 0.00567 & 0.6234 & Yes \\
Canberra (Weighted)    & 16 &  -2.8167  & 0.00295 & 0.3565 & Yes \\
Manhattan              & 16 &  -4.9467  & 0.00249 & 0.3083 & Yes \\
Minkowski ($p{=}1.5$)  & 16 &  -5.8117  & 0.00176 & 0.2232 & Yes \\
Hamming (Soft)         & 16 &  -6.3183  & 0.00348 & 0.4110 & Yes \\
Canberra (Robust)      & 16 & -16.5883  & $<0.001$ & 0.0736 & Yes \\
Canberra (Std.)        & 16 & -18.3767  & $<0.001$ & 0.0736 & Yes \\
Chebyshev (Smooth)     & 16 & -21.3817  & $<0.001$ & 0.0736 & Yes \\
Bray--Curtis (Std.)    & 16 & -34.4379  & $<0.001$ & 0.0975 & Yes \\
Mahalanobis (Std.)     & 16 & -64.8082  & $<0.001$ & 0.0736 & Yes \\
\bottomrule
\end{tabular}
\end{table*}

\begin{table}[t]
\setlength{\tabcolsep}{1pt}
\centering
\small
\caption{Wilcoxon signed--rank test comparing each non--Euclidean harmonic loss
against the Euclidean harmonic baseline on \emph{average intrinsic dimension}
(number of PCs required to reach $90\%$ EV). Median score improvement is again
the median paired difference (non--Euclidean minus Euclidean); here more negative
values correspond to fewer required components.}
\label{tab:wilcoxon-id}
\begin{tabular}{lrrrrc}
\toprule
\textbf{Comparison} & \textbf{N Pairs} & \textbf{Median Impr.} & \textbf{$p$-value} & \textbf{$p_{\text{adj}}$} & \textbf{Sig.\ ($p<0.05$)} \\
\midrule
Bray--Curtis (Norm.)   & 16 &  12.8333  & $<0.001$ & 0.1142 & Yes \\
Canberra (Robust)      & 16 &   9.8333  & 0.00162  & 0.2086 & Yes \\
Canberra (Std.)        & 16 &   8.1667  & $<0.001$ & 0.0988 & Yes \\
Chebyshev (Std.)       & 16 &   7.7083  & 0.01864  & 1.0000 & Yes \\
Manhattan              & 16 &   7.6667  & 0.00412  & 0.4764 & Yes \\
Cosine (Unst.)         & 16 &   5.8333  & 0.01043  & 1.0000 & Yes \\
Canberra (Weighted)    & 16 &   5.2083  & 0.26768  & 1.0000 &  \\
Cosine (Stable)        & 16 &   4.3333  & 0.01127  & 1.0000 & Yes \\
Mahalanobis (Std.)     & 16 &   4.0417  & 0.12716  & 1.0000 &  \\
Bray--Curtis (Std.)    & 16 &   2.1667  & 0.34869  & 1.0000 &  \\
Bray--Curtis (Abs.)    & 16 &   1.0000  & 0.66019  & 1.0000 &  \\
Hamming (Soft)         & 16 &   0.7083  & 0.77730  & 1.0000 &  \\
Minkowski ($p{=}1.5$)  & 16 &  -0.0000  & 0.80665  & 1.0000 &  \\
Minkowski ($p{=}3.0$)  & 16 &  -0.0000  & 0.75554  & 1.0000 &  \\
Mahalanobis (Diag.)    & 16 &  -0.0000  & 0.30656  & 1.0000 &  \\
Chebyshev (Smooth)     & 16 &  -0.2083  & 0.77638  & 1.0000 &  \\
Mahalanobis (Chol.)    & 16 &  -0.3333  & 0.85062  & 1.0000 &  \\
\bottomrule
\end{tabular}
\end{table}

\begin{table}[t]
\setlength{\tabcolsep}{2pt}
\centering
\small
\caption{Wilcoxon signed--rank test comparing each non--Euclidean harmonic loss
against the Euclidean harmonic baseline on \emph{average emissions} (gCO$_2$eq).
Median score improvement is the median paired difference (non--Euclidean minus
Euclidean); negative values indicate lower emissions than the Euclidean
reference.}
\label{tab:wilcoxon-emissions}
\begin{tabular}{lrrrrc}
\toprule
\textbf{Comparison} & \textbf{N Pairs} & \textbf{Median Impr.} & \textbf{$p$-value} & \textbf{$p_{\text{adj}}$} & \textbf{Sig.\ ($p<0.05$)} \\
\midrule
Mahalanobis (Std.)     & 16 &  1.8037  & $<0.001$ & 0.114  & Yes \\
Bray--Curtis (Std.)    & 16 &  1.1880  & 0.01620  & 1.000  & Yes \\
Cosine (Unst.)         & 16 &  0.2341  & 0.05249  & 1.000  &  \\
Canberra (Std.)        & 16 &  0.0351  & 0.77611  & 1.000  &  \\
Chebyshev (Smooth)     & 16 &  0.0208  & 0.73679  & 1.000  &  \\
Canberra (Robust)      & 16 & -0.0206  & 0.73679  & 1.000  &  \\
Cosine (Stable)        & 16 & -0.0239  & 0.97937  & 1.000  &  \\
Minkowski ($p{=}1.5$)  & 16 & -0.0492  & 0.36552  & 1.000  &  \\
Hamming (Soft)         & 16 & -0.0909  & 0.26625  & 1.000  &  \\
Bray--Curtis (Abs.)    & 16 & -0.1066  & 0.20520  & 1.000  &  \\
Chebyshev (Std.)       & 16 & -0.1421  & 0.14056  & 1.000  &  \\
Bray--Curtis (Norm.)   & 16 & -0.1625  & 0.11477  & 1.000  &  \\
Manhattan              & 16 & -0.1976  & 0.12716  & 1.000  &  \\
Minkowski ($p{=}3.0$)  & 16 & -0.4158  & 0.05249  & 1.000  &  \\
Mahalanobis (Diag.)    & 16 & -0.4587  & 0.05249  & 1.000  &  \\
Canberra (Weighted)    & 16 & -0.7687  & 0.03188  & 1.000  & Yes \\
Mahalanobis (Chol.)    & 16 & -0.9431  & 0.00411  & 0.476  & Yes \\
\bottomrule
\end{tabular}
\end{table}

\subsection{Accuracy with Confidence Intervals.}
To complement the aggregate tables and Wilcoxon tests, 
Figure~\ref{fig:acc:final:curves:CI} reports accuracy curves for the
top-performing losses on each dataset/backbone pair.
For every setting we re-train each candidate with three random seeds and
plot the mean trajectory together with a shaded $95\%$ confidence interval
($n{=}3$).
%This directly addresses the concern that our conclusions might rely on
%single-seed variance.

\begin{figure*}[h!]
    \begin{centering}
    \includegraphics[page=2,width=0.3\textwidth,trim=0 20 7 0,clip]{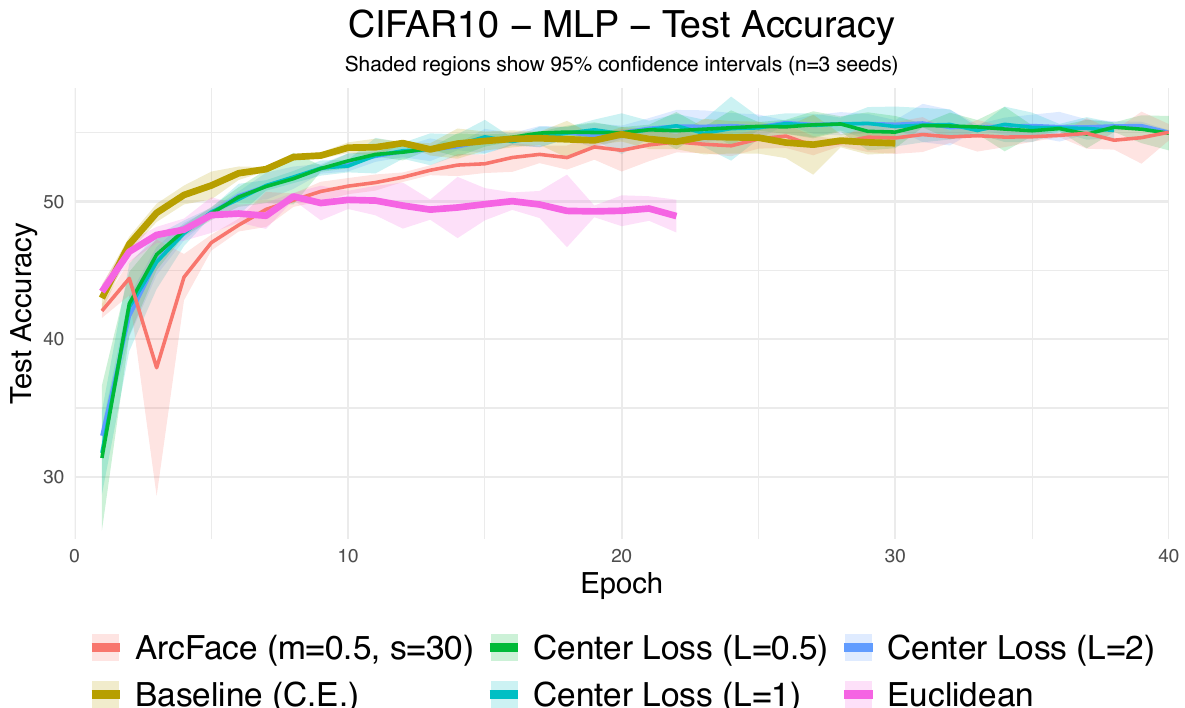}
    \includegraphics[page=6,width=0.3\textwidth,trim=0 20 7 0,clip]{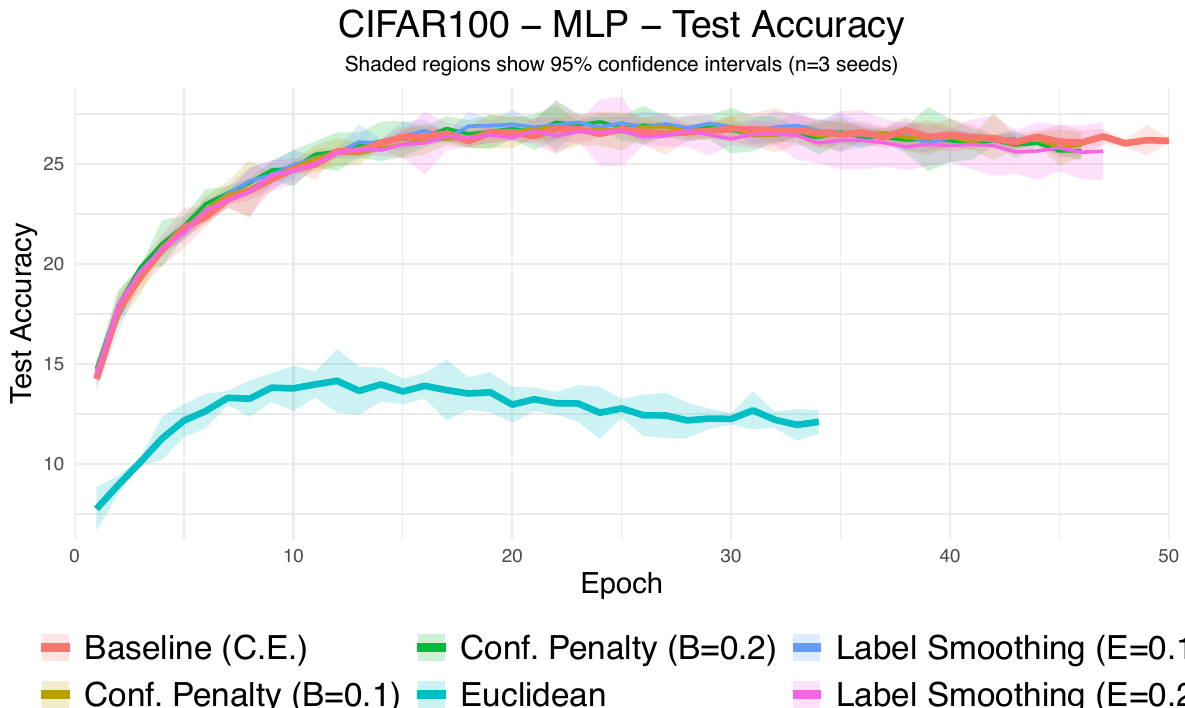}
    \includegraphics[page=10,width=0.3\textwidth,trim=0 20 7 0,clip]{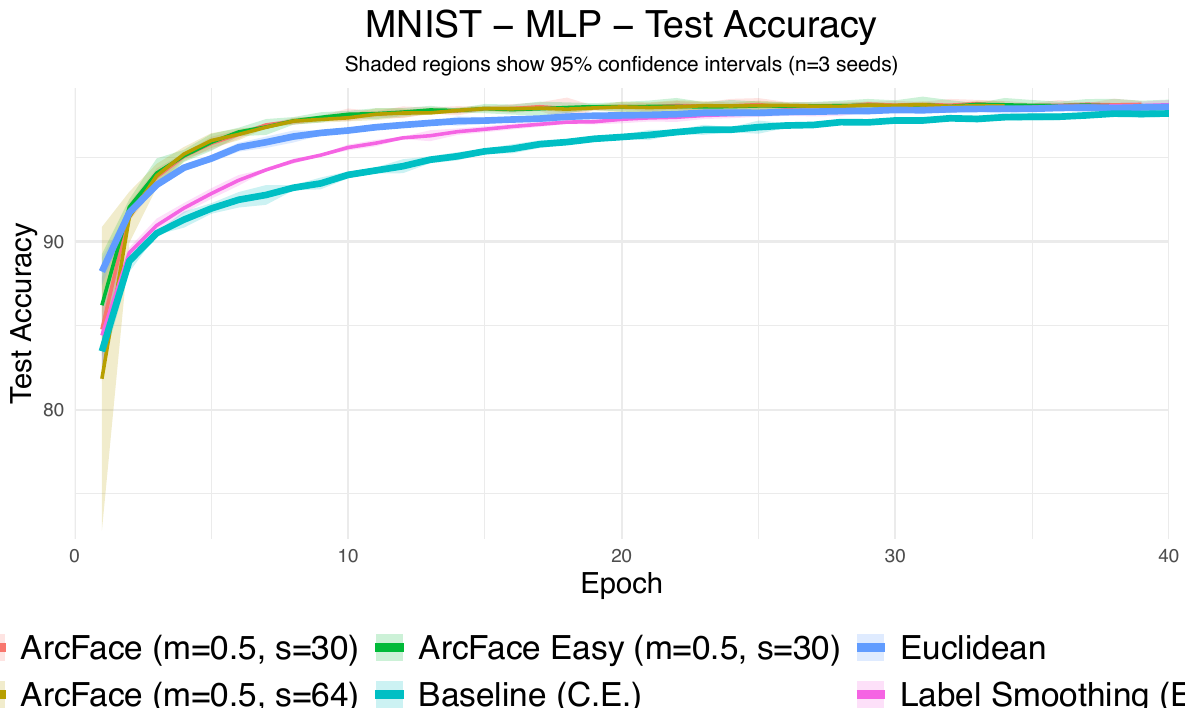}
        % \caption{Curve Plots: MLP}
        \label{fig:acc:final:mlp}
    \end{centering}\\
    
%\end{figure}
%\begin{figure}[h!]
    \begin{centering}
    \includegraphics[page=1,width=0.3\textwidth,trim=0 20 7 0,clip]{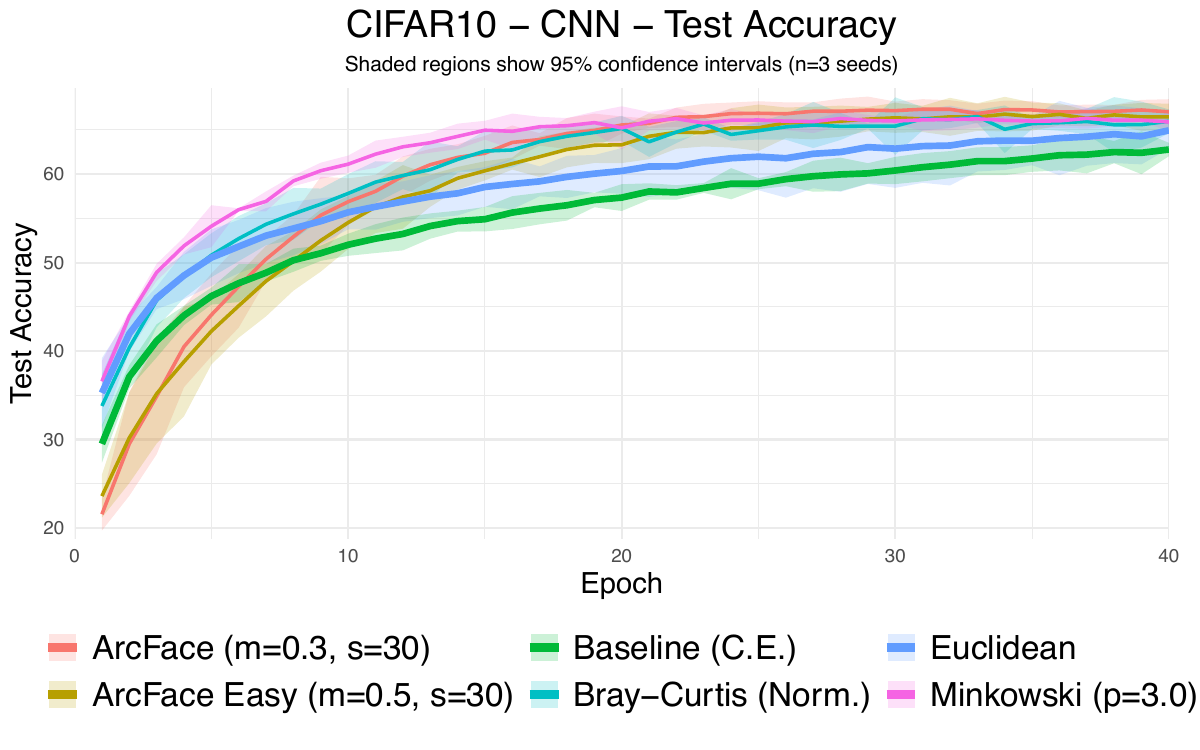}
    \includegraphics[page=5,width=0.3\textwidth,trim=0 20 7 0,clip]{figures/Image_Classification/vision_top_final_performers_with_CI_MLP_CNN_ResNet_PVT_copy.pdf}
    \includegraphics[page=9,width=0.3\textwidth,trim=0 20 7 0,clip]{figures/Image_Classification/vision_top_final_performers_with_CI_MLP_CNN_ResNet_PVT_copy.pdf}
        %\caption{Curve Plots: CNN}
        \label{fig:acc:final:cnn}
    \end{centering}\\
    
%\end{figure}
%\begin{figure}[h!]
    \begin{centering}
    \includegraphics[page=4,width=0.3\textwidth,trim=0 20 20 0,clip]{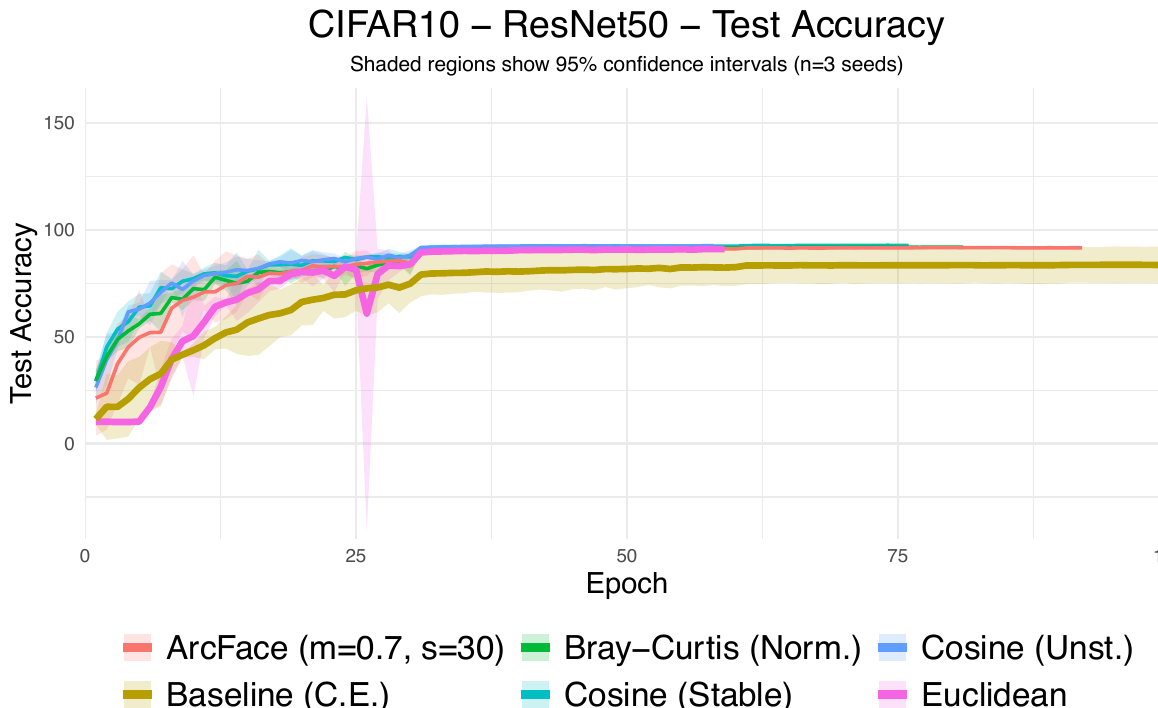}
    \includegraphics[page=8,width=0.3\textwidth,trim=0 20 20 0,clip]{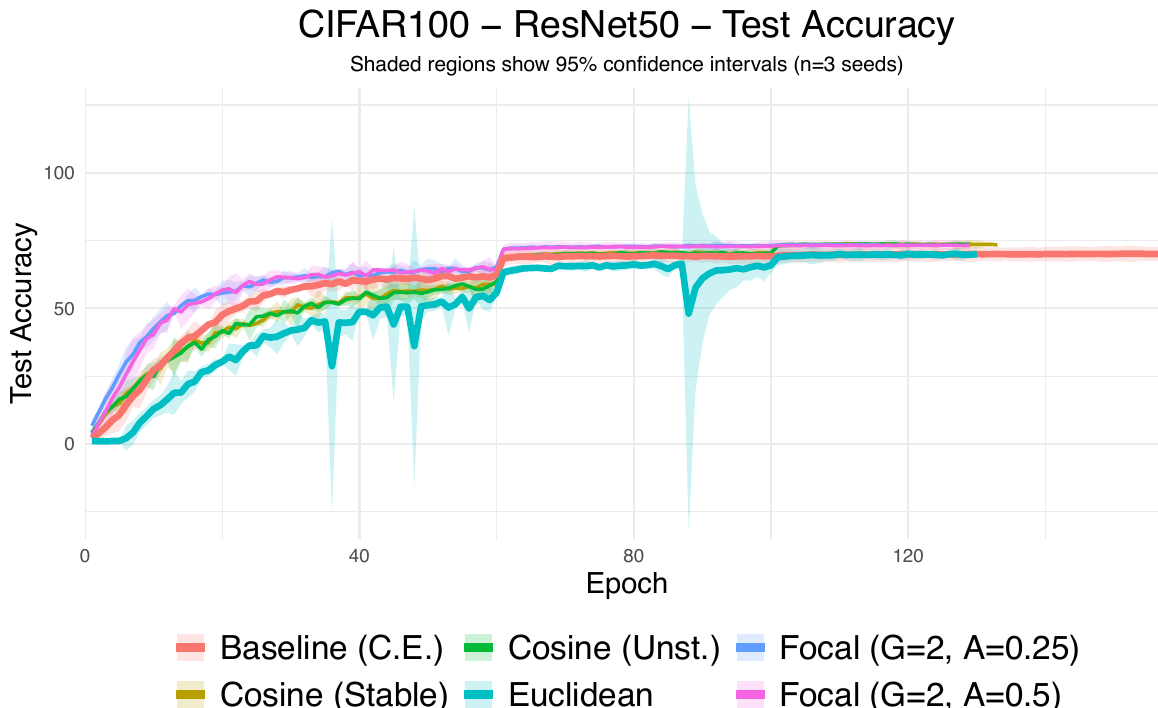}
    \includegraphics[page=12,width=0.3\textwidth,trim=0 20 20 0,clip]{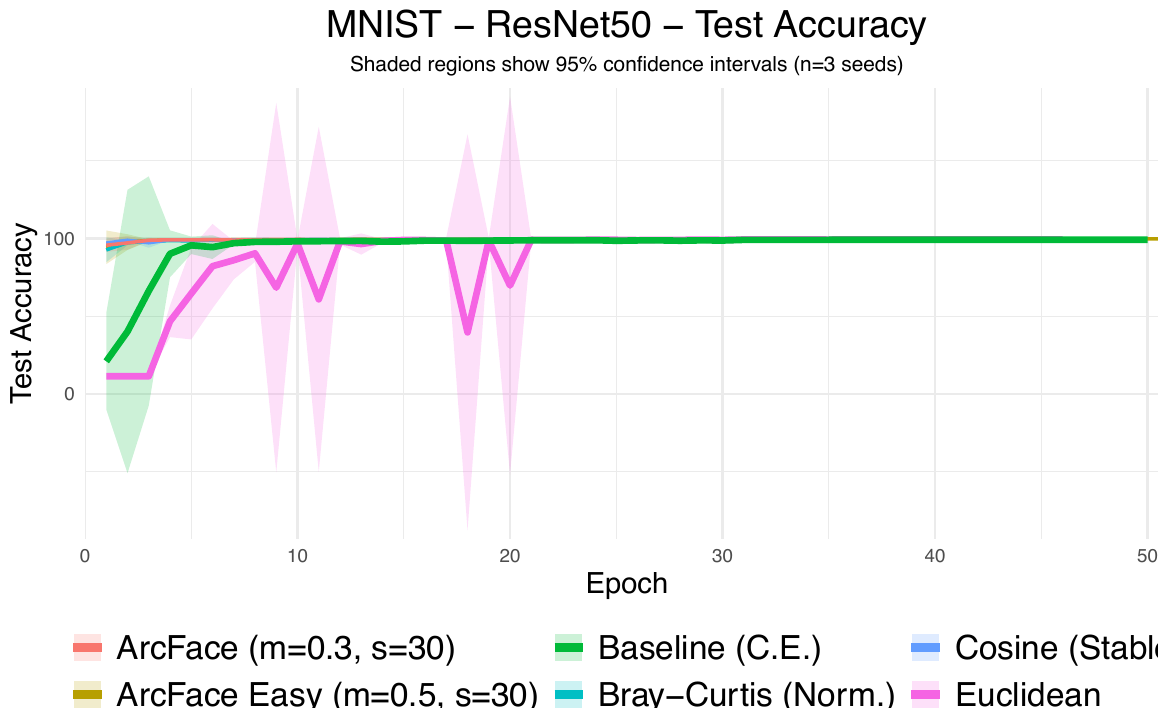}
        %\caption{Curve Plots: ResNet}
        \label{fig:acc:final:ResNet}
    \end{centering}\\
    
%\end{figure}
    \begin{centering}
	    \includegraphics[page=3,width=0.3\textwidth,trim=0 20 20 0,clip]{figures/Image_Classification/vision_top_final_performers_with_CI_MLP_CNN_ResNet_PVT_copy.pdf}
	    \includegraphics[page=7,width=0.3\textwidth,trim=0 20 20 0,clip]{figures/Image_Classification/vision_top_final_performers_with_CI_MLP_CNN_ResNet_PVT_copy.pdf}
	    \includegraphics[page=11,width=0.3\textwidth,trim=0 20 20 0,clip]{figures/Image_Classification/vision_top_final_performers_with_CI_MLP_CNN_ResNet_PVT_copy.pdf}
	        \caption{Vision: Accuracy curves with Confidence Intervals. Shaded regions show 95\% confidence intervals (n = 3 seeds)}
	        \Description[Accuracy curves with 95\% confidence intervals]{A multi-panel set of line plots showing test accuracy over training for multiple vision datasets and model backbones. Each panel overlays accuracy trajectories for selected loss functions, and translucent shaded bands around the lines indicate 95\% confidence intervals computed over three random seeds. The figure emphasizes both the mean learning dynamics and the uncertainty across runs.}
	        \label{fig:acc:final:curves:CI}
	   \end{centering}
\end{figure*}

Across datasets and architectures, two consistent patterns emerge.
First, the ranking suggested by our radar plots and summary tables
is preserved under multi-seed training: distance-based harmonic losses
that previously appeared as strong contenders (e.g., cosine, Bray--Curtis,
Minkowski) continue to track at least as well as, and often above, the
cross-entropy and Euclidean baselines throughout training.  
In several regimes (notably CIFAR-10/CIFAR-100 with ResNet50 and MNIST with
ResNet50/PVT), the confidence bands of the leading non-Euclidean harmonic
loss lie systematically above those of the baselines in the later epochs,
indicating that the final accuracy gains are not artifacts of seed choice
but persist under sampling noise.

Second, the width of the confidence intervals is often comparable
or smaller for harmonic losses than for standard baselines.  
On datasets where optimization is more fragile (e.g., CIFAR-100 with PVT), %,MarathiSign with CNN/PVT), 
cross-entropy and some regularized baselines
(Focal, Center Loss) display visibly wider bands and occasional late-epoch fluctuations, whereas harmonic distances yield smoother trajectories with tighter intervals, echoing our gradient stability findings. Importantly, we do not observe any case where a harmonic loss that outperforms Euclidean in the aggregate tables suffers a reversal when confidence intervals are taken into account.

Overall, these multi-seed curves provide statistical depth to our vision
experiments: performance improvements for non-Euclidean harmonic losses
are accompanied by tight, stable confidence bands, supporting the claim
that their advantages over Euclidean and cross-entropy are robust rather
than due to random initialization.

\clearpage
\newpage

\section{Vision: Radar Plots: Additional datasets (MNIST, CIFAR10)}
\label{sec:spider-extra}

Figure~\ref{fig:acc:spider:vision:extra} reports results for distance–based harmonic losses on MNIST and CIFAR10 across all four backbones (MLP, CNN, ResNet50, PVT).  
%
%Although MNIST is comparatively easy, the plots already illustrate the
%three-way trade-off between performance, geometry, and sustainability.
%
In summary, the MNIST and CIFAR-10 radar plots confirm what has been observed on other datasets: even on smaller benchmarks, non–Euclidean
harmonic losses (particularly cosine, Bray-Curtis, and Chebyshev) can
enhance representation structure and, on harder datasets, improve
accuracy, all while maintaining comparable or better sustainability than
Euclidean harmonic loss and the cross–entropy baseline.

\subsection{MNIST}
\textbf{RQ1: Model Performance (F1, Accuracy).}
Across all MNIST backbones, accuracy and F1 are
saturated for several distances, but cosine–based harmonic
losses (stable/unstable) and Bray--Curtis (normalized) remain
among the most reliable high–performers.  
On MLP and CNN, these distances match or slightly exceed both Euclidean
harmonic loss and the cross–entropy baseline.
On ResNet50 and PVT, where capacity is ample, almost all distances reach
nearly perfect accuracy, confirming that changing the distance in the
harmonic head does not harm performance. %  even in overparameterized
%regimes.

\textbf{RQ2: Interpretability (PC2 EV, PCA 90\%).}
Even on this simple dataset, non–Euclidean distances already reshape the
embedding geometry.  
Bray--Curtis (normalized) and Chebyshev (standard) produce noticeably
higher PC2 explained variance and reduce the number of components needed
to reach $90\%$ EV, indicating compact, prototype–aligned clusters.
Cosine harmonic losses also improve EV relative to Euclidean while
maintaining top accuracy.
Mahalanobis and Minkowski variants (on ResNet50) further concentrate
variance, but their interpretability advantage is less pronounced on
MNIST because the task is almost linearly separable.

\textbf{RQ3: Sustainability (Duration/Epoch/GFLOPs, Emissions).}
For MNIST, the harmonic head constitutes a tiny fraction of the overall
compute, so all distances exhibit similar Duration/Epoch/GFLOPs and
emissions.  
Cosine and Bray--Curtis are essentially neutral relative to Euclidean
and cross–entropy; small differences arise mainly from minor variations
in convergence speed rather than per–step cost.  
The key takeaway from MNIST is therefore that non–Euclidean harmonic
losses can improve representation structure without sacrificing
accuracy or sustainability. % on easy vision tasks.

\subsection{CIFAR10}
\textbf{RQ1: Model Performance (F1, Accuracy).}
On CIFAR-10, cosine harmonic losses become clearly advantageous.
For MLP and CNN, cosine (stable/unstable) and Bray--Curtis (normalized)
consistently occupy the highest or near–highest F1 and accuracy,
outperforming Euclidean harmonic loss and the cross–entropy baseline.
On ResNet50 and PVT, cosine again delivers strong accuracy while
remaining competitive with the best non–Euclidean alternatives
(e.g., Minkowski $p{=}3.0$).
Overall, cosine is the most robust choice across architectures once the
task requires nontrivial feature extraction.

\textbf{RQ2: Interpretability (PC2 EV, PCA 90\%).}
CIFAR-10 further highlights the interpretability benefits of
non–Euclidean geometry.  
Bray--Curtis and Chebyshev systematically
increase PC2 EV and reduce PCA~90\% dimensionality on MLP, CNN, and
ResNet50, yielding sharper, more compact embeddings than Euclidean or
cross–entropy.  
Cosine harmonic losses also improve EV over Euclidean, providing a
favorable accuracy/interpretability compromise.  
On PVT, Canberra–weighted and Bray--Curtis variants similarly enhance
variance concentration while preserving strong performance, reinforcing
the observation that prototype–friendly distances induce more structured
feature spaces.

\textbf{RQ3: Sustainability (Duration/Epoch/GFLOPs, Emissions).}
On CIFAR-10, sustainability trends mirror those seen on larger datasets.
Cosine harmonic loss is typically neutral to slightly favorable in
Duration/Epoch/GFLOPs and emissions relative to Euclidean and
cross–entropy, especially on CNN and ResNet50 where convergence is
faster.  
Bray--Curtis and Canberra variants introduce modest overhead, reflecting
their more complex computations, but remain within the same qualitative
efficiency regime.  
In all cases, the harmonic head is lightweight compared to the backbone,
so the main sustainability differences arise from reduced steps–to–high
accuracy rather than large per–step cost.

%%%%%%%%%%%%%%%%%
\begin{figure*}[h!]
    \centering
    \begin{centering}   
    \includegraphics[page=10,width=0.3\textwidth,trim=45 80 80 15,clip]{figures/Color-Blind_Accessible_Images/color_blind_vision_spiders.pdf}
    \includegraphics[page=2,width=0.3\textwidth,trim=45 80 80 15,clip]{figures/Color-Blind_Accessible_Images/color_blind_vision_spiders.pdf}\\
    \includegraphics[page=9,width=0.3\textwidth,trim=45 80 80 15,clip]{figures/Color-Blind_Accessible_Images/color_blind_vision_spiders.pdf}
    \includegraphics[page=1,width=0.3\textwidth,trim=45 80 80 15,clip]{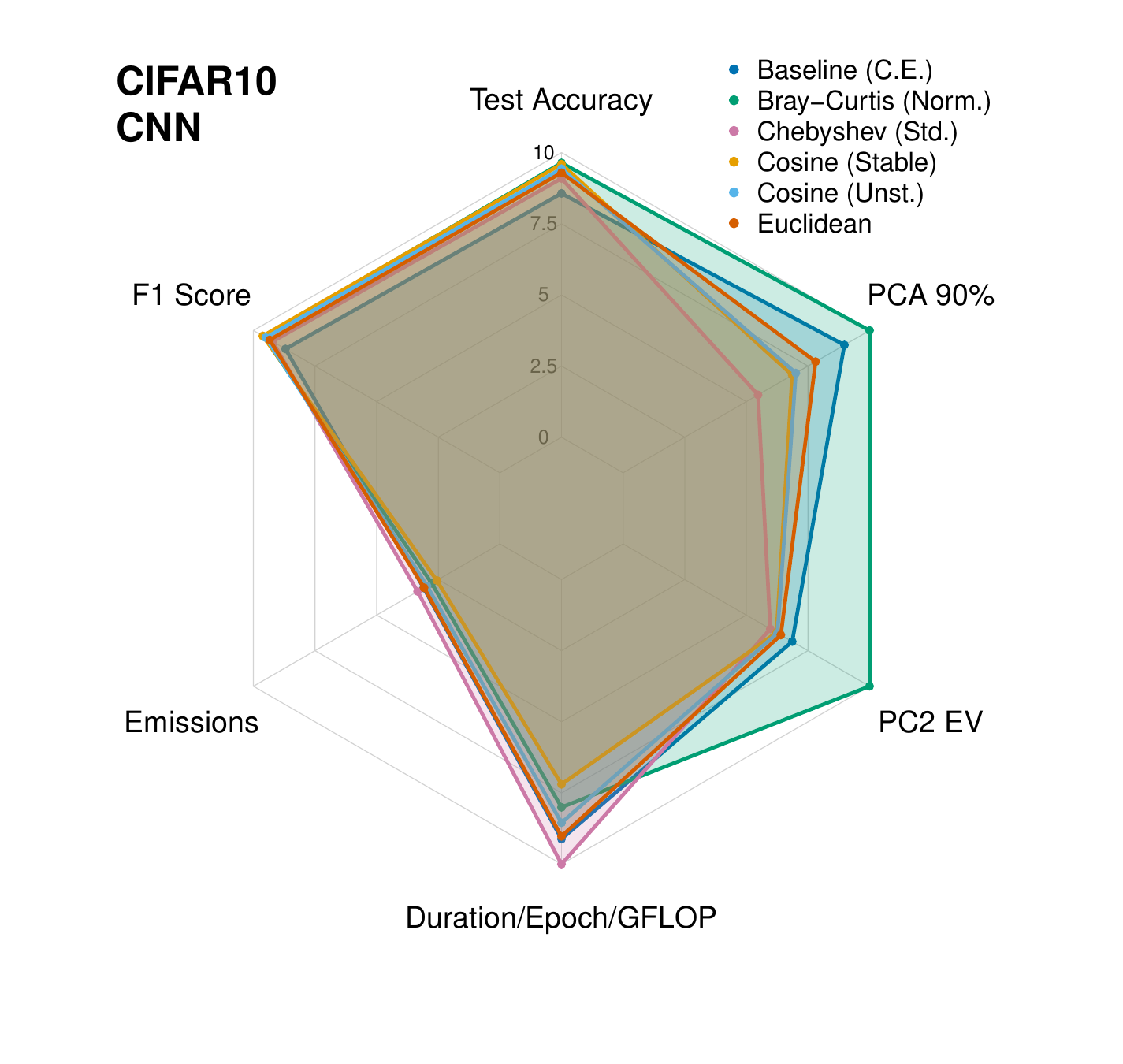}\\
    \includegraphics[page=12,width=0.3\textwidth,trim=45 80 80 15,clip]{figures/Color-Blind_Accessible_Images/color_blind_vision_spiders.pdf} 
    \includegraphics[page=4,width=0.3\textwidth,trim=45 80 80 15,clip]{figures/Color-Blind_Accessible_Images/color_blind_vision_spiders.pdf}\\
	    \includegraphics[page=11,width=0.3\textwidth,trim=45 80 80 15,clip]{figures/Color-Blind_Accessible_Images/color_blind_vision_spiders.pdf} 
	    \includegraphics[page=3,width=0.3\textwidth,trim=45 80 80 15,clip]{figures/Color-Blind_Accessible_Images/color_blind_vision_spiders.pdf}
	    \end{centering}
	    \caption{Vision: Radar plots: 1) \textit{Model Performance} (F1, Accuracy); 2) \textit{Interpretability} (PC2 EV, PCA 90\%), and 3) \textit{Sustainability} (Duration/Epoch/GFLOPs, Emissions). Plots feature Baseline (Cross-Entropy), Euclidean harmonic, and the four top-performing non-Euclidean harmonic losses.}
	    \Description[Additional vision radar plots]{A collection of radar (spider) plots for vision experiments, each summarizing three categories of metrics: performance (F1 and accuracy), interpretability (variance explained by the first two principal components and the number of PCs to reach 90\% explained variance), and sustainability (training duration per epoch, GFLOPs where applicable, and CO2 emissions). Each radar chart overlays multiple methods (cross-entropy, Euclidean harmonic, and several non-Euclidean harmonic losses) so readers can compare the shapes of trade-offs across metrics at a glance.}
	    \label{fig:acc:spider:vision:extra}
\end{figure*}

\section{Results with Alternative Loss Functions}
%\subsection{Additional Loss Baselines: Design Goals and Practical Use}
%\label{sec:baselines:preamble}
To contextualize our distance–based harmonic losses, we benchmark against four widely used alternatives that are often motivated by calibration, robustness in low–data regimes, or representational compactness. Below we summarize each loss, its objective, and why it is relevant along our three axes: effectiveness, sustainability, and interpretability. Unless otherwise noted, these baselines are applied with a conventional linear head and softmax; they can also be evaluated on distance–parameterized logits (e.g., $-\mathrm{dist}(\mathbf{h},\mathbf{w}_c)$) to ensure architectural parity.

\paragraph{Focal Loss (calibration, anti–grokking, class imbalance).}
Focal Loss reweights examples by their difficulty:
\[
\mathcal{L}_{\text{focal}}(\mathbf{z}, y)
= -\,\alpha\,(1 - p_y)^{\gamma}\,\log p_y,
\quad
p_y=\frac{e^{z_y}}{\sum_j e^{z_j}},
\]
with focusing parameter $\gamma\!\ge\!0$ and class weight $\alpha\!\in\!(0,1]$. By down–weighting well–classified (overconfident) samples, it yields smoother gradient signals, often improving calibration and mitigating late–stage overfitting behaviors akin to grokking. In our grids we consider $\gamma\!\in\!\{2,3\}$ and $\alpha\!\in\!\{0.25,0.5\}$. \emph{Sustainability:} modest compute overhead (same forward/backward shape as CE), but potentially fewer effective updates on easy samples; net carbon effect is typically neutral to slightly higher than CE, depending on convergence behavior.

\paragraph{Label Smoothing (reduced overconfidence, low–data stability).}
Label Smoothing replaces the one–hot target with
\[
\tilde{\mathbf{y}}=(1-\varepsilon)\,\mathbf{e}_y + \frac{\varepsilon}{K}\,\mathbf{1},
\quad
\mathcal{L}_{\text{LS}}(\mathbf{z}, y) = -\sum_{c=1}^K \tilde{y}_c \log p_c,
\]
where $\varepsilon\!\in\![0,1)$ controls smoothing (we use $\varepsilon\!\in\!\{0.1,0.2\}$). The softened targets reduce overconfidence and improve generalization in scarce–label settings; they also stabilize optimization by shrinking logit magnitudes. \emph{Interpretability:} mild regularization can yield more isotropic features; \emph{sustainability:} training cost matches CE, with potential reductions in steps–to–target when overconfidence previously harmed convergence.

\paragraph{Center Loss (prototype compactness, cluster interpretability).}
Center Loss explicitly penalizes the distance to a class prototype:
\[
\mathcal{L}_{\text{center}}(\mathbf{h}, y)
=\frac{1}{2}\,\big\|\mathbf{h}-\mathbf{c}_{y}\big\|_2^2,
\qquad
\mathcal{L}=\mathcal{L}_{\text{CE}}+\lambda\,\mathcal{L}_{\text{center}},
\]
with learnable class centers $\{\mathbf{c}_k\}$ and trade–off $\lambda>0$ (we test $\lambda\!\in\!\{0.5,1,2\}$). This encourages intra–class compactness and inter–class separability—properties that make feature clusters and decision prototypes easier to inspect. \emph{Sustainability:} small extra memory and updates for $\mathbf{c}_k$; the overhead is minor relative to the backbone but measurable in long runs.

\paragraph{Confidence Penalty (entropy regularization, anti–grokking).}
Confidence Penalty adds a negative–entropy term to discourage over–peaked posteriors:
\[
\mathcal{L}_{\text{CP}}(\mathbf{z}, y)
= \mathcal{L}_{\text{CE}}(\mathbf{z}, y) \;-\; \beta\,\mathcal{H}(\mathbf{p}),
\quad
\mathcal{H}(\mathbf{p})=-\sum_{c=1}^K p_c \log p_c,
\]
with $\beta\!>\!0$ (we use $\beta\!\in\!\{0.1,0.2\}$). By explicitly rewarding higher predictive entropy when appropriate, it reduces brittle overconfidence and can temper delayed generalization (grokking–like) effects. \emph{Sustainability:} essentially identical compute to CE; any carbon changes stem from altered convergence trajectories rather than per–step cost.

\medskip
\paragraph{ArcFace (angular margins, maximized class separation).}
ArcFace~\citep{deng2019arcface} introduces an \emph{additive angular margin} that enlarges the decision boundary between classes on the unit hypersphere. 
Given normalized features $\mathbf{h}$ and normalized class weights $\mathbf{w}_c$, the cosine similarity 
$\cos\theta_c = \langle \mathbf{h}, \mathbf{w}_c\rangle$ 
is modified for the target class $y$ by adding a fixed angular margin $m$:
\[
\cos(\theta_y + m)
= \cos\theta_y \cos m - \sin\theta_y \sin m.
\]
ArcFace replaces the final linear classifier with a scaled angular softmax:
\[
\mathcal{L}_{\text{arcface}}(\mathbf{h}, y)
= -\log
\frac{\exp\left(s\,\cos(\theta_y + m)\right)}
     {\exp\left(s\,\cos(\theta_y + m)\right)
      + \sum_{c\neq y}\exp\left(s\,\cos\theta_c\right)},
\]
where $s$ is a feature-scale parameter (typically $s\!=\!30$–$64$).  
By manipulating angles rather than norms, ArcFace enforces tighter class clustering and larger inter-class separation, and is widely regarded as a \emph{strong margin-based baseline} in metric learning. 
This makes it a particularly relevant comparator for harmonic losses: both approaches normalize features to a hypersphere and control geometry around class prototypes, but ArcFace explicitly pushes angular margins, whereas harmonic losses adjust the entire distance landscape.  
\emph{Sustainability:} ArcFace is lightweight (same complexity as cosine classifiers), but the angular margin can slightly increase optimization stiffness, occasionally raising per-step compute or slowing convergence. In our experiments we consider $m\!\in\!\{0.3,0.5\}$ and $s\!\in\!\{30,64\}$.

\medskip
%\noindent\textbf{Why these baselines here?}
Each baseline addresses a failure mode that harmonic losses also target but via different inductive biases: Focal/Label Smoothing emphasize calibration and data efficiency; Center Loss operationalizes prototype compactness; Confidence Penalty discourages pathological overconfidence. This makes them natural comparators for our \emph{distance–based} formulation, which subsumes prototype reasoning in its very parameterization and, as we show, can simultaneously improve accuracy, reduce emissions, and enhance interpretability.

\subsection{Vision: Fine-grained results}
%\textbf{Behavior of Additional Baseline Losses (PVT and ResNet50)}
A cross–backbone inspection in Tables \ref{tab:all-losses-cifar100-pvt}--\ref{tab:all-losses-marathi-resnet} shows that the additional baselines
(Focal, Label Smoothing, Confidence Penalty, Center Loss, ArcFace) are
competitive on accuracy, but do not displace the non-Euclidean harmonic
losses as the most balanced options.

On \textbf{CIFAR-10 and CIFAR-100 with ResNet-50}, cosine-based harmonic heads
remain among the strongest configurations: they deliver the largest accuracy and
F1 gains over cross-entropy (up to $\approx 11\%$ relative on CIFAR-10 and
$5\%$ on CIFAR-100), while simultaneously reducing emissions by
$10$--$30\%$ and sharply concentrating the representation (PC90\% dropping from
$50$ to $5$--$8$ components). Focal and ArcFace variants sometimes match top
accuracy, but typically exhibit weaker EV/PC90\% improvements or higher
emissions, so they do not dominate the multi-criteria trade-off.

For \textbf{PVT backbones}, where capacity and input resolution are higher, the
picture is more nuanced. On CIFAR-10 PVT, calibration-oriented losses (Label
Smoothing, Focal) offer slight accuracy improvements and sizable emissions
reductions, yet Euclidean harmonic achieves the most compact geometry (PC90\%
from $17$ to $3$) at only a small performance cost. On CIFAR-100 PVT, focal and
confidence-penalty losses are best in accuracy, but Euclidean harmonic again
produces the most concentrated feature spaces (PC90\% $4$ vs.\ $50$), highlighting
an interpretability advantage even when it is not the accuracy winner.

On the \textbf{high-resolution Marathi Sign dataset}, nearly all methods
saturate accuracy ($\ge 0.999$), so the comparison is driven by structure and
sustainability. Here, cosine-based harmonic losses for ResNet-50 (and to a
lesser extent for PVT) achieve substantial EV/PC90\% gains. e.g., Cosine
(Unst.) reduces PC90\% from $15.5$ to $6.5$ while slightly lowering
emissions, demonstrating that our non-Euclidean harmonic heads remain
competitive even in settings where strong baselines like ArcFace and Center
Loss are present.

\begin{table*}[htbp]
\small
\setlength{\tabcolsep}{2pt}
\centering
\caption{Results for CIFAR100 PVT (top-8 losses) and \% changes w.r.t. Baseline (CE).}
\label{tab:cifar100-pvt-top8}
\begin{tabular}{llllll}
\toprule
\textbf{Method} & \textbf{Acc} & \textbf{F1} & \textbf{gCO$_2$eq} & \textbf{EV} & \textbf{PC90\%} \\
\midrule
Baseline                 & 0.3994 & 0.3970 & 3.67  & 0.0973 & 50.0 \\
Focal ($\gamma{=}2,\alpha{=}0.25$) & 0.4017 ({\scriptsize 0.59\%}) & 0.3996 ({\scriptsize 0.65\%}) & 4.9609 ({\scriptsize -35.35\%}) & 0.1257 ({\scriptsize 29.22\%}) & 50.0 ({\scriptsize 0\%}) \\
Focal ($\gamma{=}3,\alpha{=}0.25$) & 0.4015 ({\scriptsize 0.53\%}) & 0.3999 ({\scriptsize 0.72\%}) & 4.9567 ({\scriptsize -35.24\%}) & 0.1364 ({\scriptsize 40.21\%}) & 50.0 ({\scriptsize 0\%}) \\
Focal ($\gamma{=}2,\alpha{=}0.5$)  & 0.4010 ({\scriptsize 0.40\%}) & 0.3997 ({\scriptsize 0.66\%}) & 3.9182 ({\scriptsize -6.90\%}) & 0.1271 ({\scriptsize 30.64\%}) & 50.0 ({\scriptsize 0\%}) \\
Conf. Penalty ($\beta{=}0.1$) & 0.3999 ({\scriptsize 0.13\%}) & 0.3977 ({\scriptsize 0.17\%}) & 5.2574 ({\scriptsize -43.44\%}) & 0.0963 ({\scriptsize -0.96\%}) & 50.0 ({\scriptsize 0\%}) \\
Label Smoothing ($\varepsilon{=}0.1$) & 0.3894 ({\scriptsize -2.50\%}) & 0.3888 ({\scriptsize -2.08\%}) & 2.6626 ({\scriptsize 27.36\%}) & 0.0818 ({\scriptsize -15.92\%}) & 50.0 ({\scriptsize 0\%}) \\
Conf. Penalty ($\beta{=}0.2$) & 0.3859 ({\scriptsize -3.38\%}) & 0.3834 ({\scriptsize -3.45\%}) & 2.1718 ({\scriptsize 40.74\%}) & 0.0749 ({\scriptsize -23.00\%}) & 50.0 ({\scriptsize 0\%}) \\
Label Smoothing ($\varepsilon{=}0.2$) & 0.3847 ({\scriptsize -3.67\%}) & 0.3851 ({\scriptsize -3.00\%}) & 2.2097 ({\scriptsize 39.71\%}) & 0.0742 ({\scriptsize -23.74\%}) & 50.0 ({\scriptsize 0\%}) \\
ArcFace  ($m{=}0.5,s{=}30$) & 0.3728 ({\scriptsize -6.65\%}) & 0.3772 ({\scriptsize -5.00\%}) & 2.6171 ({\scriptsize 28.60\%}) & 0.1259 ({\scriptsize 29.41\%}) & 50.0 ({\scriptsize 0\%}) \\
Euclidean              & 0.2864 ({\scriptsize -28.29\%}) & 0.2945 ({\scriptsize -25.83\%}) & 6.4329 ({\scriptsize -75.51\%}) & 0.8414 ({\scriptsize 765.15\%}) & 4.0 ({\scriptsize 92.00\%}) \\
\bottomrule
\label{tab:all-losses-cifar100-pvt}
\end{tabular}
\end{table*}

\begin{table*}[htbp]
\small
\setlength{\tabcolsep}{2pt}
\centering
\caption{Results for CIFAR100 ResNet50 (top-8 losses) and \% changes w.r.t. Baseline (CE).}
\label{tab:cifar100-resnet-top8}
\begin{tabular}{llllll}
\toprule
\textbf{Method} & \textbf{Acc} & \textbf{F1} & \textbf{gCO$_2$eq} & \textbf{EV} & \textbf{PC90\%} \\
\midrule
Baseline            & 0.7006 & 0.6993 & 89.64 & 0.1069 & 50.0 \\
Cosine (Stable)     & 0.7381 ({\scriptsize 5.35\%}) & 0.7384 ({\scriptsize 5.59\%}) & 79.2831 ({\scriptsize 11.55\%}) & 0.5915 ({\scriptsize 453.51\%}) & 8.0 ({\scriptsize 84.00\%}) \\
Focal ($\gamma{=}2,\alpha{=}0.25$) & 0.7349 ({\scriptsize 4.90\%}) & 0.7342 ({\scriptsize 5.00\%}) & 72.1795 ({\scriptsize 19.48\%}) & 0.1468 ({\scriptsize 37.32\%}) & 50.0 ({\scriptsize 0\%}) \\
Focal ($\gamma{=}2,\alpha{=}0.5$)  & 0.7341 ({\scriptsize 4.79\%}) & 0.7332 ({\scriptsize 4.85\%}) & 78.7234 ({\scriptsize 12.18\%}) & 0.1224 ({\scriptsize 14.52\%}) & 50.0 ({\scriptsize 0\%}) \\
Cosine (Unst.)      & 0.7340 ({\scriptsize 4.77\%}) & 0.7349 ({\scriptsize 5.09\%}) & 72.5413 ({\scriptsize 19.07\%}) & 0.5891 ({\scriptsize 451.21\%}) & 8.0 ({\scriptsize 84.00\%}) \\
Focal ($\gamma{=}3,\alpha{=}0.25$) & 0.7311 ({\scriptsize 4.36\%}) & 0.7308 ({\scriptsize 4.51\%}) & 81.5875 ({\scriptsize 8.98\%}) & 0.1554 ({\scriptsize 45.41\%}) & 50.0 ({\scriptsize 0\%}) \\
Label Smoothing ($\varepsilon{=}0.1$) & 0.7261 ({\scriptsize 3.64\%}) & 0.7248 ({\scriptsize 3.65\%}) & 79.8223 ({\scriptsize 10.95\%}) & 0.1469 ({\scriptsize 37.44\%}) & 50.0 ({\scriptsize 0\%}) \\
Label Smoothing ($\varepsilon{=}0.2$) & 0.7221 ({\scriptsize 3.08\%}) & 0.7206 ({\scriptsize 3.04\%}) & 81.1883 ({\scriptsize 9.43\%}) & 0.1524 ({\scriptsize 42.59\%}) & 50.0 ({\scriptsize 0\%}) \\
ArcFace ($m{=}0.7,s{=}30$) & 0.7166 ({\scriptsize 2.29\%}) & 0.7150 ({\scriptsize 2.25\%}) & 70.5590 ({\scriptsize 21.29\%}) & 0.6059 ({\scriptsize 466.93\%}) & 48.33 ({\scriptsize 3.33\%}) \\
Euclidean           & 0.7047 ({\scriptsize 0.59\%}) & 0.7055 ({\scriptsize 0.89\%}) & 87.7280 ({\scriptsize 2.13\%}) & 0.4301 ({\scriptsize 302.49\%}) & 33.67 ({\scriptsize 32.67\%}) \\
\bottomrule
\end{tabular}
\label{tab:all-losses-cifar100-resnet}
\end{table*}

\begin{table*}[htbp]
\small
\setlength{\tabcolsep}{2pt}
\centering
\caption{Results for MarathiSign PVT (top-8 losses) and \% changes w.r.t. Baseline (CE).}
\label{tab:marathi-pvt-top8}
\begin{tabular}{llllll}
\toprule
\textbf{Method} & \textbf{Acc} & \textbf{F1} & \textbf{gCO$_2$eq} & \textbf{EV} & \textbf{PC90\%} \\
\midrule
Baseline                & 0.9965 & 0.9964 & 2.85  & 0.2135 & 24.75 \\
ArcFace ($m{=}0.7,s{=}30$) & 0.9997 ({\scriptsize 0.33\%}) & 0.9997 ({\scriptsize 0.33\%}) & 4.2108 ({\scriptsize -47.88\%}) & 0.1132 ({\scriptsize -47.00\%}) & 36.0 ({\scriptsize -45.45\%}) \\
Focal ($\gamma{=}3,\alpha{=}0.25$) & 0.9997 ({\scriptsize 0.32\%}) & 0.9996 ({\scriptsize 0.32\%}) & 2.7240 ({\scriptsize 4.34\%}) & 0.2749 ({\scriptsize 28.74\%}) & 19.33 ({\scriptsize 21.89\%}) \\
Center Loss ($\lambda{=}1$) & 0.9995 ({\scriptsize 0.31\%}) & 0.9995 ({\scriptsize 0.31\%}) & 5.1746 ({\scriptsize -81.72\%}) & 0.1824 ({\scriptsize -14.56\%}) & 29.67 ({\scriptsize -19.87\%}) \\
Conf. Penalty ($\beta{=}0.2$) & 0.9995 ({\scriptsize 0.30\%}) & 0.9994 ({\scriptsize 0.30\%}) & 3.6912 ({\scriptsize -29.63\%}) & 0.1504 ({\scriptsize -29.58\%}) & 33.67 ({\scriptsize -36.03\%}) \\
Focal ($\gamma{=}2,\alpha{=}0.25$) & 0.9993 ({\scriptsize 0.29\%}) & 0.9993 ({\scriptsize 0.28\%}) & 2.5648 ({\scriptsize 9.93\%}) & 0.2853 ({\scriptsize 33.63\%}) & 19.33 ({\scriptsize 21.89\%}) \\
ArcFace ($m{=}0.3,s{=}30$) & 0.9992 ({\scriptsize 0.28\%}) & 0.9992 ({\scriptsize 0.28\%}) & 7.5158 ({\scriptsize -163.94\%}) & 0.1756 ({\scriptsize -17.76\%}) & 28.0 ({\scriptsize -13.13\%}) \\
Label Smoothing ($\varepsilon{=}0.2$) & 0.9992 ({\scriptsize 0.28\%}) & 0.9991 ({\scriptsize 0.27\%}) & 2.6977 ({\scriptsize 5.26\%}) & 0.1190 ({\scriptsize -44.26\%}) & 35.0 ({\scriptsize -41.41\%}) \\
Cosine (Unst.)         & 0.9991 ({\scriptsize 0.27\%}) & 0.9991 ({\scriptsize 0.26\%}) & 7.3447 ({\scriptsize -157.93\%}) & 0.5552 ({\scriptsize 160.04\%}) & 7.0 ({\scriptsize 71.72\%}) \\
Euclidean              & 0.9994 ({\scriptsize 0.30\%}) & 0.9994 ({\scriptsize 0.30\%}) & 4.3621 ({\scriptsize -53.19\%}) & 0.5035 ({\scriptsize 135.83\%}) & 14.25 ({\scriptsize 42.42\%}) \\
\bottomrule
\end{tabular}
\label{tab:all-losses-marathi-pvt}
\end{table*}

\begin{table*}[htbp]
\small
\setlength{\tabcolsep}{2pt}
\centering
\caption{Results for MarathiSign ResNet50 (top-8 losses) and \% changes w.r.t. Baseline (CE).}
\label{tab:marathi-resnet-top8}
\begin{tabular}{llllll}
\toprule
\textbf{Method} & \textbf{Acc} & \textbf{F1} & \textbf{gCO$_2$eq} & \textbf{EV} & \textbf{PC90\%} \\
\midrule
Baseline                  & 0.9998 & 0.9998 & 35.41 & 0.4507 & 15.5 \\
Conf. Penalty ($\beta{=}0.1$) & 0.9999 ({\scriptsize 0.01\%}) & 0.9999 ({\scriptsize 0.01\%}) & 29.6596 ({\scriptsize 16.24\%}) & 0.4393 ({\scriptsize -2.55\%}) & 17.5 ({\scriptsize -12.90\%}) \\
Cosine (Stable)           & 0.9999 ({\scriptsize 0.01\%}) & 0.9999 ({\scriptsize 0.01\%}) & 54.8139 ({\scriptsize -54.80\%}) & 0.7321 ({\scriptsize 62.42\%}) & 5.5 ({\scriptsize 64.52\%}) \\
Conf. Penalty ($\beta{=}0.2$) & 0.9998 ({\scriptsize 0.00\%}) & 0.9998 ({\scriptsize 0.00\%}) & 30.3697 ({\scriptsize 14.23\%}) & 0.3288 ({\scriptsize -27.05\%}) & 33.0 ({\scriptsize -112.90\%}) \\
Focal ($\gamma{=}2,\alpha{=}0.25$) & 0.9998 ({\scriptsize 0.00\%}) & 0.9998 ({\scriptsize 0.00\%}) & 24.95 ({\scriptsize 29.54\%}) & 0.4016 ({\scriptsize -10.89\%}) & 17.0 ({\scriptsize -9.68\%}) \\
Focal ($\gamma{=}3,\alpha{=}0.25$) & 0.9998 ({\scriptsize 0.00\%}) & 0.9998 ({\scriptsize 0.00\%}) & 28.1206 ({\scriptsize 20.58\%}) & 0.4286 ({\scriptsize -4.91\%}) & 16.5 ({\scriptsize -6.45\%}) \\
Cosine (Unst.)            & 0.9997 ({\scriptsize -0.01\%}) & 0.9997 ({\scriptsize -0.01\%}) & 33.8287 ({\scriptsize 4.46\%}) & 0.6427 ({\scriptsize 42.58\%}) & 6.5 ({\scriptsize 58.06\%}) \\
Label Smoothing ($\varepsilon{=}0.1$) & 0.9997 ({\scriptsize -0.01\%}) & 0.9997 ({\scriptsize -0.01\%}) & 33.5266 ({\scriptsize 5.31\%}) & 0.1758 ({\scriptsize -60.99\%}) & 36.0 ({\scriptsize -132.26\%}) \\
Label Smoothing ($\varepsilon{=}0.2$) & 0.9997 ({\scriptsize -0.01\%}) & 0.9997 ({\scriptsize -0.01\%}) & 38.7093 ({\scriptsize -9.32\%}) & 0.1641 ({\scriptsize -63.60\%}) & 36.0 ({\scriptsize -132.26\%}) \\
Euclidean                & 0.9984 ({\scriptsize -0.14\%}) & 0.9983 ({\scriptsize -0.15\%}) & 50.2050 ({\scriptsize -41.79\%}) & 0.2787 ({\scriptsize -38.18\%}) & 50.0 ({\scriptsize -222.58\%}) \\
\bottomrule
\end{tabular}
\label{tab:all-losses-marathi-resnet}
\end{table*}

% Finally, on \textbf{MNIST}, ArcFace and focal variants tend to capture the very
% top accuracy, but harmonic distances (Bray--Curtis, cosine, Mahalanobis,
% Euclidean harmonic) are still present among the top entries and often provide
% the sharpest prototype structure (largest EV, smallest PC90\%) with neutral or
% favorable emissions. Overall, these appendical results confirm that adding
% powerful baselines does not overturn the central message: non-Euclidean
% harmonic losses continue to offer attractive accuracy--interpretability--sustainability
% trade-offs compared to both cross-entropy and a set of widely used alternative
% classification losses.

\newpage

\subsection{Vision: Radar Plots with Additional Losses (MNIST, CIFAR10, CIFAR100)}
Figure~\ref{fig:acc:spider:vision:new:losses} presents an expanded multi-criteria analysis across MNIST, CIFAR-10, and CIFAR-100 using MLP, CNN, ResNet50, and PVT backbones.
This comparison tests whether the advantages previously attributed to harmonic losses persist when measured against widely adopted alternatives for regularization, interpretability, and robustness.

\textbf{RQ1: Model Performance (F1, Test Accuracy).}
Across datasets and architectures, the \textbf{harmonic losses} -- particularly the \emph{cosine-} and \emph{Bray--Curtis–based} variants -- remain the strongest overall performers.
While \textbf{Focal Loss} and \textbf{Label Smoothing} occasionally narrow the gap on more complex datasets such as CIFAR-100, they do not consistently surpass harmonic losses across backbones.
Cosine-based harmonic loss maintains higher accuracy and smoother convergence, especially for CNN and ResNet50, showing greater robustness to data imbalance and optimization noise than either Focal or Confidence Penalty Loss.
Even when Center Loss improves class compactness, it rarely translates into superior end-task accuracy, reinforcing that distance-based formulations bring more balanced generalization benefits.

\textbf{RQ2: Interpretability (PC2 EV, PCA 90\%).}
The advantage of harmonic losses extends beyond performance: \textbf{non-Euclidean harmonics}, especially Bray--Curtis and Chebyshev, consistently yield the most structured latent geometries.
They capture more variance with fewer principal components and align features more distinctly around class prototypes.
Although \textbf{Center Loss} achieves comparable compactness in isolated cases, its representations tend to be less stable across architectures.
\textbf{Label Smoothing} and \textbf{Confidence Penalty} slightly improve feature spread, but their effects remain shallow compared to the systematic geometric alignment achieved by harmonic formulations.
This supports the notion that explicit metric-based geometry is a stronger driver of interpretability than indirect regularization.

\textbf{RQ3: Sustainability (Duration/Epoch, Emissions).}
When considering efficiency, harmonic losses continue to hold their edge.
They achieve competitive or lower CO$_2$ emissions than both Euclidean and cross-entropy baselines. %, with \textbf{cosine} remaining the most sustainable variant.
Among the new baselines, only \textbf{Label Smoothing} approaches similar energy efficiency, while \textbf{Focal Loss} incurs additional computational cost due to its per-sample weighting.
Despite this, none of the conventional alternatives outperform the best-performing harmonic distances on a joint accuracy–emission axis, confirming that the added geometric structure of harmonic loss does not come at a sustainability penalty.

%\textbf{Cross-cutting Insights.}
Three key findings emerge:
i) \textbf{Cosine- and Bray--Curtis–based harmonic losses} remain the most consistently effective across accuracy, interpretability, and sustainability;
ii) \textbf{Conventional regularized losses} such as Focal or Label Smoothing can mitigate specific failure modes (imbalance, overconfidence) but do not achieve the same balance across criteria;
iii) The geometric grounding of harmonic losses continues to provide superior inductive structure, yielding smoother optimization, clearer feature organization, and greener training.
Overall, these results reaffirm the \emph{general dominance and stability of harmonic loss formulations}, even against strong baselines optimized for robustness and interpretability.

%%%%%%%%%%%%%%%%%
% MODEL ON ROWS
%%%%%%%%%%%%%%%%%
\begin{figure*}[h!]
    \begin{centering}
    % MLP
    \includegraphics[page=10,width=0.3\textwidth,trim=45 80 80 15,clip]{figures/Image_Classification/Spider/euclidean_plus_top5_new_losses_added.pdf}
    \includegraphics[page=2,width=0.3\textwidth,trim=45 80 80 15,clip]{figures/Image_Classification/Spider/euclidean_baseline_plus_top4.pdf}
    \includegraphics[page=6,width=0.3\textwidth,trim=45 80 80 15,clip]{figures/Image_Classification/Spider/euclidean_plus_top5_new_losses_added.pdf}
    % CNN 
    \includegraphics[page=9,width=0.3\textwidth,trim=45 80 80 15,clip]{figures/Image_Classification/Spider/euclidean_plus_top5_new_losses_added.pdf}
    \includegraphics[page=1,width=0.3\textwidth,trim=45 80 80 15,clip]{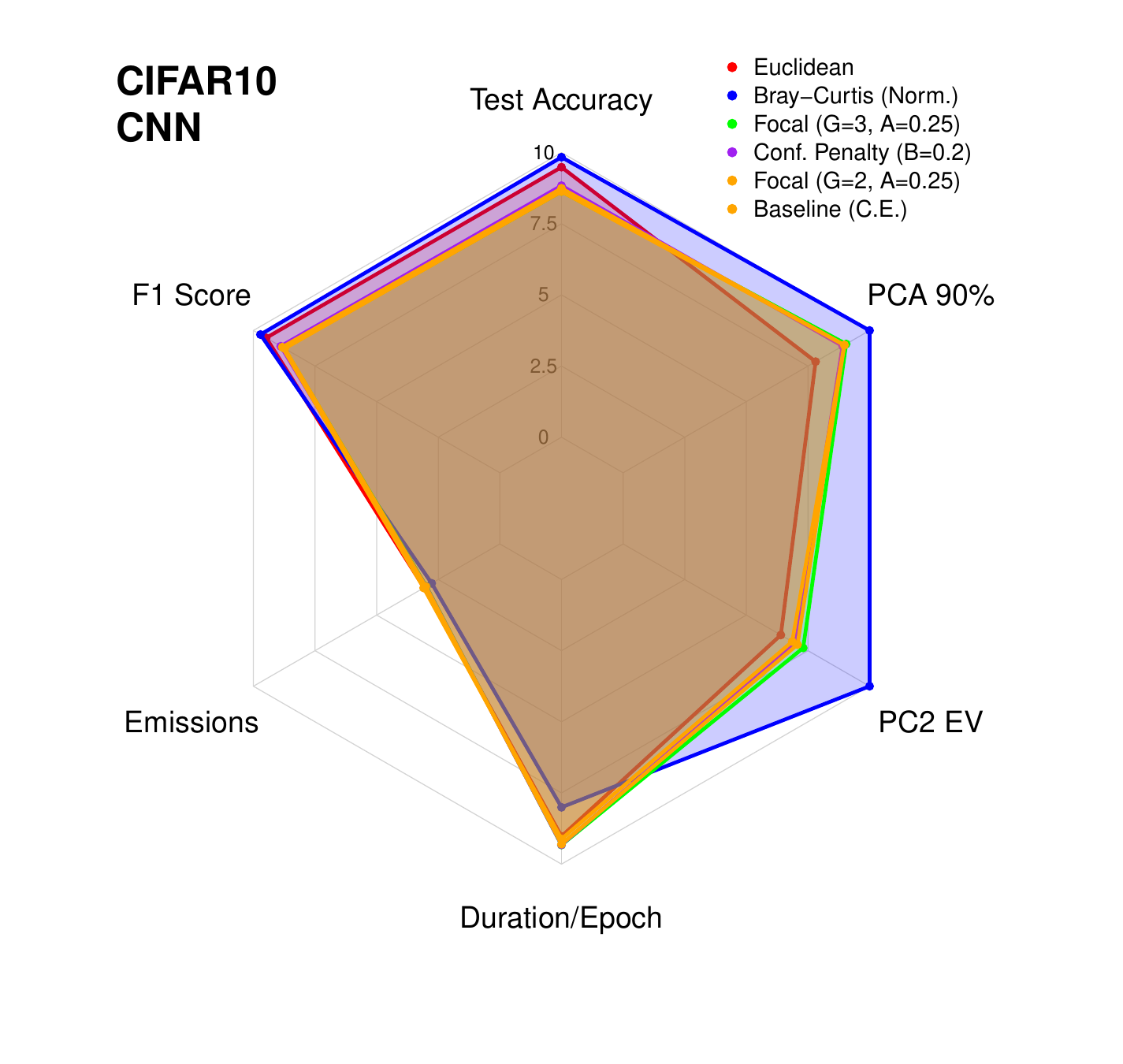}
    \includegraphics[page=5,width=0.3\textwidth,trim=45 80 80 15,clip]{figures/Image_Classification/Spider/euclidean_plus_top5_new_losses_added.pdf}\\
    % RESNET
    \includegraphics[page=12,width=0.3\textwidth,trim=45 80 80 15,clip]{figures/Image_Classification/Spider/euclidean_plus_top5_new_losses_added.pdf} 
    \includegraphics[page=4,width=0.3\textwidth,trim=45 80 80 15,clip]{figures/Image_Classification/Spider/euclidean_plus_top5_new_losses_added.pdf}
    \includegraphics[page=8,width=0.3\textwidth,trim=45 80 80 15,clip]{figures/Image_Classification/Spider/euclidean_plus_top5_new_losses_added.pdf}\\         
    % PVT
	    \includegraphics[page=11,width=0.3\textwidth,trim=45 80 80 15,clip]{figures/Image_Classification/Spider/euclidean_plus_top5_new_losses_added.pdf} 
	    \includegraphics[page=3,width=0.3\textwidth,trim=45 80 80 15,clip]{figures/Image_Classification/Spider/euclidean_plus_top5_new_losses_added.pdf}
	    \includegraphics[page=7,width=0.3\textwidth,trim=45 80 80 15,clip]{figures/Image_Classification/Spider/euclidean_plus_top5_new_losses_added.pdf}\\ 
	        \caption{Vision: Radar plots -- MNIST, CIFAR10, CIFAR100: 1) \textit{Model Performance} (F1, Accuracy); 2) \textit{Interpretability} (PC2 EV, PCA 90\%), and 3) \textit{Sustainability} (Duration/Epoch, Emissions). Plots feature Baseline (Cross-Entropy), Euclidean harmonic, and the four top-performing losses.}
	        \Description[Vision radar plots with additional loss baselines]{A grid of radar (spider) charts for MNIST, CIFAR-10, and CIFAR-100 vision experiments. Each radar compares multiple training objectives across three metric families: model performance (F1 and accuracy), interpretability (principal-component explained variance metrics), and sustainability (training duration per epoch and CO2 emissions). The overlaid shapes represent the cross-entropy baseline, Euclidean harmonic loss, and the top-performing loss variants (including additional baseline losses), highlighting how different objectives trade off accuracy, interpretability, and emissions.}
	        \label{fig:acc:spider:vision:new:losses}
	    \end{centering}
\end{figure*}

\begin{figure*}[h!]
    \begin{centering}
    \includegraphics[page=14,width=0.3\textwidth,trim=45 80 80 15,clip]{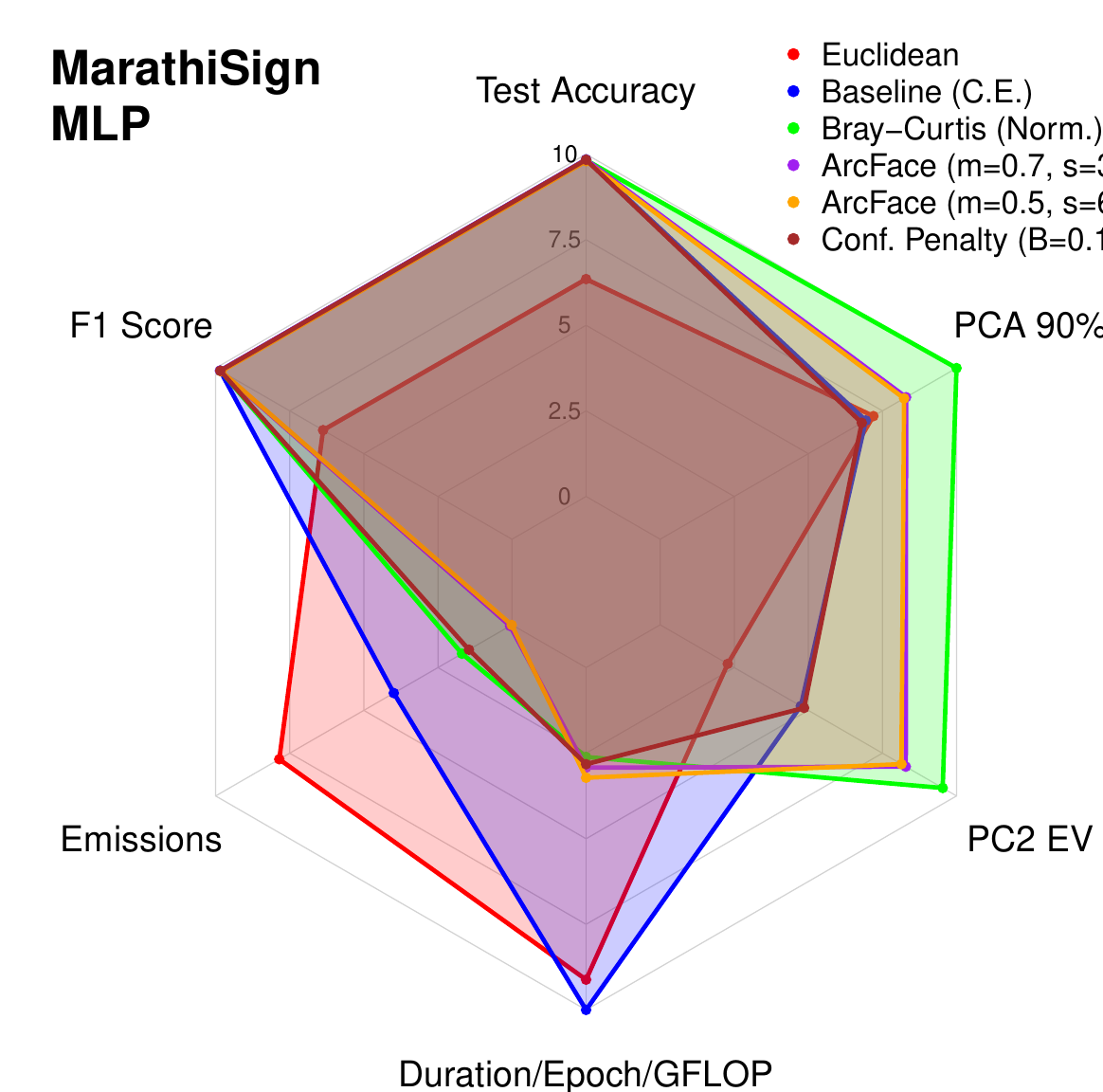}
    \includegraphics[page=18,width=0.3\textwidth,trim=45 80 80 15,clip]{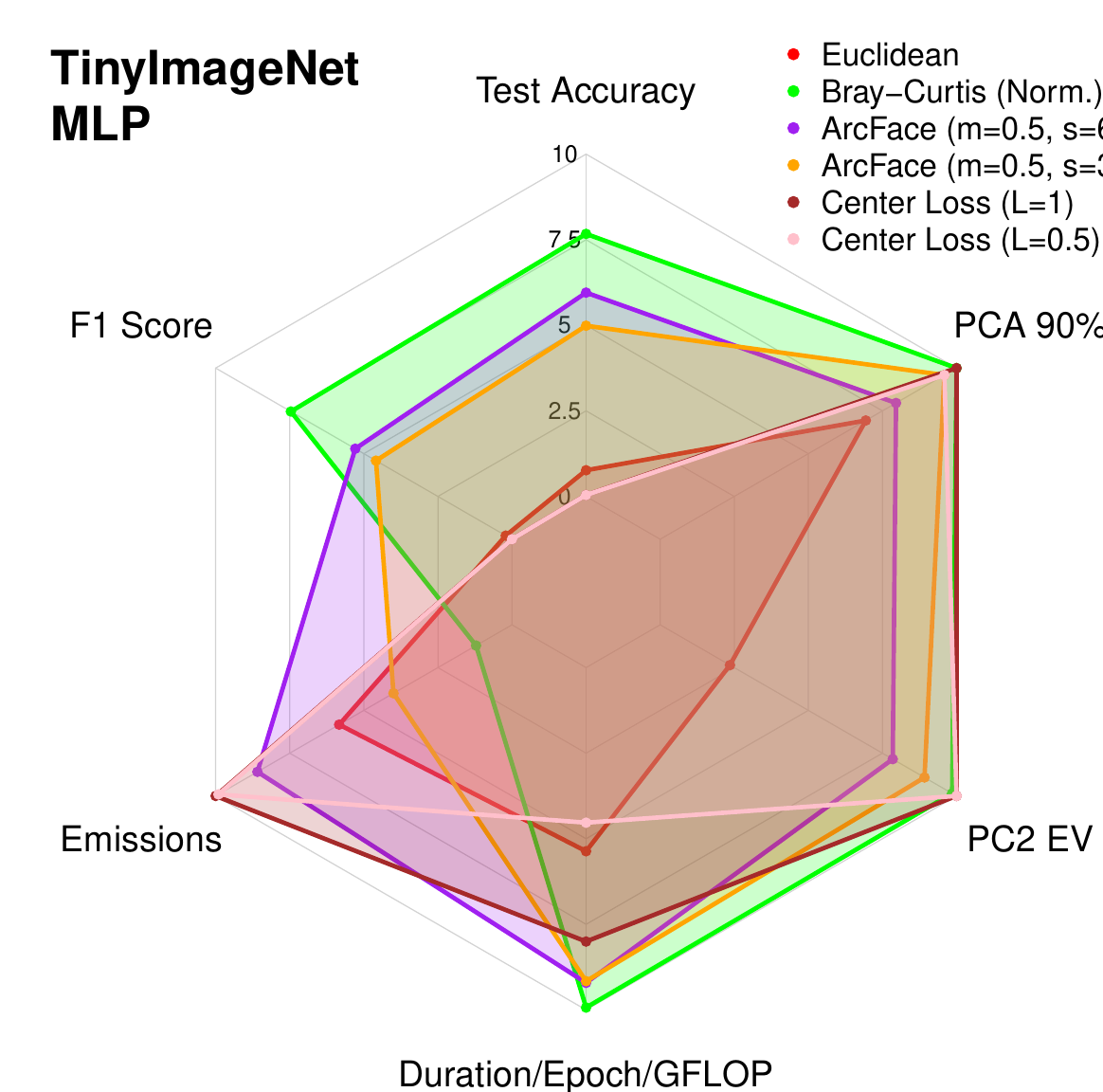}\\
    \includegraphics[page=13,width=0.3\textwidth,trim=45 80 80 15,clip]{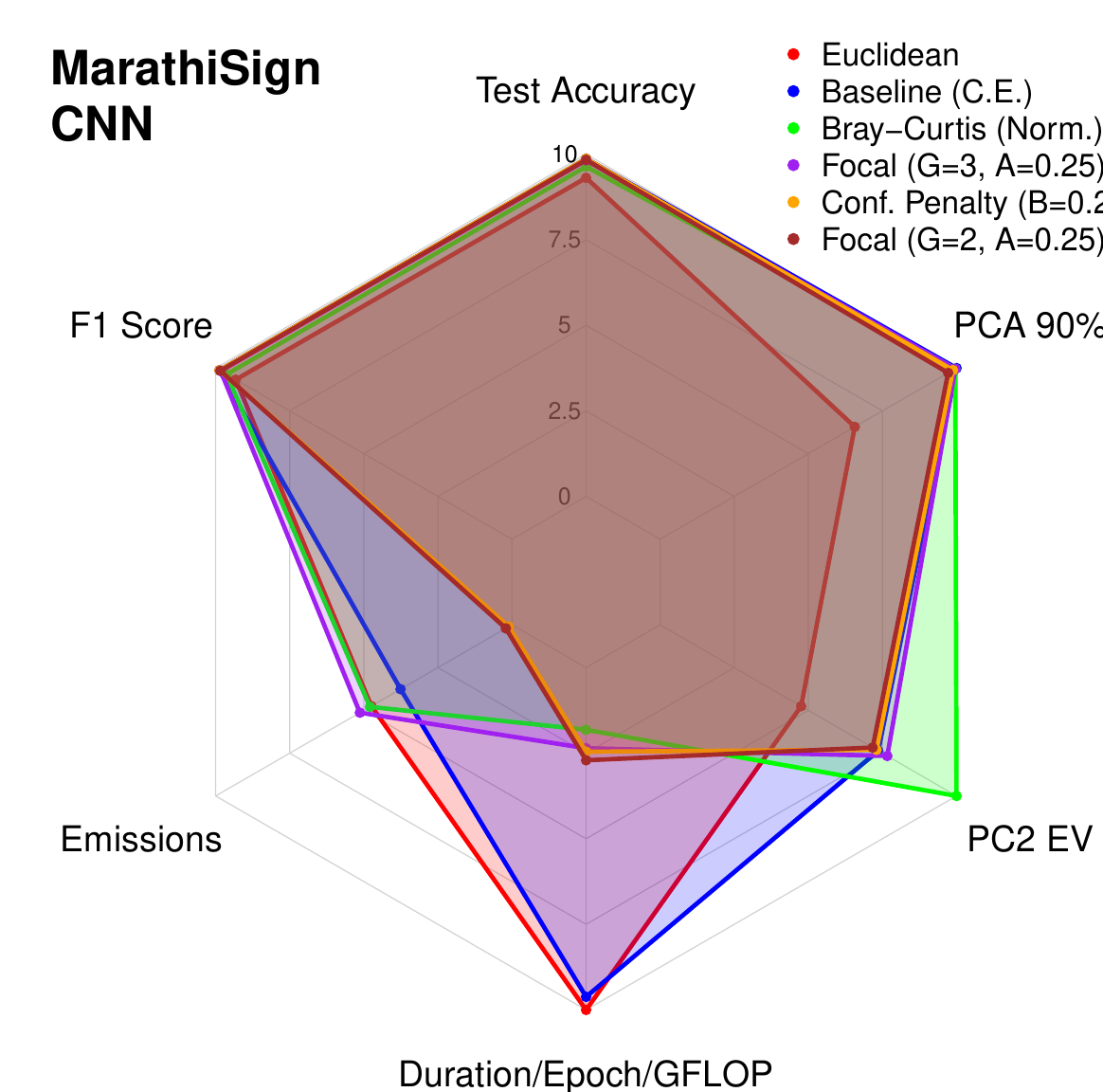}
    \includegraphics[page=17,width=0.3\textwidth,trim=45 80 80 15,clip]{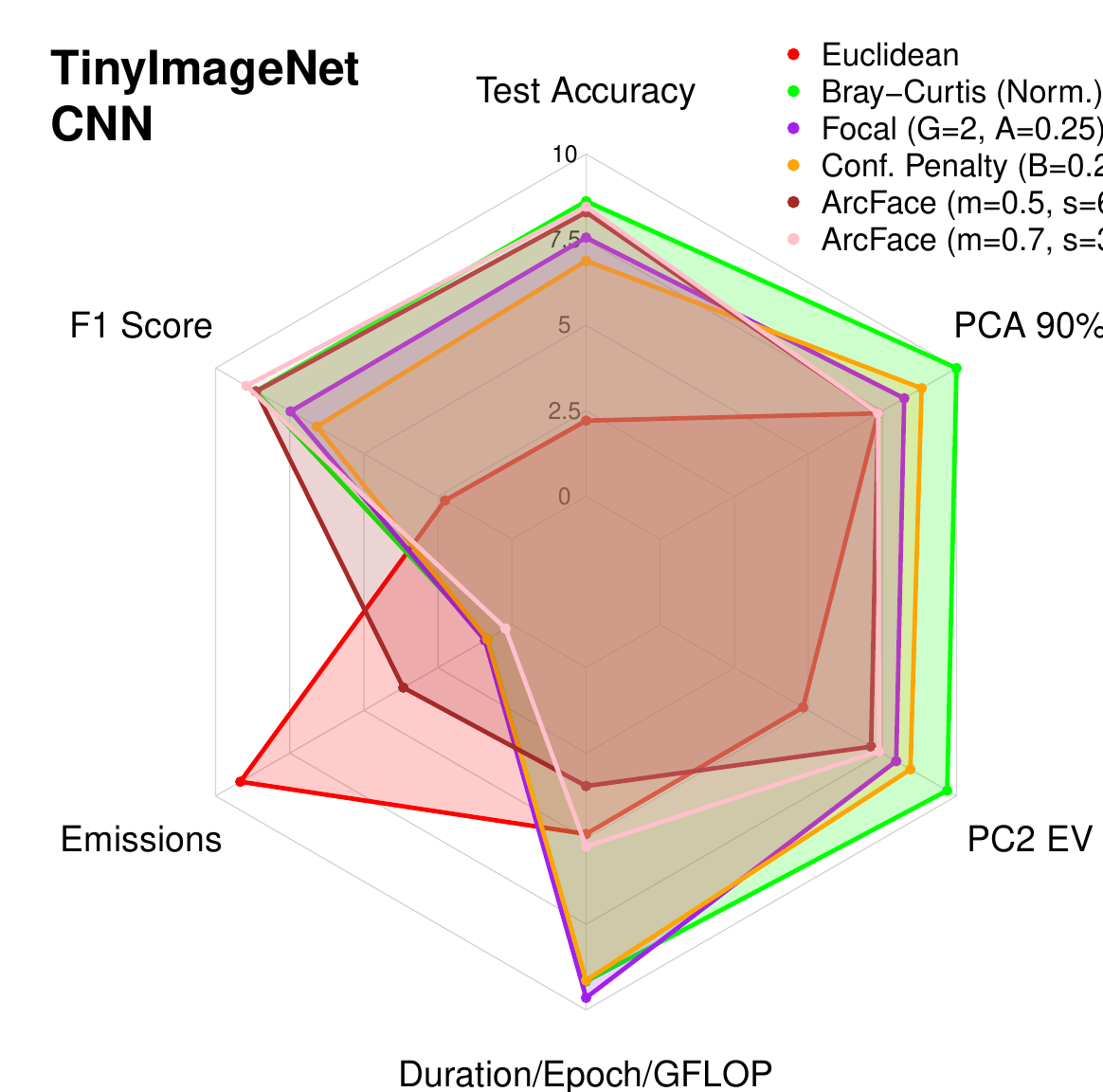} \\
	    \includegraphics[page=16,width=0.3\textwidth,trim=45 80 80 15,clip]{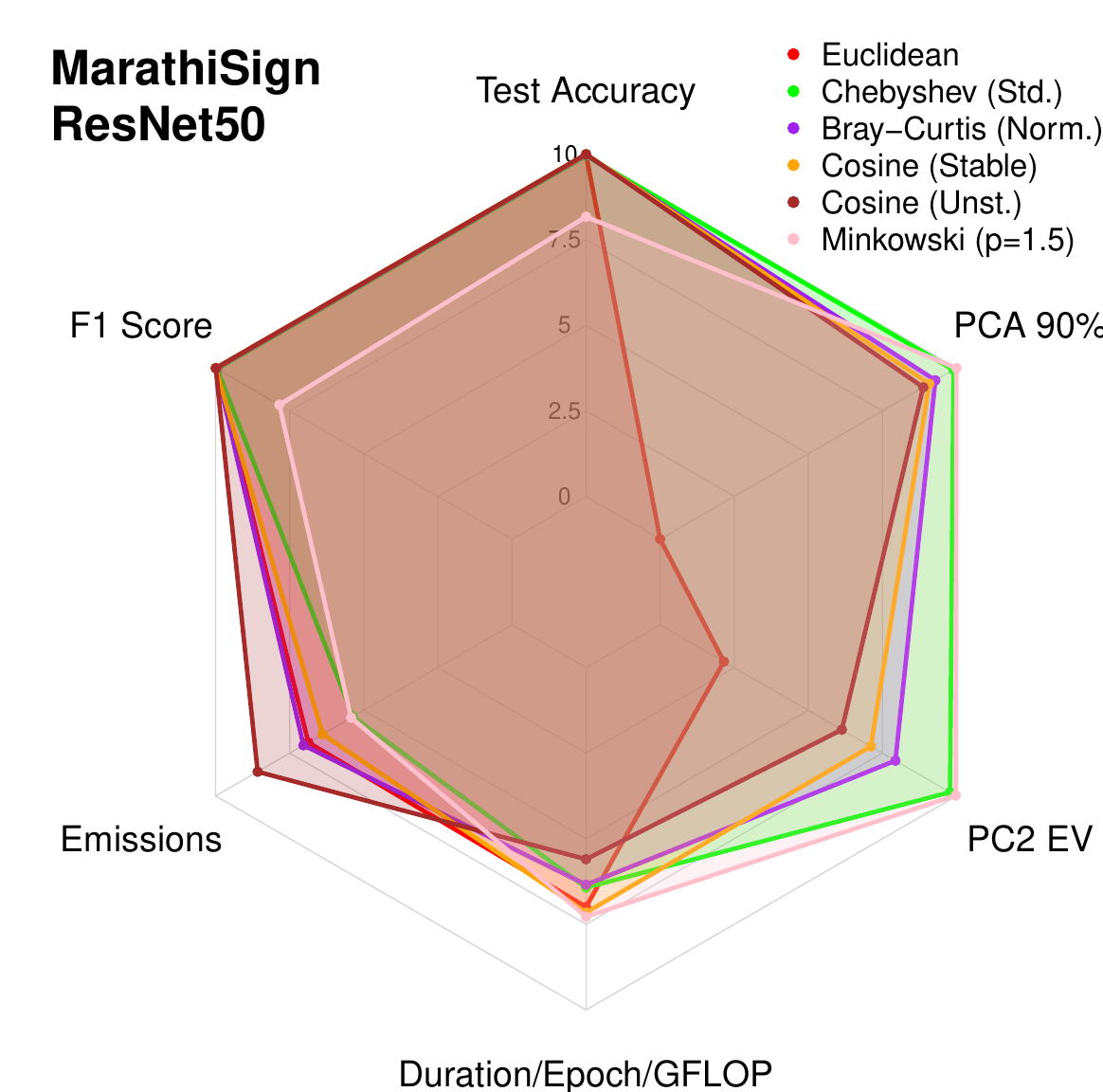}
	    \includegraphics[page=20,width=0.3\textwidth,trim=45 80 80 15,clip]{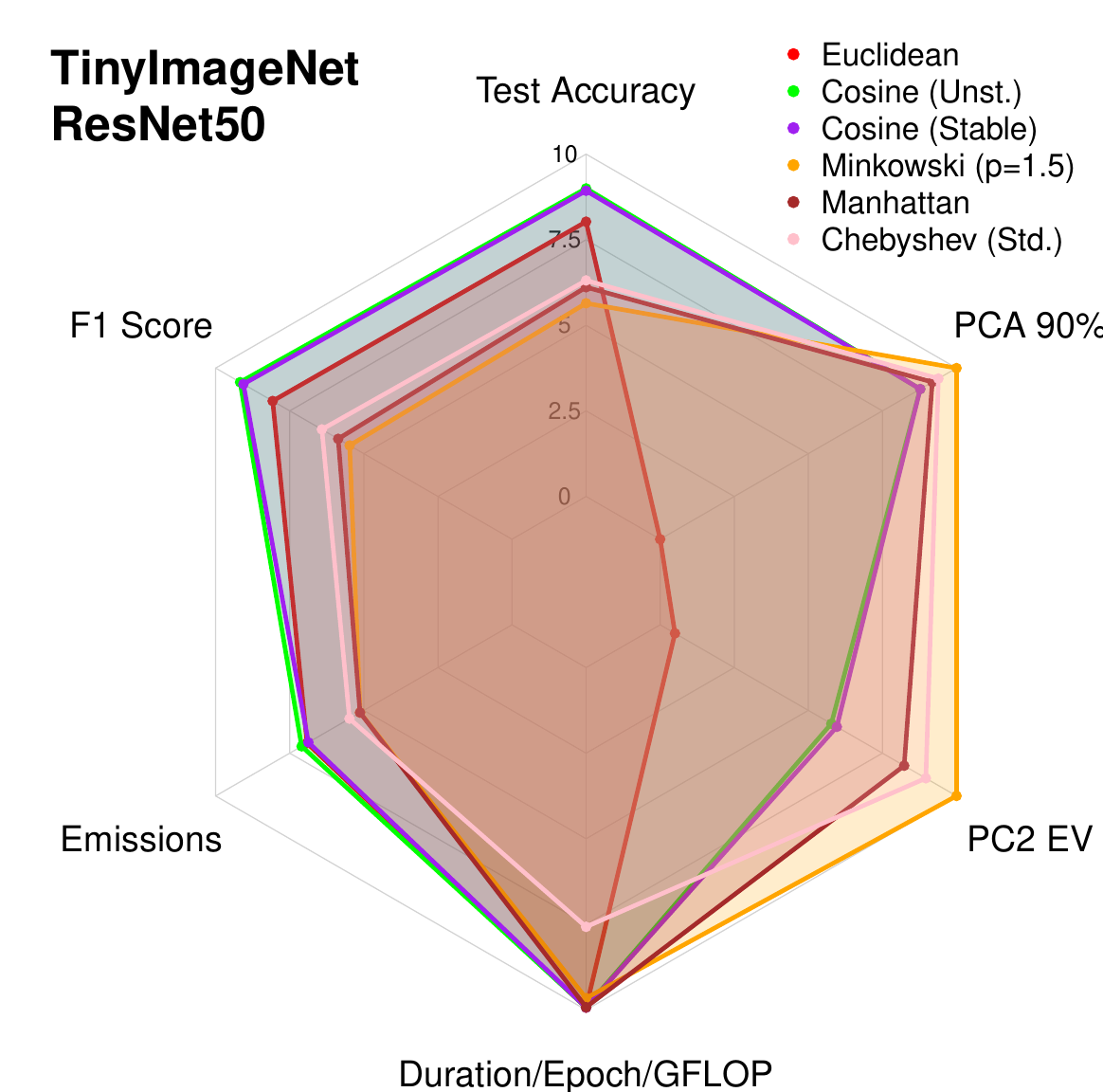} \\
	    \includegraphics[page=15,width=0.3\textwidth,trim=45 80 80 15,clip]{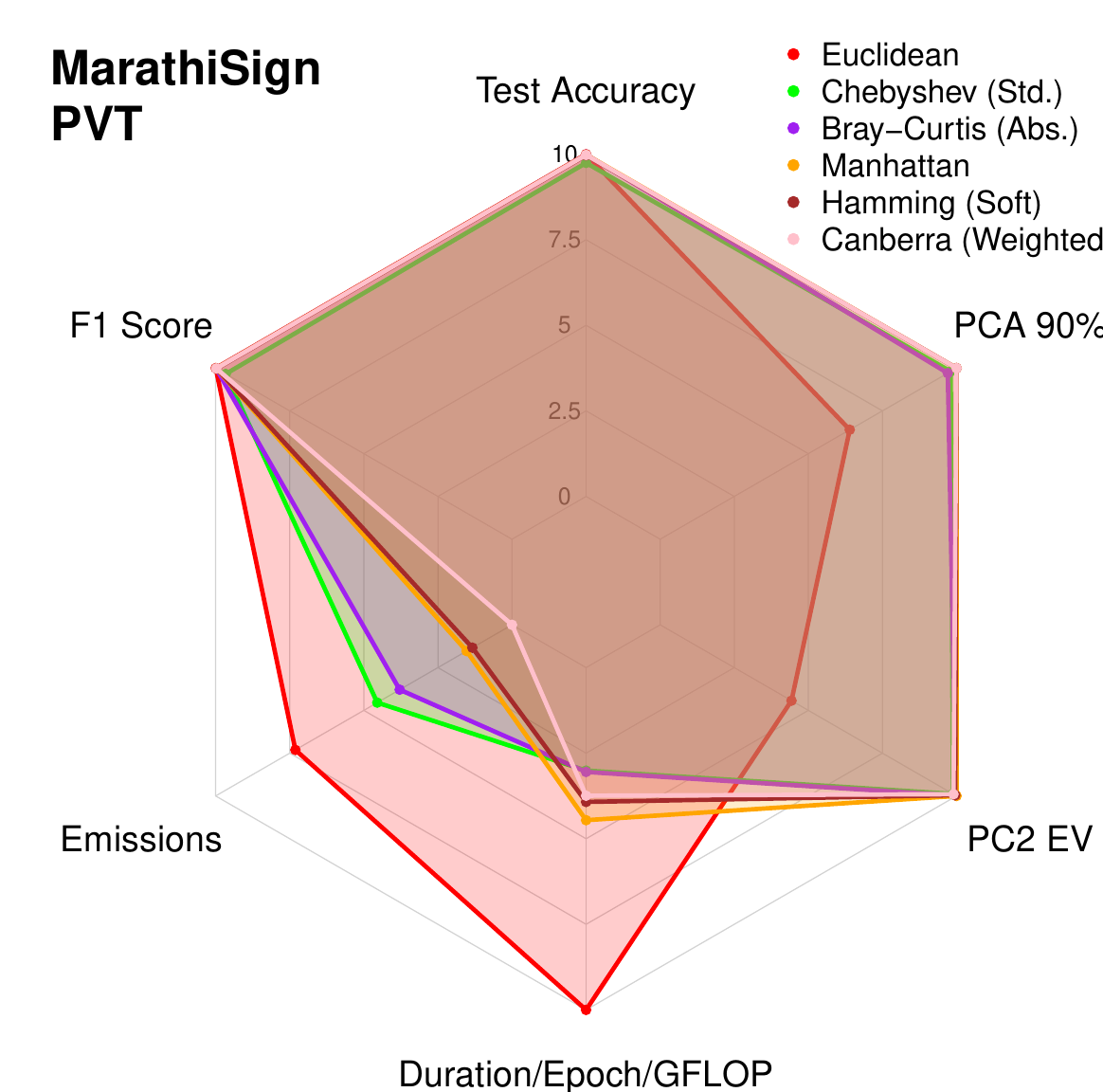}
	    \includegraphics[page=19,width=0.3\textwidth,trim=45 80 80 15,clip]{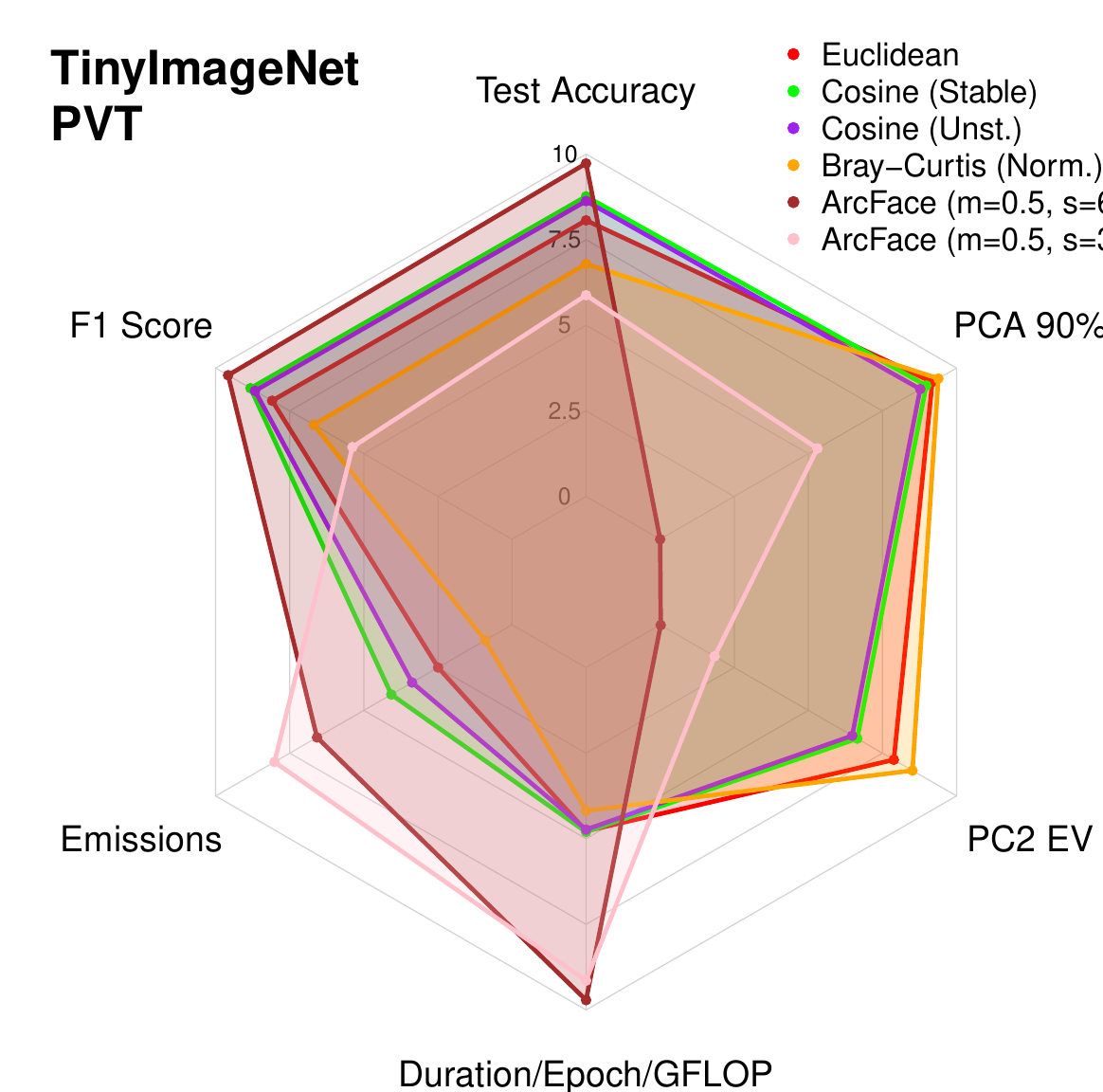}
	    \caption{Vision: Radar plots -- Marathi Sign Language, TinyImageNet: 1) \textit{Model Performance} (F1, Accuracy); 2) \textit{Interpretability} (PC2 EV, PCA 90\%), and 3) \textit{Sustainability} (Duration/Epoch, Emissions). Plots feature Baseline (Cross-Entropy), Euclidean harmonic, and the four top-performing losses.}
	        \Description[Vision radar plots for Marathi Sign and TinyImageNet]{A set of radar (spider) charts summarizing results on Marathi Sign Language and TinyImageNet. Each radar plot includes axes for performance (F1 and accuracy), interpretability (explained-variance metrics from PCA), and sustainability (training duration per epoch and CO2 emissions). Multiple overlaid polygons compare the cross-entropy baseline, Euclidean harmonic loss, and top-performing loss variants, enabling visual comparison of trade-offs across these criteria.}
	        \label{fig:acc:spider:vision:new:losses:part:2}
	    \end{centering}
\end{figure*}
}

{
\clearpage
\newpage
\subsection{Vision: Radar Plots with Additional Losses (Marathi Sign Language, TinyImageNet)}
Figure~\ref{fig:acc:spider:vision:new:losses:part:2} extends the radar analysis to
higher–resolution benchmarks (Marathi Sign Language, TinyImageNet) and augments
the comparison set with strong loss baselines such as Focal Loss, ArcFace, Center
Loss, and Confidence Penalty.  % As before, each radar summarizes 
%\emph{Model Performance} (Accuracy, F1), \emph{Interpretability}
%(PC2~EV, PCA~90\%), and \emph{Sustainability} (Duration/Epoch/GFLOPs, Emissions).
%When a baseline loss does not appear in the legend for a given backbone, it means
%it failed to enter the top–8 methods, i.e., it was dominated by harmonic configurations
%on the combined criteria.

\textbf{RQ1: Model Performance (Accuracy, F1).}
On Marathi Sign, the added losses make MLP and CNN particularly competitive:
ArcFace and Focal occasionally obtain strong accuracy/F1, yet
non–Euclidean harmonic losses remain among the best methods.
For MLP, \textbf{Bray--Curtis (Normalized)} tracks or exceeds both
cross entropy and ArcFace while preserving a smooth performance profile.
For CNN, \textbf{Bray--Curtis (Normalized)} and \textbf{cosine (stable/unstable)}
consistently occupy the top accuracy/F1 slices; Focal and Confidence Penalty
are competitive but never clearly dominate.
On deeper backbones, the picture is even clearer: for
\textbf{ResNet50} and \textbf{PVT} on Marathi Sign,
all top–performing methods are harmonic losses, indicating that distance–based harmonic heads
outperform alternative losses outright in this regime.

On TinyImageNet, a harder and more fine–grained benchmark, a similar pattern
emerges.  For MLP, ArcFace and Center Loss join \textbf{Bray--Curtis (Normalized)}
and Euclidean in the top–performing losses, but Bray--Curtis remains competitive in accuracy while providing different geometric and sustainability properties.
For CNN, the strongest Focal and ArcFace variants reach high F1, yet
\textbf{Bray--Curtis (Normalized)} again sits near the performance frontier.
On \textbf{ResNet50}, the top–performing losses are \emph{entirely harmonic} (cosine, Minkowski,
Chebyshev, Euclidean), and on \textbf{PVT} TinyImageNet the leaders are dominated
by cosine and Bray--Curtis, with ArcFace appearing only as an alternative angular
baseline.  Overall, even in the presence of sophisticated angular–margin and
confidence–shaping losses, non–Euclidean harmonic heads remain on or very near
the performance Pareto frontier.

\textbf{RQ2: Interpretability (PC2~EV, PCA~90\%).}
The higher–resolution datasets accentuate differences in representation geometry.
On Marathi Sign, harmonic distances such as
\textbf{Bray--Curtis (Normalized/Absolute)}, \textbf{Chebyshev (Standard)},
and \textbf{Canberra/Hamming} for PVT yield the strongest PCA structure:
they maximize PC2 explained variance and minimize the number of components
required to reach $90\%$ EV, indicating compact, prototype–aligned embeddings.
ArcFace and Focal improve angular separation but generally do not achieve
the same variance concentration as the best harmonic distances.

TinyImageNet confirms this trend.  On MLP and CNN, Bray--Curtis and Chebyshev
produce markedly higher PC2~EV and lower PCA~90\% dimensionality than Euclidean
and most additional baselines, including Center Loss and Focal.  For ResNet50 and
PVT, \textbf{cosine} and \textbf{Bray--Curtis} continue to enlarge the PCA wedges
relative to Euclidean, whereas ArcFace’s contribution is mainly on performance
rather than on variance concentration.  Thus, across both Marathi Sign and
TinyImageNet, the most interpretable geometries are consistently induced by
non–Euclidean harmonic losses rather than by the newly added baselines.

\textbf{RQ3: Sustainability (Duration/Epoch/GFLOPs, Emissions).}
The sustainability axes show that richer loss design does not necessarily
translate into greener training.  On Marathi Sign MLP/CNN, several harmonic
distances (e.g., \textbf{Bray--Curtis (Normalized)}, \textbf{Chebyshev}) attain
\emph{equal or lower} normalized Duration/Epoch/GFLOPs and emissions than
cross–entropy and the added baselines; Focal and ArcFace occasionally incur
slightly higher emissions due to their sharper gradients and additional
computations.  For ResNet50 and PVT, where backbone FLOPs dominate, all harmonic
variants remain sustainability–competitive. % , and no additional loss enters the
%top–8.

On TinyImageNet, the pattern persists.  ArcFace and Focal may match harmonic
losses in accuracy, but they usually do so with similar or higher emissions.
Cosine and Bray--Curtis heads on ResNet50 and PVT often achieve comparable or
better emissions than Euclidean, while still improving representation structure.
Center Loss introduces modest overhead but does not surpass harmonic distances in
overall sustainability.

%\textbf{Summary.}
Across Marathi Sign and TinyImageNet, adding strong baselines such as ArcFace,
Focal, Center Loss, and Confidence Penalty \emph{does not displace} the
non–Euclidean harmonic losses from the top tier.  Whenever these baselines are
competitive in accuracy, harmonic distances typically offer superior
interpretability and comparable or lower emissions. %; on deeper backbones they
%are more often outperformed  do not enter the top–8 at all.  
This reinforces our central claim that distance–tailored harmonic heads provide a robust, geometry–aware alternative to contemporary loss designs, remaining competitive or superior across performance, structure, and sustainability, even on challenging high–resolution
vision benchmarks.
}

\subsection{Vision: Aggregated Emissions}
\label{sec:vision:sustainability}

Figure~\ref{fig:emissions:cumulative} reports the cumulative CO$_2$ emissions
for all vision experiments (MNIST, CIFAR-10/100, Marathi Sign, and
TinyImageNet), expressed as the difference in grams of CO$_2$ relative to the
cross–entropy baseline (total CE emissions $=650.49$\,gCO$_2$eq over 680
runs).  All methods lie within a band of roughly $\pm 8\%$ of this baseline,
showing that changing the distance in the harmonic head affects emissions in a
controlled rather than catastrophic.

\paragraph{Harmonic losses remain competitive or greener.}
The most sustainable region of the plot is dominated by non–Euclidean
harmonic losses.  In particular, \textbf{Bray--Curtis (Normalized)},
\textbf{Bray--Curtis (Absolute)}, \textbf{Canberra (Weighted)}, and
\textbf{Mahalanobis (Cholesky)} consistently achieve \emph{lower} cumulative
emissions than cross–entropy, even after adding the more demanding Marathi
Sign and TinyImageNet settings.  Cosine variants (\emph{Cosine (Stable)} and
\emph{Cosine (Unstable)}) and Euclidean harmonic loss sit very close to the
baseline, indicating that distance–based heads introduce essentially no
sustainability penalty while still improving accuracy and representation
structure.

\paragraph{Behavior of additional baselines.}
Among the newly added conventional baselines, \textbf{Label Smoothing} and
\textbf{Confidence Penalty} occupy the middle of the spectrum: their
emissions are comparable to, but generally not better than, those of the
best harmonic distances.  In contrast, more aggressive objectives such as
\textbf{Focal Loss} and large–margin \textbf{Center Loss} variants tend to
cluster on the higher–emission side, reflecting the extra computation and
slower convergence induced by power–scaled gradients and auxiliary center
updates.  ArcFace configurations behave similarly: moderate settings can be
near baseline, but high margin/scale choices increase emissions relative to
the most efficient harmonic distances.

Across all four datasets and backbones, the qualitative picture is stable.
Non–Euclidean harmonic losses provide some of the \emph{greenest} options,
often achieving lower or baseline–level emissions while simultaneously
improving accuracy and interpretability.  The main exception is
\textbf{Mahalanobis (Standard)}, which remains the least sustainable
configuration due to its covariance estimation cost—consistent with our
earlier observation that Mahalanobis emphasizes representation clarity at a
computational price.  Overall, the expanded analysis confirms that
distance choice in the harmonic head materially affects the carbon
footprint of training, and that carefully chosen non–Euclidean geometries
(e.g., Bray--Curtis, Canberra, cosine) offer a favorable trade-off between 
performance, interpretability, and sustainability compared to both
Euclidean harmonic loss and modern regularized baselines.

%\begin{wrapfigure}{r}{0.5\textwidth}
\begin{figure}[h!]
\begin{centering}
\includegraphics[width=0.48\textwidth,trim=0 20 7 50,clip]{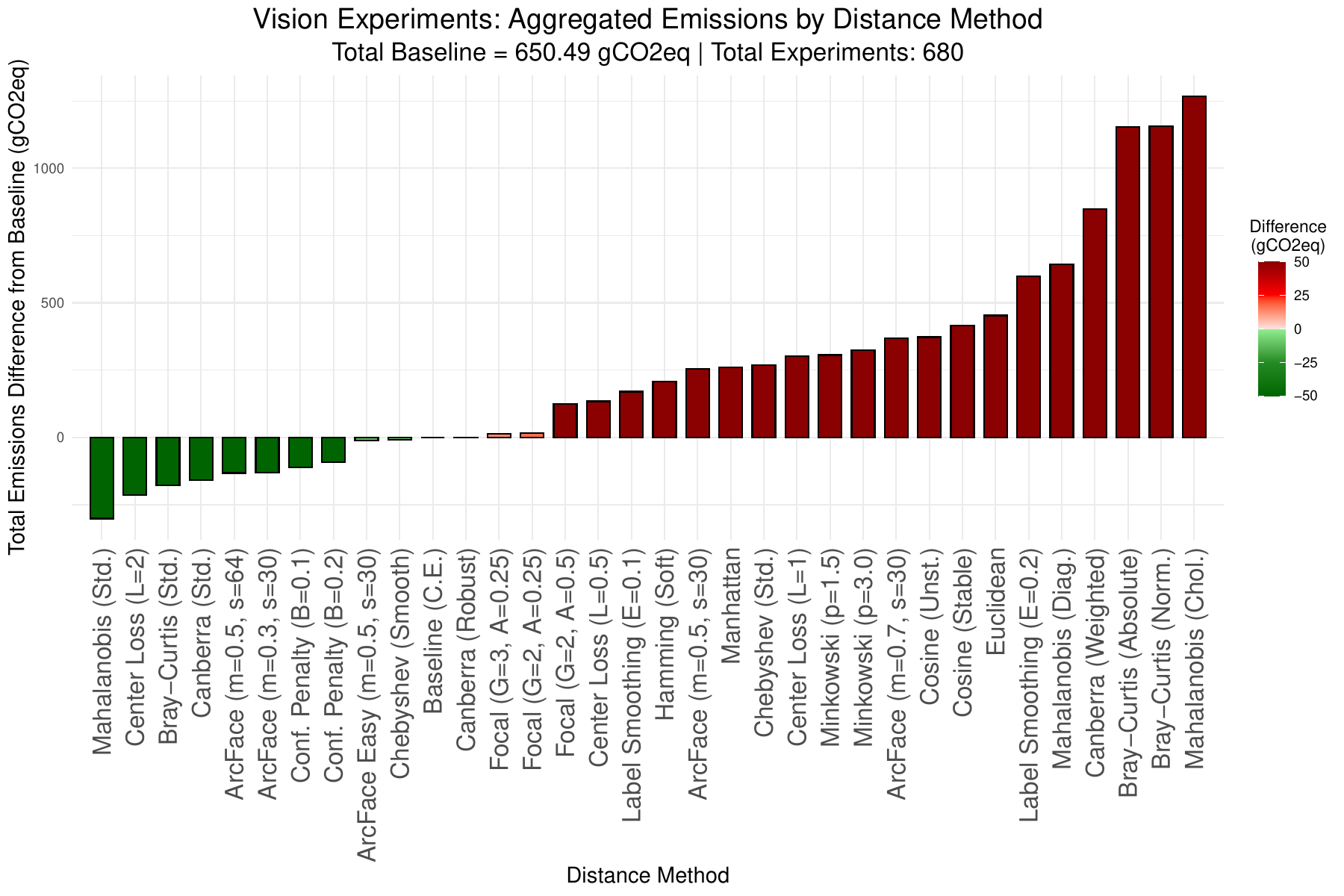}
\caption{Vision: Emissions Averaged Across Seeds and Aggregated Over all 12 Model Backbones.}
\Description[Aggregated vision emissions across losses]{A single summary plot aggregating carbon-emission differences for vision experiments across multiple datasets and 12 model backbones. The graphic groups results by loss function and shows how each distance-based harmonic loss variant changes estimated CO2 emissions relative to the cross-entropy baseline when averaged across random seeds. The overall purpose is to highlight which losses tend to be more carbon-efficient (negative deltas) or less efficient (positive deltas) when aggregated over many configurations.}
\label{fig:emissions:cumulative}
\end{centering}
\end{figure}
%\end{wrapfigure}
Figure~ \ref{fig:emissions:cumulative} reports cumulative emission \emph{differences} (gCO$_2$eq) for each custom loss function across all 12 model/dataset combinations.  (Total Baseline $=181.2$ gCO$_2$eq).
%
%We identify three patterns. 
%by distance across all 228 (without new loss) vision experiments (348 with new loss functions) 

\textbf{Lower-than-baseline emissions}: 
\emph{Mahalanobis (Standard)} shows the largest \emph{positive} delta, indicating consistently lower emissions; \emph{Bray--Curtis (Standard)} and \emph{Cosine (Unstable)} also sit on the positive side, with \emph{Canberra (Standard)} and \emph{Cosine (Stable)} slightly above zero. Euclidean and Manhattan are close to baseline. 
Of the new baseline loss functions introduced, Confidence Penalty performs on par with Cosine (Unstable) and the most efficiently compared to its counterparts. Almost all new losses are more efficient than Cross Entropy Loss, with varying degrees of success. Other distances are characterized by higher emissions, as shown by the red cluster. %distances 
%\textbf{High emissions cluster:} 
%\emph{Bray--Curtis (Normalized)} and \emph{Hamming (Soft)}  contribute a \emph{negative} aggregate delta (net savings vs.\ baseline). \emph{Canberra (Weighted)}, \emph{Mahalanobis (Diagonal)}, and the \emph{Minkowski} variants ($p{=}1.5,3.0$) also lean negative, while both \emph{Chebyshev} variants and \emph{Mahalanobis (Cholesky)} show some of the strongest cumulative reductions. 
%\textbf{iii) Takeaway.} 
Results reinforce that non-Euclidean harmonic losses can be more sustainable than their Euclidean counterpart, and that the choice of distance materially affects the carbon footprint of model training.  
%If emissions are the priority, \emph{Bray--Curtis (Normalized)}, \emph{Hamming (Soft)}, \emph{Minkowski}, and \emph{Chebyshev}/\emph{Mahalanobis (Cholesky)} are generally greener. % if accuracy or other criteria favor \emph{Cosine} or \emph{Bray--Curtis (Standard)}, the results suggest budgeting for a modest emissions premium relative to the cross-entropy baseline.

{

\section{Convergence analysis}
\label{sec:convergence}
% Figures~\ref{fig:convergence-loss-resnet-bray-canberra}--        \ref{fig:convergence-loss-resnet-manhattan-mahlanobis}

\textbf{Vision:}
Figure~\ref{fig:convergence-loss-all}
 reports the training and
validation loss trajectories for PVT and ResNet50 across all datasets
and all non-Euclidean harmonic losses.  
A key concern is whether distances that introduce
nontrivial geometric structure  such as cosine and Mahalanobis lead to unstable optimization or distorted convergence landscapes. Empirically, we observe no such issues.

Across MNIST, CIFAR-10, CIFAR-100, and Marathi Sign,  
all non--Euclidean harmonic losses exhibit smooth, monotonic
decrease in the training objective and stable validation trends, with no
oscillation, divergence, or gradient explosion.  
Even distances with stronger geometric bias (e.g., Mahalanobis, Chebyshev,
Bray-Curtis) converge at rates comparable to or faster than Euclidean harmonic
loss.  
Cosine variants in particular show the
\emph{fastest early descent}, followed by steady tightening of the validation curves, consistent with their
angularly flatter basins.

Notably, none of the distances introduce optimization barriers,
despite their differing curvature properties.  
Mahalanobis maintains stable descent even though anisotropic
curvature could, in principle, yield direction-dependent gradients.
Likewise, Canberra, Hamming, Manhattan, and Minkowski losses converge smoothly,
indicating that the harmonic formulation effectively normalizes distance
geometry into a well-conditioned optimization surface.

Overall, the loss curves demonstrate that the harmonic link function
absorbs geometric variability and translates heterogeneous distance
metrics into similarly well-behaved training dynamics.  
This provides experimental evidence that alternative geometries do not impair
convergence nor destabilize class separation boundaries.

\textbf{Language:} Figure~\ref{fig:convergence-llm} reports training/validation loss, training
accuracy, and (for GPT--2B) training and validation perplexity for
cross--entropy, Euclidean harmonic, and Minkowski ($p{=}2$) heads across
BERT--0.1B, GPT--0.1B, GPT--2B, and QWEN2--0.5B.
Across all architectures, the distance--based harmonic losses exhibit
smooth optimization dynamics: losses decrease
monotonically with no oscillatory or unstable regimes, and accuracy curves
increase steadily towards a plateau.

For BERT--0.1B, Euclidean and Minkowski harmonic losses reduce both
training and validation loss more quickly than cross--entropy and converge
to a lower plateau, while achieving higher final training accuracy.  GPT--0.1B
shows a similar pattern: all three heads converge, but the harmonic
variants reach a given accuracy earlier and with gently sloping curves,
indicating stable gradients.
For GPT--2B and QWEN2--0.5B, the three heads track each other closely in
both loss and accuracy, confirming that the change of geometry does not
impede convergence even at larger scale.
The validation loss curves mirror the training behaviour: no divergence or
late--stage degradation is observed for any harmonic configuration.

The GPT--2B perplexity panel further corroborates this picture.
Training and validation perplexity decrease rapidly and stabilize to
comparable levels for all heads; the harmonic variants sometimes achieve
slightly faster early reductions, but do not introduce pathological
behaviour.  Overall, these results show that replacing the linear
classifier with a distance--based harmonic head preserves, and in some cases marginally improves, the convergence properties
of standard cross--entropy while enabling the geometric and
interpretability benefits discussed in the main paper.

% NEW PLOT

\begin{figure*}[h!]
    \begin{centering}
	    \includegraphics[width=0.99\textwidth]{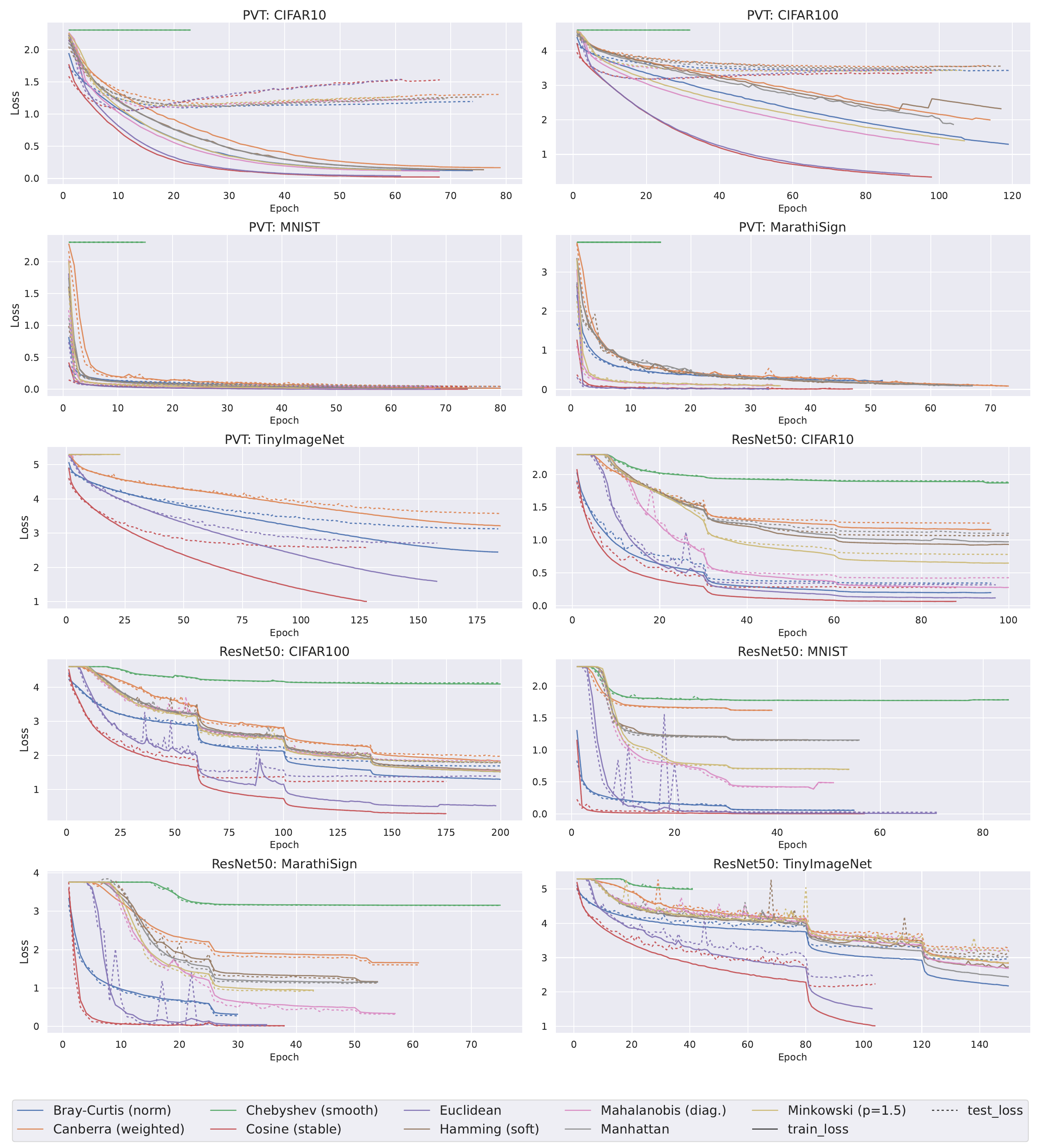}
	    \caption{Loss convergence behavior with PVT and ResNet50: Training and Validation loss across all datasets with different non-Euclidean harmonic losses.}
	        \Description[Vision-model loss convergence curves]{A multi-panel convergence figure showing training and validation loss trajectories for PVT and ResNet50 across several vision datasets. Each panel plots loss versus training progress and overlays curves for different non-Euclidean distance choices used in the harmonic loss. The figure is intended to show whether different distance metrics change convergence speed or stability, by comparing the shapes and smoothness of the loss curves across methods.}
	        \label{fig:convergence-loss-all}
	    \end{centering}
\end{figure*}

\begin{figure*}[h!]
    \begin{centering}
    \includegraphics[width=0.49\textwidth,page=1]{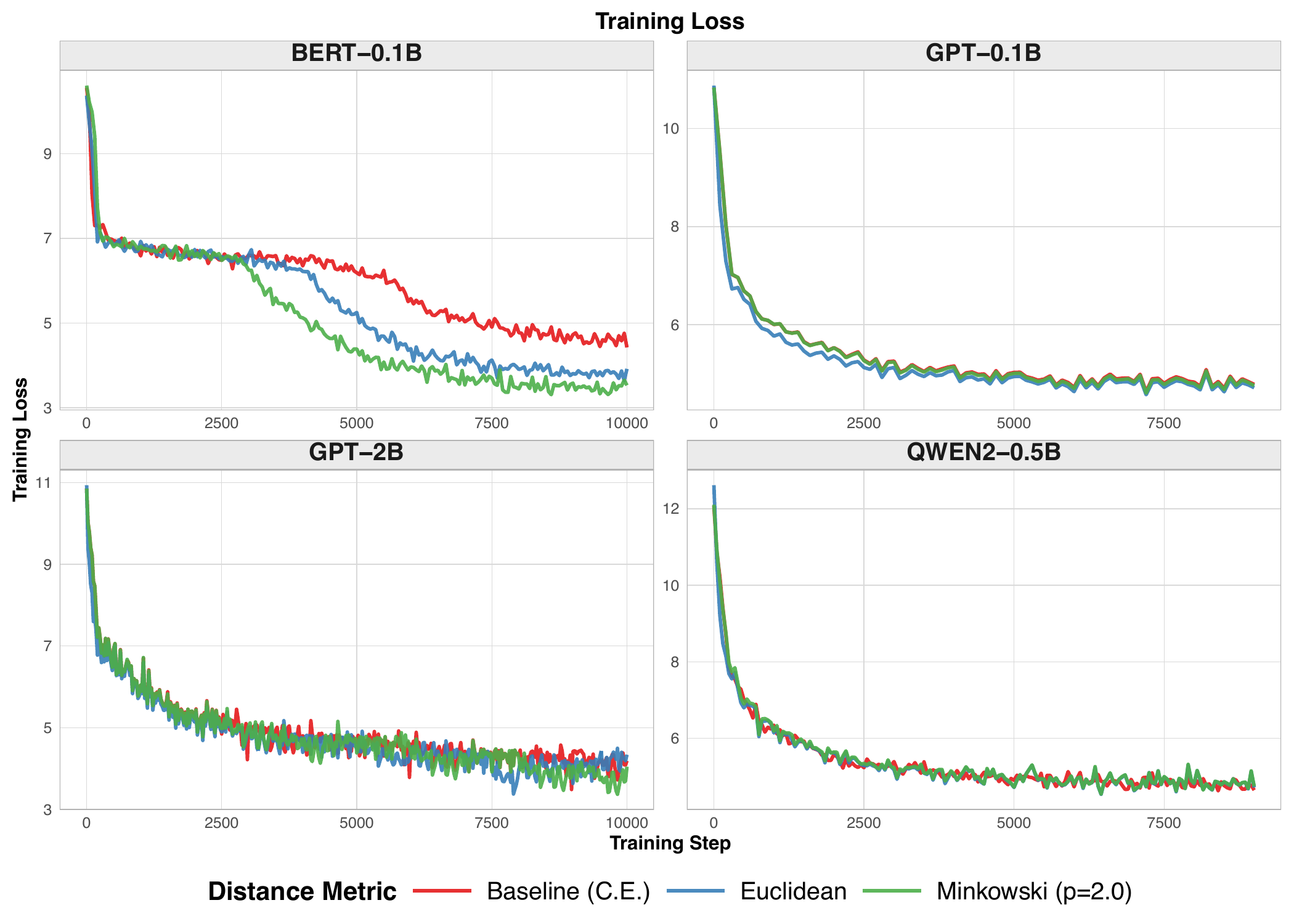}
	    \includegraphics[width=0.49\textwidth,page=2]{figures/LLMs/llm_distance_comparison.pdf}
	    \includegraphics[width=0.49\textwidth,page=3]{figures/LLMs/llm_distance_comparison.pdf}
	    \includegraphics[width=0.49\textwidth,trim=5 5 5 10,clip]{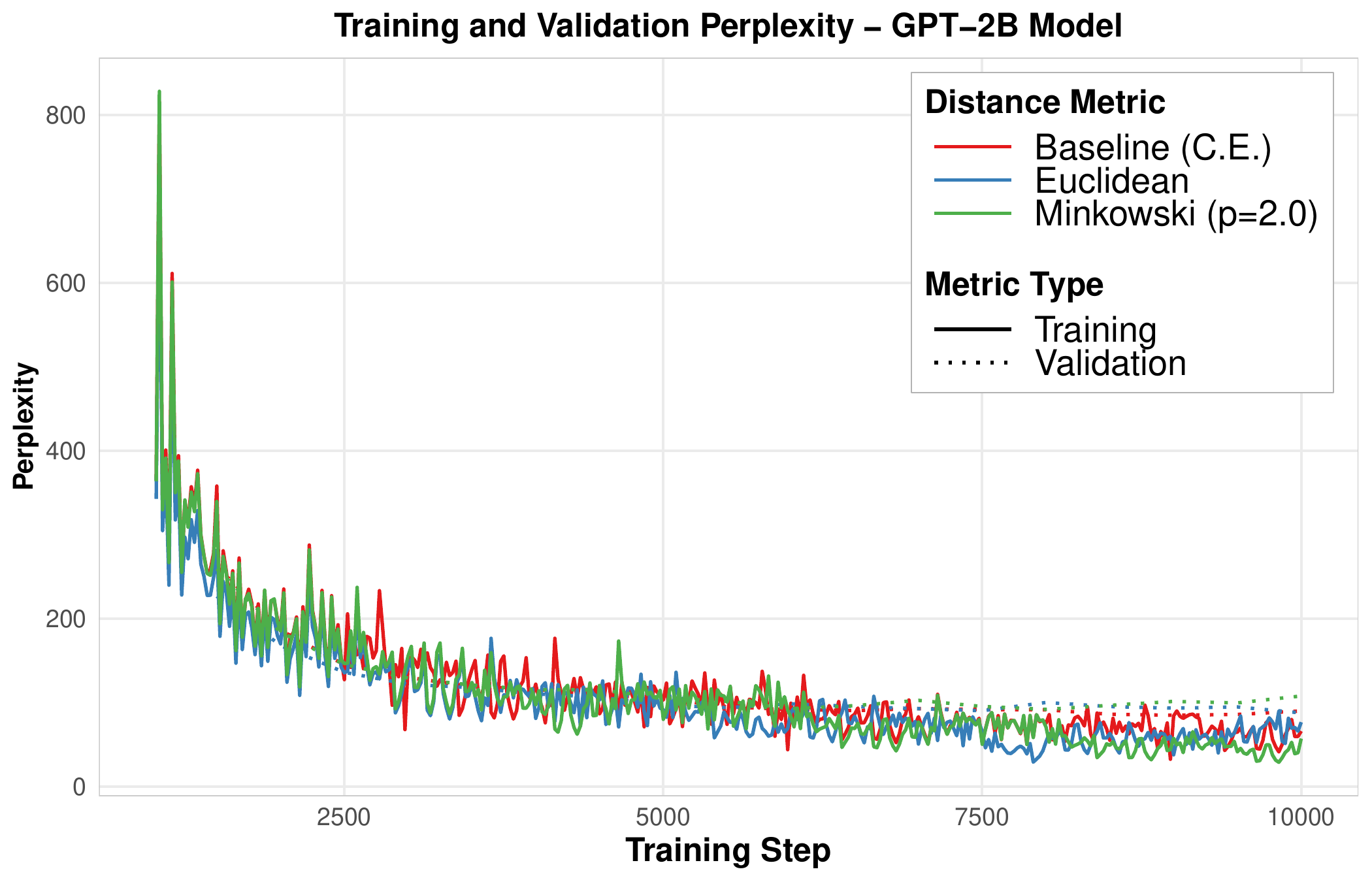}
	    \caption{Loss convergence behavior with language models (BERT-0.1B, GPT-0.1B, QWEN2-0.5B, GPT-2B).}
	        \Description[Language-model convergence curves]{A collection of convergence plots for language models, comparing cross-entropy to distance-based harmonic loss variants. Separate panels show training and validation trajectories (loss or perplexity) over training progress for multiple model families (BERT, GPT, Qwen, and GPT-2B). The panels overlay curves for different distance choices, allowing readers to visually assess convergence stability and relative rates across objectives.}
	        \label{fig:convergence-llm}
	    \end{centering}
\end{figure*}

}

\clearpage
\newpage
{
\section{Grokking analysis: Modulo Addition}

Figure~\ref{fig:circle} summarizes the behavior of standard MLPs and H--MLPs trained on the synthetic modulo--addition task, a setting known to exhibit pronounced grokking effects under cross--entropy. The top two rows illustrate training curves and corresponding 2D embeddings for baseline cross--entropy models (first two columns) and Euclidean harmonic loss (third and fourth columns). The remaining rows extend this comparison to alternative non--Euclidean harmonic losses.

\paragraph{Cross--entropy exhibits clear grokking.}
For both the standard MLP and its lightly regularized variant, cross--entropy produces the characteristic grokking pattern: training accuracy rapidly converges while test accuracy improves only after a long delay. This decoupling is consistent with prior observations in algorithmic tasks, where cross--entropy tends to overfit memorization pathways before discovering the true modular arithmetic structure. The PCA plots confirm this: the learned embeddings under cross--entropy have diffuse, irregular geometry, and the first two principal components explain only a small fraction of the variance (EV $\approx$ 20--30\%).

\paragraph{Euclidean harmonic loss eliminates grokking and induces a perfect geometric structure.}
In contrast, the Euclidean harmonic model reaches high train \emph{and} test accuracy simultaneously. No grokking delay is observed. The PCA projection reveals a striking property: the latent representation forms a \emph{perfect 2D circle}, and the first two principal components explain nearly all variance (EV $\approx 100\%$). This matches theoretical expectations for harmonic distance--based classification on cyclic group structure: the model learns an isometric embedding of $\mathbb{Z}_n$ into the plane, validating the geometric alignment induced by harmonic objectives.

\paragraph{Other distance--based harmonic losses replicate the circle structure with similarly fast generalization.}
The bottom rows show that this desirable behavior is \emph{not} unique to Euclidean distance. Cosine, stable cosine, Manhattan (1--norm), several Canberra variants, Hamming losses, Minkowski $p=3$, Chebyshev, and others all produce the same qualitative outcome:

\begin{itemize}
    \item \textbf{Immediate or near--immediate generalization}, with no grokking phase.
    \item \textbf{Highly structured 2D embeddings}, often forming a near--perfect circle.
    \item \textbf{Explained variance approaching 100\%}, indicating strong alignment to a low--dimensional manifold reflecting the algebraic symmetry of the task.
\end{itemize}

Some distances (e.g., Hamming and Chebyshev) produce slightly rotated or warped circles, but the essential geometric structure and variance concentration remain intact. This demonstrates that harmonic losses robustly recover the underlying modular arithmetic structure \emph{regardless of the distance family}.

\paragraph{Harmonic losses reduce grokking compared to cross--entropy.}
Across all non--Euclidean distances tested, harmonic losses exhibit two consistent advantages over cross--entropy:

\begin{enumerate}
    \item \textbf{Reduced grokking or complete elimination of delayed generalization}.  
    Training and test accuracy rise together, indicating that the model discovers the algorithmic rule rather than memorizing individual cases.
    \item \textbf{Improved interpretability via stable geometric structure}.  
    The emergence of a low--dimensional circular manifold with EV close to 1.0 serves as a quantitative and visual certificate of representation clarity.
\end{enumerate}

These results reinforce the core claims of the paper: harmonic losses promote structured, prototype--aligned representations and smoother, more reliable optimization dynamics, even on tasks where cross--entropy typically groks. The fact that many distances achieve EV $\approx 100\%$ highlights that the benefits of harmonic classification do not depend on Euclidean geometry alone, but arise from the broader class of distance--based harmonic objectives.
}

\begin{figure*}[h!]
	    \begin{centering}
\includegraphics[page=1,width=0.995\textwidth,trim=400 2710 250 0,clip]{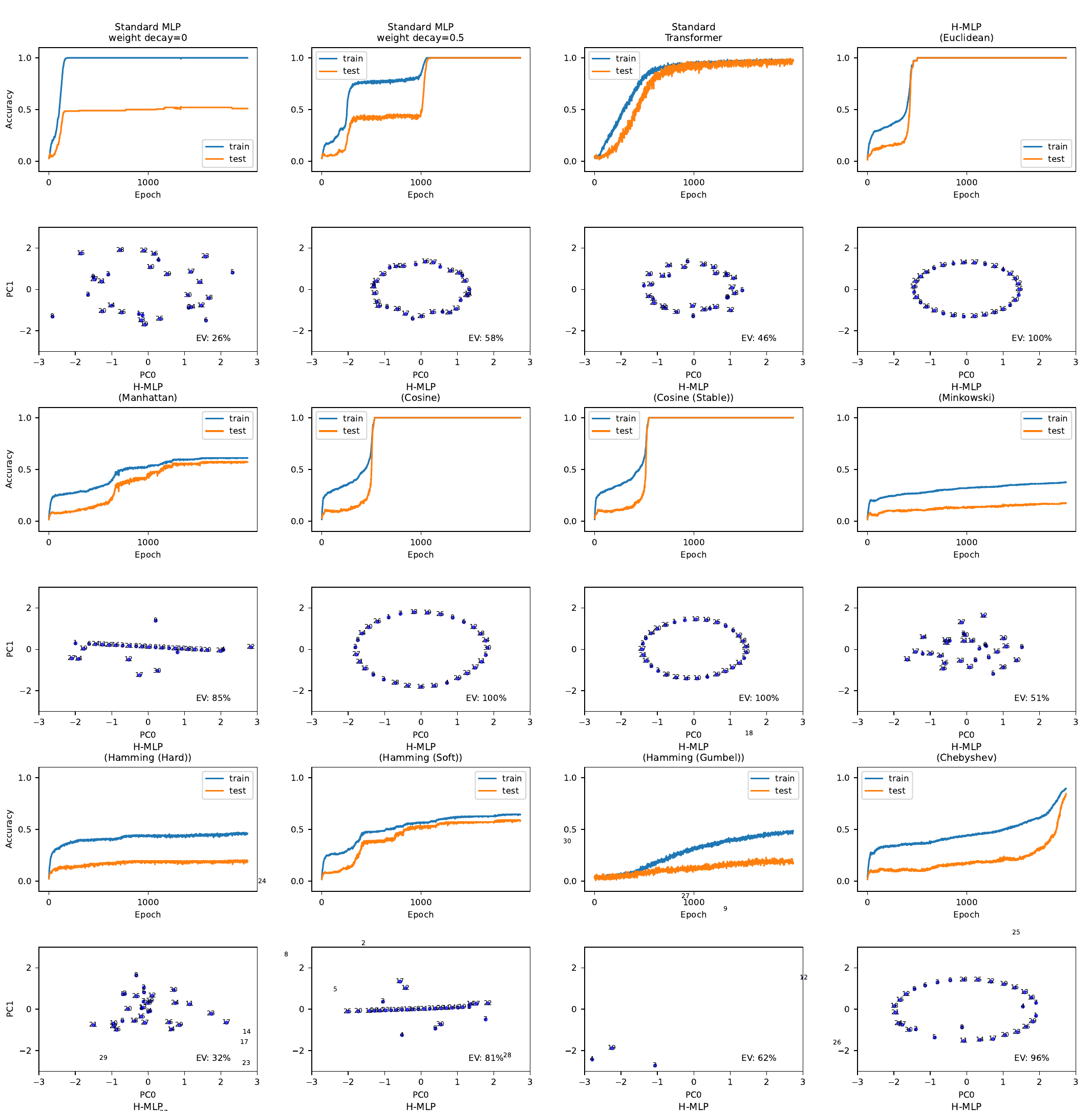}
	  \caption{Results on standard MLP trained for modular addition. 
	  %Generalization is only achieved with the addition of strong weight decay; however, (a) significant grokking occurs, and (b) while the first two principal components form an approximate circle, they explain far less than the total variance. In contrast, 
	  The harmonic model trained for modular addition generalizes quickly without grokking. Moreover, the embedding forms a perfect 2D circle. EV in the plot represents the explained variance by the first two principal components of the embedding.}
	  \Description[Modular-addition case study: grokking vs harmonic loss]{A multi-panel case study on the modular addition (toy) task using an MLP. The figure contrasts standard training behavior with a harmonic-loss setup by showing learning dynamics (training versus test performance over time) and a two-dimensional PCA visualization of the learned embeddings. In the harmonic-loss case, the plot highlights rapid generalization without a delayed “grokking” phase and an embedding that lies on a near-perfect circle in the top two principal components. The annotation “EV” reports the explained variance of the first two PCA components.}
	    \label{fig:circle}
	    \end{centering}
\end{figure*}

\section{Computational complexity: FLOPS}

% \begin{table}[t]
% \centering
% \caption{Approximate FLOPs per forward pass (in GFLOPs) for each backbone and dataset. 
% Values are computed from the convolutional/transformer backbone only; the choice of
% harmonic distance or alternative loss affects the last layer and has negligible impact on FLOPs.}
% \label{tab:flops}
% \setlength{\tabcolsep}{4pt}
% \begin{tabular}{lccccc}
% \toprule
% \textbf{Backbone} & \textbf{CIFAR-10} & \textbf{CIFAR-100} & \textbf{MarathiSign} & \textbf{MNIST} & \textbf{TinyImageNet} \\
% \midrule
% CNN      & 0.0123 & 0.0123 & 0.0123 & 0.0085 & 0.6030 \\
% MLP      & 0.0034 & 0.0035 & 0.0034 & 0.0011 & 0.3100 \\
% PVT      & 0.0383 & 0.0383 & 0.0383 & 0.0383 & 1.9000 \\
% ResNet50 & 0.0796 & 0.0800 & 4.0961 & 0.0600 & 4.0967 \\
% \bottomrule
% \end{tabular}
% \end{table}

%\paragraph{FLOPs and architectural cost.}
Table~\ref{tab:flops} reports the approximate floating–point operations per forward pass
for each backbone–dataset.
On $32{\times}32$ inputs (MNIST resized, CIFAR-10/100, MarathiSign), the FLOP hierarchy
is consistent: MLP is cheapest (${<}0.004$\,GFLOPs), CNN roughly $3{\times}$ more
expensive ($\approx 0.012$\,GFLOPs), PVT adds another ${\sim}3\times$ ($\approx
0.038$\,GFLOPs), and ResNet50 is about $2\times$ PVT ($\approx 0.08$\,GFLOPs).
Moving to high–resolution inputs ($224{\times}224$ for TinyImageNet and our
high–resolution MarathiSign runs) increases cost by two orders of magnitude:
PVT rises to ${\sim}1.9$\,GFLOPs and ResNet50 to ${\sim}4.1$\,GFLOPs per forward pass.
These numbers highlight that i) sustainability differences across \emph{architectures}
are dominated by backbone FLOPs, while ii) swapping Euclidean harmonic loss for
alternative distances or baselines changes only the final classifier head, adding an
$O(Cd)$ cost that is negligible compared to the convolutional / transformer body.
Consequently, the per–step FLOP budget is effectively distance–invariant, and
our sustainability comparisons across losses can be interpreted as differences in
optimization dynamics (steps-to-target, stability) rather than raw arithmetic cost.

\begin{table*}[htbp]
\centering
\setlength{\tabcolsep}{4pt}
%\small
\caption{Per-sample FLOPs, GFLOPs, and parameter counts for each backbone and dataset.}
\label{tab:flops-normalization}
\begin{tabular}{l l r r r r r r r}
\toprule
\textbf{Model} & \textbf{Dataset} & \textbf{In Ch.} & \textbf{H} & \textbf{W} &
\textbf{\#Cls} & \textbf{FLOPs} & \textbf{Params} & \textbf{GFLOPs} \\
\midrule
CNN      & CIFAR10      & 3 &  32 &  32 &  10 &   12307072  &   545098  & 0.0123 \\
CNN      & CIFAR100     & 3 &  32 &  32 & 100 &   12330112  &   556708  & 0.0123 \\
CNN      & MarathiSign  & 3 &  32 &  32 &  43 &   12315520  &   549355  & 0.0123 \\
CNN      & MNIST        & 1 &  28 &  28 &  10 &    8520064  &   421642  & 0.0085 \\
CNN      & TinyImageNet & 3 & 224 & 224 & 200 &  602966144  & 25735432  & 0.6029 \\
\midrule
MLP      & CIFAR10      & 3 &  32 &  32 &  10 &   3413760   &  1707274  & 0.0034 \\
MLP      & CIFAR100     & 3 &  32 &  32 & 100 &   3459840   &  1730404  & 0.0034 \\
MLP      & MarathiSign  & 3 &  32 &  32 &  43 &   3430656   &  1715755  & 0.0034 \\
MLP      & MNIST        & 1 &  28 &  28 &  10 &   1070848   &   535818  & 0.0010 \\
MLP      & TinyImageNet & 3 & 224 & 224 & 200 & 309536256   & 154769096 & 0.3095 \\
\midrule
PVT      & CIFAR10      & 3 &  32 &  32 &  10 &  38268630   & 12746560  & 0.0382 \\
PVT      & CIFAR100     & 3 &  32 &  32 & 100 &  38314710   & 12769600  & 0.0383 \\
PVT      & MarathiSign  & 3 &  32 &  32 &  43 &  38285526   & 12755008  & 0.0382 \\
PVT      & MNIST        & 3 &  32 &  32 &  10 &  38268630   & 12746560  & 0.0382 \\
PVT      & TinyImageNet & 3 & 224 & 224 & 200 & 1899590400 & 12795200  & 1.8995 \\
\midrule
ResNet50 & CIFAR10      & 3 &  32 &  32 &  10 &   79618429  & 23472480  & 0.0796 \\
ResNet50 & CIFAR100     & 3 &  32 &  32 & 100 &   79987069  & 23656800  & 0.0799 \\
ResNet50 & MarathiSign  & 3 & 224 & 224 &  43 & 4096080128  & 23540064  & 4.0960 \\
ResNet50 & MNIST        & 1 &  28 &  28 &  10 &  60007460   & 23472480  & 0.0600 \\
ResNet50 & TinyImageNet & 3 & 224 & 224 & 200 & 4096723200  & 23861600  & 4.0967 \\
\bottomrule
\end{tabular}
\label{tab:flops}
\end{table*}

\clearpage
\newpage
{
\section{Geometric Insights}

%\textbf{MNIST + ResNet50:}
To better illustrate how different harmonic distances shape the embedding geometry, we visualize the last–layer representations of ResNet50 on MNIST (see Figure \ref{fig:geometric-insights-mnist}) and CIFAR10 (see Figure \ref{fig:geometric-insights-cifar10}) using 2D PCA, with class prototypes overlaid as markers.
For the Euclidean harmonic head, the class clusters are roughly spherical and separated by (approximately) straight boundaries in the projection: decision regions are controlled mainly by radial distance to each prototype, yielding isotropic attraction basins around each center.

Under Cosine harmonic loss, the picture changes markedly. Features and prototypes concentrate on (or very near to) a common hypersphere, so the PCA plot shows clusters arranged along a circle. Classes are separated primarily by their angle rather than their norm, and decision boundaries correspond to angular bisectors between prototypes. This matches our geometric claim that cosine harmonic removes radial curvature and constrains optimization to an angular manifold: as training proceeds, points slide along the sphere towards their prototype, producing wide, smooth basins and stable gradient norms.

By contrast, Mahalanobis harmonic loss induces anisotropic curvature. After whitening by $\Sigma^{-1/2}$, the decision boundaries are linear, but in the original feature space they correspond to ellipsoidal contours. In the PCA plots this appears as elongated clusters and distorted attraction basins around prototypes, with some directions exhibiting much tighter concentration than others. When the empirical covariance is well–conditioned this yields very sharp, well–separated clusters (high variance concentration), but when eigenvalues are highly unbalanced the same anisotropy can make optimization more sensitive to particular directions.

% Overall, these visualizations support our geometric narrative: cosine harmonic leads to near–spherical (angular) class structure with smooth basins, whereas Mahalanobis introduces directional stretching that sharpens separation but can amplify curvature along poorly conditioned directions.

\begin{figure*}[h!]
    \begin{centering}
    \includegraphics[width=0.49\textwidth,trim=5 5 5 37,clip]{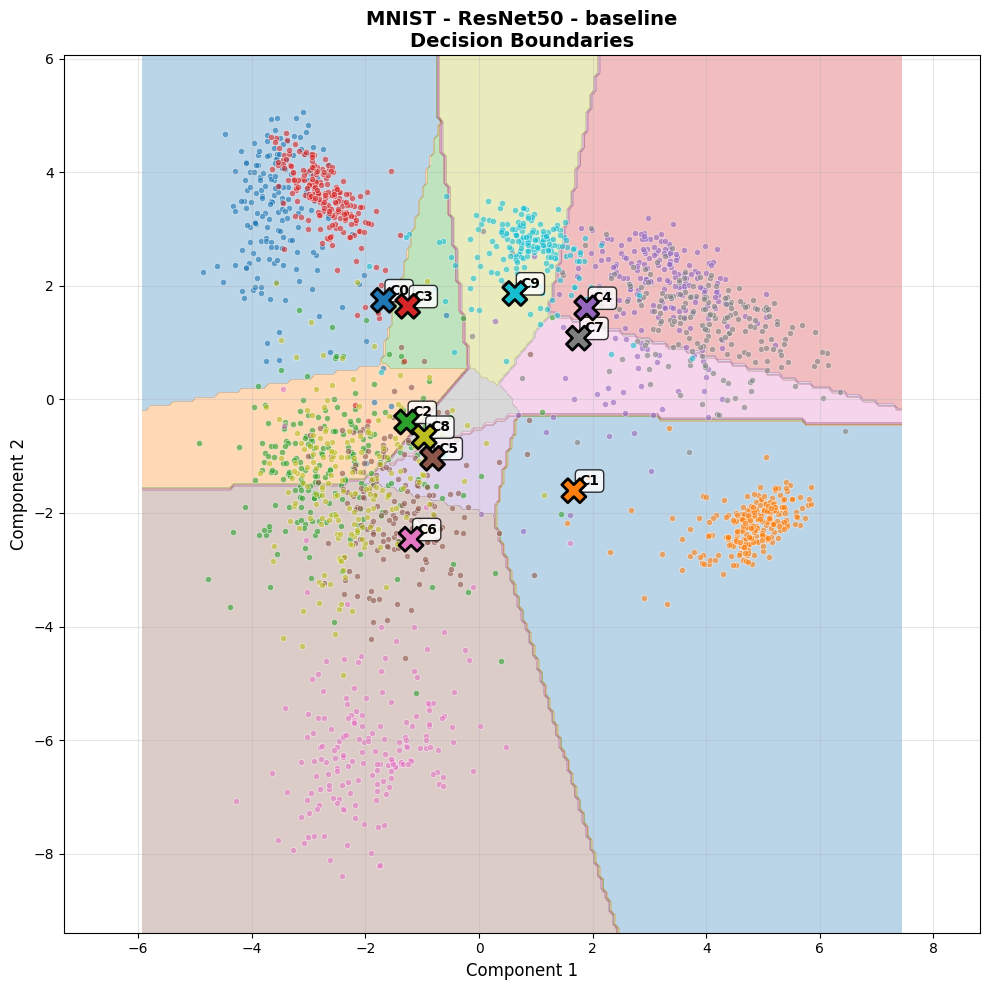}
 %  \ \ \ (a) \ \ \ \ \ \ \ \ \ \ \ \ \ \ \ \ \ \ \ \ \ \ \ \ \ \ \ \ \ \ \ \ \ \ \ \ \ \ \ \ \ \ \ \ \ \ \ \ \ \ \ \ \ \ \ \ \ \ \ \ \ \ \ \ \ \ \ \ \ \ \ \ \ \ \  (b) \\
    %
%    \includegraphics[width=0.49\textwidth,trim=5 5 5 37,clip]{figures/Geometric/MNIST_ResNet50_comparison/MNIST_ResNet50_euclidean/scatter_pca.png}
    \includegraphics[width=0.49\textwidth,trim=5 5 5 37,clip]{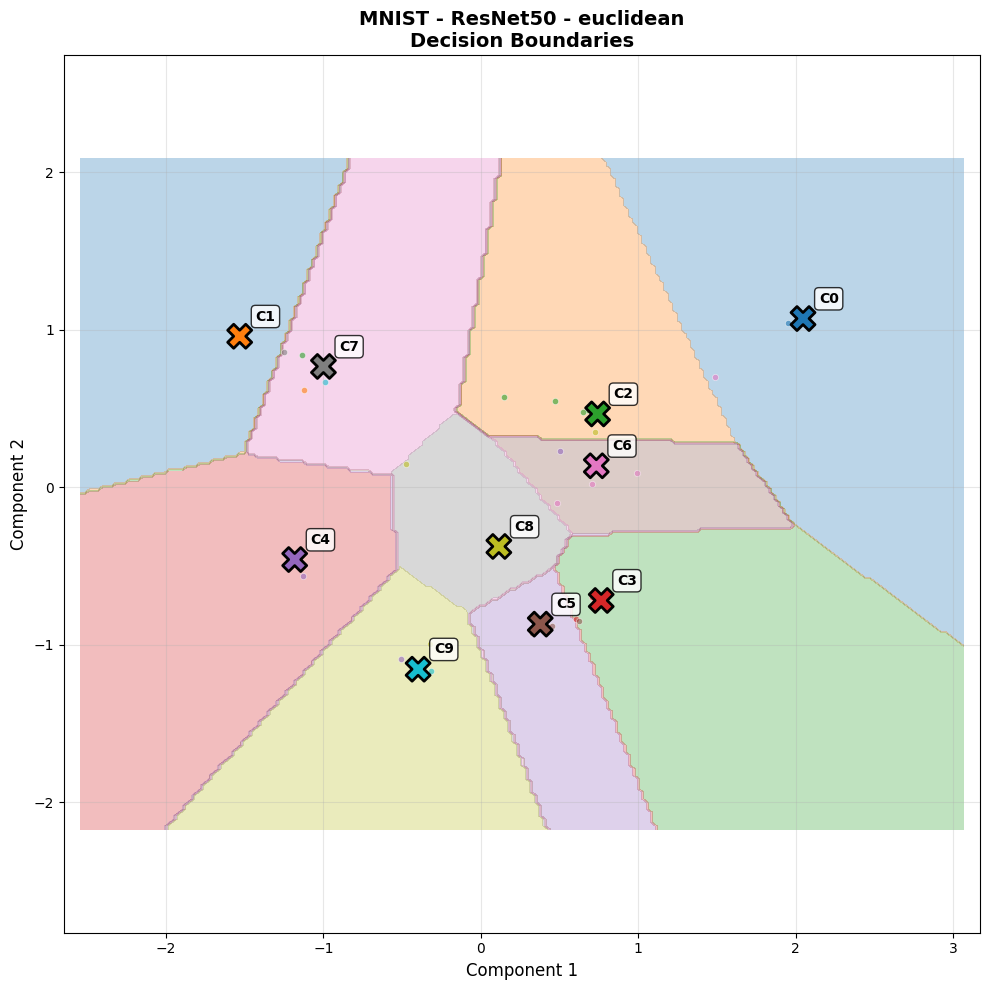}   
  \ \ \ \ \ \ (a) \ \ \ \ \ \ \ \ \ \ \ \ \ \ \ \ \ \ \ \ \ \ \ \ \ \ \ \ \ \ \ \ \ \ \ \ \ \ \ \ \ \ \ \ \ \ \ \ \ \ \ \ \ \ \ \ \ \ \ \ \ \ \ \ \ \ \ \ \ \ \ \ \ \ \  (b) \\
    \includegraphics[width=0.49\textwidth,trim=5 5 5 37,clip]{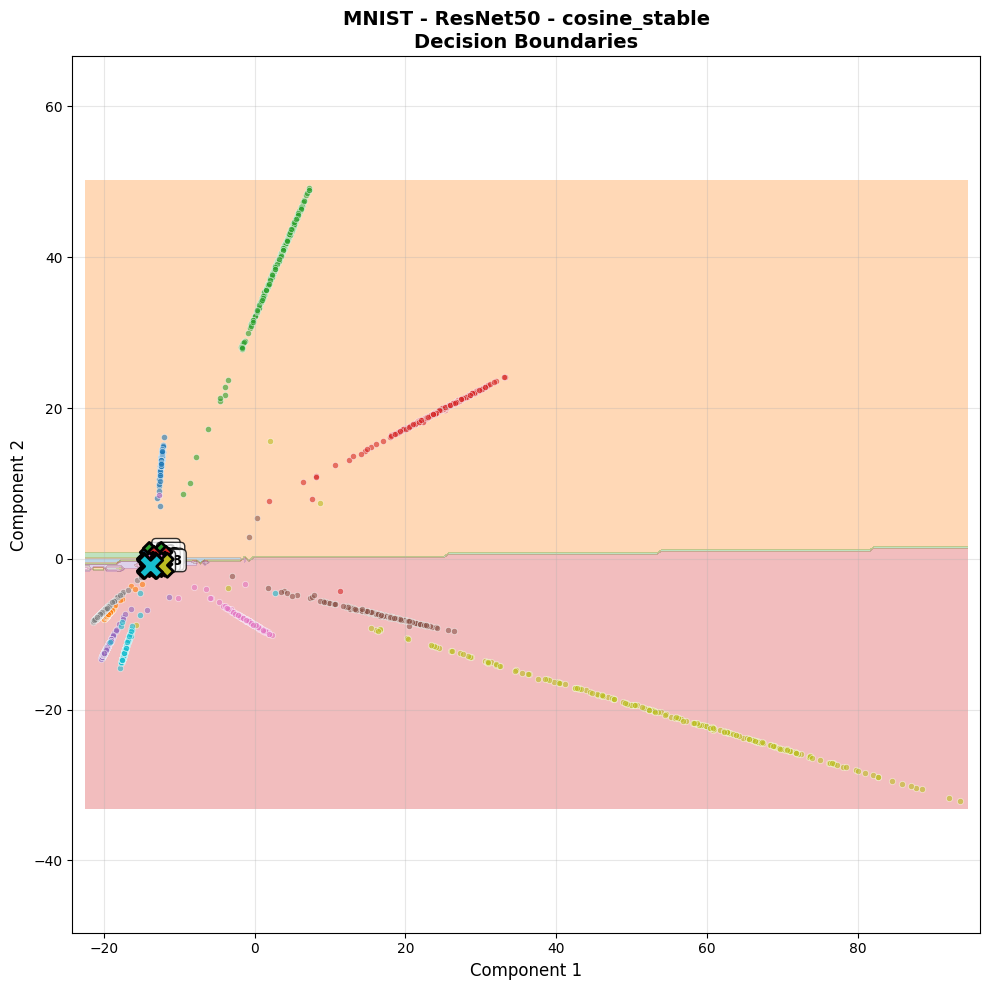}   
    \includegraphics[width=0.49\textwidth,trim=5 5 5 37,clip]{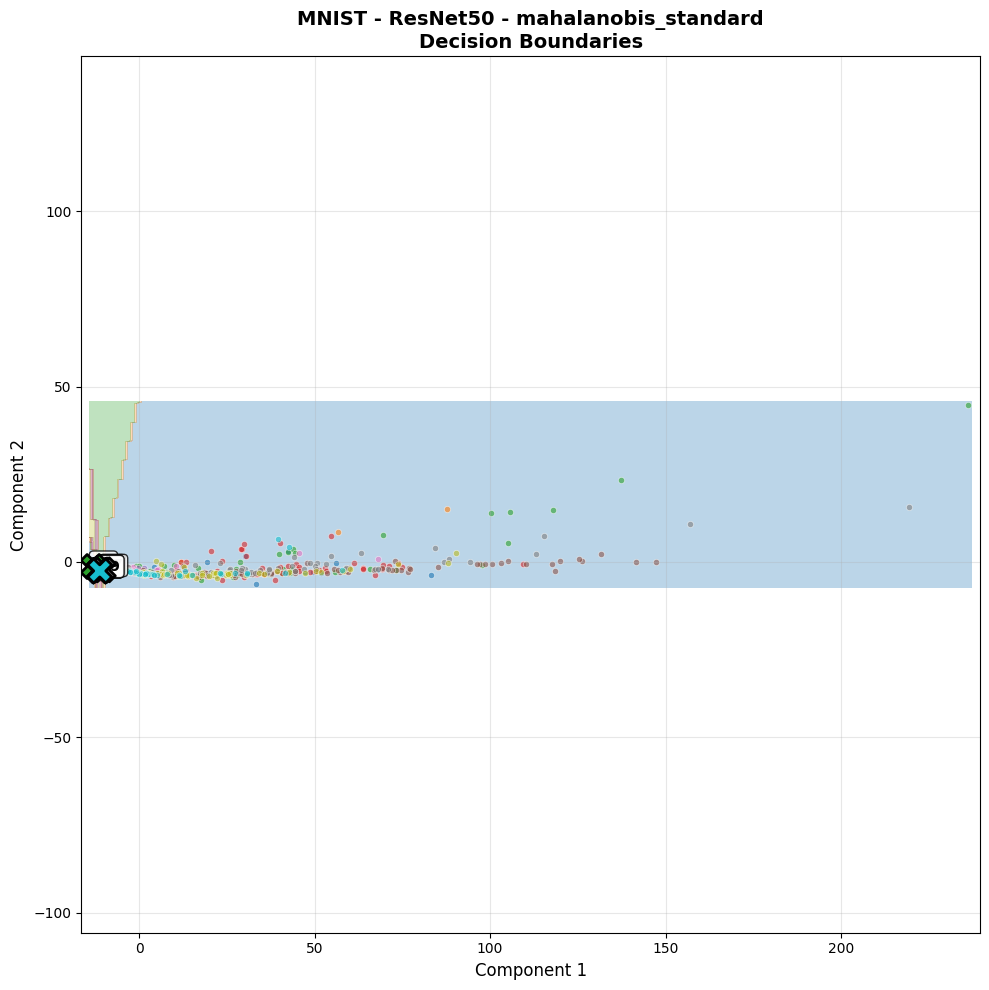}   
   \ \ \ \ \ \ \ \ \ (c) \ \ \ \ \ \ \ \ \ \ \ \ \ \ \ \ \ \ \ \ \ \ \ \ \ \ \ \ \ \ \ \ \ \ \ \ \ \ \ \ \ \ \ \ \ \ \ \ \ \ \ \ \ \ \ \ \ \ \ \ \ \ \ \ \ \ \ \ \ \ \ \ \ \ \  (d) \\
	    \caption{Geometric effect of distance--based harmonic losses on ResNet50 embeddings (MNIST). From top to bottom: Baseline (a), Euclidean harmonic loss (b),  cosine harmonic loss (c), and Mahalanobis harmonic loss (d). 
%Cross--entropy produces irregular, radius--dependent clusters with uneven norms, whereas Euclidean harmonic loss already regularizes class means toward prototype directions. 
%Cosine harmonic loss further collapses the representation onto an approximately spherical manifold, where classes are separated primarily by angle and clusters drift smoothly toward their prototypes, illustrating the ``angular'' geometry and stable convergence behaviour. 
%Mahalanobis harmonic loss induces anisotropic curvature: clusters become elongated along a few dominant directions determined by $\Sigma^{-1}$, yielding strong variance concentration but more directional sensitivity. 
%Together, these plots empirically confirm that changing the distance in the harmonic head alters the curvature of the embedding space \emph{without} introducing pathological optimization issues: all variants converge smoothly, while non--Euclidean distances produce more structured and interpretable decision geometries.
}
	    \Description[PCA embedding plots for MNIST under different losses]{Four panels visualizing the 2D PCA projection of ResNet50 penultimate-layer embeddings on MNIST, with decision regions and class prototypes overlaid. Panels (a) through (d) correspond to: a baseline classifier, Euclidean harmonic loss, cosine harmonic loss, and Mahalanobis harmonic loss. Across panels, the points form class clusters in the PCA plane and the background/contours indicate how the classifier partitions the space, illustrating how changing the distance metric in the harmonic head alters cluster shape and decision geometry.}
	        \label{fig:geometric-insights-mnist}
	    \end{centering}
\end{figure*}

\begin{figure*}[h!]
    \begin{centering}
    \includegraphics[width=0.49\textwidth,trim=5 5 5 37,clip]{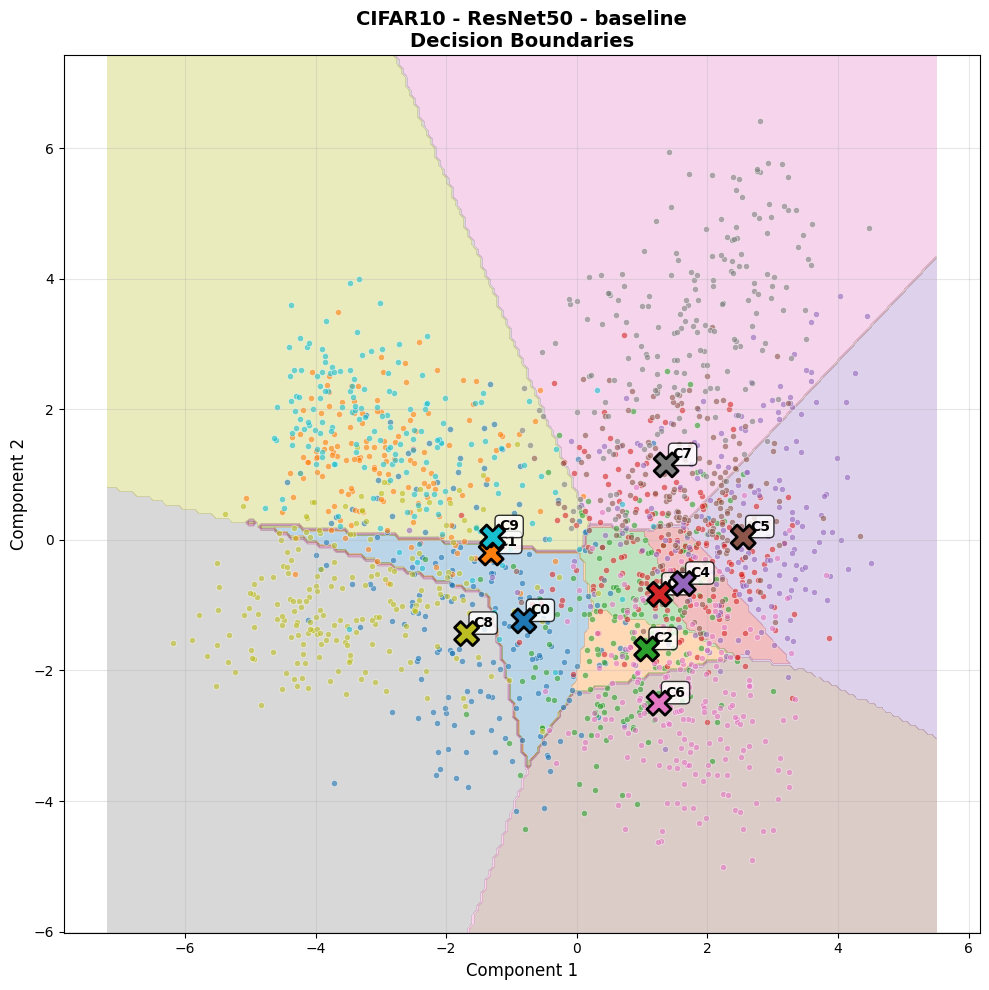}
 %  \ \ \ (a) \ \ \ \ \ \ \ \ \ \ \ \ \ \ \ \ \ \ \ \ \ \ \ \ \ \ \ \ \ \ \ \ \ \ \ \ \ \ \ \ \ \ \ \ \ \ \ \ \ \ \ \ \ \ \ \ \ \ \ \ \ \ \ \ \ \ \ \ \ \ \ \ \ \ \  (b) \\
    %
    \includegraphics[width=0.49\textwidth,trim=5 5 5 37,clip]{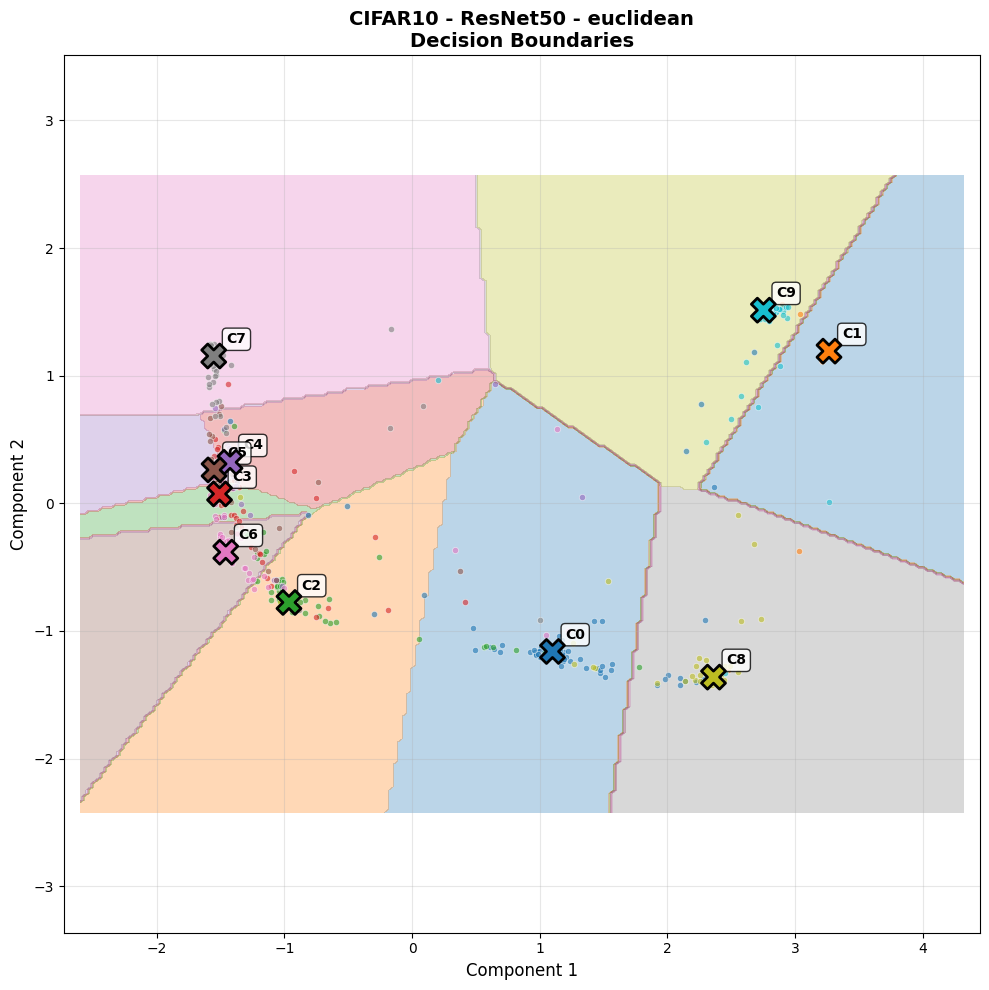}   
  \ \ \ \ \ \ (a) \ \ \ \ \ \ \ \ \ \ \ \ \ \ \ \ \ \ \ \ \ \ \ \ \ \ \ \ \ \ \ \ \ \ \ \ \ \ \ \ \ \ \ \ \ \ \ \ \ \ \ \ \ \ \ \ \ \ \ \ \ \ \ \ \ \ \ \ \ \ \ \ \ \ \  (b) \\
	    \includegraphics[width=0.49\textwidth,trim=5 5 5 37,clip]{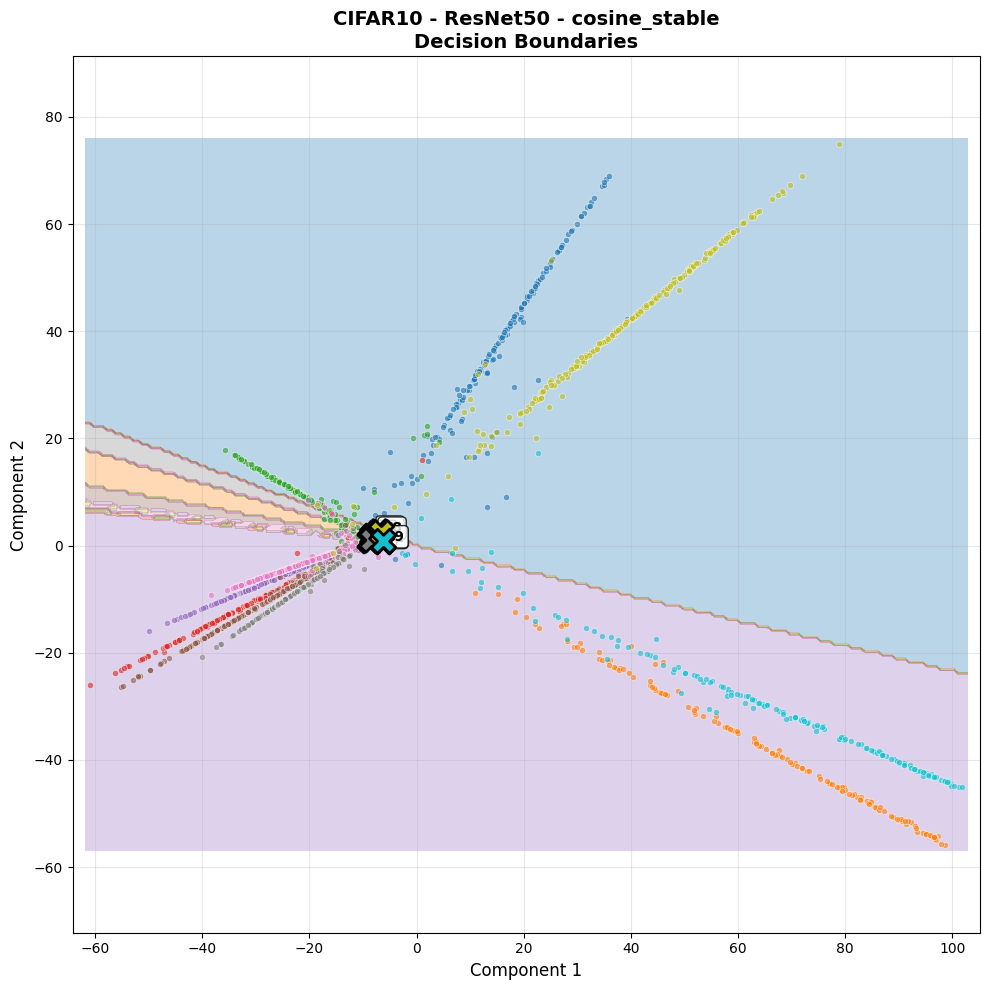}       \includegraphics[width=0.49\textwidth,trim=5 5 5 37,clip]{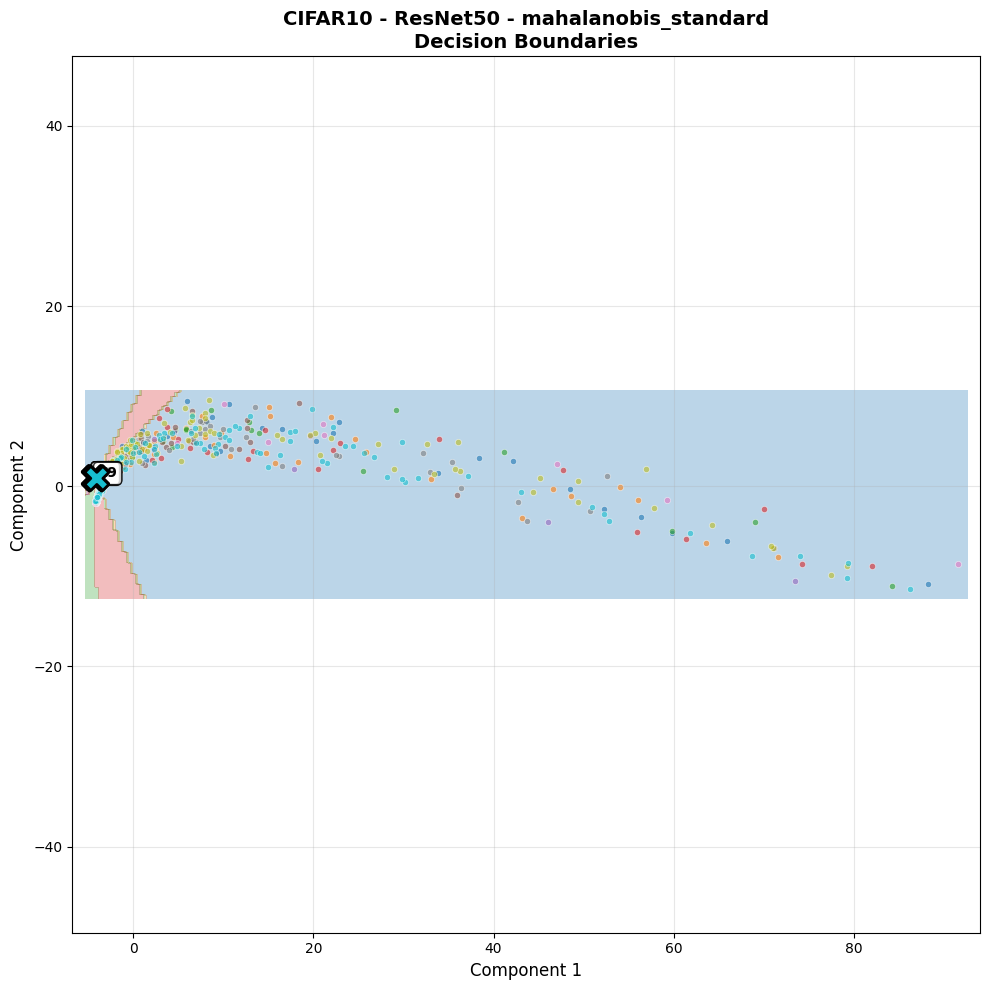}   
	   \ \ \ \ \ \ \ \ \ (c) \ \ \ \ \ \ \ \ \ \ \ \ \ \ \ \ \ \ \ \ \ \ \ \ \ \ \ \ \ \ \ \ \ \ \ \ \ \ \ \ \ \ \ \ \ \ \ \ \ \ \ \ \ \ \ \ \ \ \ \ \ \ \ \ \ \ \ \ \ \ \ \ \ \ \  (d) \\
	    \caption{Geometric effect of distance--based harmonic losses on ResNet50 embeddings (CIFAR10). From top to bottom: Baseline (a), Euclidean harmonic loss (b),  cosine harmonic loss (c), and Mahalanobis harmonic loss (d).
}
	    \Description[PCA embedding plots for CIFAR-10 under different losses]{Four panels visualizing a 2D PCA projection of ResNet50 embeddings on CIFAR-10, including the induced decision regions and class prototypes. Panels (a)–(d) show the baseline classifier and three harmonic-loss variants using different distances (Euclidean, cosine, and Mahalanobis). The figure highlights how the embedding clusters and the separating boundaries change as the distance metric in the harmonic head changes.}
	        \label{fig:geometric-insights-cifar10}
	    \end{centering}
\end{figure*}

\clearpage
\newpage

}

\newpage
\section{Additional results}
{
\color{red}
}

\subsection{Vision: Tables}
The empirical evaluation of non-Euclidean harmonic losses across MNIST, CIFAR-10, and CIFAR-100 with MLP, CNN, and ResNet50 backbones reveals several consistent patterns.

\textbf{Model Performance.}
Cosine distance emerges as the most reliable performer across architectures and datasets. In both stable and unstable variants, cosine harmonic loss consistently improves test accuracy and F1 relative to Euclidean, with gains most pronounced in deeper models (CNNs and ResNets) and in medium-complexity datasets such as CIFAR-10. Bray--Curtis offers modest gains in certain contexts but is less consistent, while Mahalanobis can improve accuracy on simple datasets (e.g., MNIST) but often lags behind cosine in more challenging regimes. Euclidean harmonic loss, while better than cross-entropy in terms of stability, is consistently outperformed by cosine-based alternatives.

\textbf{Interpretability.}
Distances strongly reshape the geometry of the learned representations. Cosine and Bray--Curtis often yield large improvements in explained variance (EV), indicating more compact feature spaces aligned with class prototypes. Mahalanobis produces the most dramatic gains in EV, frequently approaching full variance explanation, but this comes at the cost of stability and efficiency. Prototype coverage (PC90\%) tends to shrink under cosine and Mahalanobis, highlighting sharper clustering effects: models assign fewer prototypes to cover 90\% of variance, making the representation space more interpretable but less evenly distributed.

\textbf{Sustainability.}
Sustainability outcomes mirror performance trends. Cosine distances typically reduce carbon emissions relative to Euclidean, in some cases by up to 40\%, making them both effective and energy-efficient. Bray--Curtis shows mixed results, with occasional emission savings but less consistent behavior. Mahalanobis tends to incur higher emissions, reflecting the computational overhead of covariance estimation and matrix operations. Shallow architectures (MLPs) show less differentiation across distances in emissions, while deeper backbones amplify both the benefits (cosine) and costs (Mahalanobis).

\textbf{Trade-offs.}
Taken together, the results confirm that distance choice is not neutral in harmonic loss. Cosine provides the most favorable balance across performance, interpretability, and sustainability, representing the strongest general-purpose alternative to Euclidean. Bray--Curtis occupies a middle ground, offering interpretability benefits without always delivering accuracy or efficiency gains. Mahalanobis maximizes interpretability at a clear sustainability cost, making it attractive primarily when prototype clarity outweighs computational expense. Euclidean serves as a stable but suboptimal baseline. 

\textbf{Conclusion.}
This systematic study establishes that non-Euclidean harmonic losses provide a flexible and effective design space. In particular, cosine distance offers a compelling replacement for cross-entropy and Euclidean harmonic loss in vision tasks, consistently improving accuracy, interpretability, and sustainability. These findings position distance-tailored harmonic losses as a promising avenue for advancing deep learning models that are not only accurate but also more transparent and energy-conscious.

\begin{table*}
%\small
\setlength{\tabcolsep}{8pt}
\centering
\caption{Results for CIFAR100 CNN. Parentheses: \% changes w.r.t. Baseline (Cross-Entropy).}
\label{tab:cifar100-cnn}
\begin{tabular}{llllll}
\toprule
\textbf{Method} &                             \textbf{Acc} &                              \textbf{F1} &                       \textbf{gCO$_2$eq} &                             \textbf{EV} &                          \textbf{PC90\%} \\
\midrule
Baseline            &                          0.3795 &                          0.3795 &                            1.18 &                       0.459295 &                        49.3333 \\
Bray-Curtis (Norm.) &  0.3229 ({\scriptsize -14.91\%}) &  0.3182 ({\scriptsize -16.16\%}) &   0.8132 ({\scriptsize 30.94\%}) &     0.9094 ({\scriptsize 98\%}) &  2.6667 ({\scriptsize 94.59\%}) \\
Mahalanobis (Chol.) &  0.2927 ({\scriptsize -22.86\%}) &  0.2921 ({\scriptsize -23.04\%}) &    0.727 ({\scriptsize 38.25\%}) &  0.341 ({\scriptsize -25.75\%}) &      50 ({\scriptsize -1.35\%}) \\
Cosine (Unst.)      &  0.2602 ({\scriptsize -31.44\%}) &  0.2667 ({\scriptsize -29.73\%}) &  2.1156 ({\scriptsize -79.68\%}) &  0.5306 ({\scriptsize 15.52\%}) &       45 ({\scriptsize 8.78\%}) \\
Cosine (Stable)     &  0.2501 ({\scriptsize -34.09\%}) &  0.2516 ({\scriptsize -33.71\%}) &  1.4263 ({\scriptsize -21.14\%}) &  0.5216 ({\scriptsize 13.57\%}) &       45 ({\scriptsize 8.78\%}) \\
Euclidean           &   0.2413 ({\scriptsize -36.4\%}) &  0.2431 ({\scriptsize -35.95\%}) &   1.2866 ({\scriptsize -9.28\%}) &  0.4362 ({\scriptsize -5.02\%}) &      50 ({\scriptsize -1.35\%}) \\
\bottomrule
\end{tabular}
\end{table*}

% \textbf{Discussion (CIFAR100 MLP).}  \textbf{Performance:} Baseline achieves the highest acc among tested distances. Baseline also yields the best F1, confirming its robustness. \textbf{Sustainability:} Cosine (Unst.) achieves the lowest gCO2eq, indicating higher sustainability. \textbf{Interpretability:} Bray-Curtis (Norm.) produces the highest EV, suggesting more structured latent spaces. Bray-Curtis (Norm.) reduces prototype coverage most strongly, indicating sharper clustering effects. Overall, the results highlight trade-offs between accuracy, sustainability, and interpretability: cosine and related distances often provide balanced improvements, while Mahalanobis tends to emphasize interpretability at the expense of efficiency.

\begin{table*}
%\small
\setlength{\tabcolsep}{8pt}
\centering
\caption{Results for CIFAR100 MLP. Parentheses: \% changes w.r.t. Baseline (Cross-Entropy).}
\label{tab:cifar100-mlp}
\begin{tabular}{llllll}
\toprule
\textbf{Method} &                             \textbf{Acc} &                              \textbf{F1} &                       \textbf{gCO$_2$eq} &                             \textbf{EV} &                          \textbf{PC90\%} \\
\midrule
Baseline            &                          0.2617 &                          0.2582 &                             0.85 &                        0.285203 &                            50.0 \\
Bray-Curtis (Norm.) &   0.226 ({\scriptsize -13.64\%}) &  0.2191 ({\scriptsize -15.15\%}) &    0.8938 ({\scriptsize -5.35\%}) &  0.9843 ({\scriptsize 245.11\%}) &           1 ({\scriptsize 98\%}) \\
Mahalanobis (Chol.) &  0.1833 ({\scriptsize -29.96\%}) &  0.1811 ({\scriptsize -29.85\%}) &   1.0889 ({\scriptsize -28.34\%}) &  0.0354 ({\scriptsize -87.57\%}) &          50 ({\scriptsize -0\%}) \\
Bray-Curtis (Abs.)  &  0.1444 ({\scriptsize -44.81\%}) &  0.1392 ({\scriptsize -46.07\%}) &  2.1255 ({\scriptsize -150.51\%}) &   0.5317 ({\scriptsize 86.43\%}) &   47.6667 ({\scriptsize 4.67\%}) \\
Cosine (Unst.)      &  0.1237 ({\scriptsize -52.74\%}) &  0.1186 ({\scriptsize -54.06\%}) &    0.5064 ({\scriptsize 40.31\%}) &   0.3799 ({\scriptsize 33.21\%}) &  40.6667 ({\scriptsize 18.67\%}) \\
Euclidean           &   0.119 ({\scriptsize -54.53\%}) &  0.1222 ({\scriptsize -52.69\%}) &     0.589 ({\scriptsize 30.58\%}) &  0.2437 ({\scriptsize -14.55\%}) &          50 ({\scriptsize -0\%}) \\
\bottomrule
\end{tabular}
\end{table*}

% \textbf{Discussion (CIFAR100 ResNet50).}  \textbf{Performance:} Cosine (Stable) achieves the highest acc among tested distances. Cosine (Stable) also yields the best F1, confirming its robustness. \textbf{Sustainability:} Cosine (Unst.) achieves the lowest gCO2eq, indicating higher sustainability. \textbf{Interpretability:} Mahalanobis (Chol.) produces the highest EV, suggesting more structured latent spaces. Bray-Curtis (Norm.) reduces prototype coverage most strongly, indicating sharper clustering effects. Overall, the results highlight trade-offs between accuracy, sustainability, and interpretability: cosine and related distances often provide balanced improvements, while Mahalanobis tends to emphasize interpretability at the expense of efficiency.

\begin{table*}
%\small
\setlength{\tabcolsep}{8pt}
\centering
\caption{Results for CIFAR100 ResNet50. Parentheses: \% changes w.r.t. Baseline (Cross-Entropy).}
\label{tab:cifar100-resnet50}
\begin{tabular}{llllll}
\toprule
\textbf{Method} &                             \textbf{Acc} &                              \textbf{F1} &                       \textbf{gCO$_2$eq} &                             \textbf{EV} &                          \textbf{PC90\%} \\
\midrule
Baseline            &                          0.6983 &                          0.6969 &                             87.77 &                        0.107216 &                      50.0 \\
Cosine (Stable)     &    0.7357 ({\scriptsize 5.35\%}) &     0.736 ({\scriptsize 5.61\%}) &    72.9745 ({\scriptsize 16.85\%}) &  0.5979 ({\scriptsize 457.66\%}) &     8 ({\scriptsize 84\%}) \\
Cosine (Unst.)      &    0.7323 ({\scriptsize 4.87\%}) &    0.7332 ({\scriptsize 5.21\%}) &    71.7592 ({\scriptsize 18.24\%}) &  0.5857 ({\scriptsize 446.27\%}) &     8 ({\scriptsize 84\%}) \\
Bray-Curtis (Norm.) &    0.655 ({\scriptsize -6.19\%}) &   0.6513 ({\scriptsize -6.54\%}) &  106.4049 ({\scriptsize -21.24\%}) &  0.7131 ({\scriptsize 565.08\%}) &     6 ({\scriptsize 88\%}) \\
Mahalanobis (Chol.) &  0.6274 ({\scriptsize -10.15\%}) &  0.6239 ({\scriptsize -10.47\%}) &   138.9317 ({\scriptsize -58.3\%}) &  0.7353 ({\scriptsize 585.81\%}) &  17.5 ({\scriptsize 65\%}) \\
Euclidean           &    0.7055 ({\scriptsize 1.03\%}) &    0.7062 ({\scriptsize 1.33\%}) &    97.432 ({\scriptsize -11.01\%}) &  0.5679 ({\scriptsize 429.66\%}) &  25.5 ({\scriptsize 49\%}) \\
\bottomrule
\end{tabular}
\end{table*}

% \textbf{Discussion (CIFAR10 CNN).}  \textbf{Performance:} Mahalanobis (Chol.) achieves the highest acc among tested distances. Mahalanobis (Chol.) also yields the best F1, confirming its robustness. \textbf{Sustainability:} Mahalanobis (Chol.) achieves the lowest gCO2eq, indicating higher sustainability. \textbf{Interpretability:} Bray-Curtis (Norm.) produces the highest EV, suggesting more structured latent spaces. Bray-Curtis (Norm.) reduces prototype coverage most strongly, indicating sharper clustering effects. Overall, the results highlight trade-offs between accuracy, sustainability, and interpretability: cosine and related distances often provide balanced improvements, while Mahalanobis tends to emphasize interpretability at the expense of efficiency.

\begin{table*}
%\small
\setlength{\tabcolsep}{8pt}
\centering
\caption{Results for CIFAR10 CNN. Parentheses: \% changes w.r.t. Baseline (Cross-Entropy).}
\label{tab:cifar10-cnn}
\begin{tabular}{llllll}
\toprule
\textbf{Method} &                             \textbf{Acc} &                              \textbf{F1} &                       \textbf{gCO$_2$eq} &                             \textbf{EV} &                          \textbf{PC90\%} \\
\midrule
Baseline            &                        0.6278 &                        0.6269 &                           1.12 &                        0.688081 &                               9.0 \\
Mahalanobis (Chol.) &  0.6644 ({\scriptsize 5.82\%}) &  0.6642 ({\scriptsize 5.95\%}) &   1.1139 ({\scriptsize 0.68\%}) &  0.4752 ({\scriptsize -30.93\%}) &       50 ({\scriptsize -455.56\%}) \\
Bray-Curtis (Norm.) &  0.6597 ({\scriptsize 5.08\%}) &   0.6551 ({\scriptsize 4.5\%}) &  1.1489 ({\scriptsize -2.45\%}) &   0.8913 ({\scriptsize 29.54\%}) &     4.3333 ({\scriptsize 51.85\%}) \\
Minkowski (p=3.0)   &  0.6589 ({\scriptsize 4.95\%}) &  0.6593 ({\scriptsize 5.17\%}) &  1.1598 ({\scriptsize -3.42\%}) &  0.5425 ({\scriptsize -21.15\%}) &       50 ({\scriptsize -455.56\%}) \\
Cosine (Stable)     &  0.6584 ({\scriptsize 4.87\%}) &  0.6566 ({\scriptsize 4.74\%}) &  1.1663 ({\scriptsize -3.99\%}) &    0.647 ({\scriptsize -5.97\%}) &  18.6667 ({\scriptsize -107.41\%}) \\
Euclidean           &  0.6495 ({\scriptsize 3.45\%}) &  0.6476 ({\scriptsize 3.31\%}) &  1.1228 ({\scriptsize -0.12\%}) &   0.6582 ({\scriptsize -4.34\%}) &   14.3333 ({\scriptsize -59.26\%}) \\
\bottomrule
\end{tabular}
\end{table*}

% \textbf{Discussion (CIFAR10 MLP).}  \textbf{Performance:} Baseline achieves the highest acc among tested distances. Baseline also yields the best F1, confirming its robustness. \textbf{Sustainability:} Euclidean achieves the lowest gCO2eq, indicating higher sustainability. \textbf{Interpretability:} Bray-Curtis (Norm.) produces the highest EV, suggesting more structured latent spaces. Bray-Curtis (Norm.) reduces prototype coverage most strongly, indicating sharper clustering effects. Overall, the results highlight trade-offs between accuracy, sustainability, and interpretability: cosine and related distances often provide balanced improvements, while Mahalanobis tends to emphasize interpretability at the expense of efficiency.

\begin{table*}
%\small
\setlength{\tabcolsep}{8pt}
\centering
\caption{Results for CIFAR10 MLP. Parentheses: \% changes w.r.t. Baseline (Cross-Entropy).}
\label{tab:cifar10-mlp}
\begin{tabular}{llllll}
\toprule
\textbf{Method} &                             \textbf{Acc} &                              \textbf{F1} &                       \textbf{gCO$_2$eq} &                             \textbf{EV} &                          \textbf{PC90\%} \\
\midrule
Baseline            &                         0.5397 &                         0.5385 &                            0.53 &                        0.346504 &                           47.0 \\
Bray-Curtis (Norm.) &  0.5224 ({\scriptsize -3.21\%}) &  0.5201 ({\scriptsize -3.41\%}) &    0.5264 ({\scriptsize 0.81\%}) &   0.967 ({\scriptsize 179.07\%}) &       1 ({\scriptsize 97.87\%}) \\
Mahalanobis (Chol.) &  0.5087 ({\scriptsize -5.75\%}) &  0.5088 ({\scriptsize -5.51\%}) &     0.458 ({\scriptsize 13.7\%}) &  0.0522 ({\scriptsize -84.94\%}) &      50 ({\scriptsize -6.38\%}) \\
Bray-Curtis (Abs.)  &  0.4934 ({\scriptsize -8.59\%}) &  0.4924 ({\scriptsize -8.55\%}) &  0.6313 ({\scriptsize -18.96\%}) &  0.2434 ({\scriptsize -29.76\%}) &      50 ({\scriptsize -6.38\%}) \\
Bray-Curtis (Std.)  &  0.4931 ({\scriptsize -8.64\%}) &  0.4935 ({\scriptsize -8.35\%}) &  0.6435 ({\scriptsize -21.25\%}) &  0.2906 ({\scriptsize -16.14\%}) &      50 ({\scriptsize -6.38\%}) \\
Euclidean           &  0.4871 ({\scriptsize -9.74\%}) &   0.4852 ({\scriptsize -9.9\%}) &   0.4303 ({\scriptsize 18.92\%}) &   0.4303 ({\scriptsize 24.19\%}) &  42.3333 ({\scriptsize 9.93\%}) \\
\bottomrule
\end{tabular}
\end{table*}

% \textbf{Discussion (CIFAR10 ResNet50).}  \textbf{Performance:} Cosine (Stable) achieves the highest acc among tested distances. Cosine (Stable) also yields the best F1, confirming its robustness. \textbf{Sustainability:} Cosine (Unst.) achieves the lowest gCO2eq, indicating higher sustainability. \textbf{Interpretability:} Chebyshev (Std.) produces the highest EV, suggesting more structured latent spaces. Chebyshev (Std.) reduces prototype coverage most strongly, indicating sharper clustering effects. Overall, the results highlight trade-offs between accuracy, sustainability, and interpretability: cosine and related distances often provide balanced improvements, while Mahalanobis tends to emphasize interpretability at the expense of efficiency.

\begin{table*}
%\small
\setlength{\tabcolsep}{8pt}
\centering
\caption{Results for CIFAR10 ResNet50. Parentheses: \% changes w.r.t. Baseline (Cross-Entropy).}
\label{tab:cifar10-resnet50}
\begin{tabular}{llllll}
\toprule
\textbf{Method} &                             \textbf{Acc} &                              \textbf{F1} &                       \textbf{gCO$_2$eq} &                             \textbf{EV} &                          \textbf{PC90\%} \\
\midrule
Baseline            &                         0.843 &                        0.8431 &                           48.65 &                        0.257211 &                      50.0 \\
Cosine (Stable)     &  0.9262 ({\scriptsize 9.87\%}) &  0.9262 ({\scriptsize 9.86\%}) &  40.6776 ({\scriptsize 16.39\%}) &   0.7559 ({\scriptsize 193.9\%}) &     5 ({\scriptsize 90\%}) \\
Cosine (Unst.)      &  0.9234 ({\scriptsize 9.54\%}) &  0.9234 ({\scriptsize 9.53\%}) &  29.3968 ({\scriptsize 39.58\%}) &   0.761 ({\scriptsize 195.86\%}) &     5 ({\scriptsize 90\%}) \\
Bray-Curtis (Norm.) &  0.9193 ({\scriptsize 9.05\%}) &  0.9192 ({\scriptsize 9.02\%}) &   45.6222 ({\scriptsize 6.23\%}) &  0.7883 ({\scriptsize 206.49\%}) &     5 ({\scriptsize 90\%}) \\
Chebyshev (Std.)    &   0.905 ({\scriptsize 7.36\%}) &   0.905 ({\scriptsize 7.34\%}) &   48.5505 ({\scriptsize 0.21\%}) &  0.9995 ({\scriptsize 288.59\%}) &     1 ({\scriptsize 98\%}) \\
Euclidean           &  0.9185 ({\scriptsize 8.96\%}) &  0.9185 ({\scriptsize 8.94\%}) &   45.8759 ({\scriptsize 5.71\%}) &   0.683 ({\scriptsize 165.56\%}) &  25.5 ({\scriptsize 49\%}) \\
\bottomrule
\end{tabular}
\end{table*}

% \textbf{Discussion (MNIST CNN).}  \textbf{Performance:} Bray-Curtis (Norm.) achieves the highest acc among tested distances. Bray-Curtis (Norm.) also yields the best F1, confirming its robustness. \textbf{Sustainability:} Mahalanobis (Chol.) achieves the lowest gCO2eq, indicating higher sustainability. \textbf{Interpretability:} Bray-Curtis (Norm.) produces the highest EV, suggesting more structured latent spaces. Baseline reduces prototype coverage most strongly, indicating sharper clustering effects. Overall, the results highlight trade-offs between accuracy, sustainability, and interpretability: cosine and related distances often provide balanced improvements, while Mahalanobis tends to emphasize interpretability at the expense of efficiency.

\begin{table*}
%\small
\setlength{\tabcolsep}{8pt}
\centering
\caption{Results for MNIST CNN. Parentheses: \% changes w.r.t. Baseline (Cross-Entropy).}
\label{tab:mnist-cnn}
\begin{tabular}{llllll}
\toprule
\textbf{Method} &                             \textbf{Acc} &                              \textbf{F1} &                       \textbf{gCO$_2$eq} &                             \textbf{EV} &                          \textbf{PC90\%} \\
\midrule
Baseline            &                        0.9782 &                        0.9782 &                           1.19 &                        0.585633 &                           10.6667 \\
Bray-Curtis (Norm.) &  0.9889 ({\scriptsize 1.09\%}) &  0.9888 ({\scriptsize 1.09\%}) &   1.1348 ({\scriptsize 4.42\%}) &   0.7225 ({\scriptsize 23.38\%}) &   13.6667 ({\scriptsize -28.12\%}) \\
Mahalanobis (Chol.) &     0.9879 ({\scriptsize 1\%}) &  0.9879 ({\scriptsize 0.99\%}) &  1.0639 ({\scriptsize 10.39\%}) &   0.4673 ({\scriptsize -20.2\%}) &  36.3333 ({\scriptsize -240.63\%}) \\
Minkowski (p=3.0)   &  0.9877 ({\scriptsize 0.97\%}) &  0.9876 ({\scriptsize 0.96\%}) &   1.1154 ({\scriptsize 6.06\%}) &  0.4195 ({\scriptsize -28.37\%}) &   49.3333 ({\scriptsize -362.5\%}) \\
Hamming (Soft)      &  0.9833 ({\scriptsize 0.52\%}) &  0.9832 ({\scriptsize 0.51\%}) &   1.1815 ({\scriptsize 0.49\%}) &  0.3089 ({\scriptsize -47.26\%}) &       50 ({\scriptsize -368.75\%}) \\
Euclidean           &   0.9831 ({\scriptsize 0.5\%}) &   0.9831 ({\scriptsize 0.5\%}) &   1.1543 ({\scriptsize 2.78\%}) &  0.4413 ({\scriptsize -24.65\%}) &   20.3333 ({\scriptsize -90.62\%}) \\
\bottomrule
\end{tabular}
\end{table*}

% \textbf{Discussion (MNIST MLP).}  \textbf{Performance:} Euclidean achieves the highest acc among tested distances. Euclidean also yields the best F1, confirming its robustness. \textbf{Sustainability:} Euclidean achieves the lowest gCO2eq, indicating higher sustainability. \textbf{Interpretability:} Chebyshev (Std.) produces the highest EV, suggesting more structured latent spaces. Chebyshev (Std.) reduces prototype coverage most strongly, indicating sharper clustering effects. Overall, the results highlight trade-offs between accuracy, sustainability, and interpretability: cosine and related distances often provide balanced improvements, while Mahalanobis tends to emphasize interpretability at the expense of efficiency.

\begin{table*}
%\small
\setlength{\tabcolsep}{8pt}
\centering
\caption{Results for MNIST MLP. Parentheses: \% changes w.r.t. Baseline (Cross-Entropy).}
\label{tab:mnist-mlp}
\begin{tabular}{llllll}
\toprule
\textbf{Method} &                             \textbf{Acc} &                              \textbf{F1} &                       \textbf{gCO$_2$eq} &                             \textbf{EV} &                          \textbf{PC90\%} \\
\midrule
Baseline            &                          0.976 &                         0.9758 &                           0.55 &                        0.565723 &                        10.3333 \\
Cosine (Unst.)      &     0.978 ({\scriptsize 0.2\%}) &    0.9778 ({\scriptsize 0.2\%}) &   0.5264 ({\scriptsize 3.58\%}) &   0.382 ({\scriptsize -32.48\%}) &       10 ({\scriptsize 3.23\%}) \\
Mahalanobis (Chol.) &   0.9774 ({\scriptsize 0.14\%}) &   0.9771 ({\scriptsize 0.14\%}) &  0.5611 ({\scriptsize -2.78\%}) &   0.092 ({\scriptsize -83.74\%}) &    50 ({\scriptsize -383.87\%}) \\
Cosine (Stable)     &   0.9766 ({\scriptsize 0.06\%}) &   0.9764 ({\scriptsize 0.06\%}) &   0.5266 ({\scriptsize 3.54\%}) &  0.4033 ({\scriptsize -28.71\%}) &   9.3333 ({\scriptsize 9.68\%}) \\
Chebyshev (Std.)    &  0.9756 ({\scriptsize -0.04\%}) &  0.9754 ({\scriptsize -0.04\%}) &  0.5881 ({\scriptsize -7.73\%}) &   0.7865 ({\scriptsize 39.03\%}) &  5.6667 ({\scriptsize 45.16\%}) \\
Euclidean           &    0.9799 ({\scriptsize 0.4\%}) &   0.9798 ({\scriptsize 0.41\%}) &   0.5221 ({\scriptsize 4.35\%}) &   0.358 ({\scriptsize -36.72\%}) &        9 ({\scriptsize 12.9\%}) \\
\bottomrule
\end{tabular}
\end{table*}

% \textbf{Discussion (MNIST ResNet50).}  \textbf{Performance:} Bray-Curtis (Norm.) achieves the highest acc among tested distances. Bray-Curtis (Norm.) also yields the best F1, confirming its robustness. \textbf{Sustainability:} Euclidean achieves the lowest gCO2eq, indicating higher sustainability. \textbf{Interpretability:} Euclidean produces the highest EV, suggesting more structured latent spaces. Mahalanobis (Chol.) reduces prototype coverage most strongly, indicating sharper clustering effects. Overall, the results highlight trade-offs between accuracy, sustainability, and interpretability: cosine and related distances often provide balanced improvements, while Mahalanobis tends to emphasize interpretability at the expense of efficiency.

\begin{table*}
%\small
\setlength{\tabcolsep}{8pt}
\centering
\caption{Results for MNIST ResNet50. Parentheses: \% changes w.r.t. Baseline (Cross-Entropy).}
\label{tab:mnist-resnet50}
\begin{tabular}{llllll}
\toprule
\textbf{Method} &                             \textbf{Acc} &                              \textbf{F1} &                       \textbf{gCO$_2$eq} &                             \textbf{EV} &                          \textbf{PC90\%} \\
\midrule
Baseline            &                        0.9909 &                        0.9909 &                           29.36 &                        0.420353 &                   50.0 \\
Bray-Curtis (Norm.) &  0.9962 ({\scriptsize 0.52\%}) &  0.9961 ({\scriptsize 0.53\%}) &  25.2889 ({\scriptsize 13.86\%}) &  0.8453 ({\scriptsize 101.09\%}) &  4 ({\scriptsize 92\%}) \\
Cosine (Unst.)      &   0.996 ({\scriptsize 0.51\%}) &   0.996 ({\scriptsize 0.52\%}) &   26.1851 ({\scriptsize 10.8\%}) &   0.6888 ({\scriptsize 63.87\%}) &  6 ({\scriptsize 88\%}) \\
Cosine (Stable)     &  0.9953 ({\scriptsize 0.44\%}) &  0.9953 ({\scriptsize 0.45\%}) &  26.4064 ({\scriptsize 10.05\%}) &   0.6974 ({\scriptsize 65.91\%}) &  6 ({\scriptsize 88\%}) \\
Mahalanobis (Chol.) &  0.9938 ({\scriptsize 0.29\%}) &   0.9938 ({\scriptsize 0.3\%}) &  31.9246 ({\scriptsize -8.75\%}) &  0.9966 ({\scriptsize 137.09\%}) &  1 ({\scriptsize 98\%}) \\
Euclidean           &  0.9934 ({\scriptsize 0.25\%}) &  0.9934 ({\scriptsize 0.25\%}) &   24.457 ({\scriptsize 16.69\%}) &  0.9998 ({\scriptsize 137.84\%}) &  1 ({\scriptsize 98\%}) \\
\bottomrule
\end{tabular}
\end{table*}

% \subsection{Vision: Interpretability} % Regular MLP + GRAD-CAM for CNN and PVT
% % 4 architectures x 3 datasets

% \textbf{MNIST}

% \textbf{CIFAR10}

% \textbf{CIFAR100}

% \textbf{Insights across datasets}

% \textbf{Cross-architecture insights}

% \subsection{Vision: Sustainability}
% \subsubsection{Insights across datasets}
% \subsubsection{Cross-architecture insights}
% \subsubsection{Image datasets: MNIST}
% \subsubsection{Image datasets: CIFAR10}
% \subsubsection{Image datasets: CIFAR100}

\subsection{Vision: Sustainability}
% \textbf{MNIST:} The relative emissions differences are modest due to the dataset’s simplicity and the small computational footprint of training. These results are presented in Figure \ref{fig:emissions:vision:MNIST}.
% For \textbf{MLPs}, several distances such as Euclidean, Manhattan, and Canberra variants slightly increase emissions, while others (e.g., Mahalanobis diagonal or Chebyshev smooth) yield small reductions compared to cross-entropy. 
% In \textbf{CNNs}, more pronounced negative bars are observed: for example, Hamming-gumbel and Canberra-standard produce lower carbon costs than the baseline. 
% For \textbf{PVT}, emissions are generally higher than the baseline, particularly with Mahalanobis and Hamming variants (see Figure \ref{fig:emissions:vision:MNIST}), reflecting higher computational overhead. 
% In \textbf{ResNet50}, differences are small but several distances (e.g., Euclidean, Hamming-soft, Mahalanobis-cholesky) achieve slightly lower emissions, suggesting that even deep models can sometimes benefit from distance-based losses in sustainability terms.

\subsubsection{MNIST}
Figure~\ref{fig:emissions:vision:MNIST} summarizes the \emph{carbon deltas} (gCO$_2$eq relative to cross-entropy) when swapping the training objective for harmonic-loss variants on MNIST across four backbones.

\textbf{MLP.} Most distances reduce per–step emissions vs.\ cross-entropy (green bars), with the largest savings from heavier geometry that replaces the softmax/cross-entropy path (e.g., Mahalanobis/standardized, Chebyshev). Euclidean and Bray--Curtis yield modest savings; only a few variants show small positive overheads. Given MNIST’s simplicity and the near-saturation accuracies, these reductions likely translate into \emph{net} greener runs because steps-to-target are comparable.

\textbf{CNN.} A broad set of distances are carbon-negative vs.\ baseline. Again, standardized Mahalanobis/Chebyshev rank among the lowest-emission options; Bray--Curtis and Euclidean remain consistently frugal. Variants that introduce extra normalization or temperature schedules can erode part of the gain but rarely flip the sign.

\textbf{ResNet50.} The deepest convolutional model shows the \emph{largest} per–step savings: many distances deliver substantial negative deltas relative to cross-entropy, suggesting that replacing the softmax loss with metric-based objectives amortizes well at this scale. Only a handful of choices (e.g., certain Chebyshev/Canberra parameterizations) incur small positive overheads.

\textbf{PVT (vision transformer).} In contrast to the CNN family, most distances \emph{increase} per–step emissions over the baseline. The transformer’s attention and normalization stack appears less amenable to the heavier distance computations; only a couple of standardized/normalized variants produce small savings. On PVT, greener training favors the lightest geometries or retaining cross-entropy.

\textbf{Takeaways.} i) On MNIST, distance-based harmonic losses are often \emph{carbon-favorable} for MLP/CNN/ResNet50, with the biggest gains on the deepest CNN; ii) these gains are not universal—PVT tends to pay a premium; iii) because test accuracy curves on MNIST converge similarly across losses, the per–step savings for CNN/ResNet50 likely convert into lower \emph{end-to-end} energy. Practically, we recommend Euclidean/Bray--Curtis/standardized Mahalanobis for convolutional backbones, and cautious use (or kernel-fused, mixed-precision implementations) of heavier distances on transformer-style models. Reporting both per–step emissions and energy-to-target accuracy remains essential for fair sustainability claims.

%%%% EMISSIONS : VISION: MNIST 

\begin{figure*}[h!]
    \begin{centering}
	    \includegraphics[page=6,width=0.495\textwidth]{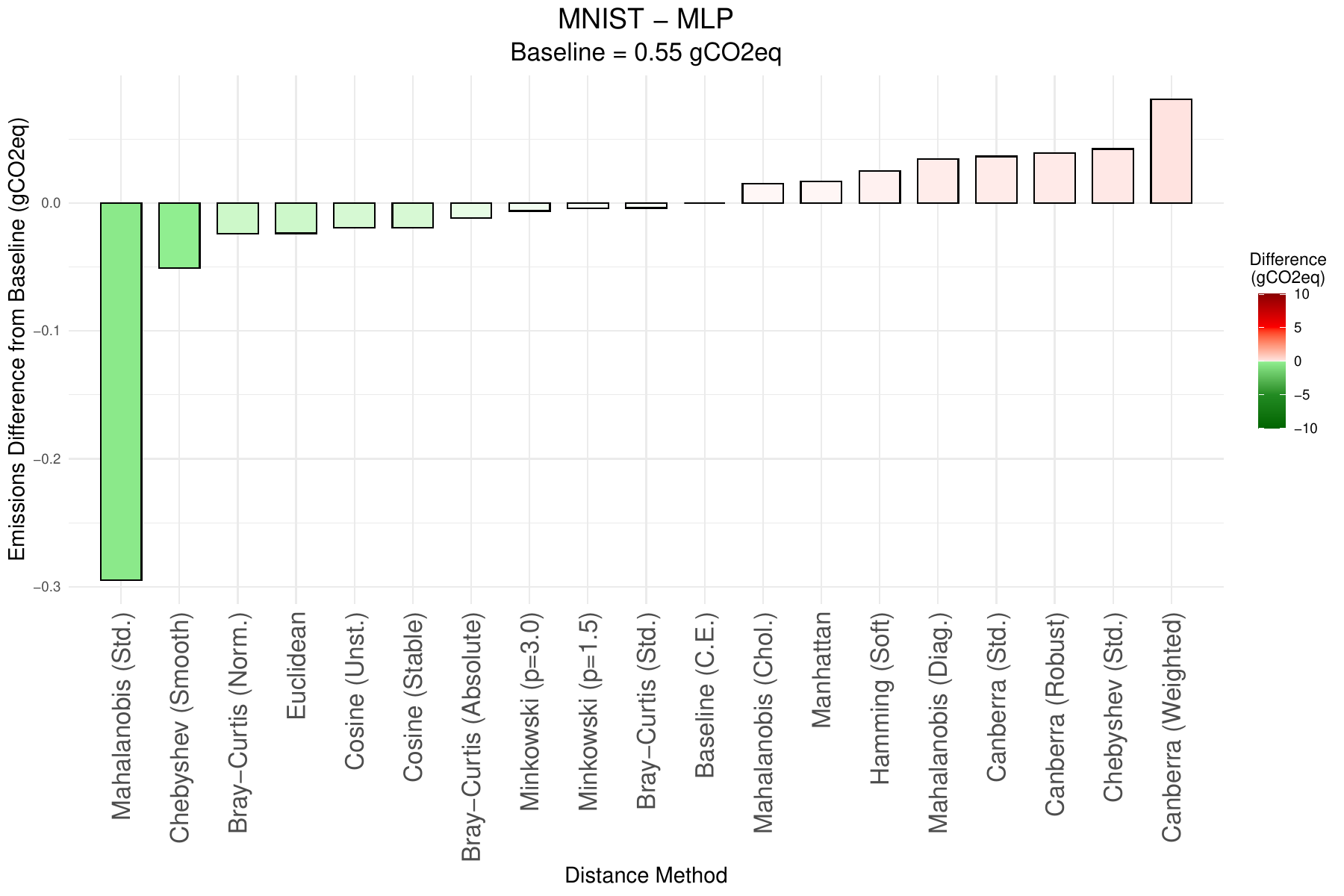}
	    \includegraphics[page=3,width=0.495\textwidth]{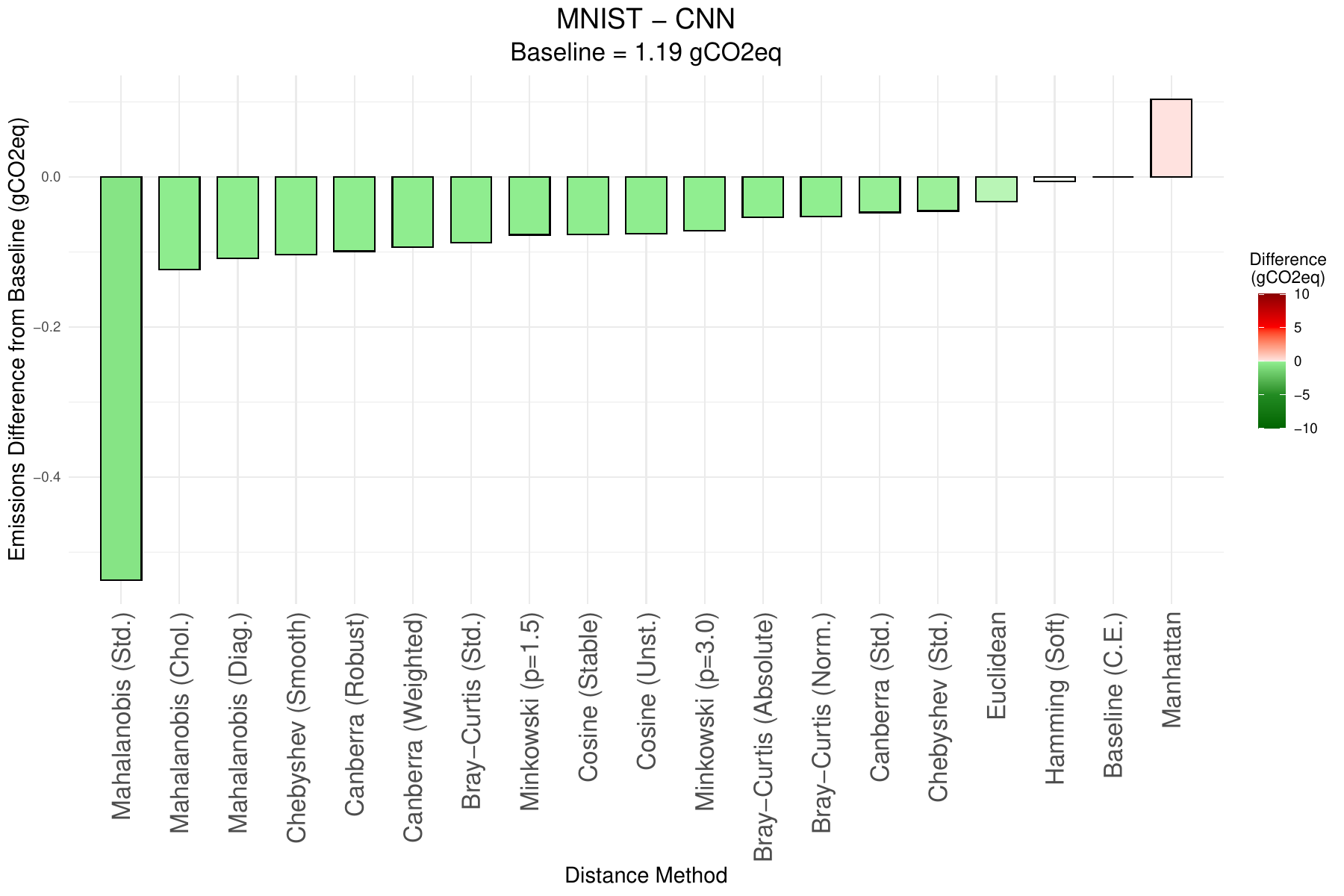}
	    \includegraphics[page=12,width=0.495\textwidth]{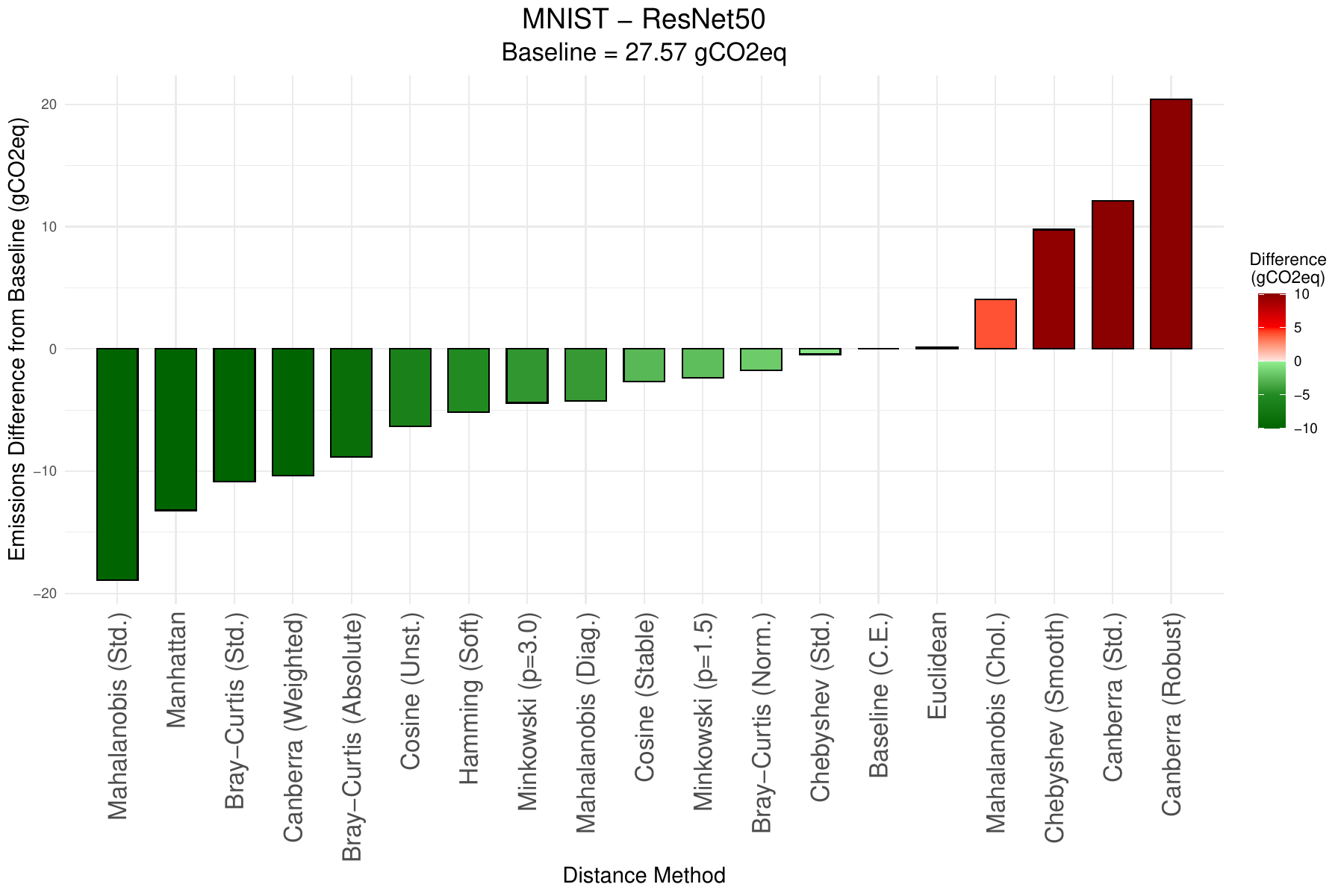}
	    \includegraphics[page=9,width=0.495\textwidth]{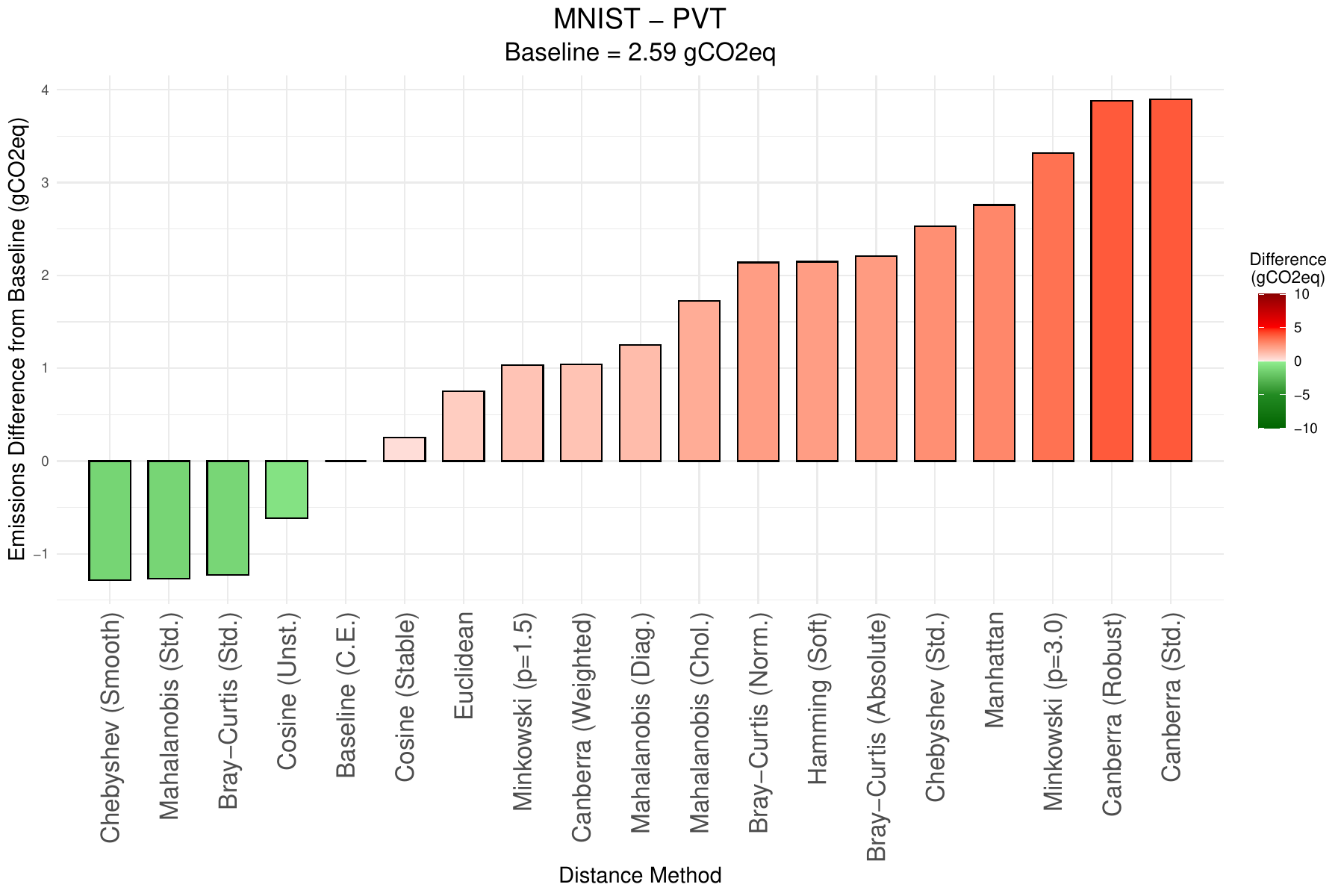}
	    \caption{Carbon emission differences for MNIST across four model backbones (MLP, CNN, ResNet50, PVT) when replacing cross-entropy with harmonic loss variants. Bars show the emission difference in grams of CO$_2$eq relative to the baseline (cross-entropy). Values above zero indicate higher emissions than baseline, while negative values indicate greener, more sustainable outcomes.} % For instance, as seen in the CNN plot, Hamming-gumbel and Canberra-standard reduce emissions relative to cross-entropy, while in ResNet, several distances such as Mahalanobis and Hamming substantially exceed baseline emissions.}
	    \Description[MNIST emissions deltas by backbone and loss]{A four-panel set of bar charts reporting per-configuration carbon emission differences on MNIST when replacing cross-entropy with harmonic loss variants. Each panel corresponds to a different model backbone (MLP, CNN, ResNet50, and PVT). Within each panel, bars represent different distance-based loss variants and show the change in emissions (in grams of CO2-equivalent) relative to the cross-entropy baseline, with bars above zero indicating higher emissions and bars below zero indicating lower emissions.}
	    \label{fig:emissions:vision:MNIST}
	    \end{centering}
\end{figure*}

\subsubsection{CIFAR-10}
%
%\textbf{CIFAR-10}: 
% Emissions differences become more substantial as model complexity increases, as shown in Figure \ref{fig:emissions:vision:CIFAR10}.
% For \textbf{MLPs}, most distance variants remain close to the baseline, with Mahalanobis and Bray--Curtis showing slight increases. 
% \textbf{CNNs} show a mixed picture: Minkowski distances and Chebyshev-standard exceed baseline emissions, while some Canberra variants remain neutral. 
% In \textbf{PVT}, nearly all distances produce lower emissions, with robust Canberra, Cosine-stable, and Bray--Curtis-normalized yielding the strongest reductions. This indicates that harmonic losses can substantially improve sustainability in transformer-style architectures. 
% For \textbf{ResNet50}, most distances result in emissions close to or slightly below baseline. Bray--Curtis-absolute and Euclidean appear particularly sustainable, with negative or near-zero differences relative to cross-entropy.

%\textbf{CIFAR-10 sustainability across architectures (per--1k-step emissions).}
Figure~\ref{fig:emissions:vision:CIFAR10} reports carbon deltas in gCO$_2$eq relative to cross-entropy when training with harmonic-loss distances on CIFAR-10.

\textbf{MLP.} Most distances are \emph{carbon–negative} versus baseline, yielding small--to–moderate per–step savings. A few choices incur mild overheads (rightmost bars), indicating that added normalization or temperature scheduling can offset the gains on shallow networks.

\textbf{CNN.} The pattern strengthens: a broad set of distances reduce per–step emissions relative to cross-entropy. Only a handful of variants sit near zero or slightly positive, suggesting that, for convolutional encoders on CIFAR-10, metric-based objectives are generally more frugal per step.

\textbf{ResNet50.} Savings are \emph{uniform and largest}: all distances fall below the baseline, with substantial negative deltas. This indicates that replacing the softmax loss amortizes particularly well at depth/width, likely due to better kernel utilization and reduced softmax/backprop overhead relative to the total compute.

\textbf{PVT (vision transformer).} Most distances are again carbon–negative, though the spread is narrower than ResNet50 and a couple of variants hover around parity or slightly positive. Transformers benefit, but less dramatically than deep CNNs.

\textbf{Takeaways.} i) On CIFAR-10, distance-based harmonic losses are typically \emph{greener per step} for CNN/ResNet50/PVT, with the strongest effect on ResNet50; ii) MLP shows mixed but mostly favorable outcomes; iii) because our accuracy-vs-epoch curves on CIFAR-10 show similar or faster convergence for several distances, these per–step gains are likely to translate into lower \emph{end-to-end} energy for deep backbones. Practically, we recommend adopting the more frugal distances for convolutional and transformer models and pairing per–step reports with \emph{energy-to-target-accuracy} to substantiate sustainability claims.

%%%% EMISSIONS : VISION: CIFAR10 

\begin{figure}[h!]
    \begin{centering}
	    \includegraphics[page=4,width=0.4\textwidth]{figures/Image_Classification/Emissions/emissions_plots_image_classification_grams.pdf}
	    \includegraphics[page=1,width=0.4\textwidth]{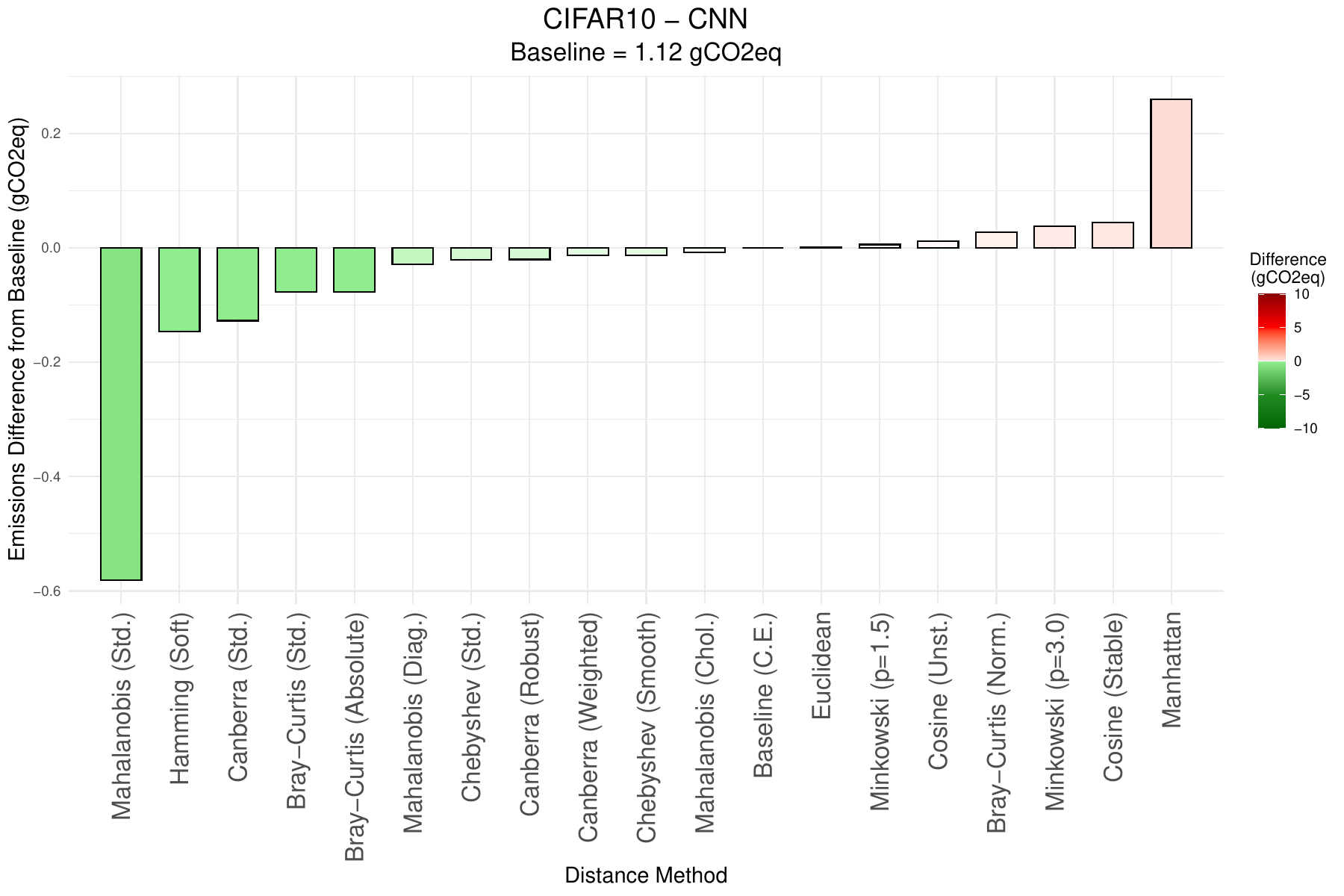}
	    \includegraphics[page=10,width=0.4\textwidth]{figures/Image_Classification/Emissions/emissions_plots_image_classification_grams.pdf}
	    \includegraphics[page=7,width=0.4\textwidth]{figures/Image_Classification/Emissions/emissions_plots_image_classification_grams.pdf}
	    \caption{Carbon emission differences for CIFAR10 across four model backbones (MLP, CNN, ResNet50, PVT) when replacing cross-entropy with harmonic loss variants. Bars show the emission difference in grams of CO$_2$eq relative to the baseline (cross-entropy). Values above zero indicate higher emissions than baseline, while negative values indicate greener, more sustainable outcomes.}
	    \Description[CIFAR-10 emissions deltas by backbone and loss]{A four-panel set of bar charts reporting per-configuration carbon emission differences on CIFAR-10 when replacing cross-entropy with harmonic loss variants. Panels correspond to four backbones (MLP, CNN, ResNet50, and PVT). Each bar shows the emissions difference (grams of CO2-equivalent) for a specific distance-based loss relative to cross-entropy, where negative values indicate reduced emissions and positive values indicate increased emissions.}
	    \label{fig:emissions:vision:CIFAR10}
	    \end{centering}
\end{figure}

\subsubsection{CIFAR-100}

% \textbf{CIFAR-100}: The sustainability patterns diverge further, as shown in Figure \ref{fig:emissions:vision:CIFAR100}. 
% For \textbf{MLPs}, results cluster near the baseline, with some minor negative differences (e.g., Minkowski-$p{=}1.5$, Canberra-weighted). 
% \textbf{CNNs} show clearer sustainability gains, with Hamming-gumbel and Canberra-weighted consistently yielding lower emissions than cross-entropy. 
% In \textbf{PVT}, most distances substantially exceed baseline emissions, with increases up to several tens of grams of CO$_2$eq. This reflects the high computational cost of transformers on complex datasets, where distance-based harmonic losses can add overhead. 
% Conversely, \textbf{ResNet50} shows strong sustainability advantages: multiple distances (e.g., Hamming-gumbel, Canberra-standard, Bray--Curtis-absolute) achieve lower emissions compared to baseline, while others remain close to neutral.

%\textbf{CIFAR-100 sustainability across architectures (per--1k-step emissions).}
Figure~\ref{fig:emissions:vision:CIFAR100} shows the carbon \emph{delta} (gCO$_2$eq vs.\ cross-entropy) when training with harmonic-loss distances on CIFAR-100.

\textbf{MLP.} Savings are modest and \emph{geometry–dependent}. Light/standardized variants (e.g., cosine, Euclidean, some Minkowski/Canberra settings) are carbon–negative, while heavier norms and covariance–based Mahalanobis parameterizations flip to positive overheads. On shallow models, extra normalization steps can outweigh gains.

\textbf{CNN.} A broad swath of distances are carbon–negative relative to the 1.18\,gCO$_2$eq baseline; several Mahalanobis and Bray–Curtis settings deliver the largest per–step reductions. A few choices (e.g., certain cosine/Canberra/Minkowski configurations) hover near parity or slightly positive, indicating mild architecture sensitivity.

\textbf{ResNet50.} The deepest convolutional model exhibits a \emph{mixed but wide} spread: many distances achieve substantial savings (left cluster of dark-green bars), yet others incur clear premiums (right cluster). Thus, distance choice materially changes footprint at scale. Notably, cosine variants are among the frugal options here, whereas some Chebyshev/Minkowski/Bray–Curtis (absolute) settings are costlier.

\textbf{PVT (vision transformer).} Most distances are \emph{carbon–positive} vs.\ the 3.67\,gCO$_2$eq baseline, with only a couple of standardized/smoothed variants slightly negative. As on MNIST/CIFAR-10, the attention/normalization stack appears less amenable to heavier metric computations.

\textbf{Takeaways.} i) On CIFAR-100, harmonic distances can be \emph{greener per step} for CNNs and selectively for ResNet50, but PVT generally pays a premium; ii) cosine tends to be frugal on deeper CNNs (and competitive on MLP), aligning with its strong accuracy dynamics, whereas several Mahalanobis/Minkowski/Chebyshev configurations increase emissions unless they deliver clear quality gains; iii) because CIFAR-100 accuracy converges differently across distances, claims of sustainability should couple per–step deltas with \emph{energy-to-target-accuracy/perplexity}. Practically, prefer cosine/Euclidean/standardized Bray–Curtis (and selected Mahalanobis settings that are both stable and frugal) for CNN/ResNet50, and use kernel fusion + mixed precision if heavier geometries are needed on transformer backbones.

%%%% EMISSIONS : VISION: CIFAR10 

\begin{figure*}[h!]
    \begin{centering}
	    \includegraphics[page=5,width=0.495\textwidth]{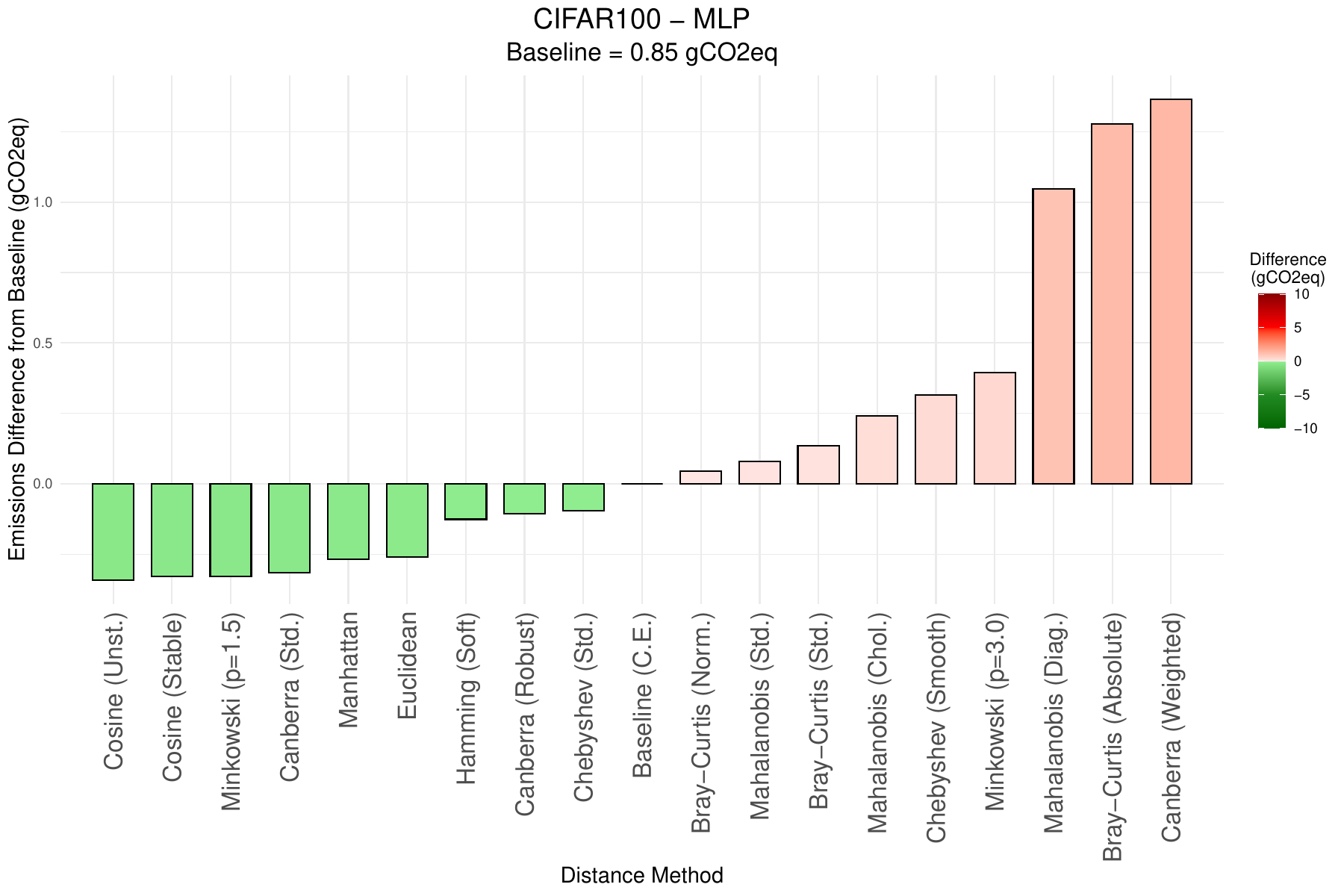}
	    \includegraphics[page=2,width=0.495\textwidth]{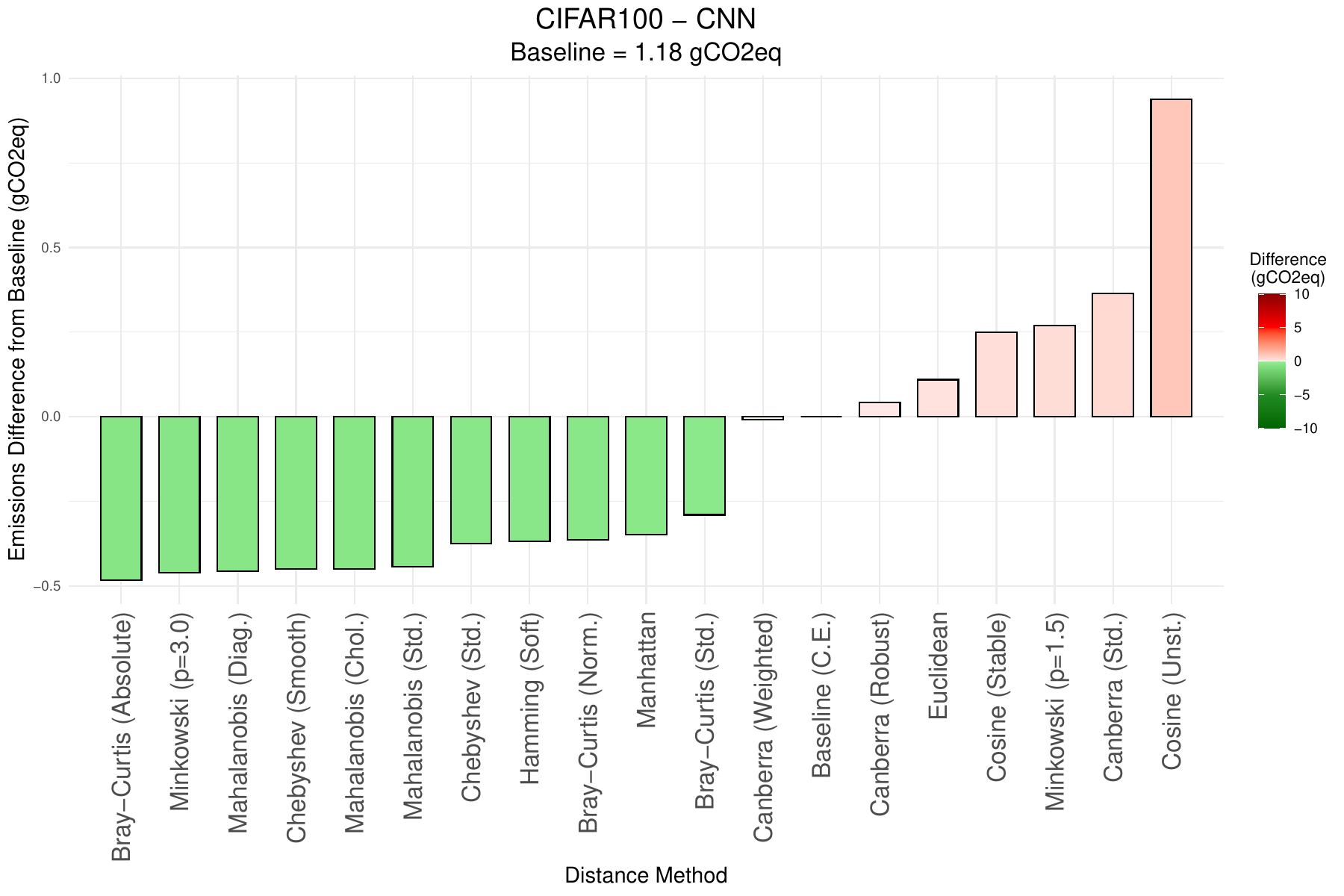}
	    \includegraphics[page=11,width=0.495\textwidth]{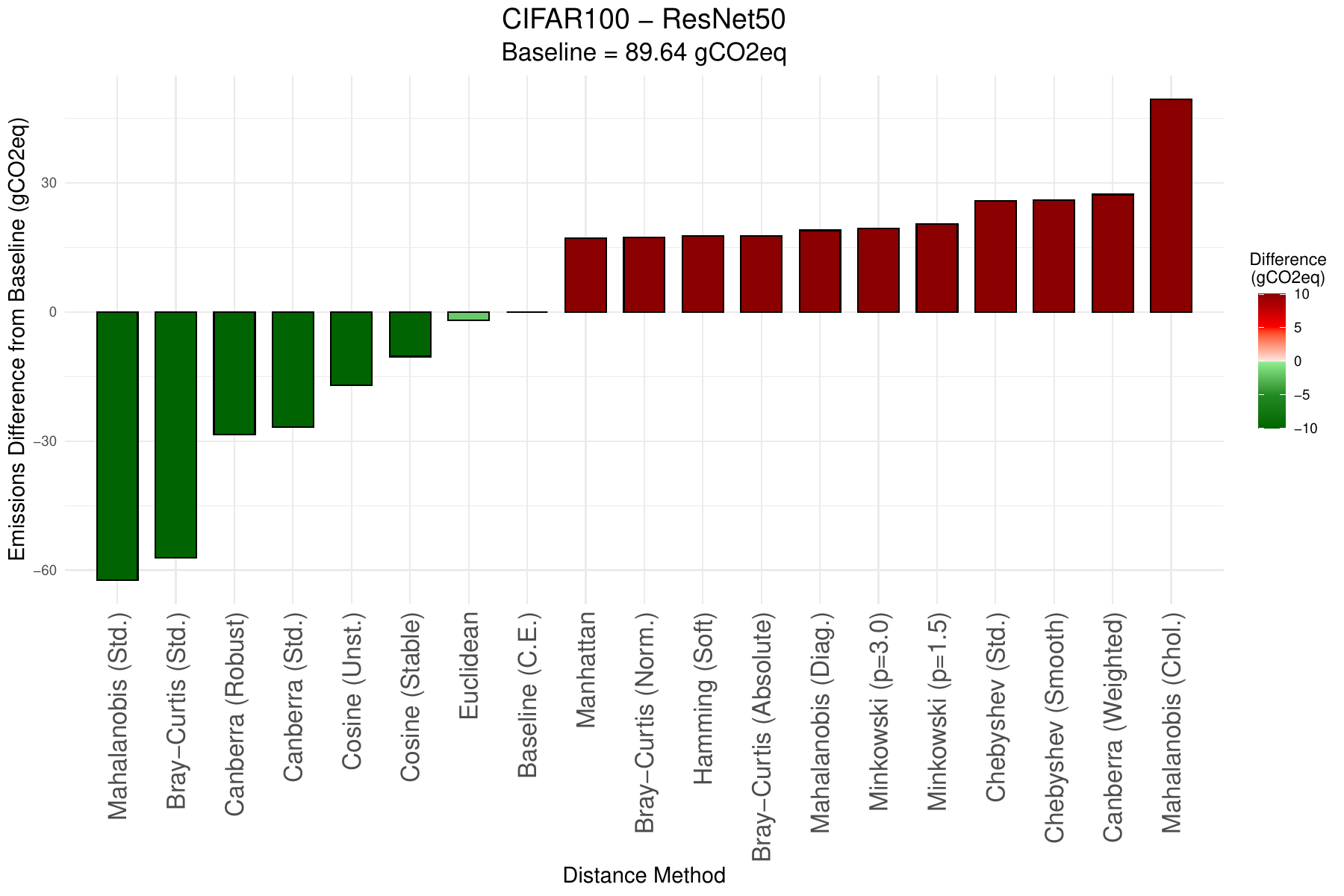}
	    \includegraphics[page=8,width=0.495\textwidth]{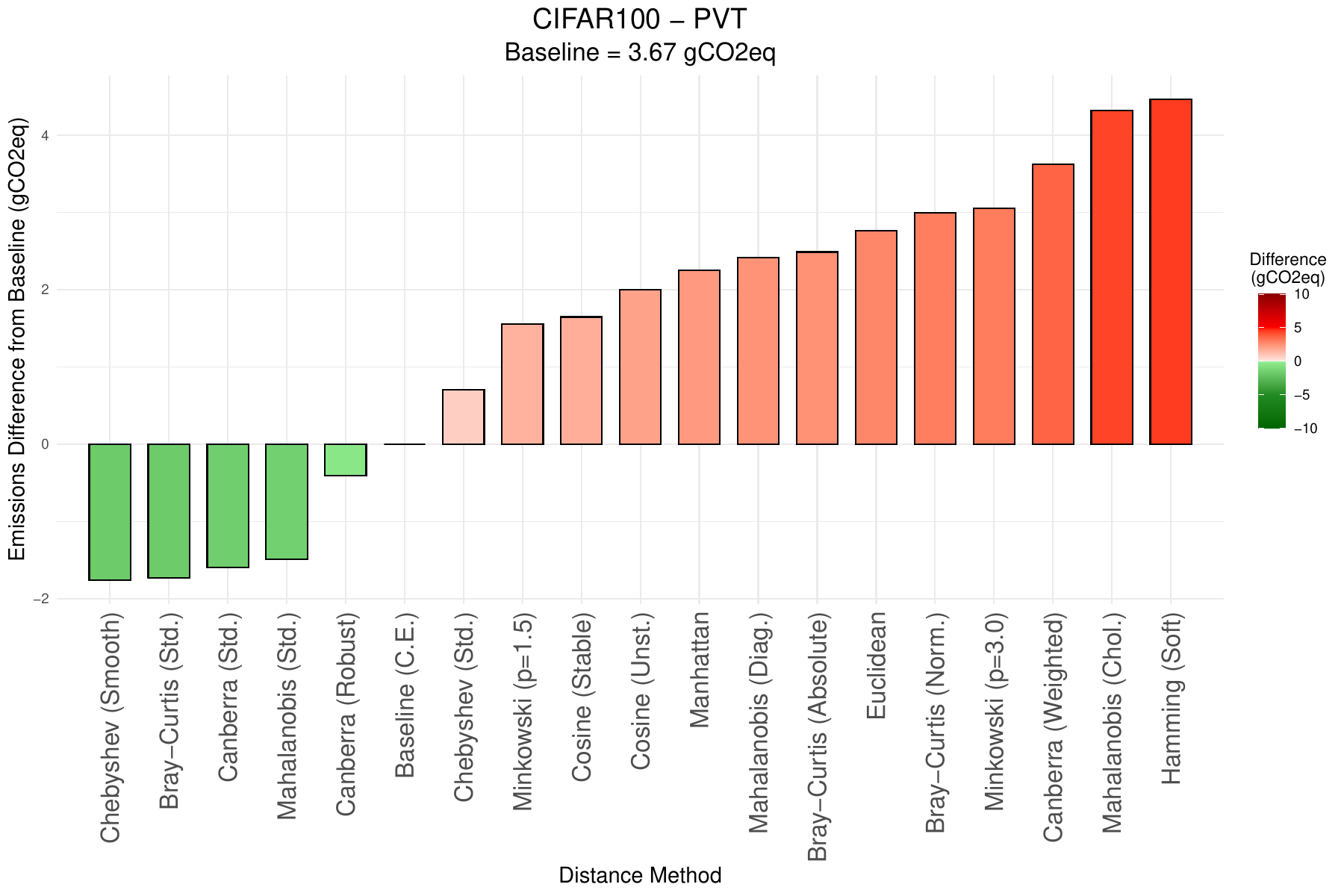}
	    \caption{Carbon emission differences for CIFAR100 across four model backbones (MLP, CNN, ResNet50, PVT) when replacing cross-entropy with harmonic loss variants. Bars show the emission difference in grams of CO$_2$eq relative to the baseline (cross-entropy). Values above zero indicate higher emissions than baseline, while negative values indicate greener, more sustainable outcomes.}
	    \Description[CIFAR-100 emissions deltas by backbone and loss]{A four-panel set of bar charts reporting per-configuration carbon emission differences on CIFAR-100 when replacing cross-entropy with harmonic loss variants. Panels correspond to the MLP, CNN, ResNet50, and PVT backbones. For each backbone, bars enumerate distance-based loss variants and report the emissions difference in grams of CO2-equivalent relative to cross-entropy; negative bars indicate lower emissions than baseline and positive bars indicate higher emissions.}
	    \label{fig:emissions:vision:CIFAR100}
	    \end{centering}
\end{figure*}

%\begin{tcolorbox}
\textbf{Insights across datasets:}
A clear trend emerges across datasets: \textbf{transformer models (PVT)} often incur higher emissions with distance-based harmonic losses, particularly on CIFAR-100 (see Figure \ref{fig:emissions:vision:CIFAR100}), whereas \textbf{convolutional and residual networks} (CNN, ResNet50) frequently yield greener outcomes (see results in Figures \ref{fig:emissions:vision:MNIST} -- \ref{fig:emissions:vision:CIFAR100}). The sustainability benefit is especially pronounced when distances incorporate robustness (Hamming-gumbel, Canberra-robust) or covariance awareness (Mahalanobis-diagonal). Simpler datasets like MNIST show limited differences, while CIFAR-10 and CIFAR-100 highlight the greater impact of distance choice on carbon footprint.
%\end{tcolorbox}

\textbf{Cross-architecture insights:} \textbf{MLPs} present a limited sustainability differences; emissions remain close to baseline across all distances. With \textbf{CNNs}, multiple distances (Hamming-gumbel, Mahalanobis-diagonal, Canberra-weighted) consistently reduce emissions, showing CNNs benefit most from harmonic loss efficiency. In \textbf{PVT}, harmonic losses generally increase emissions, especially on CIFAR-100, highlighting potential overhead in attention-based models. \textbf{ResNet50} demonstrates an effective integration with several distances (Hamming, Canberra, Bray--Curtis), which achieve significant reductions in emissions over baseline, indicating that deep CNNs can combine effectiveness with sustainability.

Overall, the sustainability analysis shows that harmonic losses can improve or degrade carbon efficiency depending on the backbone and dataset. The choice of distance measure therefore plays a critical role not only in accuracy but also in environmental impact, reinforcing the need for holistic evaluation across the accuracy--sustainability--interpretability triangle.

% \subsection{LLMs: Interpretability}
% \textbf{Cross-architecture insights}

% \textbf{GPT2 small}

% \textbf{BERT}

% \textbf{QWEN 0.5B}

\subsection{Language: Sustainability}
% \textbf{LLM pretraining emissions (OpenWebText).}
Figure~\ref{fig:emissions:LLMs} reports \emph{per–1k-step} carbon differences (gCO$_2$eq) when replacing cross-entropy with distance-based harmonic losses for BERT, GPT, and QWEN. Positive bars indicate higher emissions than the cross-entropy baseline (annotated atop each subplot).

\textbf{Overall.} Across all three backbones, distance-based losses tend to \emph{increase} per–1k-step emissions relative to cross-entropy. The magnitude of overhead correlates with the computational complexity of the distance: lightweight cosine variants add the least overhead, while Mahalanobis and Minkowski incur the most.

\textbf{BERT.} Cosine (simple or temperature-scaled) yields small overheads (low single-digit gCO$_2$eq over a 7.87 gCO$_2$eq baseline), suggesting that the extra normalization and dot-product operations have modest cost. Euclidean and Bray--Curtis sit mid-pack, whereas Mahalanobis (Cholesky/standard/diagonal) and Minkowski ($p>2$) are consistently more carbon intensive per 1k steps.

\textbf{GPT.} All distances increase emissions over the 60.36 gCO$_2$eq baseline, with a clearer spread: cosine remains the most frugal among alternatives; Euclidean and Manhattan are mid-range; Mahalanobis (any parameterization) and Minkowski/L2 are the heaviest. This indicates that the per-step FLOPs and memory traffic of covariance-related computations (and higher-order norms) become more pronounced at GPT scale.

\textbf{QWEN.} For this larger model (baseline 75.29 gCO$_2$eq), the methods we evaluated (Minkowski/L2 and Euclidean) both raise per–1k-step emissions, with Minkowski/L2 showing a substantial increase. Although the set of distances is smaller here, the pattern mirrors GPT: heavier metrics cost more per step as model width/depth grows.

\textbf{Implications.} i) If \emph{Green AI} considerations are primary, cosine-based harmonic losses are the most promising drop-in replacements, especially on encoder-style models (BERT). ii) Mahalanobis and Minkowski should be justified by clear accuracy or stability gains, as they carry the largest per-step carbon premiums. iii) Reported values are per–1k-step; end-to-end footprint also depends on \emph{steps-to-target-quality}. Thus, a distance that reduces time-to-accuracy could still yield net carbon savings even with higher per-step cost.

% \textbf{Recommendations.} For fair accounting, pair per–1k-step metrics with (a) wall-clock and energy-to-target perplexity/accuracy; (b) mixed precision and kernel fusion (e.g., fused norm/cosine) to reduce overheads; (c) batch-size tuning and \texttt{torch.compile}/XLA to lower memory-bound penalties of Mahalanobis; and (d) ablations isolating backward-pass cost, since gradient calculations dominate for the heavier distances.

\textbf{Summary.} Distance choice in harmonic loss is not carbon-neutral: cosine variants introduce minimal overhead; Euclidean/Bray--Curtis are moderate; Mahalanobis/Minkowski are expensive. Any claimed performance gains from richer geometries should be weighed against these systematic energy costs, preferably via \emph{energy-normalized} quality metrics (e.g., accuracy per kWh).

\begin{figure*}[h!]
    \begin{centering}
	    \includegraphics[page=1,width=0.495\textwidth]{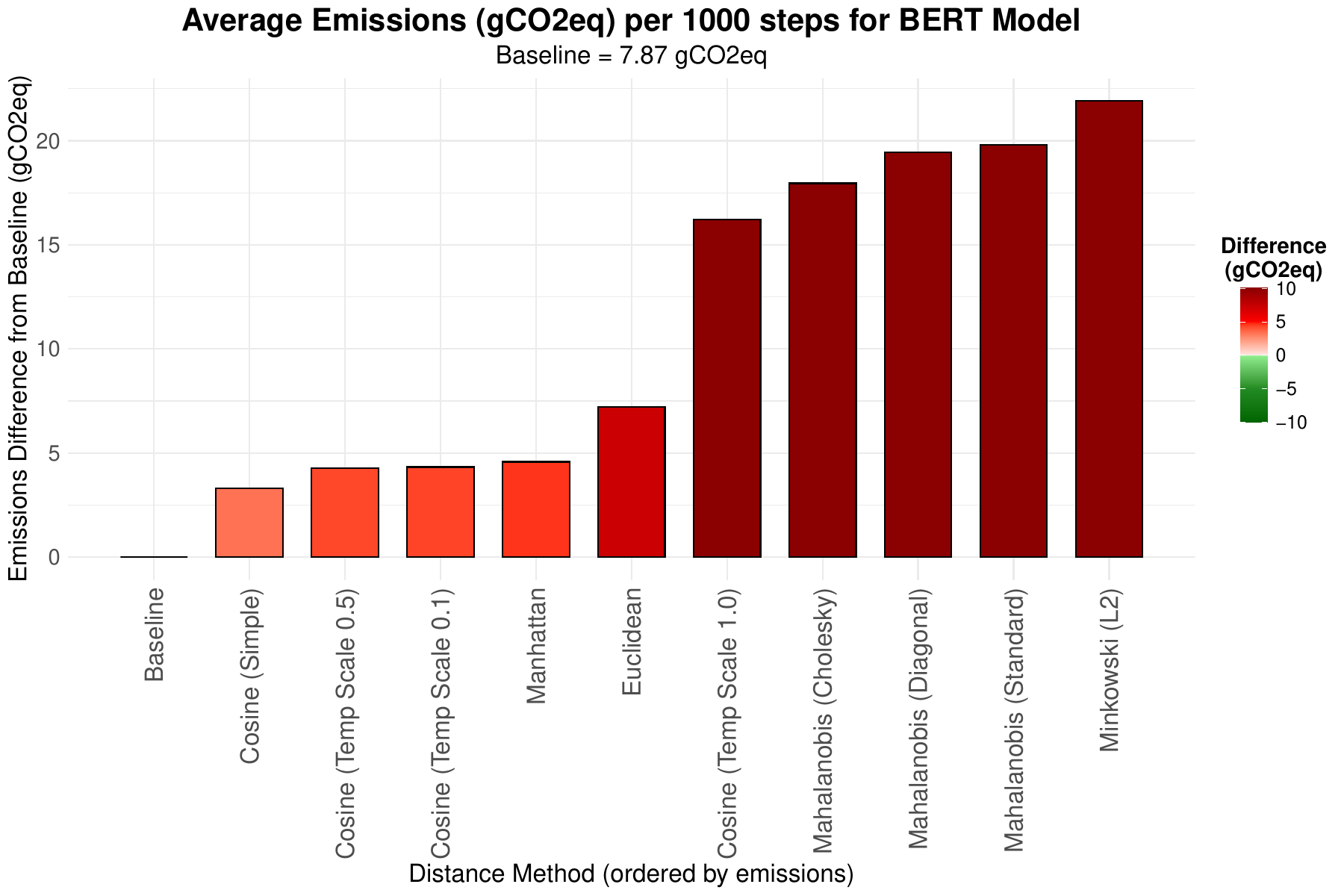}
	    \includegraphics[page=2,width=0.495\textwidth]{figures/LLMs/Emissions/emissions_plots_per_model.pdf}
	    \includegraphics[page=3,width=0.495\textwidth]{figures/LLMs/Emissions/emissions_plots_per_model.pdf}
	    \caption{Carbon emission differences for LLM pretraining on OpenWebText (BERT, GPT2, QWEN) when replacing cross-entropy with harmonic loss variants. Bars show the emission difference in grams of CO$_2$eq relative to the baseline (cross-entropy). Values above zero indicate higher emissions than baseline, while negative values indicate greener, more sustainable outcomes.}
	    \Description[LLM pretraining emissions deltas by model and loss]{A three-panel set of bar charts showing carbon emission differences during language-model pretraining on OpenWebText when swapping cross-entropy for harmonic loss variants. Each panel corresponds to a different model family (BERT, GPT-2, and Qwen). Bars represent different distance-based loss variants and report the emissions difference (grams of CO2-equivalent) relative to cross-entropy; bars above zero indicate increased emissions and bars below zero indicate reduced emissions.}
	    \label{fig:emissions:LLMs}
	    \end{centering}
\end{figure*}

% \textbf{Cross-architecture insights}

% \textbf{GPT2 small}

% \textbf{BERT}

% \textbf{QWEN 0.5B}

\subsection{Language: Interpretability}

%\textbf{Motivation.}
Mechanistic and representation-level interpretability of large language models (LLMs) increasingly leverages the hypothesis that internal activations admit \emph{approximately linear} structure: many features behave like directions in an activation space, and linear operations can steer or probe them \citep{anthropic_superposition_2022,huben2024sae,turntrout2023steering}. Within this paradigm, Principal Component Analysis (PCA) is a simple, well-understood lens for: i) summarizing dominant sources of variance in activations; ii) stabilizing analyses by denoising; and (iii) producing human-auditable axes that can be inspected, correlated with concepts, and tracked over time.

%\textbf{What PCA gives you for LLMs.}
Given a layer $\ell$ with residual-stream activations $H_\ell \in \mathbb{R}^{N \times d}$ collected across $N$ tokens (or prompts), PCA factorizes $H_\ell$ via SVD to yield orthogonal directions $\{u_k\}_{k=1}^d$ ordered by explained variance. In practice this supports:
\begin{enumerate}
    \item \textbf{Concept probing and visualization.} Projections onto top PCs often align with semantically meaningful contrasts; e.g., the first PC of GPT-style embeddings correlated with human well-being judgments in zero-shot tests \citep{farai_wellbeing_2023}, and per-layer PCA can reconstruct or predict response modes in GPT-2 \citep{jorgensen2023activations}.
    \item \textbf{Diagnosing and localizing phenomena.} Layer-wise or head-wise PCA reveals where variance concentrates, helping localize depth at which concepts emerge or consolidate (complementary to linear probing) \citep{concept_depth_2024}. Tracking \emph{subspace distance} across checkpoints detects representational drift during fine-tuning or domain shift.
    \item \textbf{Sanity checks and baselines.} With growing interest in sparse autoencoders (SAEs) for monosemantic features \citep{huben2024sae}, PCA serves as a transparent baseline decomposition: if SAEs meaningfully improve sparsity/faithfulness over PCA while matching reconstruction, that strengthens the interpretability claim \citep{alignmentforum_sae_2023}.
\end{enumerate}

%\textbf{When PCA is justified.}
PCA is most compelling under: a) approximately linear feature superposition and b) high signal-to-noise in dominant directions. Toy and empirical studies argue that Transformers often encode many features as \emph{directions} (superposition) \citep{anthropic_superposition_2022}, and even simple linear additions to activations can steer model behavior \citep{turntrout2023steering}. PCA then becomes an appropriate first-pass tool to:
\begin{itemize}
    \item extract high-variance axes that frequently correlate with coherent features or tasks,
    \item reduce dimensionality before causal tests (e.g., ablate/project-out a PC and re-evaluate behavior),
    \item build compact surrogates (e.g., PCA embeddings for downstream analyses or compression) \citep{bengtsson2025pca_compress,he2024transformers_pca}.
\end{itemize}

Under widely observed linear-structure assumptions in Transformer activations, PCA offers an interpretable, testable starting point: it surfaces dominant directions, supports hypothesis generation, and provides quantitative targets for more advanced decompositions. %Used responsibly—with causal and sparsity-based follow-ups—PCA strengthens, rather than substitutes, the evidentiary chain for claims about internal LLM mechanisms.
% \textbf{MLP}

% \textbf{Transformers}

% \textbf{Cross-architecture insights}

\newpage

\end{document}